%% file: main_to_arxiv.tex
\documentclass[11pt]{article}
\makeatletter
\let\@fnsymbol\@arabic
\makeatother
\usepackage{appendix}
\usepackage{graphicx}
\usepackage{subcaption}
\usepackage{makecell}
\usepackage{amsmath}
\usepackage{amsfonts}
\usepackage{amssymb}
\usepackage{epsfig}
\usepackage{amsmath}
\usepackage{amsfonts}
\usepackage{amssymb}
\usepackage{dsfont}
\usepackage{mathrsfs}
\usepackage{bbm}
 \usepackage{hyperref}
\usepackage{amsmath}
\usepackage{amsfonts}
\usepackage{amssymb}
\usepackage{multirow}
\usepackage{mathrsfs}
\usepackage[dvipsnames,usenames]{color}
\usepackage{booktabs}
\usepackage{subcaption}
\usepackage{caption}
\usepackage{svg}
\usepackage{bm}
\usepackage{tikz}
\usetikzlibrary{arrows.meta}
\usetikzlibrary{positioning}
\input{Tikz_flow_charts/flowchart_style.tex}

\usepackage[normalem]{ulem}
\usepackage{natbib}
\setcitestyle{numbers,square}
\allowdisplaybreaks[4]
\usepackage{titlesec}
\titleformat{\paragraph}[hang]
  {\normalfont\normalsize\bfseries}
  {\theparagraph}
  {1em}
  {}
\titlespacing*{\paragraph}
  {0pt}
  {1ex plus .2ex minus .2ex}
  {1ex plus .2ex}

\newtheorem{theorem}{Theorem}[section]
\newtheorem{theorem*}{Theorem}[subsubsection]
\newtheorem{lemma}{Lemma}[section]
\newtheorem{proposition}{Proposition}[section]
\newtheorem{example}{Example}[section]
\newtheorem{remark}{Remark}[section]
\newtheorem{definition}{Definition}[section]

\newtheorem{assumption}{Assumption}[section]
\newtheorem{corollary}{Corollary}[section]

\newcommand{\setd}{{ d \kern -.15em l}}
\newcommand{\hatsetd}{ d \hat{\kern -.15em l }}
\newcommand{\dd}{\mathsf {d\kern -0.07em l}} 

\newcommand{\bgeqn}{\begin{eqnarray}}
\newcommand{\edeqn}{\end{eqnarray}}
\newcommand{\bgeq}{\begin{eqnarray*}}
\newcommand{\edeq}{\end{eqnarray*}}
\newcommand{\bec}{\begin{center}}
\newcommand{\enc}{\end{center}}
\newcommand{\R}{{\rm I\!R}}

\newcommand{\var}{{\rm Var}}

\newcommand{\E}{{\cal E}}

\newcommand{\B}{{\cal B}}
\newcommand{\I}{{\cal I}}

\newcommand{\Y}{{\cal Y}}

\newcommand{\W}{\mathcal{W}}

\newcommand{\be}{\begin{equation}}
\newcommand{\ee}{\end{equation}}

\def\eps{\varepsilon}

\def\rP{{\rm P}}

\def\bbe{{\Bbb{E}}} 

\def\bbp{{\Bbb{P}}}

\def\bbbone{{\mathchoice {\rm 1\mskip-4mu l} {\rm 1\mskip-4mu l}
{\rm 1\mskip-4.5mu l} {\rm 1\mskip-5mu l}}}
\newcommand{\ind}{{\bbbone}} 

\newcommand{\qedbox}{\hfill\ensuremath{\square}}

\AtBeginDocument{%
  \setlength{\abovedisplayskip}{3pt plus 1pt minus 1pt}%
  \setlength{\belowdisplayskip}{3pt plus 1pt minus 1pt}%
  \setlength{\abovedisplayshortskip}{3pt plus 1pt minus 1pt}%
  \setlength{\belowdisplayshortskip}{3pt plus 1pt minus 1pt}%
}

\def\bbbone{{\mathchoice {\rm 1\mskip-4mu l} {\rm 1\mskip-4mu l}
		{\rm 1\mskip-4.5mu l} {\rm 1\mskip-5mu l}}}

\newcommand{\Boalp}{{\boldsymbol{\alpha}}}
\newcommand{\Bobet}{{\boldsymbol{\beta}}}
\newcommand{\Bothe}{{\boldsymbol{\theta}}}

\newcommand{\BoTheS}{{\boldsymbol{\theta}^{\star}}}
\newcommand{\Bov}{\boldsymbol{v}}
\newcommand{\Boxi}{\boldsymbol{\xi}}
\begin{document}

\title{Error Bounds for Statistical Estimators in BTL Model with Parametric Multivariate Utility Functions}
    
\author{
Yicheng Li\thanks{Department of Systems Engineering and Engineering Management, The Chinese University of Hong Kong. Email: \href{mailto:ycli@se.cuhk.edu.hk}{\texttt{ycli@se.cuhk.edu.hk}}.
}
\,\,and\,\,
Huifu Xu\thanks{Department of Systems Engineering and Engineering Management, The Chinese University of Hong Kong. Email: \href{mailto:hfxu@se.cuhk.edu.hk}{\texttt{hfxu@se.cuhk.edu.hk}}.
}
}

	\date{\today}

	\maketitle

	\begin{abstract}
    We study preference elicitation under the Bradley-Terry-Luce (BTL) model where the true partworth vector is unknown and has to be estimated as a parameter with elicited preference information. The set of selected pairwise queries is non-uniform, deterministic, and arbitrary over a collection of alternatives, provided that it satisfies a joint identifiability condition. We focus on understanding when the canonical maximum likelihood estimator (MLE) is finite and admits sharp error bounds without explicit compactness constraints on the feasible set or external regularizers.
    To this end, we derive minimax lower bounds under the standard bounded dynamic range condition, and find that the same Fisher-information geometry in the classic Cram\'er-Rao lower bounds underpins the finite-sample difficulty of the estimation problem. By combining a non-asymptotic expansion of the likelihood score equation with a fixed-point localization argument, we identify a design-dependent sample size threshold above which the unconstrained canonical MLE exists and is unique with high probability. The same expansion yields a decomposition of the estimation error into a linear stochastic term, an explicit second-order bias, and a higher-order remainder. A refined analysis gives sufficient sample size conditions under which the canonical MLE attains the minimax rates up to logarithmic and constant factors. These results provide a unified non-asymptotic theory for parametric utility elicitation and reveal when the inference is determined by response data alone rather than by external regularization. Preliminary numerical results are consistent with the theoretical findings. 
	\end{abstract}


\section{Introduction}
\label{sec_intro}
Preference elicitation seeks to infer a decision maker’s (DM’s) latent utility function from a limited collection of observed preference statements.
In many applications, including but not limited to marketing and transportation~\cite{ToubiaPolyCut04, YuGoosBayesSurvey}, interactive recommendation systems~\cite{BonillaGuoGPPrefElict, BaltrunasGroupRecommend}, and reinforcement learning from human feedback~\cite{RafailovDPORLHF, ZhuRLHFBTLMLEError}, such statements are collected through survey-based statistical procedures. In particular, given $d$ alternatives (e.g., products or services), the experimenter may present selected pairs $(i,j)\in \mathbb{I}:=\{(k,l)\mid 1\leq k<l\leq d\}$ to the DM and record which alternative is preferred. Let $\mathcal{X}:=\{\mathbf{x}_1,\dots,\mathbf{x}_d\}\subseteq\mathbb{R}^{n}$, where $\mathbf{x}_i$ denotes the characteristic vector of alternative $i$, and let $U:\mathcal{X}\to \R$ denote the DM's latent utility function. We are then interested in quantifying the statistical effectiveness of estimating $U(\cdot)$ from a prescribed questionnaire design and the resulting ordinal responses. 

In some practical applications, the preference statements are not necessarily consistent, which means the observed pairwise choices are not transitive~\cite{TverskyIntransitivityOfPrefs} and the DM's responses to the same pair of alternatives may be different~\cite{HeyRepetImproveConsist},
either because the DM's preferences are truly random or there are random errors in the elicitation process~\cite{SaureElliposDopt, ToubiaDynaExHessElict, BertsimasOHairLearnPrewithNoise, ChenLiuVNMElicit}.
The phenomena cannot be described by deterministic utility models (including von Neumann-Morgenstern's expected utility theory) such as~\cite{ToubiaPolyCut04, SainanModPloyH, JiaxinCoorWisePolyCut},
because no single admissible utility function is consistent with all observed responses. Likewise, set-based approaches remain non-probabilistic, but represent incomplete preference information through families of deterministic utility functions~\cite{GrecoSalvatoreSetbasedUtility, GiarlottaGrecoSetbasedUtility}. On the other hand, random utility theory models such response variability by treating utilities as random variables~\cite{CascettaRUT2009, HuJianPLAPRO}.
The most well-known one is the
Bradley-Terry-Luce (BTL) model~\cite{BradleyTerryModel,LuceIndividualChoice} in the discrete-choice literature. Bradley and Terry~\cite{BradleyTerryModel} first 
use the formula
\begin{equation}
    \label{eq_Logit_originalIIA_form}
    \bbp(i\mid\{i,j\})=\frac{U(\mathbf{x}_i)}{U(\mathbf{x}_i)+U(\mathbf{x}_j)}
\end{equation}
to calculate the probability of alternative $i$ winning a comparison $(i,j)$. Luce develops an axiomatic description of DM's preference to justify~\eqref{eq_Logit_originalIIA_form} when $U(\cdot):\mathcal{X}\to\mathbb{R}_{++}$, see~\cite[Theorem~3]{LuceIndividualChoice}.
By setting $U(\cdot):=\sigma\log U(\cdot)$, where $\sigma\in(0,\infty)$ is a positive constant,
McFadden~\cite{McFaddenCondLogit}~\cite[Chapter~2]{TrainDiscreteChoice} shows that the DM's preference can be described by the random utility function
\begin{equation}
    \label{eq_def_random_utiliy_V_i}
    V(\mathbf{x}_i) = U(\mathbf{x}_i) + \epsilon_i,
\end{equation}
where $\epsilon_i:(\Omega,\mathcal F,\mathbb P)\to \mathbb{R}$, $i\in[d]$, 
are identically and independently distributed (i.i.d.) with $\epsilon_i\sim\mathrm{Gumbel}(0,\sigma)$, and
\begin{equation}
    \label{eq_Logit_general_form}
    \bbp(i\mid\{i,j\})=\frac{e^{U(\mathbf{x}_i)/\sigma}}{e^{U(\mathbf{x}_i)/\sigma}+e^{U(\mathbf{x}_j)/\sigma}}.
\end{equation}
For repeated presentations of a pair $(i,j)$, we assume that the draws $(\epsilon_i,\epsilon_j)$ are independent.
The logit form in~\eqref{eq_Logit_general_form} is widely known as the BTL model. 

When the utility function $U(\cdot)$ is linear in its parameters, maximum likelihood estimation under the BTL model is down to solving a logistic regression problem, see~\cite[Chapter~4]{McCullaghGLM} and~\cite[Chapter~7]{DobsonIntroToGeneralLineModel}.
Nevertheless, the statistical question depends on how the questionnaire is generated. In preference elicitation, the experimenter controls the comparisons manually~\cite{GlickmanPairedComparisonDesign, ZhuRLHFBTLMLEError, MukherjeeOptDesRLHF}, or sometimes through an optimal design criterion~\cite{ToubiaDynaExHessElict, YuGoosBayesSurvey, GoosDoptimalNestedLogit}. Thus, covariates here need not satisfy specific distributional properties used in, e.g.,~\cite{OstrovsBachSelfCMEst,
ChardonLogitErrorFinSam}. In addition, many survey-based data collection procedures involve a limited number of responses because of cost, logistics, or respondent fatigue. Consequently, the reliability of asymptotic maximum likelihood theory and experimental design based on the Fisher information matrix (FIM) is questionable~\cite[Section~2]{YuGoosBayesSurvey}. These two concerns raise the basic well-posedness issue of the maximum likelihood estimator (MLE). Classical results~\cite{KonisLPCheckSeperLogit, ALBERTExistMLELogit, LesaffreExistMLELogit} show that the positive definiteness of the FIM does not exclude separation in the realized responses, in which case a finite MLE fails to exist. A common remedy to this challenge is to stabilize the estimator externally. For example, one may restrict the parameter to a prescribed compact set~\cite{ShahEstPariCompGrapTop, ZhuRLHFBTLMLEError, SuMouPropScoreDeBias, ChenLiuVNMElicit}, regularize the likelihood via ridge~\cite{BachSelfconcordantLogit, YuxinSpectralMLETopK} or Firth correction~\cite{FirthCorrect93, KosmidisFirth21}, and introduce a Bayesian prior~\cite{SaureElliposDopt, JiapengBayesPieceLinPreLean}. The constrained MLE guarantees the existence of an empirical minimizer and a uniform lower bound on the likelihood curvature over the feasible set, but it requires the diameter of that set to be specified a priori. Firth correction remains finite under separation and has a shrinkage property when applied to logistic regression~\cite[Theorem~2]{KosmidisFirth21}, yet the resulting loss function is generally non-convex.
These strategies produce finite estimates even under separation, but finiteness alone does not establish that the likelihood is informative in every identifiable direction. When the response data remain separated, the magnitude of the estimate along an escaping direction is determined primarily by the intensity of the external regularizer.
This leads to our central question:

\textit{
For a fixed and possibly nonuniform questionnaire design, how many responses suffice for the unconstrained canonical MLE to exist with high probability and exhibit sharp non-asymptotic error bounds?}

This issue has been well addressed in pairwise ranking models, where every query vector is a canonical graph edge. In this setting, the DM assigns a structureless latent preference score to each alternative, and the FIM reduces to a weighted graph Laplacian~\cite{ShahEstPariCompGrapTop} (see also Definition~\ref{def_comparGraph} and Section~\ref{subsubsec_likelihood_func}). The special geometry leads to graph-topology treatment of the estimation problem, e.g., information-theoretic lower bounds established through Laplacian spectral quantities~\cite{HajekOhMimaxInfPartilRank, ShahEstPariCompGrapTop, LeeMiniMaxGLM, LeeMinimaxGenPairComp}, finite-sample MLE existence criteria derived via Ford’s condition on the directed win-loss graph~\cite{FordCondMLE, ButlerMLEBTL_Exists, BongBTL_Ranking_MLEerror}, and upper bounds obtained from graph-based localization~\cite{ShahEstPariCompGrapTop, NegahbanRankCentrality, ChenYanXiRankRisistan, YangCongTopKMonoAd, YangRaschRandomMLE} and leave-one-out (LOO) arguments~\cite{ChenSuhSpeMLETopK, YuxinSpectralMLETopK, GaoUncertQuatiBTL}. Even though related techniques also appear in,
e.g., structured von Neumann–Morgenstern utility elicitation~\cite{ChenLiuVNMElicit}, downstream policy optimality quantification~\cite{ZhuRLHFBTLMLEError}, $K$-wise questionnaire design under the Plackett–Luce model~\cite{MukherjeeOptDesRLHF}, covariate-assisted ranking under Erd\"os-R\'enyi graph design~\cite{FanRankingWithCovar}, and $\ell_\infty$ error analysis over general comparison graphs~\cite{LIl_inf_bound_BTL_GenGraph}, the unified graph-theoretic framework does not transfer mechanically once the canonical edge vectors are replaced by general covariate directions. Some of these error bounds remain conservative when the heterogeneity of the query directions is taken into account. For example, they may depend inversely on the smallest eigenvalue of the FIM, leading to deterioration in Euclidean error and sample complexity, see e.g.,~\cite[Theorem~2]{ShahEstPariCompGrapTop} and~\cite[Lemma~6]{SuMouPropScoreDeBias}; or be inherently suboptimal when the comparison graph is irregular, see discussion in~\cite{ChenYanXiRankRisistan, YangCongTopKMonoAd}.

For general covariates, the overlapping condition~\cite{ALBERTExistMLELogit} conveys not only the graph topology but also the conic geometry of the covariates in the parameter space~\cite{SpeckmanLeeSunChoiceMLE}. A sharp and explicit sample threshold for MLE existence requires additional assumptions on how the covariates are generated. For example, Cand\`es and Sur~\cite{CandesSurPhaseTransLogit} derive an asymptotic phase transition for MLE existence under Gaussian covariates, while Chardon et
al.~\cite{ChardonLogitErrorFinSam} obtain non-asymptotic guarantees for MLE existence and excess risk under Gaussian and other regular random designs. They also show that for discrete covariates, existence may depend strongly on the orientation of the ground truth~\cite[Section~3.3]{ChardonLogitErrorFinSam}. Our aim is instead to derive a design-dependent sufficient condition that applies to an arbitrary but fixed collection of pairwise queries.

We use a two-step analysis based on the pseudo self-concordance property of the logistic likelihood function~\cite{BachSelfconcordantLogit, OstrovsBachSelfCMEst}. First, we localize the estimator uniformly along each query direction, which keeps the empirical Hessian comparable to its value at the ground truth. Second, we use a likelihood score expansion to connect this local behavior with global excess risk and Euclidean error. Due to the non-asymptotic nature of the results to be presented, the expansion retains the second-order bias of the canonical MLE. Classical works show explicit bias formulas for generalized linear models~\cite{CordeiroMcCuBiasGLM}, and show how Firth's correction removes the leading bias asymptotically~\cite{FirthCorrect93}. 
More recent finite-sample analysis for smooth functionals of Z-estimators by Lin et al.~\cite{LinSuJackQuadBarriZEst} identifies how the second-order bias produces a quadratic sample size barrier, complementing the growing-dimensional asymptotic literature by Portnoy~\cite{PortnoyAsymMpN, PortnoAsymExpFamily}.
\subsection{Main contributions}
\label{subsec_major_contrib}
In this paper, we study the parametric utility elicitation under the BTL model with a fixed and possibly nonuniform collection of pairwise queries.
To facilitate reading, we plot a flowchart (see Figure~\ref{fig_theoretical-flow}) outlining
key theoretical developments. 
The main contributions are three-fold.
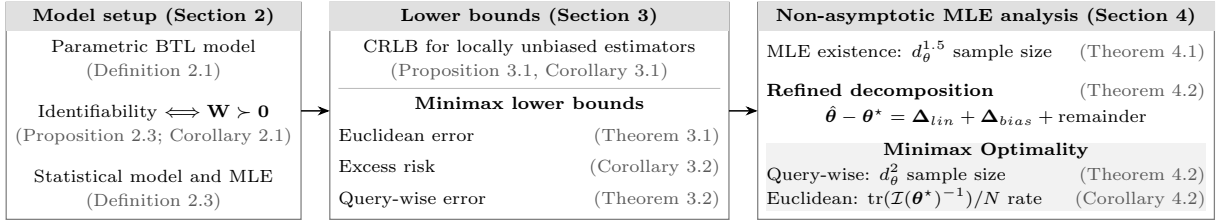
\begin{figure}[t]
    \centering
    \resizebox{\linewidth}{!}{\input{Tikz_flow_charts/theory_tikz_auto.tex}}
    \caption{Flowchart of theoretical results.}
    \label{fig_theoretical-flow}
\end{figure}
\begin{itemize}
    \item We consider a linear-in-parameter utility function subject to linear equality constraints (Definition~\ref{def_para_utility_func}). This covers 
    several classical specifications, e.g., pairwise ranking model~\cite{YuxinSpectralMLETopK, ShahEstPariCompGrapTop, FanRankingWithCovar}, multivariate linear utility function~\cite{TrainDiscreteChoice, SaureElliposDopt, JiaxinCoorWisePolyCut}, piecewise linear approximation of nonlinear univariate utilities~\cite{ChenLiuVNMElicit, SainanModPloyH}. 
    We derive a necessary and sufficient identifiability condition for the population-level MLE to be well-posed (Proposition~\ref{prop_indentifibility_complete_obv}), equivalent to the positive definiteness of the design matrix (Corollary~\ref{coro_from_identifibality}).
    Under the identifiability condition, we establish both Cramér--Rao Lower Bounds (CRLB) for locally unbiased estimators (Corollary~\ref{coro_of_CRLB}), and minimax lower bounds over bounded dynamic range parameter class (Theorems~\ref{thm_minimax_loewe_bound_tr_inv},~\ref{thm_residual_coherence_minimax} and Corollary~\ref{coro_miniax_lowerbound_excess_risk}). These results demonstrate that trace- and coherence-based Fisher-information criteria govern both classical local efficiency and finite-sample minimax lower bounds,
    when measured by the Euclidean norm and along individual query directions, respectively.
    \item 
    We develop a two-step convex localization strategy for the finite-sample analysis of the unconstrained MLE. Motivated by the induction-type localization arguments in ranking literature~\cite{YuxinSpectralMLETopK, ChenYanXiRankRisistan}, we express the likelihood score equation as a fixed-point system. By combining this construction with the expansion strategy of score equation in~\cite{LinSuJackQuadBarriZEst, SuMouPropScoreDeBias}, and the pseudo self-concordance trick in~\cite{BachSelfconcordantLogit, OstrovsBachSelfCMEst}, we derive a design-dependent sufficient sample size above which the unconstrained MLE exists and is unique with high probability (Theorem~\ref{thm_suffic_sample_l2_upper_bound}). Instead of imposing boundedness constraints or introducing explicit regularization for the MLE, we localize the likelihood minimizer in a convex and compact set around the ground truth in the proof. 
    \item We distinguish the sample sizes sufficient for existence of the MLE from those sufficient for its minimax optimality.
    A refined decomposition of the estimation error isolates the contribution of the linear stochastic term, deterministic second-order bias, and higher-order remainder (Theorem~\ref{thm_refinred_Q_infty_ell_2_residual_decompose_bound}). 
    When measured by the Euclidean norm, beyond the baseline localization requirements, the sufficient sample size for the MLE to attain the minimax rate depends on the inverse of the FIM via its effective rank but not its extreme eigenvalue (Corollary~\ref{coro_euclidean_upper_in_detailed_terms}), thereby improving the previous conclusions in~\cite{ShahEstPariCompGrapTop, SuMouPropScoreDeBias}.
    Numerical results in Section~\ref{sec_numerical_exp} illustrate how the bias and nonlinear residuals deteriorate the statistical efficiency of the MLE.
\end{itemize}

The rest of the paper is organized as follows. In Section~\ref{sec_model_setup}, we introduce the parametric utility model and its general identifiability condition. Several classical problems that this model subsumes are sequentially presented as examples. In Sections~\ref{sec_minimax_LB} and~\ref{MLE_upper_bound}, we discuss lower error bounds and finite-sample guarantees for the unconstrained MLE, respectively.
In Section~\ref{sec_numerical_exp}, we provide numerical evidence that verifies the established theory. In Section~\ref{sec_conclusion_remarks}, we conclude with some remarks. Proofs of the main results are put in Section~\ref{sec_proof_of_main_results}, with auxiliary arguments deferred to Appendix~\ref{appendi_A}.

Throughout the paper, bold lowercase and uppercase letters denote column vectors and matrices, respectively. $\mathbf{0}$, $\mathbf{1}$, and $\mathbf{I}_n$ denote the all-zero vector, the all-one vector, and the identity matrix in $\mathbb{R}^{n\times n}$, respectively. $\mathbf{e}_i$ is the $i$-th canonical Euclidean basis  with dimensions clear from context. For a matrix $\mathbf{A}$, $\mathbf{A}_{ij}=[\mathbf{A}]_{ij}$ denotes its $(i,j)$-th entry. $\mathrm{diag}(\mathbf{a})$ and $\mathrm{diag}(\mathbf{A})$ stand for a diagonal matrix formed by $\mathbf{a}$ or a vector formed by the diagonal of $\mathbf{A}$, respectively. $(a_{ij})_{i,j\leq n}$ is a matrix whose $(i,j)$-th element is $a_{ij}$. We write $\mathbf{A}^\dagger$ for the Moore-Penrose pseudoinverse of $\mathbf{A}$. For any conformable matrices $\mathbf{A}$, $\mathbf{B}$, $\langle\mathbf{A},\mathbf{B}\rangle = \mathrm{tr}(\mathbf{A}^\top\mathbf{B})$ denotes their Frobenius inner product, and $\mathbf{A}\circ\mathbf{B}$ is the Hadamard product between them. $\mathbf{A}\succ\mathbf{0}$ ($\mathbf{A}\succeq\mathbf{0}$) means $\mathbf{A}$ is positive (semi)definite. We use $\Vert\mathbf{A}\Vert$, $\Vert\mathbf{A}\Vert_F$ for the spectral and Frobenius norms. $\Vert\mathbf{a}\Vert_2$, $\Vert\mathbf{a}\Vert_{\mathbf{H}}:=\sqrt{\mathbf{a}^\top\mathbf{H}\mathbf{a}}$ stand for the Euclidean norm and $\mathbf{H}$-reweighted seminorm of $\mathbf{a}$, respectively, where $\mathbf{H}\succeq\mathbf{0}$. For $\Theta\subseteq\mathbb{R}^n$, $\mathrm{int}(\Theta)$ denotes its interior. For a random variable $Z$, let $\Vert Z\Vert_{\psi_2}$ denote its sub-Gaussian norm. For a random vector $\boldsymbol{\xi}\in\mathbb{R}^n$, define $\Vert\boldsymbol{\xi}\Vert_{\psi_2}:=\sup_{\mathbf{u}\in\mathbb S^{n-1}}\Vert\mathbf{u}^\top\boldsymbol{\xi}\Vert_{\psi_2}$, where $\mathbb{S}^{n-1}$ is the unit sphere in $\mathbb{R}^n$. For nonnegative quantities $f$ and $g$, relations $f\lesssim g$, $f\gtrsim g$ and $f\asymp g$ mean, respectively, that $f\leq Cg$, $f\geq cg$, or $cg\leq f\leq Cg$ for some absolute constants $C\geq c>0$.

\section{Model Setup and Identifiability}
	\label{sec_model_setup}
In this section, we spell out the details of the parametric BTL model, and discuss parameter identifiability issue. See Figure~\ref{fig_survey_pipeline} for an overview of the statistical estimation procedure considered in this paper. 
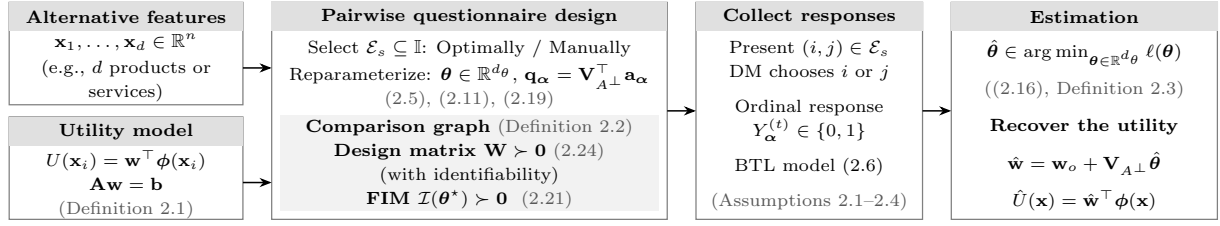
\begin{figure}[t]
  \centering
  \resizebox{\linewidth}{!}{\input{Tikz_flow_charts/survey_tikz_auto.tex}}
  \caption{Survey-based statistical procedures, from alternative features to utility estimation.}
  \label{fig_survey_pipeline}
\end{figure}
\subsection{The Parametric Random Utility Model}
\label{subsec_utlity_model}
    Under the BTL framework~\eqref{eq_Logit_general_form}, choice probabilities depend on the alternatives through their utility differences and the scale parameter $\sigma$.
    We will handle them separately.
    We first specify $U(\cdot)$ through an appropriate feature map, allowing nonlinear dependence on the characteristic vectors while retaining linearity in the unknown parameter.
    \begin{definition}[Linear-in-parameter Utility Representation]
    \label{def_para_utility_func}
    Let $\mathcal{X}:=\{\mathbf{x}_1,\dots,\mathbf{x}_d\}$ denote the characteristic vectors of the $d$ alternatives, and let $\boldsymbol{\phi}:\mathcal{X}\to\mathbb{R}^p$ be a fixed vector-valued feature map.
    The utility values $\{U(\mathbf{x}_i)\}_{i=1}^d$ 
    admit the following linear form w.r.t. parameter $\mathbf{w}\in\mathbb{R}^p$, i.e.,
    \begin{equation}
    \label{eq_para_of_U_i_by_embedding}
        U(\mathbf{x}_i) := \mathbf{w}^\top\boldsymbol{\phi}(\mathbf{x}_i) = \sum_{k=1}^p w_k[\boldsymbol{\phi}(\mathbf{x}_i)]_k,\,\;\;\text{for}\,\;i\in[d],
    \end{equation}
    where $\mathbf{w}=[w_1,\dots,w_p]^\top$
    satisfies $\mathbf{Aw}=\mathbf{b}$ for some $\mathbf{A}\in\mathbb{R}^{l\times p}$, $\mathbf{b} \in \mathbb{R}^l$.
    \end{definition}
    We mainly focus on the case where $\{\mathbf{x}_i\}_{i=1}^d$ are deterministic attribute vectors of the alternatives, though our results can be modified to cover random lotteries with finite outcomes~\cite{ChenLiuVNMElicit}, and highly nonlinear specifications~\cite[Section~4.1]{BachSelfconcordantLogit}. The equality constraints $\mathbf{Aw}=\mathbf{b}$ are adopted to encode the known structural restrictions of the utility function or to eliminate non-identifiable directions, as discussed in Proposition~\ref{prop_indentifibility_complete_obv}. They should not be confused with polyhedral cuts generated from observed preferences in deterministic elicitation methods (see e.g.,~\cite{ToubiaPolyCut04}). Definition~\ref{def_para_utility_func} places several commonly used utility specifications within a common parametric framework, as illustrated next.
    \begin{example}[Discrete Choice Model]
    \label{example_Discrete_Choice_Model}
        For every $i\in[d]$, set $\mathbf{x}_i\in\{0,1\}^p$, $\boldsymbol{\phi}(\mathbf{x}_i)=\mathbf{x}_i$ in~\eqref{eq_para_of_U_i_by_embedding}, and remove the linear equality constraint $\mathbf{Aw}=\mathbf{b}$. Then $U(\mathbf{x}_i) = \mathbf{w}^\top \mathbf{x}_i$, and it can be viewed as a multivariate linear utility function, where $\mathbf{x}_i$ is the vector of attributes and $\mathbf{w}$ represents the partworth vector. This recovers the discrete choice model used in classical conjoint analysis~\cite[Chapter 2]{TrainDiscreteChoice}, ~\cite[Section 3]{SaureElliposDopt}, and~\cite{JiaxinCoorWisePolyCut}. 
    \end{example}

    \begin{example}[Piecewise-linear Approximation]
    \label{ex_piecewise_linear_approxi}
        For each $i\in[d]$, let $x_i\in[a,b]\subset\mathbb{R}$ denote the scalar attribute value of alternative $i$. Fix a set of breakpoints $a=t_0<t_1<\cdots<t_J=b$, and let $u:[a,b]\to[0,1]$ be a normalized univariate, nonlinear utility function satisfying $u(a)=0$ and $u(b)=1$. Define $\mathbf{w}_j:=u(t_j)-u(t_{j-1}),\,j\in[J]$, and let $\mathbf{w}:=[\mathbf{w}_1,\dots,\mathbf{w}_J]^\top\in\mathbb{R}^J$. For every $\zeta\in[a,b]$, define $\boldsymbol{\phi}(\zeta)\in\mathbb{R}^J$ coordinatewise, i.e.,
        \begin{equation*}
            [\boldsymbol{\phi}(\zeta)]_j :=
                \begin{cases}
                    0,
                    & \text{for } \zeta\le t_{j-1},\\[1mm]
                    \dfrac{\zeta-t_{j-1}}{t_j-t_{j-1}},
                    & \text{for } t_{j-1}<\zeta\le t_j,\\[3mm]
                    1,
                    & \text{for } \zeta>t_j.
                \end{cases}
        \end{equation*}
        Then $U(\zeta):=\mathbf{w}^\top\boldsymbol{\phi}(\zeta)$ is a piecewise-linear approximation of $u$, and hence $U(x_i)=\mathbf{w}^\top\boldsymbol{\phi}(x_i)$ gives the corresponding utility scores~\cite{ChenLiuVNMElicit,SainanModPloyH}. The normalization is encoded by $\mathbf{1}^\top\mathbf{w}=1$, while the monotonicity of $u$ implies $\mathbf{w}\geq\mathbf{0}$. The same representation extends to additive multi-attribute utility functions following~\cite[Proposition~2]{HuJianPLAPRO}.
   \end{example}
    \begin{example}[Pairwise Ranking Model]
    \label{ex_ranking_setup}
        By setting $\boldsymbol{\phi}(\mathbf{x}_i) = \mathbf{e}_i\in\mathbb{R}^d,\,\forall\,i\in[d]$, and imposing $\mathbf{1}^\top\mathbf{w}=0$, we can interpret the $\{U(\mathbf{x}_i)\}_{i=1}^d$ in~\eqref{eq_para_of_U_i_by_embedding} as structure-less latent preference scores that the DM assigns to the items. This recovers the ranking model used in~\cite{HajekOhMimaxInfPartilRank, ShahEstPariCompGrapTop, YuxinSpectralMLETopK, YangCongTopKMonoAd}.
    \end{example}
    \begin{example}[Covariate Assisted Ranking Model]
    \label{ex_cov_assisted_ranking_1}
        Let $\boldsymbol{\phi}(\mathbf{x}_i) = [\mathbf{e}_i^\top,\mathbf{x}_i^\top]^\top\in \mathbb{R}^{d+n}$, $\mathbf{w} = [\mathbf{w}_1^\top,\mathbf{w}_2^\top]^\top$, where $\mathbf{w}_1\in\mathbb{R}^d$, $\mathbf{w}_2\in\mathbb{R}^n$. Then $U(\mathbf{x}_i) = \mathbf{e}_i^\top\mathbf{w}_1 + \mathbf{x}_i^\top \mathbf{w}_2$, where $\mathbf{e}_i^\top\mathbf{w}_1$ represents the residual scores that cannot be explained by the contribution from covariates $\mathbf{w}_2^\top \mathbf{x}_i$ alone~\cite{FanRankingWithCovar}. 
        Under the condition that 
        \begin{equation*}
            \bar{\mathbf{X}}:=\begin{bmatrix}
            1 &\cdots &1\\
            \mathbf{x}_1 & \cdots & \mathbf{x}_d 
        \end{bmatrix}\in\mathbb{R}^{(n+1)\times d},\,\,\mathrm{rank}(\bar{\mathbf{X}}) = n+1,
        \end{equation*}
        the linear equality constraint $[\bar{\mathbf{X}},\,\mathbf{0}_{(n+1)\times n}]\mathbf{w}=\mathbf{0}$ is imposed to ensure identifiability. We will come back to this in Example~\ref{ex_cov_assisted_ranking_2}.
    \end{example}
    
    We now discuss the specification of $\sigma$. To preserve strictly positive choice probabilities and a nondegenerate stochastic model, we assume $0<\sigma<\infty$ throughout. When $\sigma\downarrow 0$, the DM chooses alternative $i$ without uncertainty whenever $U(\mathbf{x}_i)>U(\mathbf{x}_j)$. This noise-free limit violates the strict positivity condition underlying Luce’s representation~\cite[Lemma~1]{LuceIndividualChoice}. As $\sigma\to \infty$, the choice probability converges to $\frac{1}{2}$, so the responses become uninformative about utility differences. From standard treatment of generalized linear models (see e.g.,~\cite[Chapter~3.2]{TrainDiscreteChoice},~\cite[Chapter~4]{McCullaghGLM}, and~\cite{LeeMiniMaxGLM}), we treat $\sigma$ as a fixed scale parameter and do not estimate it along with $\mathbf{w}$.
    This is justified by the fact that the response law is invariant within $[(\mathbf w,\sigma)]: = \{(c\mathbf w,c\sigma)\mid c\in\mathbb{R}_{++}\}$ unless the scale of $U(\cdot)$ is fixed exogenously, e.g., in Examples~\ref{example_Discrete_Choice_Model},~\ref{ex_ranking_setup}, and~\ref{ex_cov_assisted_ranking_1}, we set $\sigma=1$. When the utility function is normalized as in Example~\ref{ex_piecewise_linear_approxi}, one may select a unique representation from $[(\mathbf w,\sigma)]$ as in~\cite[Theorem 1]{ChenLiuVNMElicit}, but the interpretation of $\sigma$ remains relative to that normalization.

\subsection{Questionnaire and Responses}
Let $\Boalp:=(i,j)\in\mathbb{I}$ denote a pairwise query, and define
  \begin{subequations}
  \label{eq_def_of_U_phi_matrix}
      \begin{align}
          \label{eq_def_of_U_phi_matrix_a}
        \mathfrak{A}
        &:=[\boldsymbol{\phi}(\mathbf{x}_1),\dots,\boldsymbol{\phi}(\mathbf{x}_d)]^\top\in\mathbb{R}^{d\times p}, \\
        \label{eq_def_of_U_phi_matrix_b}
        \mathbf{a}_{\Boalp}&:=\mathfrak{A}^\top(\mathbf{e}_i-\mathbf{e}_j).
      \end{align}
  \end{subequations}
 $\mathbf{a}_{\Boalp}$ is called a {\em query vector} associated with query $\Boalp$, which is also known as a {\em query direction}.
Next, we discuss how the sequence of responses is generated and specify several mild conditions on the ground-truth parameter underlying the responses.
By substituting~\eqref{eq_para_of_U_i_by_embedding} and~\eqref{eq_def_of_U_phi_matrix_b} into \eqref{eq_Logit_general_form}, we obtain a scalar choice probability in terms of $\mathbf{a}_{\Boalp}$ parameterized by $\mathbf{w}$
		 \begin{equation}
		\label{eq_BTL_basic_scaled_alpha}
		p_\Boalp(\mathbf{w}):=\bbp(i\mid\{i,j\}) =  \frac{1}{1+\exp\left(-\frac{\mathbf{w}^\top\mathbf{a}_{\Boalp}}{\sigma}\right)}.
		\end{equation}
        Let $\E_s\subseteq\mathbb{I}$ be the selected pairs, and let $n_\Boalp\geq 1$ denote the number of responses collected for pair $\Boalp\in\E_s$. Repeated presentations of the same query are allowed but not required for our theoretical results to hold. Thus, $\E_s$ and $\{n_\Boalp\}_{\Boalp\in\E_s}$ specify the {\em questionnaire design}.
        For query $\Boalp$, let $(\Y_\Boalp,\mathscr{Y}_\Boalp) := (\{0,1\},2^{\{0,1\}})$, and let $\rP_\mathbf{w}^\Boalp$ denote the Bernoulli probability measure on this space with success rate $p_\Boalp(\mathbf{w})$. The response to presentation $t\in[n_\Boalp]$ is denoted by $Y_\Boalp^{(t)}$, with value one when $i$ is chosen and zero otherwise. Its distribution under the BTL model~\eqref{eq_BTL_basic_scaled_alpha} is $\rP_\mathbf{w}^\Boalp$. The next proposition describes this via the exponential-family representation.
		\begin{proposition}[{\cite[Chapters~2 and~4]{McCullaghGLM}}]
		   \label{observ_Bernoulli_exp_calss}
          For any but fixed $\Boalp$ and $t\in[n_\Boalp]$, suppose that $Y_{\Boalp}^{(t)}\sim \rP_\mathbf{w}^\Boalp$. 
        Then the probability mass function of $Y_{\Boalp}^{(t)}$, evaluated at $y\in\{0,1\}$, can be expressed by
            \begin{equation}
            \label{eq_T_alpha_t_loss_as_Logit}
                f_{Y_\Boalp^{(t)}}(y\mid\mathbf{w}) = \exp\left[y\frac{\mathbf{w}^\top\mathbf{a}_{\Boalp}}{\sigma} - \log\left(1+e^{\frac{\mathbf{w}^\top\mathbf{a}_{\Boalp}}{\sigma}}\right)\right].
            \end{equation}
        Moreover, let $\Psi(\eta) := \log(1+e^\eta)$. Then $\bbe_\mathbf{w}^\Boalp[Y_\Boalp^{(t)}]=\Psi^{\prime}\left(\frac{\mathbf{w}^\top\mathbf{a}_{\Boalp}}{\sigma}\right)$ and $\mathrm{Var}_\mathbf{w}^\Boalp[Y_\Boalp^{(t)}]=\Psi^{\prime\prime}\left(\frac{\mathbf{w}^\top\mathbf{a}_{\Boalp}}{\sigma}\right)$, where $\bbe_\mathbf{w}^\Boalp[Y]$ and $\mathrm{Var}_{\mathbf{w}}^\Boalp[Y]$ denote the expectation and the variance under $\rP_\mathbf{w}^\Boalp$, respectively. 
		\end{proposition}
        The proof can be found in, e.g.,~\cite[Chapter 2.2.2]{McCullaghGLM}.
        $\Psi(\eta)$ is also known as the cumulant function, as its derivatives determine the moment statistics of $Y_\Boalp^{(t)}$, 
        with $|\Psi^{(3)}(\eta)|\leq \Psi^{\prime\prime}(\eta)$, $|\Psi^{(4)}(\eta)|\leq \Psi^{\prime\prime}(\eta),\,\forall\,\eta\in\mathbb{R}$.
        Equation~\eqref{eq_T_alpha_t_loss_as_Logit} has the canonical logistic regression form given a single sample-covariate pair $(Y_\Boalp^{(t)},\frac{\mathbf{a}_{\Boalp}}{\sigma})$, and the structure of the collection of such pairs may be conveniently described by the notion of comparison graph (see e.g.~\cite{ShahEstPariCompGrapTop, BongBTL_Ranking_MLEerror, LeeMinimaxGenPairComp, ZhuRLHFBTLMLEError}).

    \begin{definition}[Comparison Graph]
        \label{def_comparGraph}  
        Consider an undirected weighted graph $\mathcal{G}([d],\mathcal{E}_s,n_\Boalp)$ with vertex set $[d]$, edge set $\mathcal{E}_s\subseteq \mathbb{I}$, and edge weight $n_{\Boalp}$. It is called a {\em comparison graph} because
        any two distinct vertices $i$ and $j$ are adjacent iff the pair $\Boalp=(i,j)$ is presented to the DM at least once. 
        The edge weight is equal to the number of times the corresponding query is presented to the DM.
    \end{definition}
        With the notion of comparison graph, we are ready to make an assumption that ensures that model~\eqref{eq_T_alpha_t_loss_as_Logit} is correctly specified.
         \begin{assumption}[Well Specified Model]
			\label{assumpt1_well_Logit}
                There exists a ground-truth parameter $\mathbf{w}^{\star}\in\mathbb{R}^p$ satisfying $\mathbf{A}\mathbf{w}^\star=\mathbf{b}$,
                such that $Y_\Boalp^{(t)}\sim\rP_{\mathbf{w}^\star}^\Boalp$ for every $\Boalp\in\mathcal{E}_s$, $t\in[n_\Boalp]$.
		\end{assumption} 
Assumption~\ref{assumpt1_well_Logit} specifies the population-level response law. Testing whether the observed choices are compatible with the model is beyond the scope of this work. Ruan et al.~\cite[Theorem~1, Section~5]{RuanNonparaMDM} formulate the representability of 
a marginal distribution model (including the BTL model as a special case) as a linear feasibility problem.
Applying such verification requires the choice probabilities or suitable estimates. Without repeated presentations of the same query, empirical choice frequencies provide no reliable estimate of the corresponding choice probabilities~\cite[Chapter 4.4.3]{McCullaghGLM}. 
        \begin{assumption}[Independent Samples]
            \label{assump_indepent_sample}
            For a fixed edge set $\mathcal{E}_s$ and the edge weights $\{n_\Boalp\}_{\Boalp\in\E_s}$ of the comparison graph defined as in Definition~\ref{def_comparGraph}, the 
            responses $\big\{Y_\Boalp^{(t)}\big\}_{\Boalp\in\mathcal{E}_s}^{t\in[n_\Boalp]}$ 
            are mutually independent, and for each fixed $\Boalp\in\E_s$, $\{Y_\Boalp^{(t)}\}_{t\in [n_\Boalp]}$ are i.i.d.
        \end{assumption}
        The next assumption states that log odds of $p_\Boalp(\mathbf{w}^\star)$, i.e., $\frac{\mathbf{a}_{\Boalp}^\top\mathbf{w}^\star}{\sigma}$ is uniformly bounded for every queried pair. 
        \begin{assumption}[Bounded Dynamic Range]
		 \label{auumpt1self_concor}
       There exists a positive constant $B>1$ such that the ground truth parameter $\mathbf{w}^\star$ has a dynamic range bounded by $\log B$ over the set of selected queries, i.e., 
          \begin{equation}
            \label{eq_assump_bound_dymic_by_logB}
            \max_{\Boalp\in\mathcal{E}_s}\big|\mathbf{a}_{\Boalp}^\top\mathbf{w}^\star/\sigma|\leq \log B<\infty.
          \end{equation}
		\end{assumption}
         The assumption follows from the bounded dynamic range condition used in ranking~\cite[Section~2.1]{ShahEstPariCompGrapTop}~\cite[Section~2.1]{YuxinSpectralMLETopK}.
         When $B\downarrow1$, the queried choice probabilities approach $\frac{1}{2}$,
       meaning that the DM becomes indifferent 
       throughout the comparisons. 
       On the other hand, as $B\to\infty$, the bound~\eqref{eq_assump_bound_dymic_by_logB} allows choice probabilities to approach zero or one. 
Assumption~\ref{auumpt1self_concor} is also needed for uniform statistical guarantees, i.e., by Proposition~\ref{prop_lim_B_to_infty_minimax_risk_Euclide_to_infty}, the minimax lower bound may go to infinity without Assumption~\ref{auumpt1self_concor}.
We therefore define
			\begin{equation}
            \label{eq_w_para_bounded_set}
				\mathcal{W}_B :=\left\{ \mathbf{w}\in\mathbb{R}^p\,\big|\, \mathbf{Aw}=\mathbf{b},\,\big|\mathbf{a}_{\Boalp}^\top\mathbf{w}/\sigma\big|\leq \log B,\,\forall\,\Boalp\in\mathcal{E}_s \right\}.
			\end{equation} 
    Proposition~\ref{prop_bound_dymic_range} below is a standard consequence of Assumption~\ref{auumpt1self_concor}, we restate it here for completeness. It shows that $\mathrm{Var}_{\mathbf{w}}^\Boalp(Y_\Boalp^{(t)})$ is uniformly bounded away from zero over $\mathcal{W}_B$. 
        \begin{proposition}[Non-degenerate Variance]
			\label{prop_bound_dymic_range}
            Under Assumption \ref{auumpt1self_concor}, 
            \begin{equation}
                \frac{1}{4B}\leq \mathrm{Var}_{\mathbf{w}}^\Boalp(Y_\Boalp^{(t)}) \leq \frac{1}{4},\,\forall\,\mathbf{w}\in \mathcal{W}_B,\,\Boalp\in\mathcal{E}_s,\,t\in[n_\Boalp].
            \end{equation}
           
		\end{proposition}
        We omit the proof as it follows directly by mimicking~\cite[Appendix~A.2]{ShahEstPariCompGrapTop} or~\cite[Lemma 13]{YangCongTopKMonoAd}.
   \subsection{Identifiability and Likelihood Function}
   \label{subsec_model_assumps_and_identib}
   This section serves a twofold purpose. First, we remove the equality constraints by reparameterization and establish the identifiability condition of the general model introduced above through the joint structure of the features, constraints, and comparison graph. We then deduce the likelihood function and express the same condition equivalently by positive definiteness of an approximation of the FIM.
    \subsubsection{Identifiability}
        Let $r_A:=\mathrm{rank}(\mathbf{A})$, and let $\mathbf{V}_{A^{\perp}}\in\mathbb{R}^{p\times (p-r_A)}$ have orthonormal columns that span $\mathrm{Null}(\mathbf{A})$. With $\mathbf{w}_o :=\mathbf{A}^\dagger\mathbf{b}$ being the minimum-norm solution of $\mathbf{Aw}=\mathbf{b}$, every feasible $\mathbf{w}$ has the representation
		\begin{equation}
		  \label{eq_transfrom_w_to_theta}
		  \mathbf{w}=\mathbf{w}_o + \mathbf{V}_{A^\perp}\Bothe,\quad\Bothe\in\mathbb{R}^{d_\theta},\quad d_\theta := p-r_A.
		\end{equation}
        Since $\mathbf{V}_{A^{\perp}}$ has orthogonal columns, the map $\boldsymbol{\theta}\mapsto \mathbf{w}_o + \mathbf{V}_{A^{\perp}}\boldsymbol{\theta}$ is injective.
        Therefore $\mathcal{W}_B$ induces the equivalent parameter space
		\begin{equation}
        \label{eq_def_of_theta_B_set}
			\Theta_B :=\left\{ \boldsymbol{\theta}\in\mathbb{R}^{d_{\theta}}\,\,\big|\,\, \big|\mathbf{a}_{\Boalp}^\top(\mathbf{w}_o + \mathbf{V}_{A^{\perp}}\Bothe)/\sigma\big|\leq \log B,\,\forall\,\Boalp\in\mathcal{E}_s \right\}.
		\end{equation}
        \begin{assumption}
		\label{assump_gap_wp_w_B_r0}
            Let $B$ be given as in Assumption~\ref{auumpt1self_concor} and $B_0$ be a positive constant such that  $\max_{\Boalp\in\mathcal{E}_s}\big|\mathbf{w}_o^\top\mathbf{a}_{\Boalp}/\sigma\big|= \log(B_0)$. Assume that $B > B_0$.
		\end{assumption}
        By Assumption~\ref{assumpt1_well_Logit}, we may denote by $\BoTheS\in\Theta_B$ the reparameterized ground truth satisfying $\mathbf{w}^\star = \mathbf{w}_o+\mathbf{V}_{A^{\perp}}\BoTheS$.
       For the pairwise ranking model, setting $\mathbf{A}=\mathbf{1}^\top$ and $\mathbf{b}=0$ gives rise to $\mathbf{w}_o=\mathbf{0}$ and makes the columns of $\mathbf{V}_{A^\perp}$ an orthonormal basis of $\{\mathbf{1}\}^\perp$, Assumption~\ref{assump_gap_wp_w_B_r0} then follows from $B>1$. The self-centering constraint removes the translation ambiguity inherent in the ranking model, otherwise, the vectors $\mathbf{w}$ and $\mathbf{w}+c\mathbf{1}$ induce the same choice probability for every constant $c$, see~\cite{LeeMinimaxGenPairComp, ShahEstPariCompGrapTop}. An analogous result holds under the general setup in this paper. Its proof is deferred to Appendix~\ref{proof_prop_indentifibility_complete_obv}.
        \begin{proposition}[Identifiability]
        \label{prop_indentifibility_complete_obv}
             Assume that the comparison graph $\mathcal{G}([d],\mathcal{E}_s,n_\Boalp)$ has $K$ connected components, 
             i.e., there exist subsets $\mathcal{C}_k\subseteq[d]$, for $k\in[K]$ such that $[d] = \bigcup_{k\in[K]}\mathcal{C}_k$ and $\mathcal{C}_k\cap \mathcal{C}_l =\varnothing$ for any $k\neq l$.
             Let $\mathbf{C}_I:=[\mathds{1}_{\mathcal{C}_1},\dots,\mathds{1}_{\mathcal{C}_K}]\in\mathbb{R}^{d\times K}$ where the $k-$th column of $\mathbf{C}_I$ is the indicator vector of $\mathcal{C}_k$, 
             i.e., 
                $\left[\mathds{1}_{\mathcal{C}_k}\right]_i =\begin{cases}
                     1,\,&\text{ for } i\in {\mathcal{C}_k}\\
                     0,\,&\text{ for } i\notin {\mathcal{C}_k}
                 \end{cases}.
         $
            Then the following two statements are equivalent.
             \begin{itemize}
                 \item[(i)] 
                 For any $\mathbf{w},\,\mathbf{w}^\prime\in\{\mathbf{w}\in\mathbb{R}^p\mid\mathbf{Aw}=\mathbf{b}\}$, $\{p_\Boalp(\mathbf{w}) = p_\Boalp(\mathbf{w}^\prime)\}_{\Boalp\in\E_s}$ implies $\mathbf{w}=\mathbf{w}^\prime$.

                 \item[(ii)] The joint rank condition holds, i.e.,
                        \begin{equation}
                            \label{eq_identifiable_rank_constraint}
                            \mathrm{rank}\left([\mathbf{C}_I,\mathfrak{A}\mathbf{V}_{A^\perp}]\right) = d_\theta+K.
                        \end{equation}
             \end{itemize}
        \end{proposition}
        Consequently, under Assumptions~\ref{assumpt1_well_Logit} and~\ref{auumpt1self_concor}, condition~\eqref{eq_identifiable_rank_constraint} uniquely identifies the ground-truth parameter $\mathbf{w}^\star\in\mathcal{W}_B$ from the pairwise choice probabilities $\{p_\Boalp(\mathbf{w}^\star)\}_{\Boalp\in\E_s}$.
        Part (ii) of the proposition, partly inspired by \cite[Proposition~1]{FanRankingWithCovar}, 
        provides a verifiable condition for identifiability, and is equivalent to the positive definiteness of the design matrix, as shown in Corollary~\ref{coro_from_identifibality} below. Unlike classical ranking models~\cite{ShahEstPariCompGrapTop, HajekOhMimaxInfPartilRank, YuxinSpectralMLETopK, YangCongTopKMonoAd}, the general linear utility model in Definition~\ref{def_para_utility_func} can remain identifiable on a disconnected comparison graph as the utility parameters are shared across alternatives.
        Examples~\ref{ex_identify_disr_choice_model} and~\ref{ex_cov_assisted_ranking_2} illustrate the disconnected and connected cases. 
        \begin{example}[Discrete Choice Model]
        \label{ex_identify_disr_choice_model}
            This example corresponds to the unconstrained case, where $r_A=0$, $\mathbf{V}_{A^\perp}=\mathbf{I}_p$, and $\mathbf{w}_o=\mathbf{0}$. For instance, let $n=p=d_\theta=2$, $d=4$, and the comparison graph have two connected components $\mathcal{C}_1=\{1,2\}$ and $\mathcal{C}_2=\{3,4\}$. Let $\mathbf{x}_1=[0,0]^\top,\,\,\mathbf{x}_2=[1,1]^\top,\,\,\mathbf{x}_3=[0,1]^\top,\,\,\mathbf{x}_4=[1,0]^\top$, and $\mathfrak{A}=[\mathbf{x}_1,\,\mathbf{x}_2,\,\mathbf{x}_3,\,\mathbf{x}_4]^\top$. Then the matrix
            \begin{equation*}
                [\mathds{1} _{\mathcal{C}_1},\mathds{1}_{\mathcal{C}_2},\mathfrak{A}] = \begin{bmatrix}
                    1&0&0&0\\
                    1&0&1&1\\
                    0&1&0&1\\
                    0&1&1&0
                \end{bmatrix},      
            \end{equation*}
            has rank $4=d_\theta+K$, where $K=2$ is the number of connected components.   
        \end{example}
        \begin{example}[Covariate Assisted Ranking]
        \label{ex_cov_assisted_ranking_2}
            In~\cite[Proposition~1]{FanRankingWithCovar}, the authors show that, if the comparison graph is further assumed to be connected, then the equality constraint in Example~\ref{ex_cov_assisted_ranking_1} always ensures identifiability. To see this, let $d=4$, $n=1$, $K=1$, and take the non-binary covariates as $x_1=0,\, x_2=7,\, x_3=2,\, x_4=1$. Then the feature matrix $\mathfrak{A}$ and linear constraint matrix $\mathbf{A}$ are
            \begin{equation*}
                \mathfrak{A}=\begin{bmatrix}
                    1&0&0&0&0\\
                    0&1&0&0&7\\
                    0&0&1&0&2\\
                    0&0&0&1&1
                \end{bmatrix},\,
                \mathbf{A} = \begin{bmatrix}
                    1&1&1&1&0\\
                    0&7&2&1&0
                \end{bmatrix},
            \end{equation*}
            which gives $\mathrm{rank}(\mathbf{A})=2$, $p=d+n=5$, and $d_\theta=3$.
            We may use any basis of $\mathrm{Null}(\mathbf{A})$ in place of the orthonormal basis $\mathbf{V}_{A^\perp}$ for the rank check, e.g.,
            \begin{equation*}
                \mathbf{v}_1=[-6, -1, 0, 7, 0]^\top,\,\mathbf{v}_2=[-5, -2, 7, 0, 0]^\top,\,\mathbf{v}_3=[0,0,0,0,1]^\top,
            \end{equation*}
            and therefore
            \begin{equation*}
                [\mathbf{1},\,\mathfrak{A} [\mathbf{v}_1,\,\mathbf{v}_2,\,\mathbf{v}_3] ] = \begin{bmatrix}
                    1& -6& -5&0\\
                    1&-1& -2& 7\\
                    1& 0& 7& 2\\
                    1& 7& 0& 1
                \end{bmatrix},       
            \end{equation*}
            has rank $4=d_\theta+K =3+1$.
        \end{example}
    \subsubsection{Likelihood Function}
    \label{subsubsec_likelihood_func}
    We now derive the MLE of $\BoTheS$ under the above reparameterization. Let $N:=\sum_{\Boalp\in\mathcal{E}_s}n_{\Boalp}$ denote the total number of comparisons. 
    For the fixed deterministic design $(\mathcal{E}_s,\{n_\Boalp\}_{\Boalp\in\mathcal{E}_s})$, define the full sample space $(\mathcal{Y}_N,\mathscr{Y}_N)$, where
        \begin{equation*}
            \mathcal{Y}_N : = \prod_{\Boalp\in\mathcal{E}_s} \prod_{t=1}^{n_\Boalp} \mathcal{Y}_\Boalp,\quad \mathscr{Y}_N :
            = \bigotimes_{\Boalp\in\mathcal{E}_s} \bigotimes_{t=1}^{n_\Boalp} \mathscr{Y}_\Boalp.
        \end{equation*}
    The canonical sample variable $Y^{(N)}:\cal{Y}_N\to \cal{Y}_N$ consists of the coordinate projections $Y_\Boalp^{(t)}:\cal{Y}_N \to \cal{Y}_\Boalp$ and is identified with an element in $\{0,1\}^N$ under a fixed ordering of the index set.
    Let $\mathscr{P}(\mathcal{Y}_N,\mathscr{Y}_N)$ denote the set of all probability measures over $(\mathcal{Y}_N,\mathscr{Y}_N)$. For each $\Bothe\in\mathbb{R}^{d_\theta}$ and $\Boalp\in\E_s$, define $\rP_{\Bothe}^{\Boalp} := \rP_{\mathbf{w}_o + \mathbf{V}_{A^\perp}\Bothe}^\Boalp$.
    Under Assumption~\ref{assump_indepent_sample}, the joint distribution of $Y^{(N)}$ is the product of the marginals $\rP_{\Bothe}^{\alpha}$ 
        \begin{equation}
            \label{eq_joint_distribution_law_of_samples}\mathbb{P}_\Bothe^{(N)}:=\bigotimes_{\Boalp\in\mathcal{E}_s}\left(\rP_{\Bothe}^{\alpha}\right)^{\otimes n_{\Boalp}} \in \mathscr{P}(\mathcal{Y}_N,\mathscr{Y}_N).
        \end{equation}
        Any measurable statistic $\hat{T}=\hat{T}(Y^{(N)})$ on $(\mathcal{Y}_N,\mathscr{Y}_N)$ is a random element under $\bbp_{\Bothe}^{(N)}$, whose explicit dependence on $Y^{(N)}$ is suppressed whenever there is no confusion in context.
         For an integrable random variable $Z$ on $(\mathcal{Y}_N,\mathscr{Y}_N)$, let $\bbe_{\Bothe}^{(N)}[Z] : = \int_{\mathcal{Y}_N} Z(\boldsymbol{y}) \mathbb{P}_\Bothe^{(N)}(\mathrm{d} \boldsymbol{y})$, and we use $\mathrm{Var}_{\Bothe}^{(N)}(\cdot)$ and $\mathrm{Cov}_{\Bothe}^{(N)}(\cdot)$ analogously. 
         
       Under Assumptions~\ref{assumpt1_well_Logit} and~\ref{assump_indepent_sample}, it follows by Proposition~\ref{observ_Bernoulli_exp_calss} that the likelihood function of $Y^{(N)}$ given parameter $\mathbf{w}$ is
        \begin{equation}
            f_{Y^{(N)}}(Y^{(N)}\mid \mathbf{w}) = 
            \label{eq_Y_N_likelihood}
            \prod_{\Boalp\in\mathcal{E}_s}\prod_{t=1}^{n_\Boalp}\exp\left[Y_{\Boalp}^{(t)}\frac{\mathbf{w}^\top\mathbf{a}_{\Boalp}}{\sigma} - \log\left(1+\exp\left(\frac{\mathbf{w}^\top\mathbf{a}_{\Boalp}}{\sigma}\right)\right)\right].
        \end{equation}
     The reparameterized negative log-likelihood can then be written as
    \begin{align}
        \ell(\Bothe)&:=-\frac{1}{N}\log\left(f_{Y^{(N)}}\left(Y^{(N)}\mid \mathbf{w}_o + \mathbf{V}_{A^\perp}\Bothe\right)\right)\nonumber\\
        \label{eq_def_of_ell_theta}
        &=\frac{1}{N}\sum_{\Boalp\in\mathcal{E}_s} \left[ -Y_\Boalp\left(\frac{\mathbf{w}_o^\top\mathbf{a}_{\Boalp}+\Bothe^\top\mathbf{V}_{A^\perp}^\top\mathbf{a}_{\Boalp}}{\sigma} \right) + n_\Boalp\Psi\left(\frac{\mathbf{w}_o^\top\mathbf{a}_{\Boalp}+\Bothe^\top\mathbf{V}_{A^\perp}^\top\mathbf{a}_{\Boalp}}{\sigma}\right)\right],
    \end{align}
    where $Y_{\Boalp}:=\sum_{t=1}^{n_{\Boalp}}Y_{\Boalp}^{(t)}$. $\{Y_{\Boalp}\}_{\Boalp\in\E_s}$ is a sufficient statistic under the model specified above and $\Psi(\cdot)$ is defined as in Proposition~\ref{observ_Bernoulli_exp_calss}.
        \begin{definition}[MLE and Firth Correction]
            \label{def_MLE_func}
            Under the identifiability condition~\eqref{eq_identifiable_rank_constraint}, the MLE is the unique solution of the unconstrained convex program
            \begin{equation}
                \label{eq_MLE_def_as_minimizer}
                \hat{\Bothe}\in\arg\min_{\Bothe\in\mathbb{R}^{d_\theta}}\ell(\Bothe),
            \end{equation}
            provided that it exists. The corresponding estimator in the original parameterization is recovered with $\hat{\mathbf{w}} = \mathbf{w}_o + \mathbf{V}_{A^\perp}\hat{\Bothe}$.
 When $\ell(\Bothe)$ is regularized by the Jeffreys' invariant prior,
\begin{equation}
\label{eq_Firth_correction}
\hat{\Bothe}_{\mathrm{Firth}}\in\arg\min_{\Bothe\in\mathbb{R}^{d_\theta}}\ell(\Bothe) - \frac{1}{2N}\log\det(\nabla^2\ell(\Bothe)),
\end{equation}
 is called the Firth correction~\cite{FirthCorrect93}. 
    \end{definition}
    Note that MLE may not exist in that 
    $\ell(\Bothe)$ is not coercive due to the separation in the realization of response data~\cite{ALBERTExistMLELogit}. In contrast, since the regularization term in~\eqref{eq_Firth_correction} goes to infinity when $\Vert\Bothe\Vert_\infty\to\infty$,
    the existence of $\hat{\Bothe}_{\mathrm{Firth}}$ is always guaranteed under the identifiability condition~\cite[Theorem~1]{KosmidisFirth21}. In Sections~\ref{sec_minimax_LB} and~\ref{MLE_upper_bound} we will derive error bounds of MLE under various metrics, and use Firth correction for numerical comparison in Section~\ref{sec_numerical_exp}.

    For notational simplicity, define
    \begin{equation}
    \label{eq_definition_q_alp}
        \mathbf{q}_{\Boalp} := \mathbf{V}_{A^\perp}^{\top}\mathbf{a}_{\Boalp}, \quad \eta_{\Boalp}(\Bothe) :=  \frac{\mathbf{w}_o^{\top}\mathbf{a}_{\Boalp}+\mathbf{q}_{\Boalp}^{\top}\Bothe}{\sigma}, \quad \Boalp\in\mathcal{E}_s,
    \end{equation}
where $\eta_{\Boalp}(\Bothe)$ denotes the log odds of choice probability $p_\Boalp(\mathbf{w}_o+\mathbf{V}_{A^\perp}\Bothe)$.
        By the chain rule
		\begin{subequations}
        \label{eq_gradAndHess_mle_in_theta}
        \begin{align}
            \label{eq_grad_mle_in_theta}
            \nabla\ell(\Bothe)&=\frac{1}{N\sigma}\sum_{\Boalp\in\mathcal{E}_s}\left[-Y_{\Boalp}+ n_{\Boalp}\Psi^\prime\bigl(\eta_{\Boalp}(\Bothe)\bigr)\right]\mathbf{q}_{\Boalp},\\
            \label{eq_hess_mle_in_theta}
            \nabla^2\ell(\Bothe)&=\frac{1}{N\sigma^2}\sum_{\Boalp\in\mathcal{E}_s}n_{\Boalp}\Psi^{\prime\prime}\bigl(\eta_{\Boalp}(\Bothe)\bigr)\mathbf{q}_{\Boalp}\mathbf{q}_{\Boalp}^{\top}.
        \end{align}
        \end{subequations}
        Equation~\eqref{eq_hess_mle_in_theta} shows that
        the difference between $\nabla^2\ell(\Bothe)$ and $\nabla^2\ell(\BoTheS)$ is governed by the query-wise predictor deviations $\eta_{\Boalp}(\Bothe)-\eta_{\Boalp}(\BoTheS)=\frac{\mathbf{q}_{\Boalp}^{\top}(\Bothe-\BoTheS)}{\sigma}$. This observation motivates us to develop lower and upper bounds on this quantity, we will come back to this in Theorems~\ref{thm_residual_coherence_minimax} and~\ref{thm_suffic_sample_l2_upper_bound}, respectively. The FIM can be subsequently calculated from~\eqref{eq_hess_mle_in_theta}
        \begin{align}
			\label{eq_FIM_in_theta}
			\mathcal{I}(\boldsymbol{\theta}^\star) &:= \mathbb{E}_{\BoTheS}^{(N)} [\nabla^2 \ell(\boldsymbol{\theta}^\star)] = \frac{1}{\sigma^2}\mathbf{V}_{A^{\perp}}^\top\mathfrak{A}^\top\left[\frac{1}{N}\sum_{\Boalp\in\mathcal{E}_s} n_{\Boalp}\Psi^{\prime\prime}\left(\eta_{\Boalp}(\BoTheS)\right)(\mathbf{e}_i-\mathbf{e}_j)(\mathbf{e}_i-\mathbf{e}_j)^\top\right]\mathfrak{A}\mathbf{V}_{A^{\perp}},
		\end{align} 
        where the second equality follows from~\eqref{eq_definition_q_alp}. Note that the graph Laplacian of the comparison graph $\mathcal{G}([d],\mathcal{E}_s,\{n_\Boalp\})$ is defined as 
        \begin{equation}
            \label{eq_graph_laplace_def}
            \mathbf{L}:=\frac{1}{N}\sum_{\Boalp\in\mathcal{E}_s}n_{\Boalp}(\mathbf{e}_i-\mathbf{e}_j)(\mathbf{e}_i-\mathbf{e}_j)^\top.
        \end{equation}
        By Propositions~\ref{observ_Bernoulli_exp_calss} and~\ref{prop_bound_dymic_range}, it is then straightforward to verify that the bracketed term in~\eqref{eq_FIM_in_theta} is an approximation of $\mathbf{L}$ up to multiplicative constants depending only on $B$, i.e., for each $\Boalp\in\E_s$, $\frac{1}{4B}\leq\Psi^{\prime\prime}(\eta_{\Boalp}(\BoTheS))\leq\frac{1}{4}$ implies that
        \begin{equation}
        \label{eq_upperlower_sorogate_by_L}
            \frac{1}{4B}\mathbf{L}\preceq \frac{1}{N}\sum_{\Boalp\in\mathcal{E}_s} n_{\Boalp}\Psi^{\prime\prime}\left(\eta_{\Boalp}(\BoTheS)\right)(\mathbf{e}_i-\mathbf{e}_j)(\mathbf{e}_i-\mathbf{e}_j)^\top \preceq \frac{1}{4}\mathbf{L}.
        \end{equation}
        For fixed $\E_s$, the empirical frequencies $\frac{Y_\Boalp}{n_\Boalp}$ converge in probability to $p_\Boalp(\mathbf{w}^\star)$ as $n_\Boalp\to\infty$.
        Thus the population-level problem concerns the recovery of parameters from these choice probabilities. 
        The identifiability of population-level MLE is stated as a corollary of Proposition~\ref{prop_indentifibility_complete_obv}. Its proof is deferred to Appendix~\ref{proof_coro_from_identifibality}.
       \begin{corollary}
           \label{coro_from_identifibality}
           Let $\mathbf{V}_{A^\perp}$, $\mathfrak{A}$, and $\mathbf{L}$ be defined as in~\eqref{eq_transfrom_w_to_theta},~\eqref{eq_def_of_U_phi_matrix}, and~\eqref{eq_graph_laplace_def}, respectively. The rank condition~\eqref{eq_identifiable_rank_constraint} holds iff the matrix 
           \begin{equation}
           \label{eq_def_of_matrix_K}
               \mathbf{W}:=\mathbf{V}_{A^{\perp}}^\top\mathfrak{A}^\top \mathbf{L}\mathfrak{A}\mathbf{V}_{A^{\perp}}\in\mathbb{R}^{d_\theta\times d_\theta},
           \end{equation}
           is positive definite. Moreover, if~\eqref{eq_identifiable_rank_constraint} fails, let $\mathbf{Q}_+$ and $\mathbf{Q}_\perp$ have orthonormal columns spanning the range and null space of $\mathbf{W}$, respectively. Augmenting $\mathbf{Aw}=\mathbf{b}$ with $(\mathbf{V}_{A^\perp}\mathbf{Q}_\perp)^\top\mathbf{w}=\mathbf{0}$ yields an identifiable model. Further, for any ground-truth parameter $\mathbf{w}^\star=\mathbf{w}_o+\mathbf{V}_{A^\perp}\BoTheS$ satisfying Assumption \ref{assumpt1_well_Logit}, there exists an identifiable representative
           \begin{equation}
               \label{eq_projected_identify_rep}\tilde{\mathbf{w}}^\star:=\mathbf{w}_o+\mathbf{V}_{A^\perp}\mathbf{Q}_+\mathbf{Q}_+^\top\BoTheS,
           \end{equation}
           such that $\tilde{\mathbf{w}}^\star$ is the unique parameter in the augmented model satisfying $p_\Boalp(\tilde{\mathbf{w}}^\star)=p_\Boalp(\mathbf{w}^\star)$ for every $\Boalp\in \E_s$.
       \end{corollary}
        Equation~\eqref{eq_projected_identify_rep} projects $\BoTheS$ onto the span of the selected query directions without changing the queried choice probabilities. 
        Unlike the classical ranking models, however, the comparison graph enters the information geometry through both the feature matrix $\mathfrak{A}$ and the null space of $\mathbf{A}$, so that connectivity alone is no longer decisive. For fixed $\mathfrak{A}$ and $\mathbf{A}$, the selected edges and their multiplicities determine $\mathbf{W}$, which approximates the FIM through~\eqref{eq_FIM_in_theta}-\eqref{eq_upperlower_sorogate_by_L}. In Sections~\ref{sec_minimax_LB} and~\ref{MLE_upper_bound}, we will describe this observation with finite-sample theory. 
		\section{Lower Bounds}
        \label{sec_minimax_LB}
        Classical likelihood theory states that the MLE is asymptotically normal, with covariance $\frac{\mathcal{I}(\BoTheS)^{-1}}{N}$ to first order~\cite[Chapters~5,~7]{AsympvanDerVaart} under regularity and non-singularity conditions, but does not specify the finite-sample accuracy of this approximation.
        In this section, we address part of these limitations from two complementary perspectives. In Proposition~\ref{prop_the_CRLB}, we derive the Cramér–Rao inequality for locally unbiased estimators at a fixed ground-truth parameter.
        Theorems~\ref{thm_minimax_loewe_bound_tr_inv} and~\ref{thm_residual_coherence_minimax} then establish minimax lower bounds over $\Theta_B$ for estimators with finite worst-case risk. Together, these results identify the information geometry relevant to the local estimation and worst-case finite-sample difficulty.
        We work in the reparameterized model\footnote{The reparameterization removes the linear equality constraints and avoids working with the relative interior of $\mathcal{W}_B$. 
        The same reduction applies to the minimax risk.}. 
        \subsection{The Cram\'{e}r-Rao Lower Bound}
        We begin with a classic local
        result stated for a fixed ground-truth parameter. Fix $q\in\mathbb{Z}_+$ with $q\leq d_\theta$, and consider a continuously differentiable vector-valued mapping $h:\Theta_B\to \mathbb{R}^q$. An estimator of $h(\Bothe)$ is a measurable mapping $\hat{\boldsymbol{\psi}}:(\mathcal{Y}_N,\mathscr{Y}_N)\to (\mathbb{R}^q,\mathcal{B}(\mathbb{R}^q))$. Given $\BoTheS\in\mathrm{int}(\Theta_B)$, we say $\hat{\boldsymbol{\psi}}$ is a locally unbiased estimator if\footnote{The $k$-th column of $\nabla h(\BoTheS)$ is $\nabla [h(\boldsymbol{\theta})]_k\in\mathbb{R}^{d_\theta}$, i.e., $\nabla h(\BoTheS) = [\nabla [h(\boldsymbol{\theta})]_1,\dots,\nabla [h(\boldsymbol{\theta})]_q]\in\mathbb{R}^{d_\theta\times q}$.}
        \begin{equation*}
            \mathbb{E}_{\BoTheS}^{(N)}[\hat{\boldsymbol{\psi}}]=h(\BoTheS), \quad\text{and} \quad \nabla_{\boldsymbol{\theta}}\mathbb{E}_{\Bothe}^{(N)}[\hat{\boldsymbol{\psi}}]\,\Big|_{\boldsymbol{\theta}=\BoTheS}
        =\nabla h(\BoTheS).
        \end{equation*}
        \begin{proposition}[CRLB]
        \label{prop_the_CRLB}
        Let Assumptions~\ref{assumpt1_well_Logit} and~\ref{assump_indepent_sample} hold, and $\mathcal{I}(\BoTheS)$ be defined as in~\eqref{eq_FIM_in_theta} where $\BoTheS\in\mathrm{int}(\Theta_B)$.
    Let $\hat{\boldsymbol{\psi}}: \mathcal{Y}_{N}\to \mathbb{R}^q$
    be any locally unbiased estimator with $\mathbb{E}_{\BoTheS}^{(N)}[\Vert\hat{\boldsymbol{\psi}}\Vert_2^2]<\infty$.
        Then 
        \begin{subequations}
        \begin{align}
            \label{eq_CRLB_Cov_form}
            \mathrm{Cov}_{\BoTheS}^{(N)}(\hat{\boldsymbol{\psi}}) &\succeq \frac{\nabla h(\BoTheS)^\top\mathcal{I}(\BoTheS)^\dagger\nabla h(\BoTheS)}{N},\\
            \label{eq_CRLB_EuclidNorm_form}
            \mathbb{E}_{\BoTheS}^{(N)} \left[\Vert\hat{\boldsymbol{\psi}} - h(\BoTheS)\Vert_2^2\right] &\geq \frac{\mathrm{tr}(\nabla h(\BoTheS)^\top\mathcal{I}(\BoTheS)^\dagger\nabla h(\BoTheS))}{N}.
        \end{align}           
        \end{subequations}
        \end{proposition}
        We omit the proof as it directly follows from the proof of the Cram\'{e}r–Rao inequality via Schur complement, see~\cite[Section 6.1.1]{PuntanenSchurProofCRLB}.
        Proposition~\ref{prop_the_CRLB} covers two specific cases when $\hat{\boldsymbol{\psi}}$ and $h(\Bothe)$ take particular forms, the next corollary addresses this. These lower error bounds will be used to compare with the minimax rates in Section~\ref{subsec_minimax_lower_bounds}.
        \begin{corollary}[Lower error bounds for locally unbiased estimators of $\BoTheS$]
            \label{coro_of_CRLB}
            Let $\hat{\Bothe}$ be any estimator of $\Bothe$ which is locally unbiased at $\BoTheS$.
            Then the following assertions hold.
            \begin{itemize}
                \item[(i)] 
                Let 
                $\Boalp\in\E_s$ and $\mathbf{q}_\Boalp$ be defined as in \eqref{eq_definition_q_alp}.
                 If $h(\Bothe)=\mathbf{q}_\Boalp^\top\Bothe$, then $\hat{\boldsymbol{\psi}}=\mathbf{q}_\Boalp^\top\hat{\Bothe}$ is a locally unbiased estimator of $h(\Bothe)$ at $\BoTheS$, 
                and $\hat{\Bothe}$ satisfies
                \begin{equation}
                \label{eq_CRLB_pesudo_leverage}
                    \mathbb{E}_{\BoTheS}^{(N)} \left[|\mathbf{q}_\Boalp^\top(\hat{\Bothe}-\BoTheS)|^2\right] \geq \frac{\mathbf{q}_\Boalp^\top\I(\BoTheS)^\dagger\mathbf{q}_\Boalp}{N}.
                \end{equation}

                \item[(ii)] If $\mathcal{I}(\BoTheS)\succ\mathbf{0}$ and $h(\boldsymbol{\theta})=\boldsymbol{\theta}$, then $\hat{\Bothe}$ satisfies
                \begin{equation}
                \label{eq_CRLB_full_rank_ell2}
                    \mathbb{E}_{\BoTheS}^{(N)}  \left[\Vert\hat{\Bothe} - \BoTheS\Vert_2^2\right] \geq \frac{\mathrm{tr}(\mathcal{I}(\BoTheS)^{-1})}{N}.
                \end{equation}
            \end{itemize}
        \end{corollary}
        \begin{proof}
            The conclusion follows straightforwardly from Proposition~\ref{prop_the_CRLB}. \qedbox
        \end{proof}
        \subsection{Minimax Lower Bounds}
        \label{subsec_minimax_lower_bounds}
        To formally link the Fisher information geometry with the intrinsic difficulty of the estimation problem at hand, we follow~\cite[Chapter 2]{TsybakovNonPara} and~\cite[Chapter 15]{WainwrightHighDim} to introduce the minimax lower bound to the utility elicitation problem. Let $\rho:\mathbb{R}^{d_\theta}\times\mathbb{R}^{d_\theta}\to[0,\infty)$ be a measurable semi-distance, i.e., $\rho$ is $\mathcal{B}(\mathbb{R}^{d_\theta})\otimes\mathcal{B}(\mathbb{R}^{d_\theta})/\mathcal{B}([0,\infty))$-measurable, and satisfies all properties of a distance, except that $\rho(\Bothe,\Bothe^\prime)=0$ does not necessarily imply $\Bothe=\Bothe^\prime$. 
        Consider any estimator $\hat{\Bothe}$ of $\Bothe$ as a measurable mapping $\hat{\Bothe}:(\Y_N,\mathscr{Y}_N)\to(\mathbb{R}^{d_\theta},\B(\mathbb{R}^{d_\theta}))$. Define the class of candidate estimators 
        \begin{equation}
        \label{eq_admissible_estimators_A_N}
            \mathcal{A}_N:=\left\{\hat{\Bothe}:\mathcal{Y}_N\to \mathbb{R}^{d_\theta}\,\,\Big|\,\,\sup_{\Bothe\in\Theta_B} \mathbb{E}_\Bothe^{(N)}\left[\rho^2(\hat{\Bothe},\Bothe)\right]< \infty\right\}.
        \end{equation} 
      If $\Theta_B$ has finite $\rho-$diameter, i.e., $\sup_{\Bothe,\Bothe^\prime\in\Theta_B}\rho(\Bothe,\Bothe^\prime)<\infty$, then every finite-valued estimator belongs to $\mathcal{A}_N$. One special case is when the estimator is Lipschitz continuous w.r.t.~samples under the semi-distance $\rho$. To see this, let $\boldsymbol{y}^{(N)}$, $\tilde{\boldsymbol{y}}^{(N)}\in \mathcal{Y}_N$ be two realizations of the response sequence. If the estimator satisfies
      \begin{equation*}
          \rho(\hat{\Bothe}(\boldsymbol{y}^{(N)}),\hat{\Bothe}(\tilde{\boldsymbol{y}}^{(N)})) \leq L \sum_{\Boalp\in\E_s} n_\Boalp\bigg|\frac{\sum_{t=1}^{n_\Boalp}(y_\Boalp^{(t)}-\tilde{y}_\Boalp^{(t)})}{n_\Boalp}\bigg|,\,\,\forall\, \hat{\Bothe}(\boldsymbol{y}^{(N)}),\hat{\Bothe}(\tilde{\boldsymbol{y}}^{(N)})\in\Theta_B,
      \end{equation*}
      where $L<\infty$ is a constant independent of the realizations, then
      \begin{equation*}
          \rho\left(\hat{\Bothe}(\boldsymbol{y}^{(N)}),\Bothe\right)\leq \rho\left(\hat{\Bothe}(\boldsymbol{y}^{(N)}),\hat{\Bothe}(\tilde{\boldsymbol{y}}^{(N)})\right) + \rho\left(\hat{\Bothe}(\tilde{\boldsymbol{y}}^{(N)}),\Bothe\right) <\infty,\,\forall\,\Bothe\in\Theta_B,
      \end{equation*}
       which implies $\hat{\Bothe}\in\mathcal{A}_N$.
      \begin{definition}[Minimax Risk]
      For any $\hat{\Bothe}\in\mathcal{A}_N$ and $\Bothe\in\Theta_B$, let 
        \begin{equation}
		  \mathcal{R}_N(\hat{\Bothe},\Bothe) : = \mathbb{E}_{\Bothe}^{(N)} \left[\rho^2(\hat{\Bothe},\Bothe)\right],
		\end{equation}
        denote the risk when $\hat{\Bothe}$ is used to estimate $\Bothe$. The minimax risk under the squared {\color{black} semi-distance} $\rho^2$ 
        is defined by
        \begin{equation}
			\label{eq_minimax_general_def}
			\mathcal{R}^{\star}_N(\Theta_B,\rho^2):=\inf_{\hat{\Bothe}\in\mathcal{A}_N}\sup_{\Bothe\in\Theta_B}\mathcal{R}_N(\hat{\Bothe},\Bothe).
		\end{equation}   
        \end{definition}
Equation~\eqref{eq_minimax_general_def} describes the performance of the best possible estimator over the adversarial design of the ground truth. By specializing $\rho$ to be the Euclidean norm and assuming $\mathbf{W}\succ\mathbf{0}$\footnote{Indeed, $\mathcal{R}^{\star}_N(\Theta_B,\Vert\cdot\Vert_2^2)$ can be driven to infinity if $\mathbf{W}$ has non-trivial null space, i.e., $\sup_{\Bothe\in\Theta_B}\Vert\Bothe-\Bothe^\prime\Vert_2=\infty$ for any finite $\Bothe^\prime\in\mathbb{R}^{d_\theta}$, rendering any lower bound trivial.}, where $\mathbf{W}$ is defined as in \eqref{eq_def_of_matrix_K}, 
we have
        \begin{equation}
        \label{eq_minimax_risk_ThetaB_ell2}
		  \mathcal{R}^{\star}_N(\Theta_B,\Vert\cdot\Vert_2^2) = \inf_{\hat{\Bothe}\in\mathcal{A}_N}\sup_{\Bothe\in\Theta_B}\mathbb{E}_\Bothe^{(N)}\left[\Vert\hat{\Bothe}-\Bothe\Vert_2^2\right].
		\end{equation}
    The first result in this section aims to find a sharp lower bound on $\mathcal{R}^{\star}_N(\Theta_B,\Vert\cdot\Vert_2^2)$. To this end, we introduce the coherence level.
    \begin{definition}[Coherence Level]
        \label{Def_incoherent_level_of_leverages}
        Let $\mathbf{q}_\Boalp$ and  $\mathbf{W}$ be defined as in  \eqref{eq_definition_q_alp} and \eqref{eq_def_of_matrix_K} respectively. When $\mathbf{W}\succ\mathbf{0}$, the quantity
            \begin{equation}
            \label{eq:conh-level-mu_Q}
                \mu_Q:=\max_{\Boalp\in\E_s} \frac{1}{d_\theta}\mathbf{q}_\Boalp^\top \mathbf{W}^{-1} \mathbf{q}_\Boalp
            \end{equation}
            is called the {\em coherence level} of the questionnaire design. 
    \end{definition}
    $\mu_Q$ quantifies how evenly the selected query vectors cover the parameter space.  A similar concept has also appeared in~\cite[Section 3]{FanRankingWithCovar}, and our definition of $\mu_Q$ coincides with theirs. 
    By the standard matrix leverage identities~\cite[Definition 1.2]{CandesRechtMC}~\cite[Section II]{ChenIncoMC}, $\mu_Q$ satisfies $1\leq\mu_Q\leq \frac{N}{(\min_{\Boalp\in\E_s}n_\Boalp )d_\theta}$\footnote{To see this, observe that $\sum_{\Boalp\in\E_s} \frac{n_\Boalp}{N} \mathbf{q}_\Boalp^\top \mathbf{W}^{-1} \mathbf{q}_\Boalp = \mathrm{tr}\left(\mathbf{W}^{-1} \frac{1}{N}\sum_{\Boalp\in\E_s}n_\Boalp\mathbf{q}_\Boalp\mathbf{q}_\Boalp^\top \right) = d_\theta$, and therefore $\mu_Q \geq1$. On the other hand, since $\mathbf{W}\succeq \frac{n_\Boalp}{N} \mathbf{q}_\Boalp\mathbf{q}_\Boalp^\top$, multiplying both sides by $\mathbf{W}^{-1/2}$ yields the upper bounds, i.e., $\mathbf{q}_\Boalp^\top\mathbf{W}^{-1}\mathbf{q}_\Boalp \leq \frac{N}{n_\Boalp}\leq \frac{N}{\min_{\Boalp\in\E_s}n_\Boalp}$.}. 
    Both upper and lower bounds are attainable, as the following example shows.
\begin{example}
    Consider the setting of Example~\ref{ex_identify_disr_choice_model}, with $d_\theta=2$. Let $\mathbf{q}_{(1,2)} = [-1,-1]^\top$ and $\mathbf{q}_{(3,4)} = [-1,1]^\top$.
    With one observation per pair,
    \begin{equation*}
        \mathbf{W} = \frac{1}{2}\left(\mathbf{q}_{(1,2)}\mathbf{q}_{(1,2)}^\top + \mathbf{q}_{(3,4)}\mathbf{q}_{(3,4)}^\top\right) = \mathbf{I}_2\succ\mathbf{0},
    \end{equation*}
    and $\mathbf{q}_{(1,2)}^\top \mathbf{W}^{-1}\mathbf{q}_{(1,2)} = \mathbf{q}_{(3,4)}^\top \mathbf{W}^{-1}\mathbf{q}_{(3,4)} = 2$, therefore $\mu_Q =1$. Consider now the other case where $n_{(1,2)}=1$, $n_{(3,4)}\geq 2$, and $N=n_{(1,2)} + n_{(3,4)}$.
    In this case,
    \begin{equation*}
        \mathbf{W} = \frac{2}{N} \left( \frac{\mathbf{q}_{(1,2)}}{\sqrt{2}}\left(\frac{\mathbf{q}_{(1,2)}}{\sqrt{2}}\right)^\top + n_{(3,4)}\frac{\mathbf{q}_{(3,4)}}{\sqrt{2}} \left(\frac{\mathbf{q}_{(3,4)}}{\sqrt{2}}\right)^\top \right) \succ\mathbf{0},
    \end{equation*}
    where the middle term is the spectral decomposition of $\mathbf{W}$. Thus $\mathbf{q}_{(1,2)}^\top\mathbf{W}^{-1}\mathbf{q}_{(1,2)} = N$ and $\mathbf{q}_{(3,4)}^\top\mathbf{W}^{-1}\mathbf{q}_{(3,4)} = \frac{N}{n_{(3,4)}}$. Subsequently, $\mu_Q = \frac{1}{2}\max(N,\frac{N}{n_{(3,4)}})= \frac{N}{2}$, which attains the upper bound.
    Moreover, we can see that $\mu_Q$ is affected by both the structure of $\{\mathbf{q}_\Boalp\}_{\Boalp\in\E_s}$ and $\{n_\Boalp\}_{\Boalp\in\E_s}$.
\end{example}
    We will come back to discuss the role of $\mu_Q$ in Theorem~\ref{thm_residual_coherence_minimax}. With all the notation introduced above, we prove the following result, its proof is deferred to Section~\ref{proof_thm_minimax_loewe_bound_tr_inv}.
    \begin{theorem}
		\label{thm_minimax_loewe_bound_tr_inv}
		Let Assumptions~\ref{assumpt1_well_Logit}--\ref{assump_gap_wp_w_B_r0} hold, and $\mathbf{W}\succ\mathbf{0}$. Let $\mathcal{R}^{\star}_N(\Theta_B,\Vert\cdot\Vert_2^2)$ be defined as in~\eqref{eq_minimax_risk_ThetaB_ell2}.
        Then
        \begin{subequations}
        \label{eq_minimax_exp_prob_lowerbound}
            \begin{gather}
                \label{eq_minimax_exp_lowerbound}
				\mathcal{R}^{\star}_N(\Theta_B,\Vert\cdot\Vert_2^2) \gtrsim \sigma^2\mathrm{tr}(\mathbf{W}^{-1})\min\left\{\frac{1}{N}, \frac{\log^2(B/B_0)}{4\mu_Qd_\theta \log(4d)}\right\},\\
                \label{eq_minimax_prob_lowerbound}    \inf_{\hat{\Bothe}\in\mathcal{A}_N}\sup_{\Bothe\in\Theta_B}\mathbb{P}_{\Bothe}^{(N)}\left\{\Vert\hat{\Bothe}-\Bothe\Vert_2^2 \geq \frac{\sigma^2\mathrm{tr}(\mathbf{W}^{-1})}{16}\min\left\{ \frac{1}{N}, \frac{\log^2(B/B_0)}{4\mu_Qd_\theta \log(4d)} \right\} \right\}  \geq \frac{1}{15}.
            \end{gather}
        \end{subequations}
	\end{theorem}
        We are particularly interested in the case where the minimum on the right-hand side of~\eqref{eq_minimax_exp_lowerbound} is attained by the first term, namely when $N \geq \frac{4\mu_Q d_\theta \log(4d)}{\log^2(B/B_0)}$, the minimax lower bound scales as $\frac{\sigma^2 \operatorname{tr}(\mathbf{W}^{-1})}{N}$. For fixed $B$, this isolates the geometry of the questionnaire design in the minimax lower bound, including the case when $\mathbf{W}$ is nearly singular. We will discuss this phenomenon in detail in Corollary~\ref{coro_euclidean_upper_in_detailed_terms}.
    
     To see necessities of the assumptions, we note that 
Assumptions~\ref{assumpt1_well_Logit} and~\ref{assump_indepent_sample} specify the basic probability model whereas Assumptions~\ref{auumpt1self_concor} and~\ref{assump_gap_wp_w_B_r0} play an important role. For fixed $\mathbf{W}$, the lower bound in~\eqref{eq_minimax_exp_prob_lowerbound} tends to zero when $B\downarrow B_0$. This trivial case is excluded by Assumption~\ref{assump_gap_wp_w_B_r0}. On the other hand, if we allow $B=\infty$ (violating Assumption~\ref{auumpt1self_concor}), then the following proposition holds. 
\begin{proposition}
    \label{prop_lim_B_to_infty_minimax_risk_Euclide_to_infty}
    Suppose Assumptions~\ref{assumpt1_well_Logit} and~\ref{assump_indepent_sample} hold, and let $\mathbf{W}\succ\mathbf{0}$ be fixed. Then for every fixed $N<\infty$, and any estimator $\hat{\Bothe}$ of $\Bothe$ that may depend on $B$
    \begin{equation}
        \label{eq_sent_B_infty_minimax_infty}
        \lim_{B\to\infty}\sup_{\Bothe\in\Theta_B}\mathbb{E}_\Bothe^{(N)}\left[\Vert\hat{\Bothe}-\Bothe\Vert_2^2\right]=\infty.
    \end{equation}
\end{proposition}
The proof is deferred to Appendix~\ref{proof_prop_lim_B_to_infty_minimax_risk_Euclide_to_infty}. Equation~\eqref{eq_sent_B_infty_minimax_infty} implies that if $\Theta_B$ is unbounded, every finite-valued estimator has infinite worst-case risk. Moreover, when $B=\infty$, $\mathcal{A}_N=\varnothing$. Both cases render the minimax risk being infinity.
\begin{remark} 
        \label{remark_minimax_loewe_bound_tr_inv}
        It might be interesting to discuss the case when $B$ is fixed whereas $\mathbf{W}$ and $N$ are allowed to vary. We concentrate on the case when $N\geq \frac{4\mu_Qd_\theta \log(4d)}{\log^2(B/B_0)}$.
          \begin{itemize}
              \item[(i)] 
              Theorem~\ref{thm_minimax_loewe_bound_tr_inv} gives a lower bound on the worst-case risk over $\Theta_B$. In particular, it provides a non-asymptotic lower bound for the MLE analysis in Section~\ref{MLE_upper_bound}. From~\eqref{eq_FIM_in_theta}--\eqref{eq_def_of_matrix_K}, we have $\frac{1}{4\sigma^2}\mathbf{W}\succeq \mathcal{I}(\BoTheS)\succeq \frac{1}{4B\sigma^2}\mathbf{W}$, i.e., $\sigma^2\mathrm{tr}(\mathbf{W}^{-1})$ and $\mathrm{tr}(\mathcal{I}(\BoTheS)^{-1})$ are comparable when $B$ is treated as an absolute constant. Thus, the finite-sample Euclidean norm lower error bound of any admissible estimator can be characterized by the same Fisher-information geometry as in \eqref{eq_CRLB_full_rank_ell2}.

              \item[(ii)]  From the experimental design perspective, Theorem ~\ref{thm_minimax_loewe_bound_tr_inv} shows that the statistical quality of questionnaire design is underpinned by $\mathbf{W}=\mathbf{V}_{A^{\perp}}^\top\mathfrak{A}^\top \mathbf{L}\mathfrak{A}\mathbf{V}_{A^{\perp}}$ in~\eqref{eq_def_of_matrix_K}. 
              Unlike a graph Laplacian by itself, $\mathbf{W}$ can be positive definite even when the graph is disconnected, as proved in Proposition~\ref{prop_indentifibility_complete_obv}. Hence, for fixed $\mathfrak{A}$ and $\mathbf{A}$, an A-optimal design allocates the limited budget to comparisons that minimize $\mathrm{tr}(\mathbf{W}^{-1})$, thereby targeting the design criterion in~\eqref{eq_minimax_exp_prob_lowerbound}.
              
              \item[(iii)] Inequality~\eqref{eq_minimax_exp_lowerbound} improves the best known bound in pairwise ranking literature under the BTL model~\cite[Theorem 3]{LeeMinimaxGenPairComp} proved by a Bayesian approach~\cite[Lemma 10]{LeeMiniMaxGLM}. The proof strategy of Theorem~\ref{thm_minimax_loewe_bound_tr_inv} can be applied to generalized linear models with canonical link function and yields similar results by assuming that the cumulant function has a bounded second-order derivative. Inequality~\eqref{eq_minimax_prob_lowerbound} shows that the lower bound is not driven by rare events. In Section~\ref{MLE_upper_bound}, we further show that the canonical MLE attains this minimax benchmark up to logarithmic factors with high probability once $N$ exceeds an explicit design-dependent threshold involving $B$ and $\mu_Q$.
          \end{itemize}
\end{remark}
        We close the discussion of Theorem~\ref{thm_minimax_loewe_bound_tr_inv} by studying the minimax lower bound for the excess risk. Let $\Bov\in\mathbb{R}^{d_\theta}$ be any but fixed. Consider the population-level negative log-likelihood function $\mathcal{L}_\Bothe(\Bov):=\mathbb{E}_\Bothe^{(N)}[\ell(\Bov)]$, and the excess risk $\E_{\mathcal{L}}(\Bov,\Bothe):=\mathcal{L}_\Bothe(\Bov)-\mathcal{L}_\Bothe(\Bothe)$. We can mimic the proof of Theorem~\ref{thm_minimax_loewe_bound_tr_inv} to derive a lower bound for\footnote{Here, with slight abuse of notation,  we let $\mathcal{A}_N=\left\{
        \hat{\Bothe}:\mathcal{Y}_N\to\mathbb{R}^{d_\theta}\, \middle|\,\sup_{\Bothe\in\Theta_B} \mathbb{E}_{\Bothe}^{(N)}\left[\mathcal{E}_{\mathcal{L}}(\hat{\Bothe},\Bothe)\right]<\infty\right\}.$}
        \begin{equation}
        \label{eq_minimax_bound_in_excess_risk}
            \mathcal{R}_N^\star(\Theta_B,\E_{\mathcal{L}}(\cdot,\cdot)) : = \inf_{\hat{\Bothe}\in\mathcal{A}_N}\sup_{\Bothe\in\Theta_B} \mathbb{E}_\Bothe^{(N)}\left[ \mathcal{L}_{\Bothe}(\hat{\Bothe})-\mathcal{L}_{\Bothe}(\Bothe) \right].
        \end{equation}
        $\E_{\mathcal{L}}(\hat{\Bothe},\BoTheS)$ is 
        widely used to measure the performance of an estimator $\hat{\Bothe}$ of $\BoTheS$ in the literature of logistic regression, see e.g.,~\cite[Section 2]{ChardonLogitErrorFinSam} and~\cite[Theorem 1.1]{OstrovsBachSelfCMEst}. 
        The result presented below can be viewed as an analogue of~\cite[Theorem 1]{ShahEstPariCompGrapTop} but stated and proved in a different form. Its proof is deferred to Section~\ref{proof_coro_miniax_lowerbound_excess_risk}.
        \begin{corollary}
        \label{coro_miniax_lowerbound_excess_risk}
            Assume the settings and conditions of Theorem~\ref{thm_minimax_loewe_bound_tr_inv}, let $\mathcal{R}_N^\star(\Theta_B,\E_{\mathcal{L}}(\cdot,\cdot))$ be defined as in~\eqref{eq_minimax_bound_in_excess_risk}. Then
            \begin{subequations}
            \label{eq_minimax_excess_risk}
                \begin{gather}
                \label{eq_minimax_excess_risk_in_exp_form}
                    \mathcal{R}_N^\star(\Theta_B,\E_{\mathcal{L}}(\cdot,\cdot)) \gtrsim \frac{d_\theta}{B}\min\left\{ \frac{1}{N},\frac{\log^2(B/B_0)}{4\mu_Qd_\theta \log(4d)} \right\},\\
                    \label{eq_minimax_excess_risk_in_prob_form} \inf_{\hat{\Bothe}\in\mathcal{A}_N}\sup_{\Bothe\in\Theta_B} \bbp_{\Bothe}^{(N)} \left\{ \E_{\mathcal{L}}(\hat{\Bothe},\Bothe) \geq \frac{d_\theta}{128B}\min\left\{\frac{1}{N},\frac{\log^2(B/B_0)}{4\mu_Qd_\theta \log(4d)}\right\} \right\} \geq \frac{1}{15}.
                \end{gather}
            \end{subequations}
        \end{corollary}
       
       The second result in this section concerns a minimax lower bound in the form of 
        \begin{equation}
        \label{eq_minimax_risk_ThetaB_Q_infty}
            \mathcal{R}_N^\star(\Theta_B,\Vert\cdot\Vert_{Q,\infty}^2) = \inf_{\hat{\Bothe}\in\mathcal{A}_N}\sup_{\Bothe\in\Theta_B}\mathbb{E}_\Bothe^{(N)}\left[\max_{\Boalp\in\E_s} \Big|\mathbf{q}_\Boalp^\top(\hat{\Bothe}-\Bothe)\Big|^2\right],
        \end{equation}
        where $\Vert\cdot\Vert_{Q,\infty}:=\max_{\Boalp\in\E_s} |\mathbf{q}_\Boalp^\top(\cdot)|$measures the largest error along queried directions, and it serves as an analogue of the $\ell_\infty$ norm bound used in ranking literature~\cite{YuxinSpectralMLETopK}. This geometry is also central to the MLE analysis in Section~\ref{MLE_upper_bound}, where we will show that controlling $\Vert\hat{\Bothe}-\BoTheS\Vert_{Q,\infty}$ bounds the drift of $\nabla^2\ell(\hat{\Bothe})$ when compared with $\nabla^2\ell(\BoTheS)$. 
        We assume $\mathbf{W}\succ\mathbf{0}$, 
        under which $\Vert\cdot\Vert_{Q,\infty}$ is a norm. This slightly restrictive assumption is well-justified by Corollary~\ref{coro_from_identifibality}. 
        \begin{theorem}
        \label{thm_residual_coherence_minimax}
              Let $\mathcal{R}_N^\star(\Theta_B,\Vert\cdot\Vert_{Q,\infty}^2)$ be defined as in~\eqref{eq_minimax_risk_ThetaB_Q_infty} and $\mathbf{W}\succ\mathbf{0}$.
             Under Assumptions~\ref{assumpt1_well_Logit}--\ref{assump_gap_wp_w_B_r0},
            \begin{subequations}
                \begin{gather}
                \label{eq_minimax_Q_infty_norm_Exp}
                    \mathcal{R}_N^\star(\Theta_B,\Vert\cdot\Vert_{Q,\infty}^2) \gtrsim \min\left\{ \frac{\mu_Qd_\theta\sigma^2}{N},\sigma^2 \log^2(B/B_0)\right\},\\
                    \label{eq_minimax_Q_infty_norm_Prob}
                    \inf_{\hat{\Bothe}\in\mathcal{A}_N}\sup_{\Bothe\in\Theta_B}\mathbb{P}_{\Bothe}^{(N)}\left\{\Vert\hat{\Bothe}-\Bothe\Vert_{Q,\infty}^2 \geq \min\left\{ \frac{\mu_Qd_\theta\sigma^2}{N},\sigma^2 \log^2(B/B_0)\right\} \right\}  \geq \frac{1}{4}.
                \end{gather}
            \end{subequations}
        \end{theorem}
        Its proof is deferred to Section~\ref{proof_thm_residual_coherence_minimax}.
    \begin{remark} 
    Theorem~\ref{thm_residual_coherence_minimax} complements Theorem~\ref{thm_minimax_loewe_bound_tr_inv} and identifies the coherence level $\mu_Q$ as the design quantity affecting the query-wise minimax rate.
    Some subsequent conclusions can be drawn.
    \begin{itemize}
        \item[(i)] When $N\geq \frac{\mu_Qd_\theta}{\log^2(B/B_0)}$, 
        the lower bound in~\eqref{eq_minimax_Q_infty_norm_Exp} has scale $\frac{\mu_Qd_\theta\sigma^2}{N}$.
        By~\eqref{eq_FIM_in_theta}--\eqref{eq_upperlower_sorogate_by_L}, this is comparable, for fixed $B$, to the largest query-wise CRLB deduced in~\eqref{eq_CRLB_pesudo_leverage}.
        Together with Theorem~\ref{thm_minimax_loewe_bound_tr_inv}, this observation 
        implies that the trace- and coherence-type Fisher information geometry characterizes the best possible performance of any estimator in the non-asymptotic regime, when measured by the Euclidean and $\Vert\cdot\Vert_{Q,\infty}$ norms, respectively.
        
        \item[(ii)] Under
        the pairwise ranking model, the coherence level reduces, up to normalization, to the maximum effective resistance over queried edges~\cite{ChenYanXiRankRisistan}. One may check $\mu_Q=\mathcal{O}(1)$ with high probability under the widely used Erd\"os-R\'enyi comparison graph design, see, e.g.,~\cite{YuxinSpectralMLETopK}~\cite[Section 5.3.3]{TroppIntroMCIFoundT}. Moreover, $\Vert\hat{\Bothe} - \BoTheS\Vert_{Q,\infty}$ can be upper bounded by $\Vert\hat{\mathbf{w}}-\mathbf{w}^\star\Vert_\infty$ up to an absolute constant~\cite[Section~2.3]{ChenSuhSpeMLETopK}. Thus, together with the Erd\"os-R\'enyi graph design, Theorem~\ref{thm_residual_coherence_minimax} recovers a lower bound for the usual coordinatewise $\ell_\infty$ norm error bound up to logarithmic factors, see, e.g.,~\cite[Theorem 7]{YuxinSpectralMLETopK} and~\cite[Theorem 1]{ChenSuhSpeMLETopK}. For general query directions, Madan et al.~\cite{MadanCombiAlgOptiDesign} show that the discrete D-optimal construction yields an explicit bound on $\mu_Q$, see Section~\ref{sec_numerical_exp}.
\end{itemize}
\end{remark}
    
        \section{The MLE Error Rate}
        \label{MLE_upper_bound}
        With theoretical preparations for lower error bounds of general estimators which are applicable to MLE, we are ready to discuss the MLE error rate by deriving upper error bounds and conditions under which the upper and lower bounds meet up to $\log d$ and $B$ factors. A precondition in Section~\ref{sec_minimax_LB} is that the estimators are finite-valued. However, MLE may not exist (see e.g.,~\cite{KonisLPCheckSeperLogit, ALBERTExistMLELogit, LesaffreExistMLELogit}). This motivates us to investigate 
        existence of MLE in the first place before discussing its upper error bounds. In view of Corollary~\ref{coro_from_identifibality}, we focus on the identifiable case $\mathbf{W}\succ \mathbf{0}$ throughout this section.
        We start by specifying the linear stochastic and deterministic bias terms
        \begin{subequations}
            \label{eq_def_of_Delta_lin_and_Delta_bias}
                \begin{align}
                \label{eq_def_of_Delta_lin}
                \Delta_{lin} &:= -\nabla^2\ell(\BoTheS)^{-1}\nabla\ell(\BoTheS),\\
                \label{eq_def_of_Delta_bias}
                \Delta_{bias} & := -\nabla^2\ell(\BoTheS)^{-1}\left(\frac{1}{2N}\sum_{\Boalp\in\E_s}n_\Boalp \frac{\mathbf{q}_\Boalp^\top\nabla^2\ell(\BoTheS)^{-1}\mathbf{q}_\Boalp}{N\sigma^2}\Psi^{(3)}\left( \frac{\mathbf{w}_o^\top \mathbf{a}_{\Boalp} +\mathbf{q}_\Boalp^\top \BoTheS}{\sigma} \right) \frac{\mathbf{q}_\Boalp}{\sigma}\right),
            \end{align} 
        \end{subequations}
        where $\Psi^{(3)}(\cdot)$ is the third-order derivative of the smooth cumulant function $\Psi(\cdot)$.
        The two quantities in~\eqref{eq_def_of_Delta_lin_and_Delta_bias} will be extensively used throughout the analysis in this section. Observe that the linear term is centered, i.e., $\bbe_{\BoTheS}^{(N)} \left[\Delta_{lin}\right]=\mathbf{0}$, and its covariance satisfies
        \begin{equation*}
            \bbe_{\BoTheS}^{(N)} \left[\Delta_{lin}\Delta_{lin}^\top\right] = \frac{\nabla^2\ell(\BoTheS)^{-1}}{N},\,\bbe_{\BoTheS}^{(N)}\left[ \big|\mathbf{q}_\Boalp^\top\Delta_{lin}\big|^2 \right] = \frac{\mathbf{q}_\Boalp^\top \nabla^2\ell(\BoTheS)^{-1}\mathbf{q}_\Boalp}{N},\,\forall\,\Boalp\in\E_s.
        \end{equation*}
        Thus, $\Delta_{lin}$ has exactly the Fisher information geometry predicted by classical asymptotic likelihood theory. The quantity $\Delta_{bias}$ corresponds to the deterministic second-order bias of the canonical logistic regression, see, e.g.,~\cite[Chapter~4]{McCullaghGLM},~\cite[Section~4]{CordeiroMcCuBiasGLM} and~\cite{PortnoAsymExpFamily, FirthCorrect93, KosmidisFirth21}. It can be alternatively identified from the second-order Taylor expansion of the likelihood score equation, as detailed in Section~\ref{subsubsec_taylor_expasion_of_high_order_res}. For a fixed questionnaire design and ground truth, $\Delta_{bias}$ is deterministic. However, it may dominate the linear stochastic term over a nontrivial finite-sample regime in the absence of additional structural restrictions on the questionnaire (Lemmas~\ref{lema_Bias_Envelope_upperbound} and~\ref{lema_bias_term_in_semi_norm_upper}). 
        Consequently, the resulting sufficient sample size for the MLE to be minimax optimal depends quadratically on $d_\theta$.
        We refer interested readers to~\cite{LinSuJackQuadBarriZEst} and references therein for more discussion.
        \begin{theorem}
        \label{thm_suffic_sample_l2_upper_bound}
            Let Assumptions~\ref{assumpt1_well_Logit}--\ref{auumpt1self_concor} hold and $\mathbf{W}\succ\mathbf{0}$,
            let $\mu_Q$ be the coherence level defined in \eqref{eq:conh-level-mu_Q}. If the number of independent samples satisfies $N\gtrsim B^3\mu_Q d_\theta^{3/2}\log^3 d$, then the following hold.
            \begin{itemize}
                \item[(i)] There exists a positive absolute constant $c>1$, such that with probability at least $1-d^{-c}$, the MLE $\hat{\Bothe}$ in Definition~\ref{def_MLE_func} exists and is unique.
                \item[(ii)] The
                estimation error $\hat{\Bothe}-\BoTheS$ satisfies the following inequalities:
                \begin{subequations}
                \label{eq_ell_Q_seminorm_error_upper_hp}
                    \begin{align}
                    \label{eq_ell_Q_error_upper_hp}
                        \mathbb{P}_{\BoTheS}^{(N)}\left\{               \Vert\hat{\Bothe}-\BoTheS\Vert_{Q,\infty} \lesssim \sqrt{\frac{\sigma^2 B \mu_Q d_\theta \log d}{N}} + \frac{\sigma B\mu_Q d_\theta^{3/2}}{N}\right\} &\geq 1-d^{-c},\\
                        \label{eq_seminorm_error_upper_hp}
                        \mathbb{P}_{\BoTheS}^{(N)}\left\{\Vert\hat{\boldsymbol{\theta}} - \BoTheS\Vert_{\nabla^2\ell(\BoTheS)}\lesssim\sqrt{\frac{B d_\theta\log d}{N}}\right\} &\geq 1-d^{-c}.
                    \end{align}
                \end{subequations}
            \end{itemize}
		\end{theorem} 
       We do not integrate~\eqref{eq_ell_Q_seminorm_error_upper_hp} to obtain error bounds in expectation form as in~\cite{ShahEstPariCompGrapTop}, since the unconstrained MLE may not exist. We assess minimax optimality by comparing our high-probability upper bounds with the corresponding probability lower bounds in~\eqref{eq_minimax_prob_lowerbound},~\eqref{eq_minimax_excess_risk_in_prob_form}, and~\eqref{eq_minimax_Q_infty_norm_Prob}. See, e.g.,~\cite[Section~3]{PinhanChaoOptimalMLESuboptSpect} for a similar treatment. 
       
       The proof of Theorem~\ref{thm_suffic_sample_l2_upper_bound} is deferred to Section~\ref{proof_thm_suffic_sample_l2_upper_bound}, where Assumption~\ref{auumpt1self_concor} enters the proof via Proposition~\ref{prop_bound_dymic_range} and Definition~\ref{Def_incoherent_level_of_leverages}, see Lemma~\ref{lema_property_of_A_matrix}.
        The first term on the r.h.s.~of~\eqref{eq_ell_Q_error_upper_hp} matches the minimax rate in Theorem~\ref{thm_residual_coherence_minimax} up to $\log d$ and $B$ factors. The second term in~\eqref{eq_ell_Q_error_upper_hp} signifies that the sufficient condition for existence alone does not immediately yield the same error rate with minimax scaling.
        By isolating the second-order bias and higher-order remainder, Theorem~\ref{thm_refinred_Q_infty_ell_2_residual_decompose_bound} below makes this distinction precise.
        In the present fixed-design setting, the $d_\theta^{3/2}$ sample size scaling improves upon the quadratic one by Ostrovskii and Bach~\cite[Theorem 1.1]{OstrovsBachSelfCMEst}, up to $\log d$ and $B$ factors. Consequently, inequalities~\eqref{eq_ell_Q_error_upper_hp} and~\eqref{eq_seminorm_error_upper_hp}, together with Lemma~\ref{lema_single_variate_self_concord}, yield the excess risk bound stated in the following corollary.
        \begin{corollary}
        \label{coro_excess_risk_from_hessian_seminorm}
            Assume the settings and conditions of Theorem~\ref{thm_suffic_sample_l2_upper_bound}. If the number of independent samples satisfies $N\gtrsim B^3\mu_Q d_\theta^{3/2}\log^3 d$, then there exists a positive absolute constant $c>1$, such that
            \begin{equation}
            \label{eq_excess_risk_whp_upper_bound}
                \bbp_{\BoTheS}^{(N)}\left\{ \mathcal{L}_{\BoTheS}(\hat{\Bothe}) - \mathcal{L}_{\BoTheS}(\BoTheS)\lesssim\frac{B d_\theta\log d}{N} \right\} \geq 1-d^{-c},
            \end{equation}
            where $\mathcal{L}_{\BoTheS}(\cdot):=\mathbb{E}_{\BoTheS}^{(N)}[\ell(\cdot)]$ is defined as in Corollary~\ref{coro_miniax_lowerbound_excess_risk}.
        \end{corollary}
        We omit the proof as it follows directly from mimicking the proof of~\cite[Theorem 3.1]{OstrovsBachSelfCMEst}. Under the assumptions of Corollary~\ref{coro_miniax_lowerbound_excess_risk}, the minimax lower bound measured in excess risk has scale $\frac{d_\theta}{BN}$ whenever $N\geq \frac{4\mu_Q d_\theta \log(4d)}{\log^2(B/B_0)}$ (see~\eqref{eq_minimax_excess_risk}). Comparing this with Corollary~\ref{coro_excess_risk_from_hessian_seminorm} shows the optimality of~\eqref{eq_excess_risk_whp_upper_bound} up to logarithmic and $B$ factors, see also~\cite[Section 2]{ChardonLogitErrorFinSam} for related excess risk guarantees in the logistic regression literature. 
        
        The main technical challenge in proving Theorem~\ref{thm_suffic_sample_l2_upper_bound} is that positive definiteness of $\nabla^2\ell(\Bothe^\star)$ does not provide a uniform non-trivial lower bound on $\nabla^2\ell(\Bothe)$ over $\mathbb{R}^{d_\theta}$, since $\Psi^{\prime\prime}(\eta)\downarrow 0$ as $|\eta|\to\infty$. We must localize the MLE near $\BoTheS$ before obtaining a curvature lower bound of the negative log-likelihood function.
         However, by performing a constrained MLE with prescribed compact feasible set as in~\cite{ShahEstPariCompGrapTop, ZhuRLHFBTLMLEError, ChenLiuVNMElicit, SuMouPropScoreDeBias}, we can derive an error bound of the form~\eqref{eq_seminorm_error_upper_hp} from the standard argument below.
		\begin{lemma}[Strong Convexity Localization]
        \label{lema_basic_cvx_loc_M_ester}
            Let $\Theta\subseteq\mathbb{R}^{d_{\theta}}$, $\BoTheS\in\Theta$, and $\nabla^2\ell(\BoTheS)\succ\mathbf{0}$. Assume that, for some $\kappa>0$, the negative log-likelihood function $\ell(\Bothe)$ is $\kappa$-strongly convex at $\BoTheS$ under the norm $\Vert\cdot\Vert_{\nabla^2\ell(\BoTheS)}$, i.e.,
            \begin{equation*}
                \label{eq_kappa_local_strong_cvx_semiH}
                \ell(\Bothe) - \ell(\BoTheS) -  \langle\nabla \ell(\BoTheS),\Bothe-\BoTheS\rangle \geq  \kappa\Vert\Bothe-\BoTheS\Vert_{\nabla^2\ell(\BoTheS)}^2,\,\forall\,\Bothe\in\Theta.
            \end{equation*}
            Then every $\Bothe\in\{\Bothe\in\Theta\mid \ell(\Bothe)\leq \ell(\BoTheS)\}$ satisfies $\Vert\Bothe-\BoTheS\Vert_{\nabla^2\ell(\BoTheS)}\leq \frac{1}{\kappa} \Vert\nabla \ell(\BoTheS)\Vert_{\nabla^2\ell(\BoTheS)^{-1}}$.
        \end{lemma}
        This localization strategy appears in~\cite[Lemma 9]{ShahEstPariCompGrapTop} and~\cite[Lemma 12]{YuxinSpectralMLETopK}, which underlies the seminorm bounds for constrained estimators. For the fixed design, $\Vert\nabla \ell(\BoTheS)\Vert_{\nabla^2\ell(\BoTheS)^{-1}}^2$ tightly concentrates around its mean, and $\sqrt{N}\nabla^2\ell(\BoTheS)^{-1/2}\nabla\ell(\BoTheS)$ is isotropic, 
        see e.g., Lemma~\ref{lema_Delta_lin_Q_infty_H_bounds}. Consequently, the main results in~\cite{ChenLiuVNMElicit, ShahEstPariCompGrapTop, ZhuRLHFBTLMLEError} may be interpreted, up to logarithmic and constant factors, as high-probability upper bounds in form of $\Vert\Delta\Vert_{\nabla^2\ell(\BoTheS)}^2\lesssim \frac{Bd_{\theta}}{N}$, thus $\Vert\Delta\Vert_2^2 \lesssim  \frac{Bd_{\theta}\Vert\nabla^2\ell(\BoTheS)^{-1}\Vert}{N}$, where $\Delta = \hat{\Bothe}-\BoTheS$. This bound can be substantially larger than the 
        minimax lower bound in Theorem~\ref{thm_minimax_loewe_bound_tr_inv}. We therefore seek conditions under which the MLE is minimax optimal under both $\Vert\cdot\Vert_{Q,\infty}$ and $\Vert\cdot\Vert_2$.
        \begin{theorem}
        \label{thm_refinred_Q_infty_ell_2_residual_decompose_bound}
            Assume the settings and conditions of Theorem~\ref{thm_suffic_sample_l2_upper_bound}. If the number of independent samples satisfies $N\gtrsim B^3\mu_Q d_\theta^{3/2}\log^3 d$, then there exists a positive absolute constant $c>1$, such that
            \begin{align}
                \label{eq_Delta_minu_Deltalin_minu_Delta_bias_Q_infty}
                \bbp_{\BoTheS}^{(N)}\Bigg\{\Vert \hat{\Bothe} - \BoTheS - \Delta_{lin} - \Delta_{bias}\Vert_{Q,\infty} \lesssim &\sigma\frac{B\mu_Q d_\theta^{3/2}}{N}\left[ \frac{1}{d_\theta^{1/4}} + \frac{1}{B^{3/2}\log^2 d}\right] \nonumber\\
                &\quad+ \sqrt{\frac{\sigma^2 B\mu_Q d_\theta \log d}{N}} \frac{1}{d_\theta^{1/4} \log d}\Bigg\}\geq 1-d^{-c}.
            \end{align}
            If, in addition, $N\gtrsim \max\left\{B^3\mu_Q d_\theta^{3/2}\log^3 d,B^4\mu_Q d_\theta \log d\right\}$, then
            \begin{align}
                \bbp_{\BoTheS}^{(N)}\Bigg\{\Vert\hat{\Bothe}-\BoTheS - \Delta_{lin} - \Delta_{bias}\Vert_2 &\lesssim \frac{\sqrt{\Vert\nabla^2\ell(\BoTheS)^{-1}\Vert B\mu_Q}d_\theta}{N}\left[ \frac{1}{d_\theta^{1/4}} + \frac{1}{B^{3/2}\log^{2}d} \right] \nonumber\\
                \label{eq_Delta_minus_lin_bias_upper_bound}
                & \quad\quad\quad\quad\,\,\,\,\,+ \sqrt{\frac{\mathrm{tr}(\nabla^2\ell(\BoTheS)^{-1})\log d_\theta}{N}}\Bigg\}\geq 1-d^{-c}.
            \end{align}
        \end{theorem}
        The proof is deferred to Section~\ref{proof_thm_refinred_Q_infty_ell_2_residual_decompose_bound}. The role of Theorem~\ref{thm_refinred_Q_infty_ell_2_residual_decompose_bound} is parallel to that of~\cite[Theorem~1]{LinSuJackQuadBarriZEst}, which refines the coarse error bound of Theorem~\ref{thm_suffic_sample_l2_upper_bound} by a second-order expansion and makes the transition toward the minimax regime explicit. By Lemmas~\ref{lema_Bias_Envelope_upperbound} and~\ref{lema_Delta_lin_Q_infty_H_bounds}
        \begin{equation*}
            \Vert\Delta_{bias}\Vert_{Q,\infty}\lesssim\sigma\frac{B\mu_Q d_\theta^{3/2}}{N},\quad \Vert\Delta_{lin}\Vert_{Q,\infty}\lesssim \sqrt{\frac{\sigma^2 B\mu_Q d_\theta \log d}{N}},
        \end{equation*}
        where the second inequality holds with high probability under the assumptions of Theorem~\ref{thm_suffic_sample_l2_upper_bound}. By comparing these bounds with~\eqref{eq_Delta_minu_Deltalin_minu_Delta_bias_Q_infty}, we can argue that the canonical MLE attains the $\Vert\cdot\Vert_{Q,\infty}$ minimax rate up to $\log d$ and $B$ factors whenever $N\gtrsim\max\left\{\frac{B\mu_Q d_\theta^2}{\log d}, B^3\mu_Q d_\theta^{3/2}\log^3 d\right\}$\footnote{Indeed, by adapting Bach's convex localization strategy~\cite{OstrovsBachSelfCMEst}, one may also show the MLE attains the minimax optimal rate up to $B$ and $\log d$ factors under $\Vert\cdot\Vert_{Q,\infty}$ when $N\gtrsim B^2\mu_Q d_\theta^2 \log d$.}, and the resulting $d_\theta^2$ dependency is consistent with the quadratic barrier for smooth plug-in estimators based on the MLE~\cite{LinSuJackQuadBarriZEst}. 
        
        We pause to relate Corollary~\ref{coro_excess_risk_from_hessian_seminorm} and Theorem~\ref{thm_refinred_Q_infty_ell_2_residual_decompose_bound} to Theorem~3.2 and Proposition~3.1 of~\cite{PortnoAsymExpFamily}, respectively. For a canonically parameterized exponential family with dimension $d_\theta$ and sample size $N$, Portnoy shows~\cite[Section~3]{PortnoAsymExpFamily}, under suitable moment conditions, that $\frac{d_\theta^{3/2}}{N}\to 0$ is sufficient to guarantee normal approximation of the centered likelihood ratio statistic, but it is in general not enough for obtaining the asymptotic normality of a fixed linear function of the uncorrected MLE, and hence for the query-wise linear predictor in the present setting. Inequality~\eqref{eq_Delta_minu_Deltalin_minu_Delta_bias_Q_infty} may be viewed as a design-dependent, non-asymptotic analogue of~\cite[Proposition~3.1]{PortnoAsymExpFamily}. Both results isolate the leading linear stochastic term and the second-order bias, while controlling the higher-order remainders. 
        On the other hand,~\eqref{eq_excess_risk_whp_upper_bound} bounds the excess risk, whereas~\cite[Theorem~3.2]{PortnoAsymExpFamily} gives a distributional likelihood-ratio approximation. Both results exhibit a $d_\theta^{3/2}/N$ dimensional scaling\footnote{In discussions of dimensional scaling, we omit $\log d$ or $\log d_\theta$ factors and treat $B$ and $\mu_Q$ as uniformly bounded and $\sigma$ as fixed, thus suppressing their effects on $N$.}. The same $d_\theta^{3/2}$ scaling has also appeared in recent analysis of broader classes of $Z$-estimators under stronger regularity conditions~\cite{SuMouPropScoreDeBias, LinSuJackQuadBarriZEst}, see Section~\ref{proof_thm_suffic_sample_l2_upper_bound} for further discussion. 
        
        Similar to~\eqref{eq_Delta_minu_Deltalin_minu_Delta_bias_Q_infty}, inequality~\eqref{eq_Delta_minus_lin_bias_upper_bound} admits a related but more geometry-sensitive structure stated in the following corollary. 
        The proof is deferred to Appendix~\ref{proof_coro_euclidean_upper_in_detailed_terms}.
        \begin{corollary}
        \label{coro_euclidean_upper_in_detailed_terms}
            Assume the settings and conditions of Theorem~\ref{thm_refinred_Q_infty_ell_2_residual_decompose_bound}. Suppose further 
            \bgeqn 
            \label{eq_sample_size_for_trace_Euc_upper_error}
            N\gtrsim \max\left\{ B^3\mu_Q d_\theta^{3/2}\log^3 d,\,
            B^4\mu_Q d_\theta \log d,\,
            \frac{B\mu_Qd_\theta \log d_\theta}{r_{\mathrm{eff}}},\,
            \frac{B\mu_Qd_\theta^2}{r_{\mathrm{eff}}\log d_\theta}\right\},
            \edeqn 
            where $r_{\mathrm{eff}}:=\frac{\mathrm{tr}(\nabla^2\ell(\BoTheS)^{-1})}{\Vert\nabla^2\ell(\BoTheS)^{-1}\Vert}$ denotes the effective rank of $\nabla^2\ell(\BoTheS)^{-1}$. Then there exists a positive absolute constant $c>1$, such that
            \begin{equation}
            \label{eq_delta_hat_upper_bound_ell2_to_minimax_rate}
                \bbp_{\BoTheS}^{(N)}\left\{\Vert\hat\Bothe-\BoTheS\Vert_2^2\lesssim\frac{\mathrm{tr}(\nabla^2\ell(\BoTheS)^{-1})\log d_\theta}{N}\right\} \geq 1-2d^{-c} - d_\theta^{-c}.
            \end{equation}
        \end{corollary}
        Comparing~\eqref{eq_delta_hat_upper_bound_ell2_to_minimax_rate} with~\eqref{eq_minimax_exp_prob_lowerbound},  we can see that the upper bound coincides with the lower bound up to $\log d_\theta$ and $B$ factors. Beyond the baseline sample-size requirements inherited from Theorems~\ref{thm_suffic_sample_l2_upper_bound} and~\ref{thm_refinred_Q_infty_ell_2_residual_decompose_bound}, the dependence of the sample-size requirement on $\nabla^2\ell(\BoTheS)$ is entirely summarized by the effective rank of $\nabla^2\ell(\BoTheS)^{-1}$. Since $1\leq r_{\rm eff}\leq d_\theta$, the resulting dimensional scaling in~\eqref{eq_sample_size_for_trace_Euc_upper_error} ranges from $d_\theta^{3/2}$ to $d_\theta^{2}$ if $B$ and $\mu_Q$ are independent of $d_\theta$. For example, when $r_{\rm eff}=\mathcal{O}(1)$, the quadratic dependence on $d_\theta$ in~\eqref{eq_sample_size_for_trace_Euc_upper_error} is then recovered. 
        Neither the smallest eigenvalue nor the condition number of $\nabla^2\ell(\BoTheS)$ appears in the sample complexity for the MLE to attain the Euclidean minimax rate in Theorem~\ref{thm_minimax_loewe_bound_tr_inv}.
        Consequently, near-singularity of the FIM can enlarge the absolute estimation error at the r.h.s.~of~\eqref{eq_delta_hat_upper_bound_ell2_to_minimax_rate} without causing a corresponding explosion of sample size at the r.h.s.~of~\eqref{eq_sample_size_for_trace_Euc_upper_error}.
        \section{Numerical Experiments}
        \label{sec_numerical_exp}
        This section presents some numerical studies that evaluate the finite-sample performance of the MLE and the Firth correction. We examine how the gap between empirical error of the MLE and the CRLB (\eqref{eq_CRLB_pesudo_leverage} and~\eqref{eq_CRLB_full_rank_ell2} in Corollary~\ref{coro_of_CRLB}) changes with the sample size. To this end, we first use synthetic query pools to control the design spectrum and parameter orientation, and then consider an
        experiment based on flight survey data in~\cite{SanguinettiGreenFlight}.  Figure~\ref{fig_numerical-flow} gives a flowchart of the key steps in 
        this section.
        \begin{figure}[t]
            \centering
            \resizebox{\linewidth}{!}{\input{Tikz_flow_charts/numerical_tikz_auto.tex}}
            \caption{Flowchart of the numerical experiments.}
            \label{fig_numerical-flow}
        \end{figure}
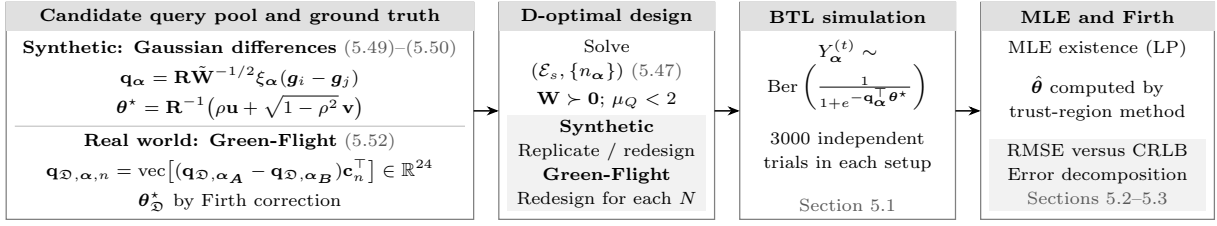
        \subsection{Implementation Details}
        \label{subsec_implemen_detail}
 		
 		We consider a linear utility model without equality constraints throughout this section, i.e., $\mathbf{V}_{A^\perp} = \mathbf{I}_p$, $\mathbf{w}_o = \mathbf{0}$, $\sigma=1$. 
        Let $\mathbb{I}=\{(i,j)\mid 1\leq i<j\leq d\}$. In each experiment, we specify a fixed pool of candidate queries $\{\mathbf{q}_\Boalp\}_{\Boalp\in\mathbb{I}}$ that spans $\mathbb{R}^{d_\theta}$ and a ground-truth parameter $\BoTheS\in\mathbb{R}^{d_\theta}$. Their constructions are given in Sections~\ref{subsec_Synthetic_Questionnaires} and~\ref{subsec_green_flight_question}. Since $\mathbf{V}_{A^\perp} = \mathbf{I}_p$, $\mathbf{q}_\Boalp=\mathbf{a}_{\Boalp}$ and $\BoTheS= {\bf w}^\star.$
        The MLE and Firth correction are solved by standard trust-region method. To obtain an identifiable design with controlled coherence level, we identify the edge weights $\{n_\Boalp\}_{\Boalp\in\E_s}$ by solving the discrete D-optimal problem~\cite{MadanCombiAlgOptiDesign}, i.e.,
        \begin{align}
        \label{eq_D_optimal_problem}
            \max_{n_\Boalp}\,& \log \det\left( \sum_{\Boalp\in\mathbb{I}} n_\Boalp \mathbf{q}_\Boalp\mathbf{q}_\Boalp^\top \right)\\
            \mathrm{s.t.}\,& \sum_{\Boalp\in\mathbb{I}}n_\Boalp \leq N_{base},\,n_\Boalp\geq 0,\,n_\Boalp\in\mathbb{Z},\,\forall\,\Boalp\in\mathbb{I}, \nonumber
        \end{align}
        where $N_{base}\asymp d_\theta$ is the design budget to be specified later. Using the solution, we define $\E_s:=\{\Boalp\in\mathbb{I}\mid n_\Boalp > 0\}$, and set
        \begin{equation}
        \label{eq_def_of_W_by_Doptimal_design_sol}
            \mathbf{W}:=\frac{1}{N_{base}}\sum_{\Boalp\in\E_s} n_\Boalp\mathbf{q}_\Boalp\mathbf{q}_\Boalp^\top.
        \end{equation} 
        We repeat the base blocks $\{n_\Boalp\mathbf{q}_\Boalp\mathbf{q}_\Boalp^\top\}_{\Boalp\in\E_s}$ by a replication factor $r\in \mathbb{Z}_+$ to increase the sample size from $N = N_{base}$ to $N = rN_{base}$ while keeping all other design-dependent coefficients unchanged. Alternatively, we may adopt D-optimal design without repetition\footnote{It can be solved by a similar method as in~\cite{SelinColGenExpDesign}, and the local search algorithm has an analogous approximation guarantee as the one that solves~\eqref{eq_D_optimal_problem}, see, e.g.,~\cite[Section~A.1]{MadanCombiAlgOptiDesign}.}, i.e., replace $n_\Boalp\in\mathbb{Z}$ by $n_\Boalp\in\{0,1\}$ in~\eqref{eq_D_optimal_problem}, but this would involve unnecessary subtlety in the design of numerical experiments for validating the effectiveness of both the MLE and Firth correction. 
        
        By the Kiefer-Wolfowitz theorem (see e.g.,~\cite{KieferWolfowitzThm},~\cite[Chapter~2.2]{ToddMVEEBook}, and~\cite[Chapter~21]{LattimoreBanditAlg}), and the strong duality between the convex relaxation of the D-optimal design and the minimum volume enclosing ellipsoid problem~\cite[Chapter~2]{ToddMVEEBook}, the solution to the convex relaxation of~\eqref{eq_D_optimal_problem} is automatically G-optimal. Ahipasaoglu et al.~\cite{SelinColGenExpDesign} use Fedorov's exchange method (see, e.g.~\cite[Section~4.1.1]{HuanJagalurOptiDesignForandComp}) to refine an initial integer allocation obtained by rounding a continuous solution of the convex relaxation of~\eqref{eq_D_optimal_problem}. We adopt their approach to solve~\eqref{eq_D_optimal_problem} due to its adaptation to candidate sets with large cardinality. Moreover, Algorithm~3 in~\cite{SelinColGenExpDesign} guarantees $\mu_Q \leq \min\left\{ \frac{N_{base}}{d_\theta},\frac{N_{base}}{N_{base}-d_\theta+1}\right\}<2$ for any $N_{base}\geq d_\theta$ based on intermediate results in~\cite[Section~2.2]{MadanCombiAlgOptiDesign}\footnote{Let $\mathbb{I}_s\subseteq \mathbb{I}$ be the search set containing $\E_s$. At a local optimum of the Fedorov's exchange method, Madan et al.~\cite[Lemma~9]{MadanCombiAlgOptiDesign} show that $\max_{\Boalp\in\mathbb{I}_s} \mathbf{q}_\Boalp^\top \mathbf{W}^{-1}\mathbf{q}_\Boalp \leq \frac{N_{base}d_\theta}{N_{base} - d_\theta +1}$. The upper bound on $\mu_Q$ can be derived from combining $\mu_Q\leq \frac{N_{base}}{d_\theta}$, and Madan et al.'s result.}. 
        
        For given $(d_\theta,N,\{\mathbf{q}_\Boalp\}_{\Boalp\in\E_s},\{n_\Boalp\}_{\Boalp\in\E_s})$, we simulate the response $Y_\Boalp^{(t)}\in\{0,1\}$ as a realization of a Bernoulli random variable with success rate $p_\Boalp(\BoTheS)$ for every $t\in[n_\Boalp], \Boalp\in\E_s$, independently. For each questionnaire design and fixed sample size, we run $L=3000$ independent Monte Carlo trials for the response vector $Y^{(N)}$ and calculate the MLE and Firth correction in each trial. Since MLE may not exist when the sample size is small, we use Konis’s linear programming method~\cite{KonisLPCheckSeperLogit} to test the existence of MLE in each trial. When there is no MLE, we use the optimal solution obtained in the box $[-100\Vert\BoTheS\Vert_\infty, 100\Vert\BoTheS\Vert_\infty]^{d_\theta}$. 
        For trial $l$, define 
        \begin{equation*}
            \Delta_l := \hat{\Bothe}^l - \BoTheS,\quad \Delta_{\mathrm{Firth}}^l := \hat{\Bothe}_{\mathrm{Firth}}^l - \BoTheS,
        \end{equation*}
        where $\hat{\Bothe}^l$ ($\hat{\Bothe}_{\mathrm{Firth}}^l$) denotes the MLE (Firth correction) for the $l$-th trial. Let $\mathfrak{L}\subseteq[L]$ index the trials included in a summary. We report the root mean square errors (RMSE)
        \begin{align*}
            \Vert \Delta \Vert_2 \text{ RMSE: }& \sqrt{\frac{1}{|\mathfrak{L}|}\sum_{l\in\mathfrak{L}}\Vert\Delta_l \Vert_2^2},\quad \Vert\Delta\Vert_{Q,\infty} \text{ RMSE: } \max_{\Boalp\in\E_s} \sqrt{\frac{1}{|\mathfrak{L}|}\sum_{l\in\mathfrak{L}} |\mathbf{q}_\Boalp^\top \Delta_l |^2}.
        \end{align*}
        The same definitions apply to $\Delta_{\mathrm{{Firth}}}$ and the random residuals identified in Section~\ref{MLE_upper_bound}. For RMSEs conditioned on MLE existence, $\mathfrak{L}$ contains only trials for which Konis's method reports no separation, so that the MLE exists. We compare the estimators' RMSEs with the corresponding CRLBs in Corollary~\ref{coro_of_CRLB}.

        \subsection{Synthetic Questionnaire}
        \label{subsec_Synthetic_Questionnaires}
        We construct synthetic query pools to examine how $\Delta_{lin}$, $\Delta_{bias}$, and the higher-order terms depend on the design spectrum (see~\eqref{eq_candidate_q_alpha_and_W_mat_def}) and parameter orientation (see~\eqref{eq_candidate_theta_star_def}) in the non-asymptotic regime.
        
        Let $\boldsymbol{g}_i\overset{i.i.d.}{\sim}\mathcal{N}(\mathbf{0},\mathbf{I}_{d_\theta})$, for $i\in[d]$. For a user-specified parameter $m\leq \binom{d}{2}$, we sample uniformly without replacement from the set $\mathbb{I}=\{(i,j)\mid 1\leq i<j\leq d\}$ to construct a subset $\tilde{\mathbb{I}}\subseteq\mathbb{I}$ with cardinality $m$, and solve~\eqref{eq_D_optimal_problem} with $\mathbb{I}$ replaced by $\tilde{\mathbb{I}}$.
        With a Rademacher sequence $\{\xi_{\Boalp}\}_{\Boalp\in\tilde{\mathbb{I}}}\in\{-1,1\}^{m}$, define
        \begin{equation*}
            \tilde{\mathbf{q}}_\Boalp : = \xi_{\Boalp}(\boldsymbol{g}_i - \boldsymbol{g}_j),\quad \tilde{\mathbf{W}} : = \frac{1}{m}\sum_{\Boalp\in\tilde{\mathbb{I}}}\tilde{\mathbf{q}}_{\Boalp}\tilde{\mathbf{q}}_{\Boalp}^\top.
        \end{equation*}
        Let $\mathbf{u}\sim\mathrm{Unif}\,(\mathbb{S}^{d_\theta-1})$ and $\mathbf{v}\sim\mathrm{Unif}\,(\mathbb{S}^{d_\theta-1}\cap \mathbf{u}^\perp)$ such that $\mathbf{u}^\top\mathbf{v}=0$.
        For $\epsilon\in(0,1]$, set $\mathbf{R} :=\mathbf{I}_{d_\theta} + (\sqrt{\epsilon}-1)\mathbf{uu}^\top$ and reshape the spectrum of $\tilde{\mathbf{W}}$ as
        \begin{equation}
        \label{eq_candidate_q_alpha_and_W_mat_def}
            \mathbf{q}_\Boalp := \mathbf{R}\tilde{\mathbf{W}}^{-1/2}\tilde{\mathbf{q}}_{\Boalp},\,\hat{\mathbf{W}}: = \frac{1}{m}\sum_{\Boalp\in\tilde{\mathbb{I}}}\mathbf{q}_\Boalp\mathbf{q}_\Boalp^\top = \mathbf{RR}^\top = \mathbf{I}_{d_\theta} + (\epsilon-1)\mathbf{uu}^\top.
        \end{equation}
        \begin{table}[t]
        \centering
        \caption{Exponent $\alpha_{\mathrm{cross}}:=\log_{d_\theta}N_{\mathrm{cross}}$ across the parameter grid. Each entry reports the Monte Carlo mean $\pm$ standard deviation of $\alpha_{\mathrm{cross}}$.}
        \label{table_CRLB_bias_Euc_crossover_mean_std}
        \scriptsize
        \setlength{\tabcolsep}{1.5pt}
        \renewcommand{\arraystretch}{1}
        \begin{minipage}[t]{0.49\textwidth}
        \vspace{0pt}
        \centering
        \textbf{(a)} $d_\theta=30$, $N_{\mathrm{base}}=45$\par
        \vspace{2pt}
        \resizebox{\linewidth}{!}{%
        \begin{tabular}{@{}c@{\hspace{2pt}}|@{\hspace{2pt}}*{5}{c}@{}}
          \toprule
          \multicolumn{1}{c}{} &
          \multicolumn{5}{c}{$r_{e,tar}$} \\
          \cmidrule(lr){2-6}
          \multicolumn{1}{c}{$\rho$} & 2 & 3.94 & 7.75 & 15.2 & 30 \\
          \midrule
          0.5 & $1.13 \pm 0.037$ & $0.99 \pm 0.031$ & $0.88 \pm 0.028$ & $0.80 \pm 0.028$ & $0.74 \pm 0.029$ \\
          0.6 & $1.22 \pm 0.032$ & $1.07 \pm 0.028$ & $0.93 \pm 0.027$ & $0.82 \pm 0.027$ & $0.74 \pm 0.028$ \\
          0.7 & $1.30 \pm 0.030$ & $1.13 \pm 0.027$ & $0.98 \pm 0.026$ & $0.85 \pm 0.027$ & $0.74 \pm 0.028$ \\
          0.8 & $1.36 \pm 0.029$ & $1.19 \pm 0.027$ & $1.03 \pm 0.027$ & $0.87 \pm 0.027$ & $0.74 \pm 0.028$ \\
          0.9 & $1.42 \pm 0.029$ & $1.25 \pm 0.026$ & $1.07 \pm 0.027$ & $0.90 \pm 0.028$ & $0.74 \pm 0.030$ \\
          \bottomrule
        \end{tabular}%
        }
      \end{minipage}\hfill
      \begin{minipage}[t]{0.49\textwidth}
        \vspace{0pt}
        \centering
        \textbf{(b)} $d_\theta=50$, $N_{\mathrm{base}}=75$\par
        \vspace{2pt}
        \resizebox{\linewidth}{!}{%
        \begin{tabular}{@{}c@{\hspace{2pt}}|@{\hspace{2pt}}*{5}{c}@{}}
          \toprule
          \multicolumn{1}{c}{} &
          \multicolumn{5}{c}{$r_{e,tar}$} \\
          \cmidrule(lr){2-6}
          \multicolumn{1}{c}{$\rho$} & 2 & 4.47 & 10 & 22.4 & 50 \\
          \midrule
          0.5 & $1.18 \pm 0.026$ & $1.02 \pm 0.020$ & $0.88 \pm 0.017$ & $0.78 \pm 0.016$ & $0.72 \pm 0.016$ \\
          0.6 & $1.26 \pm 0.025$ & $1.09 \pm 0.019$ & $0.93 \pm 0.017$ & $0.80 \pm 0.016$ & $0.72 \pm 0.016$ \\
          0.7 & $1.33 \pm 0.025$ & $1.15 \pm 0.019$ & $0.98 \pm 0.017$ & $0.83 \pm 0.016$ & $0.72 \pm 0.016$ \\
          0.8 & $1.39 \pm 0.024$ & $1.21 \pm 0.019$ & $1.03 \pm 0.017$ & $0.86 \pm 0.017$ & $0.72 \pm 0.017$ \\
          0.9 & $1.44 \pm 0.024$ & $1.26 \pm 0.019$ & $1.07 \pm 0.017$ & $0.89 \pm 0.017$ & $0.72 \pm 0.018$ \\
          \bottomrule
        \end{tabular}%
        }
      \end{minipage}
    
      \vspace{0.8em}
    
      \begin{minipage}[t]{0.49\textwidth}
        \vspace{0pt}
        \centering
        \textbf{(c)} $d_\theta=70$, $N_{\mathrm{base}}=105$\par
        \vspace{2pt}
        \resizebox{\linewidth}{!}{%
        \begin{tabular}{@{}c@{\hspace{2pt}}|@{\hspace{2pt}}*{5}{c}@{}}
          \toprule
          \multicolumn{1}{c}{} &
          \multicolumn{5}{c}{$r_{e,tar}$} \\
          \cmidrule(lr){2-6}
          \multicolumn{1}{c}{$\rho$} & 2 & 4.86 & 11.8 & 28.8 & 70 \\
          \midrule
          0.5 & $1.21 \pm 0.019$ & $1.04 \pm 0.014$ & $0.89 \pm 0.013$ & $0.78 \pm 0.011$ & $0.71 \pm 0.011$ \\
          0.6 & $1.29 \pm 0.018$ & $1.11 \pm 0.013$ & $0.94 \pm 0.012$ & $0.80 \pm 0.012$ & $0.71 \pm 0.011$ \\
          0.7 & $1.35 \pm 0.017$ & $1.17 \pm 0.013$ & $0.99 \pm 0.012$ & $0.83 \pm 0.012$ & $0.71 \pm 0.012$ \\
          0.8 & $1.41 \pm 0.017$ & $1.22 \pm 0.012$ & $1.03 \pm 0.012$ & $0.86 \pm 0.012$ & $0.71 \pm 0.012$ \\
          0.9 & $1.46 \pm 0.017$ & $1.27 \pm 0.012$ & $1.08 \pm 0.012$ & $0.89 \pm 0.013$ & $0.71 \pm 0.013$ \\
          \bottomrule
        \end{tabular}%
        }
      \end{minipage}\hfill
      \begin{minipage}[t]{0.49\textwidth}
        \vspace{0pt}
        \centering
        \textbf{(d)} $d_\theta=90$, $N_{\mathrm{base}}=135$\par
        \vspace{2pt}
        \resizebox{\linewidth}{!}{%
        \begin{tabular}{@{}c@{\hspace{2pt}}|@{\hspace{2pt}}*{5}{c}@{}}
          \toprule
          \multicolumn{1}{c}{} &
          \multicolumn{5}{c}{$r_{e,tar}$} \\
          \cmidrule(lr){2-6}
          \multicolumn{1}{c}{$\rho$} & 2 & 5.18 & 13.4 & 34.7 & 90 \\
          \midrule
          0.5 & $1.24 \pm 0.017$ & $1.05 \pm 0.013$ & $0.89 \pm 0.012$ & $0.78 \pm 0.010$ & $0.71 \pm 0.010$ \\
          0.6 & $1.31 \pm 0.016$ & $1.12 \pm 0.012$ & $0.95 \pm 0.011$ & $0.80 \pm 0.010$ & $0.71 \pm 0.010$ \\
          0.7 & $1.37 \pm 0.015$ & $1.18 \pm 0.011$ & $0.99 \pm 0.010$ & $0.83 \pm 0.010$ & $0.71 \pm 0.010$ \\
          0.8 & $1.43 \pm 0.015$ & $1.23 \pm 0.010$ & $1.04 \pm 0.010$ & $0.86 \pm 0.010$ & $0.71 \pm 0.010$ \\
          0.9 & $1.47 \pm 0.015$ & $1.28 \pm 0.010$ & $1.08 \pm 0.009$ & $0.89 \pm 0.009$ & $0.71 \pm 0.009$ \\
          \bottomrule
        \end{tabular}%
        }
      \end{minipage}
    \end{table}
    Thus, $\hat{\mathbf{W}}$ has eigenvalue $\epsilon$ along $\mathbf{u}$ and eigenvalue $1$ on $\{\mathbf{u}\}^\perp$. Its inverse has effective rank
     \begin{equation*}
         r_{e,tar}:=\frac{\mathrm{tr}(\hat{\mathbf{W}}^{-1})}{\Vert\hat{\mathbf{W}}^{-1}\Vert} = 1 + (d_\theta-1)\epsilon.
     \end{equation*}
     Note that $\hat{\mathbf{W}}$ uniformly averages all queries in the candidate pool, whereas $\mathbf{W}$ in~\eqref{eq_def_of_W_by_Doptimal_design_sol} uses the selected D-optimal design, and $\mathcal{I}(\BoTheS)$ additionally incorporates the ground-truth variance, see~\eqref{eq_FIM_in_theta}--\eqref{eq_def_of_matrix_K}. The effective ranks of their inverses therefore may not coincide. We specify the ground truth by
        \begin{equation}
        \label{eq_candidate_theta_star_def}
            \BoTheS = \mathbf{R}^{-1}\left(\rho \mathbf{u} + \sqrt{1-\rho^2}\mathbf{v}\right),\,\rho\in[0,1].
        \end{equation}
        For fixed $\{\tilde{\mathbf{q}}_\Boalp\}_{\Boalp\in\tilde{\mathbb{I}}}$, the pair $(\mathbf{u},\mathbf{v})$ and $\rho$, 
        \begin{equation*}
            \eta_\Boalp(\BoTheS) =\frac{\mathbf{w_o}^\top\mathbf{a}_\Boalp + (\mathbf{V}_{A^\perp}\Bothe^\star)^\top\mathbf{a}_\Boalp}{\sigma}= \mathbf{q}_\Boalp^\top \BoTheS = \tilde{\mathbf{q}}_\Boalp^\top \tilde{\mathbf{W}}^{-1/2}\left(\rho \mathbf{u} + \sqrt{1-\rho^2}\mathbf{v}\right)
        \end{equation*}
        and it satisfies $\frac{1}{m}\sum_{\Boalp\in\tilde{\mathbb{I}}} \eta_\Boalp(\BoTheS)^2 = 1$. Since $\eta_\Boalp(\BoTheS)$ does not depend on $\epsilon$, $\epsilon$ does not affect the choice probabilities in~\eqref{eq_BTL_basic_scaled_alpha}. From direct calculation, $\frac{|\mathbf{u}^\top \BoTheS|}{\Vert\BoTheS\Vert_2} = \frac{\rho}{\sqrt{\rho^2 + \epsilon(1-\rho^2)}}$, thus $\rho$ and $\epsilon$ jointly control the alignment between $\BoTheS$ and $\mathbf{u}$.
        Moreover, the solution of~\eqref{eq_D_optimal_problem} is also invariant to $\epsilon$ because
        \begin{equation}
        \label{eq_Doptimal_sol_not_effected_by_eps_and_reff}
            \log\det\left( \sum_{\Boalp\in\tilde{\mathbb{I}}} n_\Boalp \mathbf{q}_\Boalp\mathbf{q}_\Boalp^\top \right) = \log\epsilon + \log\det\left( \sum_{\Boalp\in\tilde{\mathbb{I}}} n_\Boalp \tilde{\mathbf{W}}^{-1/2}\tilde{\mathbf{q}}_{\Boalp}\tilde{\mathbf{q}}_{\Boalp}^\top\tilde{\mathbf{W}}^{-1/2}\right).
        \end{equation}
  We draw $100$ independent realizations of the candidate query pool and the ground-truth parameter as described above with $r_{e,tar}=d_\theta$\footnote{The choice $r_{e,tar}=d_\theta$ implies $\epsilon=1$, from which $\mathbf{R}$ reduces to $\mathbf{I}_{d_\theta}$, and the ground truth is purely determined by $\rho$ and the random directions $\mathbf{u}$, $\mathbf{v}$.}, $\rho=0.5$, $m=2000$, $d=4d_\theta$, $d_\theta \in \{30,50,70,90\}$ and solve~\eqref{eq_D_optimal_problem} with $N_{base}=1.5d_\theta$. According to~\eqref{eq_candidate_q_alpha_and_W_mat_def},~\eqref{eq_candidate_theta_star_def} and~\eqref{eq_Doptimal_sol_not_effected_by_eps_and_reff}, the solution of~\eqref{eq_D_optimal_problem} is invariant to $(r_{e,tar}, \rho)$. Moreover, for a fixed design $\{n_\Boalp\}_{\Boalp\in\E_s}$, the CRLB (see~\eqref{eq_CRLB_full_rank_ell2}) and $\Vert\Delta_{bias}\Vert_2$ (see~\eqref{eq_def_of_Delta_bias}) are functions of $(r_{e,tar}, \rho)$. Thus, for different $(r_{e,tar}, \rho)$ values, we can calculate the CRLB and $\Vert\Delta_{bias}\Vert_2$ based on the case when $(r_{e,tar}, \rho) = (d_\theta,0.5)$.
  Note that when a base query design is replicated, $\mathcal{I}(\BoTheS)$ remains unchanged ($n_\Boalp$ and $N$ has the same scale in~\eqref{eq_FIM_in_theta}). Thus 
  \begin{equation*}
      \sqrt{\frac{\mathrm{tr}(\mathcal{I}(\BoTheS)^{-1})}{N}}= \sqrt{\frac{\mathrm{tr}(\mathcal{I}(\BoTheS)^{-1})}{N_{base}}} \frac{1}{\sqrt{r}},\quad
  \Vert\Delta_{bias}(N)\Vert_2 =\frac{\Vert\Delta_{bias}(N_{base})\Vert_2}{r}.
  \end{equation*}
  The two quantities are equal when
  $N_{cross} = N_{base}^2 \frac{\Vert\Delta_{bias}(N_{base})\Vert_2^2}{\mathrm{tr}(\mathcal{I}(\BoTheS)^{-1})}$. We report the mean values and standard deviations of $\alpha_{\mathrm{cross}}:=\log_{d_\theta}N_{\mathrm{cross}}$ in Table~\ref{table_CRLB_bias_Euc_crossover_mean_std} for different values of $(r_{e,tar}, \rho)$. In general, $N_{cross}$ increases with $\rho$ and decreases with $r_{e,tar}$, and can be substantially larger than $d_\theta$. When $r_{e,tar} = d_\theta$, the pool spectrum is isotropic and the distribution of $N_{cross}$ is invariant to $\rho$.
        \begin{figure}[p]
	       \centering
            \setlength{\abovecaptionskip}{0cm}
            \setlength{\belowcaptionskip}{0cm}
            \begin{minipage}{15.5cm}
            \centering
            \begin{subfigure}[t]{0.5\linewidth}
            \centering
            \includegraphics[width=\linewidth]{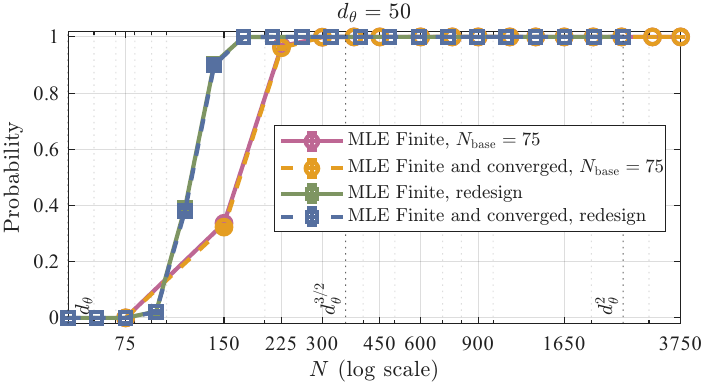}
            \caption{}
            \end{subfigure}\hfill
            \begin{subfigure}[t]{0.5\linewidth}
            \centering
            \includegraphics[width=\linewidth]{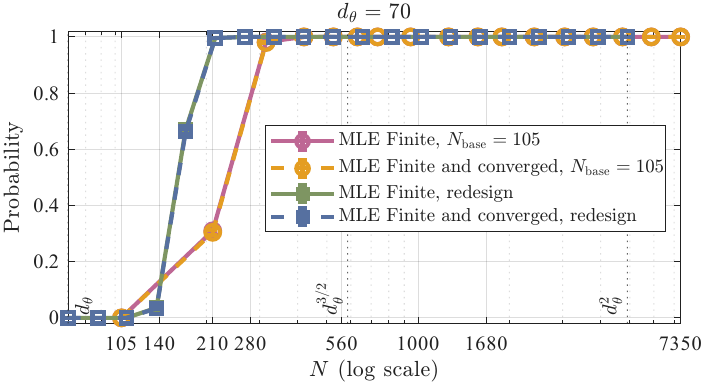}
            \caption{}
            \end{subfigure}
            \end{minipage}
	       \caption{MLE existence and trust-region convergence probabilities versus $N$ under base-design replication and D-optimal redesign, estimated from 3000 Monte Carlo trials, using the synthetic query pool: (a) $d_\theta = 50$; (b) $d_\theta = 70$. An MLE is classified as finite when Konis’s method reports no separation. An MLE is classified as converged if the trust-region method returns a solution with gradient norm below $10^{-6}$. The vertical dashed lines indicate the values of $d_\theta$, $d_\theta^{3/2}$ and $d_\theta^2$.}
	       \label{fig_Dopt_MLE_existence_convergence_transition}
        \end{figure}
        \begin{figure}[p]
        \centering
        \setlength{\abovecaptionskip}{0cm}
        \setlength{\belowcaptionskip}{0cm}
            \begin{minipage}{15.5cm}
            \centering
            \begin{subfigure}[t]{0.5\linewidth}
            \centering
            \includegraphics[width=\linewidth]{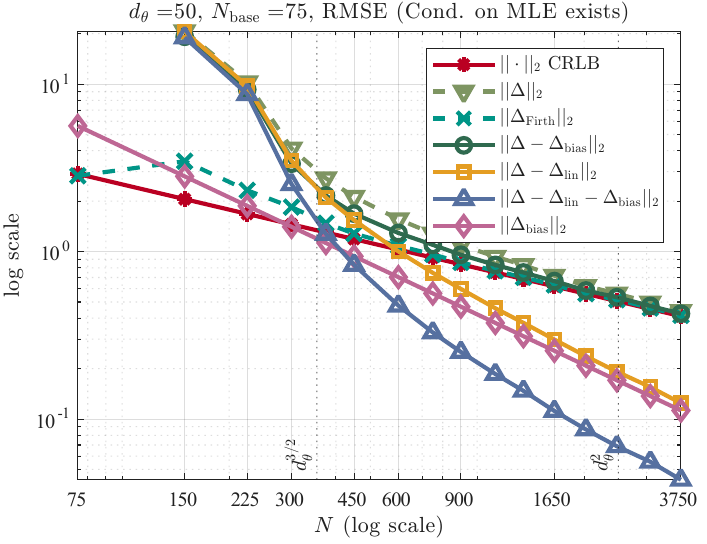}
            \caption{}
            \end{subfigure}\hfill
            \begin{subfigure}[t]{0.5\linewidth}
            \centering
            \includegraphics[width=\linewidth]{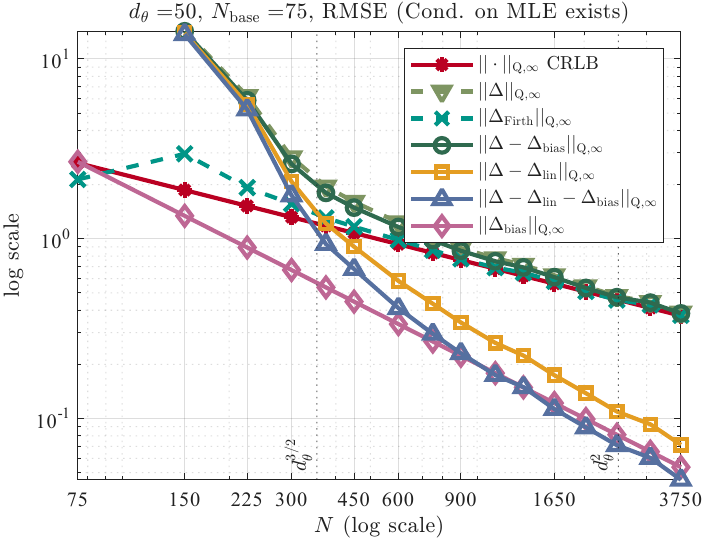}
            \caption{}
            \end{subfigure}
            \par\vspace{2pt}
            \begin{subfigure}[t]{0.5\linewidth}
            \centering
            \includegraphics[width=\linewidth]{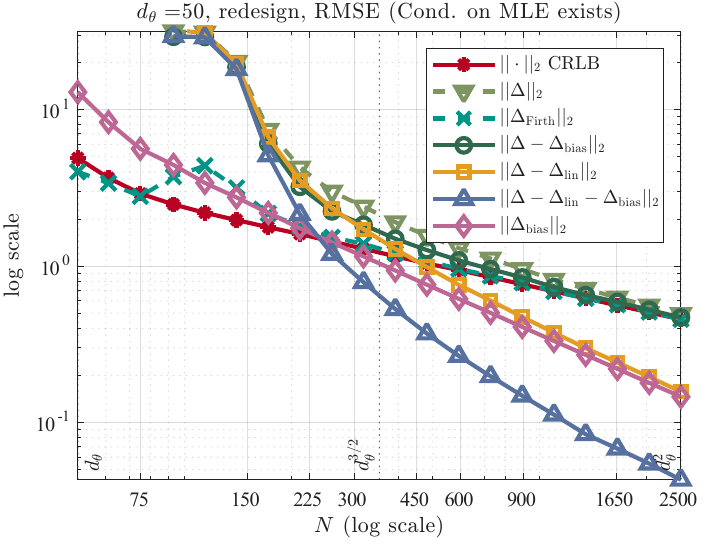}
            \caption{}
            \end{subfigure}\hfill
            \begin{subfigure}[t]{0.5\linewidth}
            \centering
            \includegraphics[width=\linewidth]{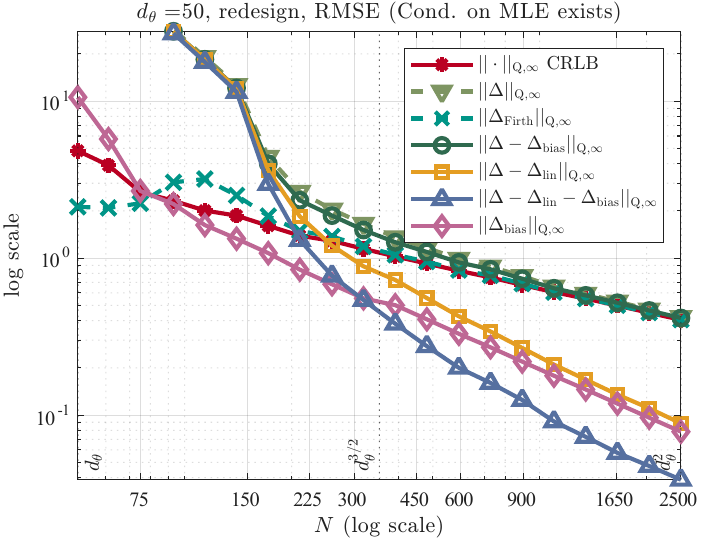}
            \caption{}
            \end{subfigure}
            \end{minipage}
        \caption{Finite-sample performance using the synthetic query pool under a replicated D-optimal design with $(d_\theta,N_{base}) = (50,75)$ (subplots (a)--(b)) and a D-optimal design recomputed for each $N$ (subplots (c)--(d)). These plots compare the MLE and Firth RMSEs with the CRLB, the second-order bias, and higher-order residuals. For each $N$, the MLE-based RMSEs are calculated conditional on the existence of a finite MLE, and are omitted when the MLE existence probability is less than $0.02$. Firth RMSEs are computed over all 3000 trials. The vertical dashed lines indicate the values of $d_\theta$, $d_\theta^{3/2}$ and $d_\theta^2$. For compact figure labels, ``$\Vert\cdot\Vert_{Q,\infty}$ CRLB'' denotes the square root of the r.h.s. of~\eqref{eq_CRLB_pesudo_leverage}, maximized over $\Boalp\in\E_s$.}
        \label{fig_Dopt_d_50_Nbase_75}
        \end{figure}
        
        We next fix $\rho=0.9$, $r_{e,tar} = 2$ and examine $d_\theta \in \{50, 70\} $ with $N_{base} = 1.5d_\theta$. 
        For each $d_\theta$, we select the underlying realization in the Table~\ref{table_CRLB_bias_Euc_crossover_mean_std} experiment whose $N_{cross}$ is closest to the empirical median,
        and fix the corresponding candidate query pool and ground-truth parameter in the subsequent Monte Carlo simulation. We record the trials in which the MLE exists and the trust-region algorithm converges, and plot the corresponding empirical probabilities in Figure~\ref{fig_Dopt_MLE_existence_convergence_transition}.
        To determine whether the observed behavior is an artifact of replicating a base design, we compare the replicated design with a discrete D-optimal design recomputed for each $N$. 
        Under both schemes, we evaluate the MLE and Firth correction using the RMSE criterion defined in Section~\ref{subsec_implemen_detail}, with results for $d_\theta=50$ and $70$ plotted in Figures~\ref{fig_Dopt_d_50_Nbase_75} and~\ref{fig_Dopt_d_70_Nbase_105}, respectively.

        \begin{figure}[!t]
	       \centering
            \setlength{\abovecaptionskip}{0cm}
            \setlength{\belowcaptionskip}{0cm}
            \begin{minipage}{15.5cm}
            \centering
            \begin{subfigure}[t]{0.5\linewidth}
            \centering
            \includegraphics[width=\linewidth]{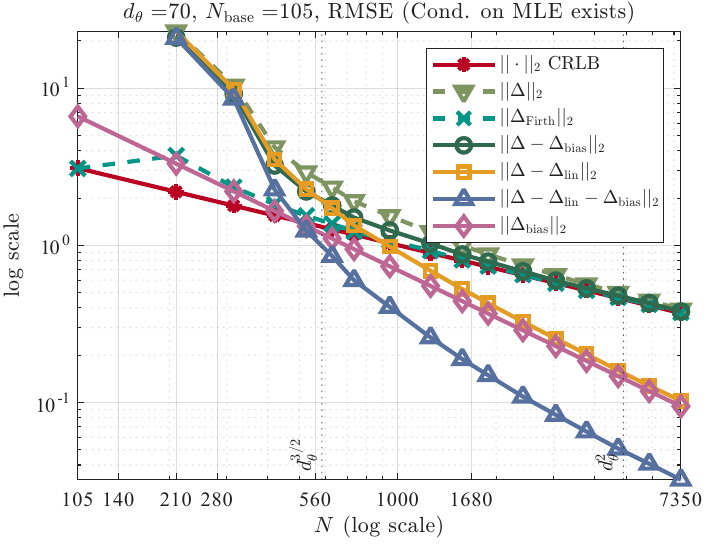}
            \caption{}
            \end{subfigure}\hfill
            \begin{subfigure}[t]{0.5\linewidth}
            \centering
            \includegraphics[width=\linewidth]{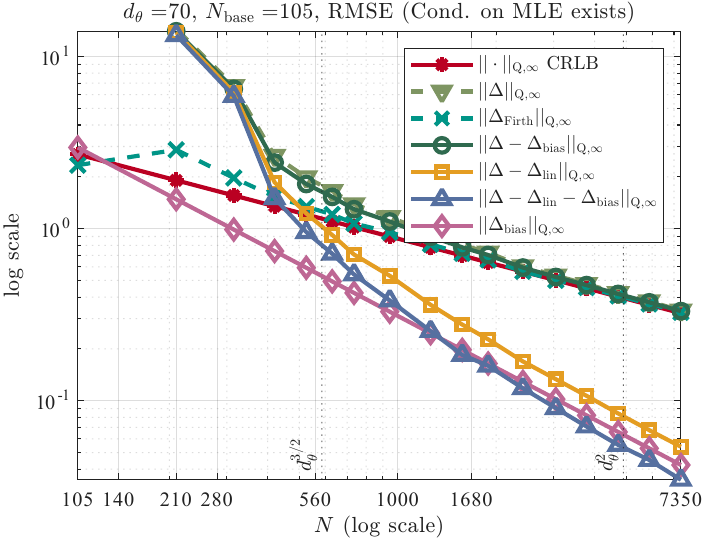}
            \caption{}
            \end{subfigure}
            \par\vspace{2pt}
            \begin{subfigure}[t]{0.5\linewidth}
            \centering
            \includegraphics[width=\linewidth]{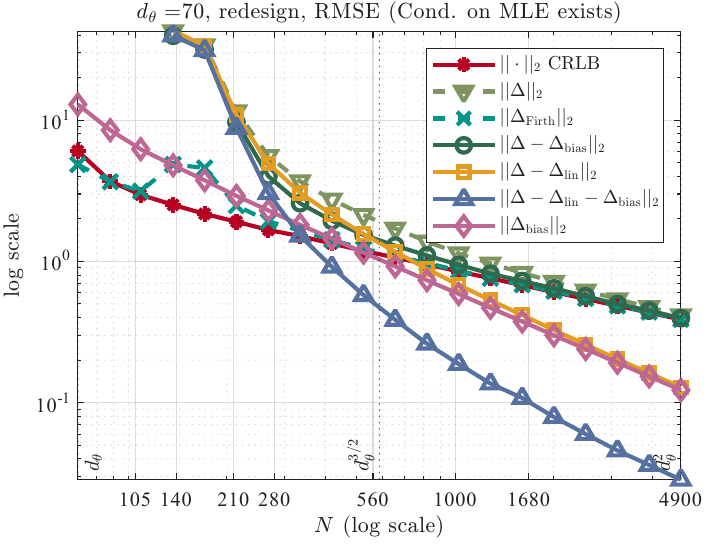}
            \caption{}
            \end{subfigure}\hfill
            \begin{subfigure}[t]{0.5\linewidth}
            \centering
            \includegraphics[width=\linewidth]{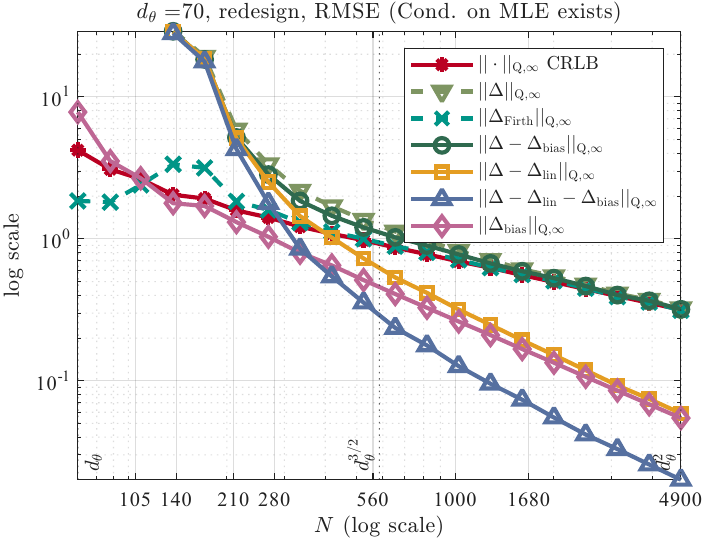}
            \caption{}
            \end{subfigure}
            \end{minipage}
	       \caption{Same as Figure~\ref{fig_Dopt_d_50_Nbase_75}, but with $(d_\theta, N_{base}) = (70,105)$.}
	       \label{fig_Dopt_d_70_Nbase_105}
        \end{figure}
        Figures~\ref{fig_Dopt_MLE_existence_convergence_transition}--\ref{fig_Dopt_d_70_Nbase_105} reveal that the sample size required for finiteness of the MLE is smaller than the sample size required for statistical efficiency (i.e., its RMSE approaches the CRLB).
        Figure~\ref{fig_Dopt_MLE_existence_convergence_transition} shows that the MLE does not exist in most Monte Carlo trials when $d_\theta\leq N\leq 2 d_\theta$ ($d_\theta = 50$ in Figure~\ref{fig_Dopt_MLE_existence_convergence_transition}~(a), and $d_\theta = 70$ in Figure~\ref{fig_Dopt_MLE_existence_convergence_transition}~(b)).
        On the other hand, the probability of the MLE being finite approaches one before $N$ reaches $d_\theta^{3/2}$. Thus, the sufficient $d_\theta^{3/2}$-scale condition in Theorem~\ref{thm_suffic_sample_l2_upper_bound} is conservative for existence in this specific instance. Next, we use Figure~\ref{fig_Dopt_d_50_Nbase_75}~(a) to describe the finite-sample behavior of the MLE. The RMSE curves involving $\Delta$ are computed conditional on MLE existence and are displayed only when the MLE existence probability exceeds $0.02$. The $\Vert\cdot\Vert_2$ CRLB, $\Vert\Delta_{bias}\Vert_2$, and $\Vert\Delta_{\mathrm{Firth}}\Vert_2$ RMSE curves are plotted over the full sample-size range. The $\Vert\Delta\Vert_2$ RMSE curve can be described in roughly four stages according to sample size. When $N\leq 150$, the MLE does not exist in most trials. For $150\leq N\leq 450$, the decrease in $\Vert\Delta\Vert_2$ RMSE primarily reflects the decay of the nonlinear residuals, as shown by the $\Vert\Delta-\Delta_{lin}-\Delta_{bias}\Vert_2$ and $\Vert\Delta-\Delta_{lin}\Vert_2$ RMSE curves. 
        At $N\approx 450$ (about $1.27d_\theta^{3/2}$), the $\Vert\Delta-\Delta_{lin}-\Delta_{bias}\Vert_2$ RMSE falls below $\Vert\Delta_{bias}\Vert_2$ and tends to vanish, while the $\Vert\Delta-\Delta_{lin}\Vert_2$ RMSE eventually approaches $\Vert\Delta_{bias}\Vert_2$ as the sample size grows. This observation is consistent with the error bound in Theorem~\ref{thm_refinred_Q_infty_ell_2_residual_decompose_bound}. 
        In the case $450\leq N\leq 900$, the log-log slope of the $\Vert\Delta\Vert_2$ RMSE curve closely tracks that of $\Vert\Delta_{bias}\Vert_2$, although the bias is already smaller than the CRLB. Finally, for $N\geq 1650$, the $\Vert\Delta\Vert_2$ RMSE nearly coincides with the CRLB. This final transition is consistent with Corollary~\ref{coro_euclidean_upper_in_detailed_terms}, which requires both the higher-order remainder and the second-order bias to be small relative to the linear stochastic term. Similar behavior is observed under the query-wise metric and when the D-optimal design is recomputed for each $N$, as shown in the remaining subplots of Figures~\ref{fig_Dopt_d_50_Nbase_75} and~\ref{fig_Dopt_d_70_Nbase_105}.

        By contrast, as seen from Figure~\ref{fig_Dopt_d_50_Nbase_75}~(a), the Firth correction exhibits more stable finite-sample behavior. When $150\leq N\leq 450$, the $\Vert\Delta_{\mathrm{Firth}}\Vert_2$ RMSE is comparable to $\Vert\Delta_{bias}\Vert_2$, and then it approaches the CRLB as $N$ increases beyond $450$. Analogous phenomena are observed in the remaining subplots of Figures~\ref{fig_Dopt_d_50_Nbase_75} and~\ref{fig_Dopt_d_70_Nbase_105}. These results support the Firth correction as a more reliable default for the finite-sample, fixed-design preference elicitation problem considered here, and we conjecture that it is minimax optimal when $N\gtrsim d_\theta^{3/2}$, up to $B$, $\mu_Q$, and logarithmic factors. To the best of our knowledge, however, non-asymptotic error bounds for the Firth correction in high-dimensional settings are not currently available.
        \subsection{The Green-Flight Questionnaire}
        \label{subsec_green_flight_question}
        \begin{figure}[t]
            \centering
            \setlength{\abovecaptionskip}{0cm}
            \setlength{\belowcaptionskip}{0cm}
            \begin{minipage}{15.5cm}
            \centering
            \begin{subfigure}[t]{0.5\linewidth}
            \centering
            \includegraphics[width=\linewidth]{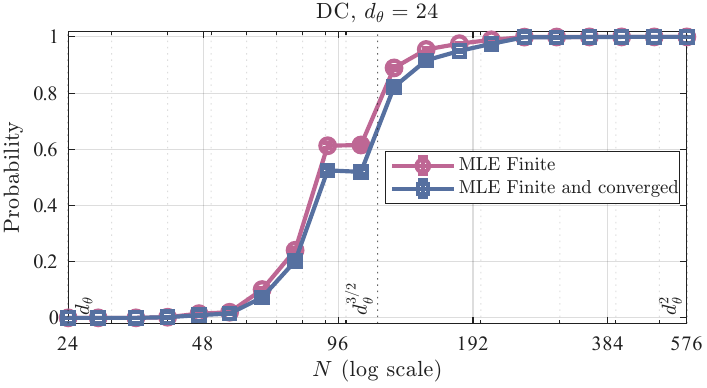}
            \caption{}
            \end{subfigure}\hfill
            \begin{subfigure}[t]{0.5\linewidth}
            \centering
            \includegraphics[width=\linewidth]{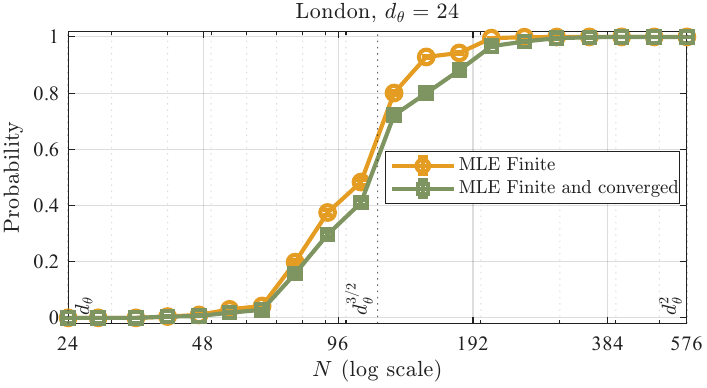}
            \caption{}
            \end{subfigure}
            \end{minipage}
            \caption{MLE existence and trust-region convergence probabilities versus $N$ under D-optimal redesign estimated from $3000$ Monte Carlo trials, using the Green-Flight query pool: (a) DC; (b) London. An MLE is classified as finite when Konis’s method reports no separation. An MLE is classified as converged if the trust-region method returns a solution with gradient norm below $10^{-6}$. The vertical dashed lines indicate the values of $d_\theta$, $d_\theta^{3/2}$ and $d_\theta^2$.
            }
            \label{fig_Greenfly_MLE_existence_convergence_transition}
        \end{figure}
        \begin{figure}[t]
	       \centering
	       \setlength{\abovecaptionskip}{0cm}
	       \setlength{\belowcaptionskip}{0cm}
           \begin{minipage}{15.5cm}
            \centering
            \begin{subfigure}[t]{0.5\linewidth}
            \centering
            \includegraphics[width=\linewidth]{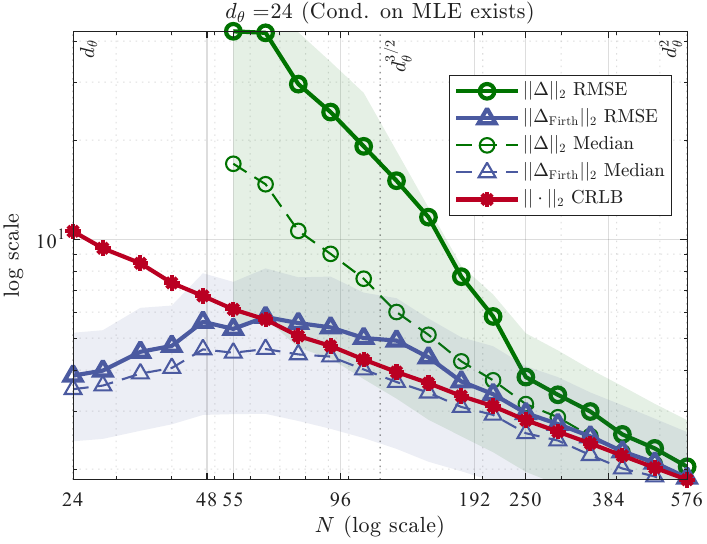}
            \caption{}
            \end{subfigure}\hfill
            \begin{subfigure}[t]{0.5\linewidth}
            \centering
            \includegraphics[width=\linewidth]{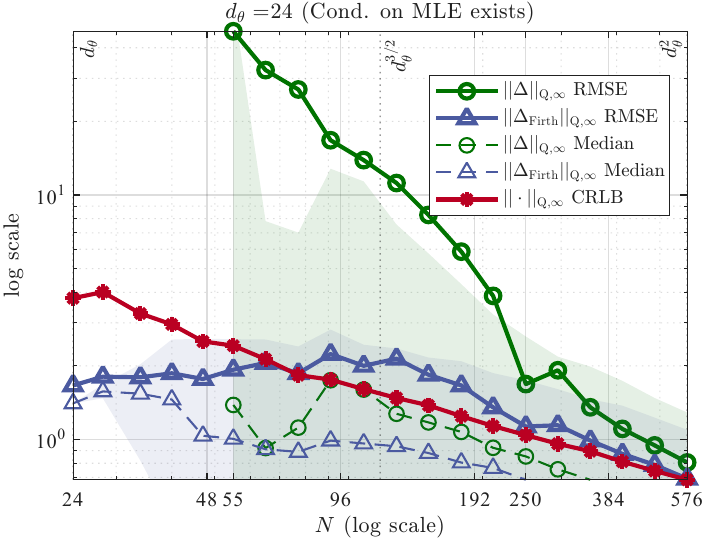}
            \caption{}
            \end{subfigure}
            \par\vspace{2pt}
            \begin{subfigure}[t]{0.5\linewidth}
            \centering
            \includegraphics[width=\linewidth]{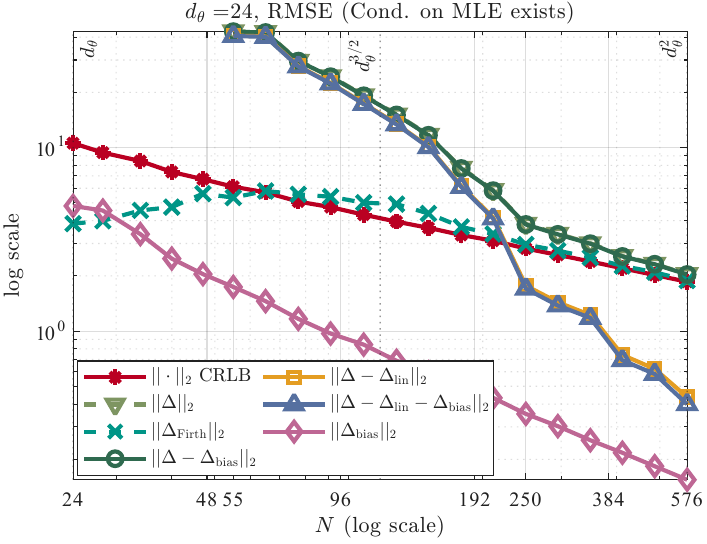}
            \caption{}
            \end{subfigure}\hfill
            \begin{subfigure}[t]{0.5\linewidth}
            \centering
            \includegraphics[width=\linewidth]{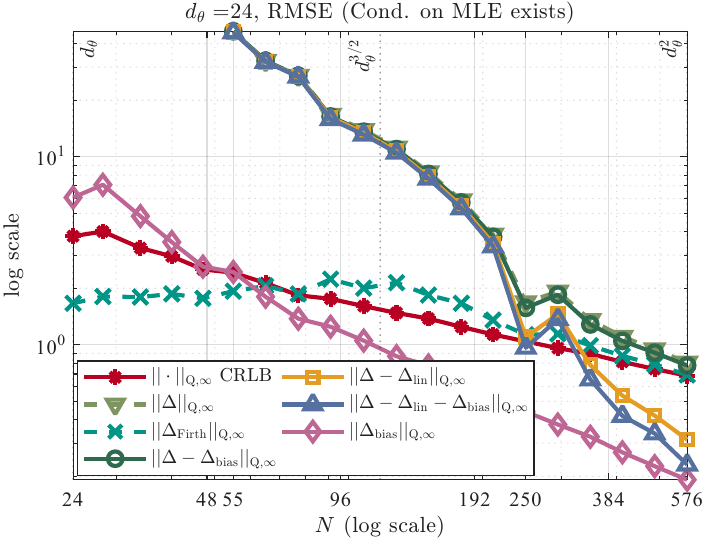}
            \caption{}
            \end{subfigure}
            \end{minipage}
	       \caption{Finite-sample performance for DC flights using the Green-Flight query pool, with the D-optimal design recomputed for each $N$. MLE-based RMSEs are calculated conditional on the existence of a finite MLE and are omitted when the MLE existence probability is below $0.02$. Firth RMSEs are computed over all $3000$ trials. The dashed curves and shaded regions in subplots (a)–-(b) report the corresponding medians and empirical $10\%–90\%$ percentile ranges, using the same trial sets as their RMSE counterparts. For each $N$, the query-wise medians and percentile intervals are computed from the absolute errors along the query direction that attains the corresponding maximum RMSE. Subplots (c)--(d) show the error decompositions defined as in Figure~\ref{fig_Dopt_d_50_Nbase_75}. The vertical dashed lines indicate the values of $d_\theta$, $d_\theta^{3/2}$ and $d_\theta^2$.}
	       \label{fig_DC_VarN_design}
        \end{figure}
        We construct a semi-synthetic experiment from a real-world survey conducted among 450 employees of the University of California, Davis~\cite{SanguinettiGreenFlight}. Respondents compared hypothetical flights to Washington DC and London, described by price (in \$), $\mathrm{CO}_2$ emissions (in $\mathrm{lb}$), departure airport (SMF or SFO), and non-stop status. Certain features of the respondents are also recorded, from which we retain age, number of flights during the past year, preferred airport (SMF or SFO), current position in the university (staff, faculty, post-doc, graduate student, and others), and preferred flight arrangement method (institution-mediated or via web portal), as the contextual information of a respondent. The survey data are publicly available at~\url{https://doi.org/10.25338/B81S5M}.
        
        After excluding respondents with missing contextual information or incomplete response sequences, we obtain a dataset with $367$ respondents and $48$ pairwise flight ticket templates for each destination. Each respondent answered $5$ DC and $6$ London questions, yielding $367\times 5 =1835$ and $367\times 6 =2202$ observed choices, respectively. For a destination $\mathfrak{D}\in\{\mathrm{DC},\mathrm{London}\}$, a pairwise flight ticket template $\Boalp$, and an alternative $j\in\{A,B\}$, define feature vectors of tickets $\mathbf{q}_{\mathfrak{D},\Boalp_j}$ and respondents $\mathbf{c}_n$ as
        \begin{align*}
            \mathbf{q}_{\mathfrak{D},\Boalp_j} &= \left[\frac{{\mathrm{CO}_2}_{\mathfrak{D},\Boalp_j}}{1000 \text{ lb}}, \frac{\mathrm{Cost}_{\mathfrak{D},\Boalp_j}}{100\text{ \$}}, \mathds{1}\{\text{non-stop}\},\mathds{1}\{ \text{departure from SFO} \}\right]^\top\in\mathbb{R}^4,\\
            \mathbf{c}_n & = \left[ 1,\text{age}_n,\text{flights}_n, \mathds{1}\{\text{preferred SFO}\},\mathds{1}\{\text{Faculty}\},\mathds{1}\{\text{institution-mediated}\}\right]^\top \in\mathbb{R}^6,
        \end{align*}
        where the last indicator in $\mathbf{c}_n$ equals one when booking is made through the Aggie Travel portal or arranged by administrative staff, and $\Boalp_A$ ($\Boalp_B$) denotes the first (second) ticket in a pairwise template. Following the contextual choice model in~\cite{LeeLowrankContexPrefer}, we adopt a bilinear utility specification with the query vector
        \begin{equation}
        \label{eq_contextual_ques_pool_set}
            \mathbf{q}_{\mathfrak{D},\Boalp,n} = \mathrm{vec}\left((\mathbf{q}_{\mathfrak{D},\Boalp_A}-\mathbf{q}_{\mathfrak{D},\Boalp_B})\mathbf{c}_n^\top\right)\in\mathbb{R}^{24},
        \end{equation}
        which encodes interactions between the ticket feature differences in template $\Boalp$ and contextual information of respondent $n$. With $\boldsymbol{\Theta}_\mathfrak{D}^\star\in\mathbb{R}^{4\times 6}$ and $\boldsymbol{\theta}^\star_\mathfrak{D} : =\mathrm{vec}(\boldsymbol{\Theta}^\star_\mathfrak{D})$, the corresponding utility score difference is
        \begin{equation*}
            U(\mathbf{q}_{\mathfrak{D},\Boalp_A},\mathbf{c}_n) - U(\mathbf{q}_{\mathfrak{D},\Boalp_B},\mathbf{c}_n) := \left(\mathbf{q}_{\mathfrak{D},\Boalp_A}-\mathbf{q}_{\mathfrak{D},\Boalp_B}\right)^\top\boldsymbol{\Theta}_\mathfrak{D}^\star \mathbf{c}_n = (\boldsymbol{\theta}^\star_\mathfrak{D})^\top\mathbf{q}_{\mathfrak{D},\Boalp,n}.
        \end{equation*}
        Here we assume that there is no model mis-specification error. Statistical guarantees of this contextual choice model then follow from results established in Sections~\ref{sec_minimax_LB}--\ref{MLE_upper_bound} with $d_\theta = 24$.
        \begin{figure}[!t]
            \centering
            \setlength{\abovecaptionskip}{0cm}
            \setlength{\belowcaptionskip}{0cm}
            \begin{minipage}{15.5cm}
            \centering
            \begin{subfigure}[t]{0.5\linewidth}
            \centering
            \includegraphics[width=\linewidth]{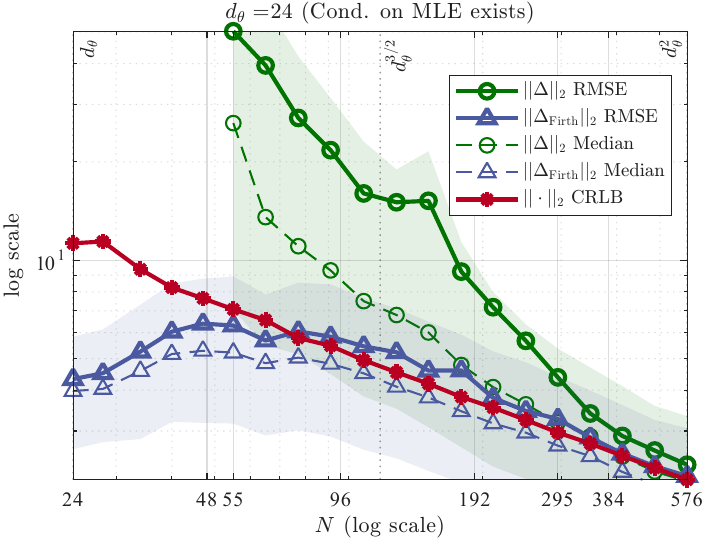}
            \caption{}
            \end{subfigure}\hfill
            \begin{subfigure}[t]{0.5\linewidth}
            \centering
            \includegraphics[width=\linewidth]{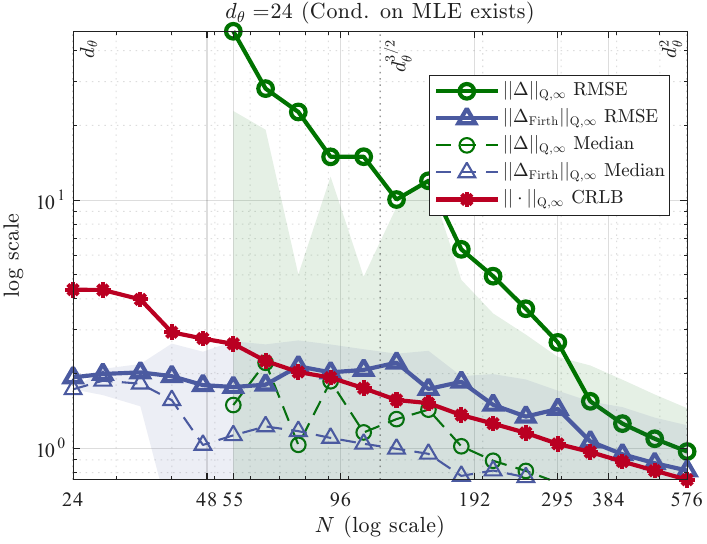}
            \caption{}
            \end{subfigure}
            \par\vspace{2pt}
            \begin{subfigure}[t]{0.5\linewidth}
            \centering
            \includegraphics[width=\linewidth]{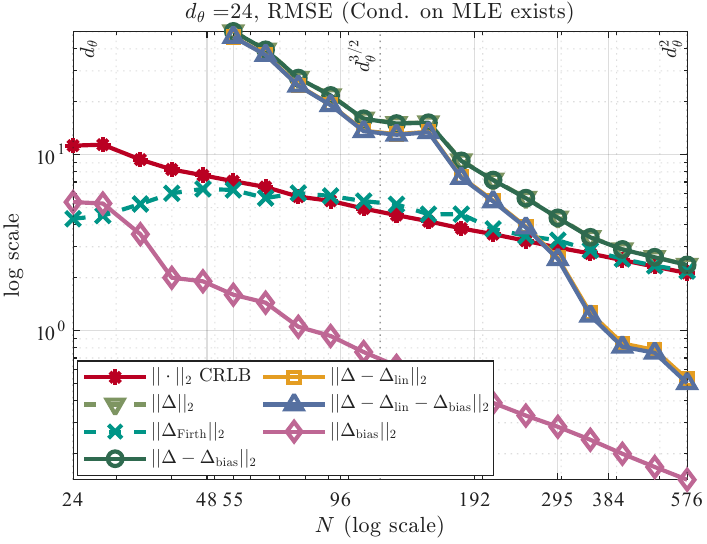}
            \caption{}
            \end{subfigure}\hfill
            \begin{subfigure}[t]{0.5\linewidth}
            \centering
            \includegraphics[width=\linewidth]{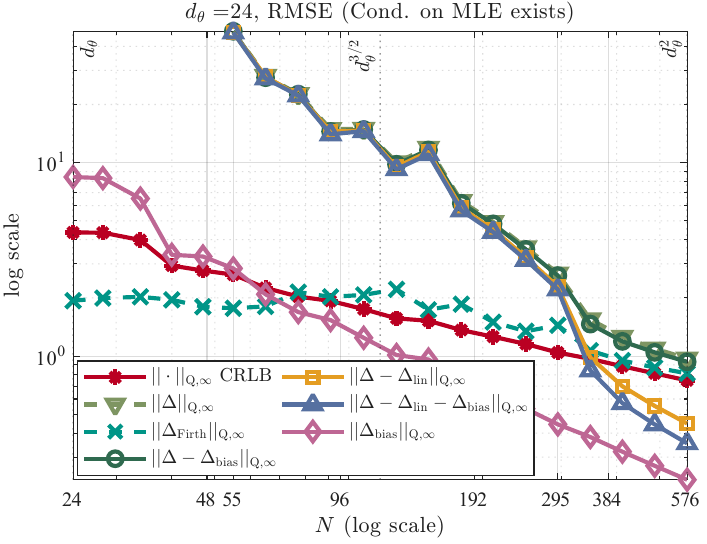}
            \caption{}
            \end{subfigure}
            \end{minipage}
            \caption{Same as Figure~\ref{fig_DC_VarN_design}, but for flights to London.}
            \label{fig_London_VarN_design}
        \end{figure}
        
        For a destination $\mathfrak{D}\in\{\mathrm{DC},\mathrm{London}\}$, we cross each of the $367$ retained $\mathbf{c}_n$ with all $48$ pairwise-comparison templates $\mathbf{q}_{\mathfrak{D},\Boalp_A} - \mathbf{q}_{\mathfrak{D},\Boalp_B}$ as in~\eqref{eq_contextual_ques_pool_set}, yielding $367\times 48 = 17616$ candidate query instances, which include ticket–respondent combinations that are not observed in the survey. We then solve the D-optimal design problem over this candidate query pool with design budget $N$.
        Since the ground truth is unknown, we follow~\cite[Section~6]{MukherjeeOptDesRLHF} and use the Firth correction to estimate $\boldsymbol{\theta}^\star_{\mathrm{DC}}$ and $\boldsymbol{\theta}^\star_{\mathrm{London}}$ 
        based on the $1835$ and $2202$ real-world responses. 
        We rescale $\boldsymbol{\theta}^\star_{\mathrm{DC}}$ ($\boldsymbol{\theta}^\star_{\mathrm{London}}$) to ensure $\log B =3$ over its candidate query pool and use the resulting parameter as the ground truth for simulation. Responses are then generated from the corresponding BTL model.
        
        Figure~\ref{fig_Greenfly_MLE_existence_convergence_transition} plots the existence and convergence probabilities of the canonical MLE, while Figures~\ref{fig_DC_VarN_design} and~\ref{fig_London_VarN_design} report the RMSEs, medians, and empirical $10\%$--$90\%$ percentile ranges of the Euclidean and query-wise errors over $3000$ Monte Carlo trials for the DC and London query pools, respectively. Figure~\ref{fig_Greenfly_MLE_existence_convergence_transition} shows that when $24 \leq N\leq 48$, the MLE fails to exist in almost all trials, and the MLE existence probabilities increase with $N$ and become close to one by approximately $N=192$, but the Euclidean and query-wise RMSEs remain substantially above the corresponding CRLBs at this sample size, as depicted by Figures~\ref{fig_DC_VarN_design}~(a)--(b) and~\ref{fig_London_VarN_design}~(a)--(b). For example, Figure~\ref{fig_DC_VarN_design}~(a) shows a gap between the $\Vert\Delta\Vert_2$ RMSE and the corresponding median for $55\leq N\leq295$, together with wide empirical percentile ranges depicted by the shaded-in-green region. This observation indicates that large variability across trials and right-skewed errors inflate the $\Vert\Delta\Vert_2$ RMSE even when the MLE exists.
        
        We apply the same decomposition method as in Section~\ref{subsec_Synthetic_Questionnaires} when plotting Figure~\ref{fig_DC_VarN_design}~(c) and examine the source of these large errors. After rescaling the ground-truth parameter, $\Vert\Delta_{bias}\Vert_2$ lies below the $\Vert\cdot\Vert_2$ CRLB throughout the plotted sample-size range, see~\eqref{eq_def_of_Delta_bias}. 
        When $N$ is small, the $\Vert\Delta-\Delta_{lin}\Vert_2$ and $\Vert\Delta - \Delta_{lin} - \Delta_{bias}\Vert_2$ RMSE curves nearly coincide and remain close to the $\Vert\Delta\Vert_2$ RMSE curve. Thus, the large error in this regime is driven primarily by higher-order nonlinear terms. 
        For $N\geq 250$, these nonlinear residuals fall below the CRLB and continue to decay more rapidly, after which the $\Vert\Delta\Vert_2$ RMSE approaches the CRLB. Similar behavior is observed for destination London in Figure~\ref{fig_London_VarN_design}~(c) and under the query-wise metric in Figures~\ref{fig_DC_VarN_design}~(d) and~\ref{fig_London_VarN_design}~(d). The Firth correction, however, exhibits substantially more stable behavior. Its Euclidean RMSE decreases steadily for $N\geq128$ and nearly coincides with the CRLB at $N=576=d_\theta^2$.

        \section{Concluding Remarks}
        \label{sec_conclusion_remarks}
In this paper, we investigate error rates of the maximum likelihood estimator for eliciting linear-in-parameter utility functions under 
the BTL framework.
Under some moderate 
conditions, we derive the Cram\'{e}r-Rao and minimax lower bounds for general estimators including MLE. We then establish existence of the MLE and conditions under which the upper and lower bounds meet up to some constants.

There are several promising directions for future research. First, the feasible parameter set considered in this paper is characterized by a system of linear equations. Extending the current analysis to more general polyhedral parameter spaces would accommodate linear inequalities that encode monotonicity or convexity of the utility function~\cite{HuJianPLAPRO, ChenLiuVNMElicit}.
Second, our results are developed for linear-in-parameter multivariate utility functions under the BTL framework. It remains unclear whether analogous theoretical treatments apply to nonlinear-in-parameter univariate utility functions or to the Plackett-Luce model. Third, the MLE provides a point estimate of the unknown parameters. An important direction for future work is to investigate whether the theoretical results can be extended to Bayesian learning approaches~\cite{SaureElliposDopt, BonillaGuoGPPrefElict, JiapengBayesPieceLinPreLean} to describe posterior concentration rates with finite samples, which may also include posterior based on Jeffreys' invariant prior, whose mode yields the Firth correction, as one special case.
        \section{Proofs of Main Results}
        \label{sec_proof_of_main_results}
        We now turn to the proofs of our main results developed in Sections~\ref{sec_minimax_LB} and~\ref{MLE_upper_bound}.
        \subsection{Proofs of Minimax Lower Bounds}
        We use two standard tools to transform the minimax risk into a hypothesis-testing problem, namely, Assouad’s lemma and Le Cam’s two-point method.
        \begin{lemma}[{\cite[Lemma~2.12]{TsybakovNonPara}}]
	    \label{lemma_Assouad_lemma}
		Let $\mathcal{S}=\{-1,1\}^m$ be the set of all $\{\pm1\}$ sequences of length $m\in\mathbb{Z}_+$, and equip this space with Hamming distance $\rho_H(\mathbf{s}^1,\mathbf{s}^2):=\sum_{j=1}^m\ind\{s_j^1\neq s_j^2\}$, for any $\mathbf{s}_1$, $\mathbf{s}_2\in\mathcal{S}$. Let $\{\bbp_\mathbf{s}^{(N)}\mid \mathbf{s}\in\mathcal{S}\}\subset \mathscr{P}(\mathcal{Y}_N,\mathscr{Y}_N)$ contain $2^m$ probability measures indexed by $\mathbf{s}\in\mathcal{S}$. With slight notation abuse, let $\bbe_\mathbf{s}^{(N)}[\cdot]$ denote expectation under $\bbp_\mathbf{s}^{(N)}$. Consider all measurable estimators $\hat{\mathbf{s}}:(\mathcal{Y}_N,\mathscr{Y}_N)\to (\mathcal{S},2^\mathcal{S})$ and some positive weights $\{a_j\}_{j=1}^m$. Then
            \begin{equation}
				\frac{1}{2^m}\sum_{\mathbf{s}\in\mathcal{S}}\mathbb{E}_\mathbf{s}^{(N)}\left[\sum_{j=1}^m a_j\ind\left\{\hat{s}_j\neq s_j\right\}\right] \geq \frac{1}{2}\sum_{j=1}^m a_j \left(1-\sup_{\substack{\mathbf{s},\mathbf{s}^{\prime}\in\mathcal{S}\\
                \rho_H(\mathbf{s},\mathbf{s}^{\prime})=1}}\dd_{TV}\left(\bbp_\mathbf{s}^{(N)},\bbp_{\mathbf{s}^{\prime}}^{(N)}\right)\right).
			\end{equation}
		\end{lemma}
        This weighted form of Assouad’s lemma follows from the proof of~\cite[Lemma~2.12]{TsybakovNonPara}.
        \begin{lemma}[{\cite[Theorem~2.2]{TsybakovNonPara}}]
                \label{lema_le_Cam_twopoints}
                Let $\rho$ be a semi-distance, and $\{\bbp_{\Bothe}^{(N)}\mid \Bothe\in\{\Bothe_1,\Bothe_2\}\}\subset\mathscr{P}(\mathcal{Y}_N,\mathscr{Y}_N)$ be the set of parametric probability measures indexed by $\{\Bothe_1,\Bothe_2\}$. 
                If there exists a $2\delta-$separated pair $\Bothe_1,\Bothe_2\in\Theta_B$ with  $\rho(\Bothe_1,\Bothe_2)\geq 2\delta$, then  
                \begin{subequations}
                    \begin{gather}
                    \label{eq_LeCam_two_point_Exp}
                        \mathcal{R}_N^\star(\Theta_B,\rho^2)\geq \frac{\delta^2}{2}(1-\dd_{TV}(\bbp_{\Bothe_1}^{(N)},\bbp_{\Bothe_2}^{(N)})),\\
                        \label{eq_LeCam_two_point_Prob}
                        \inf_{\hat{\Bothe}\in\mathcal{A}_N}\sup_{\Bothe\in\Theta_B}\mathbb{P}_{\Bothe}^{(N)}\left\{\rho^2(\hat{\Bothe},\Bothe) \geq \delta^2 \right\} \geq \frac{1}{2}(1-\dd_{TV}(\bbp_{\Bothe_1}^{(N)},\bbp_{\Bothe_2}^{(N)})).
                    \end{gather}
                \end{subequations}
        \end{lemma}
        We also need to upper bound the TV distance between two joint Bernoulli distributions $\bbp_{\Bothe_1}^{(N)},\bbp_{\Bothe_2}^{(N)}$ induced by some $\Bothe_1,\,\Bothe_2\in\mathbb{R}^{d_\theta}$. By Pinsker's inequality \cite[Lemma 2.5]{TsybakovNonPara}, i.e., $\dd_{TV}(\rP,\mathrm{Q})\leq \sqrt{\dd_{KL}(\rP,\mathrm{Q})/2}$, and the additivity of KL divergence for product measures, i.e., $\dd_{KL}(\bbp_{\Bothe_1}^{(N)},\bbp_{\Bothe_2}^{(N)}) = \sum_{\Boalp\in\E_s}n_\Boalp \dd_{KL}(\rP_{\Bothe_1}^{\alpha},\rP_{\Bothe_2}^{\alpha})$, it is enough to find an upper bound on $\dd_{KL}(\rP_{\Bothe_1}^{\alpha},\rP_{\Bothe_2}^{\alpha})$. This estimate is standard and due to~\cite{ShahEstPariCompGrapTop}. 
        \begin{lemma}[{\cite[Lemma~8]{ShahEstPariCompGrapTop}}]
		      \label{lemma_KL_between_general_exp_fam}
		      Let $\Bothe_1,\Bothe_2\in\mathbb{R}^{d_\theta}$ be arbitrary but fixed. For two joint distributions of the full observation $Y^{(N)}$, i.e., $\bbp_{\Bothe_1}^{(N)}$, $\bbp_{\Bothe_2}^{(N)}$ defined as in \eqref{eq_joint_distribution_law_of_samples}, we have 
        \begin{equation}
		  \dd_{KL}(\bbp_{\Bothe_1}^{(N)},\bbp_{\Bothe_2}^{(N)}) \leq \frac{N}{8\sigma^2}(\boldsymbol{\theta}_1-\boldsymbol{\theta}_2)^\top \mathbf{V}_{A^{\perp}}^\top\mathfrak{A}^\top\mathbf{L}\mathfrak{A}\mathbf{V}_{A^{\perp}} (\boldsymbol{\theta}_1-\boldsymbol{\theta}_2).
		\end{equation}
		\end{lemma}
        The following univariate lower bound on the sum of two KL divergences between Bernoulli distributions complements Lemma~\ref{lemma_KL_between_general_exp_fam}. 
            \begin{lemma}
            \label{lema_sum_of_excess_loss_lower_bound_by_semi_one_sample}
                Let $a,b\in[-\log B,\log B]$, $z\in\mathbb{R}$, and $\rP_{\Psi^\prime}^a$, $\rP_{\Psi^\prime}^b$, $\rP_{\Psi^\prime}^z$ be Bernoulli distributions with success rates $\Psi^\prime(a)$, $\Psi^\prime(b)$, and $\Psi^\prime(z)$, respectively. Then
                \begin{equation*}
                    \dd_{KL}(\rP_{\Psi^\prime}^a,\rP_{\Psi^\prime}^z) + \dd_{KL}(\rP_{\Psi^\prime}^b,\rP_{\Psi^\prime}^z) \geq \frac{1}{16B}(a-b)^2.
                \end{equation*}
                Moreover, let $\Bothe_1$, $\Bothe_2\in\Theta_B$, and let
                \begin{align*}
                    \mathcal{L}^{\Boalp}_{\Bothe_1}(\Bov) &= -\Psi^\prime(\frac{\mathbf{w}_o^\top\mathbf{a}_{\Boalp} + \mathbf{q}_\Boalp^\top \Bothe_1}{\sigma}) \frac{\mathbf{w}_o^\top\mathbf{a}_{\Boalp} + \mathbf{q}_\Boalp^\top \Bov}{\sigma} + \Psi(\frac{\mathbf{w}_o^\top\mathbf{a}_{\Boalp} + \mathbf{q}_\Boalp^\top \Bov}{\sigma}),\\
                    \mathcal{L}^{\Boalp}_{\Bothe_2}(\Bov) &= -\Psi^\prime(\frac{\mathbf{w}_o^\top\mathbf{a}_{\Boalp} + \mathbf{q}_\Boalp^\top \Bothe_2}{\sigma})\frac{\mathbf{w}_o^\top\mathbf{a}_{\Boalp} + \mathbf{q}_\Boalp^\top \Bov}{\sigma} + \Psi(\frac{\mathbf{w}_o^\top\mathbf{a}_{\Boalp} + \mathbf{q}_\Boalp^\top \Bov}{\sigma}),
                \end{align*}
                denote the population-level negative log-likelihood of one observation associated with query $\Boalp\in\E_s$, evaluated at some $\boldsymbol{v}\in\mathbb{R}^{d_\theta}$. Then 
                \begin{equation}
                \label{eq_sum_of_excess_loss_lower_bound_by_semi_one_sample}
                    \mathcal{L}^{\Boalp}_{\Bothe_1}(\Bov) -\mathcal{L}^{\Boalp}_{\Bothe_1}(\Bothe_1) + \mathcal{L}^{\Boalp}_{\Bothe_2}(\Bov)-\mathcal{L}^{\Boalp}_{\Bothe_2}(\Bothe_2) \geq \frac{1}{16B\sigma^2}\Big|\mathbf{q}_\Boalp^\top(\Bothe_1-\Bothe_2)\Big|^2, \,\forall\,\Bov\in\mathbb{R}^{d_\theta}. 
                \end{equation}
            \end{lemma}
            The proof is deferred to Appendix~\ref{proof_lema_sum_of_excess_loss_lower_bound_by_semi_one_sample}.
        \subsubsection{Proof of Theorem \ref{thm_minimax_loewe_bound_tr_inv}}
        \label{proof_thm_minimax_loewe_bound_tr_inv}
        \noindent \textbf{Proof of \eqref{eq_minimax_exp_lowerbound}.} 
        The proof proceeds in two main steps:
        we first define a binary precoder/decoder pair and relate the mean square error to Hamming distance, and then use Lemmas~\ref{lemma_Assouad_lemma} and~\ref{lemma_KL_between_general_exp_fam} to 
        obtain a lower bound on the weighted Hamming distance and upper bound on the TV distance of two properly designed probability measures, respectively.

        \noindent\underline{\textbf{STEP 1.}} We start by defining the precoder, which is similar to those in~\cite[Theorem 1]{ShahEstPariCompGrapTop} and 
        \cite[Theorem 5]{LIl_inf_bound_BTL_GenGraph}, but we use a Rademacher concentration argument to retain a sufficiently large subset of feasible points in $\Theta_B$. Let $\mathcal{S}:=\{-1,1\}^{d_{\theta}}$, and $(\mathcal{S},2^\mathcal{S},\bbp_{\Boxi})$ denote the probability space that supports a Rademacher sequence $\Boxi:=[\xi_1,\dots,\xi_{d_\theta}]^\top$. That is, for $A\subseteq\mathcal{S}$, $\bbp_{\Boxi}(A) = 2^{-d_\theta}|A|$, and $j\in [d_\theta]$, $\xi_j:\mathcal{S}\to\{-1,1\}$ is the coordinate projection. We denote a realization of $\Boxi$ as $\mathbf{s}\in\mathcal{S}$, and set $\bbe_{\Boxi}[\cdot]$ be the expectation under $\bbp_{\Boxi}$. 
        Let $\mathbf{W}:=\mathbf{H}\boldsymbol{\Lambda}\mathbf{H}^\top$ denote the 
        spectral decomposition of $\mathbf{W}$, with $\boldsymbol{\Lambda}=\mathrm{diag}(\lambda_1,\dots,\lambda_{d_\theta})$ and $\mathbf{H} = [\mathbf{h}_1,\dots,\mathbf{h}_{d_\theta}]$, where $\mathbf{h}_j\in\mathbb{R}^{d_\theta}$. For some $\delta>0$, define
        \begin{equation}
            \label{eq_def_of_theta_s_w_s}
			\Bothe^{(\Boxi)} := \delta\mathbf{H}\boldsymbol{\Lambda}^{-1/2}\Boxi = \sum_{j=1}^{d_\theta}\frac{\delta \xi_j\mathbf{h}_j}{\sqrt{\lambda_j}},\quad \mathbf{w}^{(\Boxi)} := \mathbf{w}_o + \mathbf{V}_{A^{\perp}} \Bothe^{(\Boxi)}. 
		\end{equation}
        Consider the set $\mathcal{Q}_B:=\{\mathbf{s}\in\mathcal{S}\mid \max_{\Boalp\in\E_s}\big| \mathbf{q}_\Boalp^\top \Bothe^{(\mathbf{s})} \big|\leq \sigma\log(B/B_0) \}$. By Assumption~\ref{assump_gap_wp_w_B_r0}, for any $\mathbf{s}\in\mathcal{Q}_B$
        \begin{equation*}
            \max_{\Boalp\in\E_s}\frac{\big|\mathbf{w}_o^\top \mathbf{a}_{\Boalp} + \mathbf{q}_\Boalp^\top \Bothe^{(\mathbf{s})}\big|}{\sigma} \leq \log B_0 + \log(B/B_0) = \log B.
        \end{equation*}
        Thus, $\{\Bothe^{(\mathbf{s})}\mid \mathbf{s}\in\mathcal{Q}_B\}\subseteq \Theta_B$ by Assumption~\ref{auumpt1self_concor} and~\eqref{eq_def_of_theta_B_set}, from which we have
        \begin{align}
            \sup_{\Bothe\in\Theta_B}\bbe_\Bothe^{(N)}\left[ \Vert\hat{\Bothe}-\Bothe\Vert_2^2 \right] &\geq \frac{1}{|\mathcal{Q}_B|} \sum_{\mathbf{s}\in\mathcal{Q}_B} \bbe_{\Bothe^{(\mathbf{s})}}^{(N)} \left[ \Vert\hat{\Bothe}-\Bothe^{(\mathbf{s})}\Vert_2^2 \right]\nonumber\\
            & = \frac{2^{-d_\theta}\sum_{\mathbf{s}\in\mathcal{Q}_B}\bbe_{\Bothe^{(\mathbf{s})}}^{(N)}\left[ \Vert\hat{\Bothe}-\Bothe^{(\mathbf{s})}\Vert_2^2 \right]}{2^{-d_\theta}|\mathcal{Q}_B|} = \frac{2^{-d_\theta} \sum_{\mathbf{s}\in\mathcal{S}} \ind\{\mathbf{s}\in \mathcal{Q}_B\} \bbe_{\Bothe^{(\mathbf{s})}}^{(N)}\left[ \Vert\hat{\Bothe}-\Bothe^{(\mathbf{s})}\Vert_2^2 \right]}{2^{-d_\theta}|\mathcal{Q}_B|}\nonumber\\
            \label{eq_minimax_lower_by_Cond_on_feasible_theta}
            & = \frac{\bbe_{\Boxi}\left[ \ind\{\mathbf{\Boxi}\in \mathcal{Q}_B\} \bbe_{\Bothe^{(\Boxi)}}^{(N)}\left[\Vert\hat{\Bothe}-\Bothe^{(\Boxi)}\Vert_2^2\right] \right]}{\bbp_{\Boxi}\{\Boxi\in \mathcal{Q}_B\}}.
        \end{align}
        Then our primary interest 
        is on a single realization $\mathbf{s}$ of $\Boxi$ and $\Bothe^{(\mathbf{s})}$. By~\eqref{eq_def_of_theta_s_w_s} and the projection property of $\mathbf{HH}^\top$, we have\footnote{For any estimator $\hat{\mathbf{w}}$ of $\mathbf{w}\in\W_B$, a similar lower bound argument as in \eqref{eq_MSE_w_s_to_decoder_1} for $\Vert\hat{\mathbf{w}}-\mathbf{w}^{(\mathbf{s})}\Vert_2^2$ can be obtained by noticing the projection property of $(\mathbf{V}_{A^{\perp}}\mathbf{H})(\mathbf{V}_{A^{\perp}}\mathbf{H})^\top$.} 
        \begin{align}
			\Vert\hat{\Bothe}-\Bothe^{(\mathbf{s})}\Vert_2^2 &
			\geq \Vert\mathbf{HH}^\top(\hat{\Bothe}-\Bothe^{(\mathbf{s})})\Vert_2^2=\Vert\mathbf{H}^\top\hat{\Bothe}-\mathbf{H}^\top \Bothe^{(\mathbf{s})}\Vert_2^2\nonumber\\
            \label{eq_MSE_w_s_to_decoder_1}
			& = \Vert\mathbf{H}^\top\hat{\Bothe}-\delta \boldsymbol{\Lambda}^{-1/2}\mathbf{s}\Vert_2^2
			= \sum_{j=1}^{d_\theta} \left( \mathbf{h}_j^\top\hat{\Bothe} - \frac{\delta}{\sqrt{\lambda_j}}s_j \right)^2,\,\forall\,\mathbf{s}\in\mathcal{S}.
		\end{align}
        We use $\mathbf{h}_j^\top \hat{\Bothe}$ 
        to define the decoder
        \begin{equation*}
			\hat{s}_j :=\begin{cases}
			    1,\,&\text{ if } \mathbf{h}_j^\top \hat{\Bothe} \geq 0,\\
                -1,\,&\text{ if } \mathbf{h}_j^\top \hat{\Bothe} < 0.
			\end{cases}
		\end{equation*}
        For any $j\in[d_\theta]$, if $\hat{s}_j\neq s_j$, then $\Big|\mathbf{h}_j^\top\hat{\Bothe}-\frac{\delta}{\sqrt{\lambda_j}}s_j\Big|\geq \frac{\delta}{\sqrt{\lambda_j}}$. Thus 
        \begin{equation}
			\label{eq_MSE_w_s_to_decoder_2}
			\sum_{j=1}^{d_\theta} \left( \mathbf{h}_j^\top\hat{\Bothe} - \frac{\delta}{\sqrt{\lambda_j}}s_j \right)^2\geq \sum_{j=1}^{d_\theta}\frac{\delta^2}{\lambda_j}\ind\{\hat{s}_j\neq s_j\}, \,\forall\,\mathbf{s}\in\mathcal{S}.
		\end{equation}
 By substituting~\eqref{eq_MSE_w_s_to_decoder_1} and~\eqref{eq_MSE_w_s_to_decoder_2} into~\eqref{eq_minimax_lower_by_Cond_on_feasible_theta}, we note that it is enough to derive an upper bound on $\bbp_{\Boxi}\{\mathbf{\Boxi}\notin \mathcal{Q}_B\}$ and a lower bound on $\bbe_{\Boxi}\left[ \bbe_{\Bothe^{(\Boxi)}}^{(N)}\left[\sum_{j=1}^{d_\theta}\frac{\delta^2}{\lambda_j}\ind\{\hat{s}_j\neq \xi_j\}\right] \right]$. Indeed 
        \begin{align}
             \sup_{\Bothe\in\Theta_B}\bbe_\Bothe^{(N)}\left[ \Vert\hat{\Bothe}-\Bothe\Vert_2^2 \right]&\geq \bbe_{\Boxi}\left[ \ind\{\mathbf{\Boxi}\in \mathcal{Q}_B\} \bbe_{\Bothe^{(\Boxi)}}^{(N)}\left[\sum_{j=1}^{d_\theta}\frac{\delta^2}{\lambda_j}\ind\{\hat{s}_j\neq \xi_j\}\right] \right] \nonumber\\
             &= \bbe_{\Boxi}\left[ \bbe_{\Bothe^{(\Boxi)}}^{(N)}\left[\sum_{j=1}^{d_\theta}\frac{\delta^2}{\lambda_j}\ind\{\hat{s}_j\neq \xi_j\}\right] \right] - \bbe_{\Boxi}\left[ \ind\{\mathbf{\Boxi}\notin \mathcal{Q}_B\}\bbe_{\Bothe^{(\Boxi)}}^{(N)}\left[\sum_{j=1}^{d_\theta}\frac{\delta^2}{\lambda_j}\ind\{\hat{s}_j\neq \xi_j\}\right] \right]\nonumber\\
             \label{eq_lower_on_minimax_Euclid_diff_by_two_terms}
             &\geq \bbe_{\Boxi}\left[ \bbe_{\Bothe^{(\Boxi)}}^{(N)}\left[\sum_{j=1}^{d_\theta}\frac{\delta^2}{\lambda_j}\ind\{\hat{s}_j\neq \xi_j\}\right] \right] - \bbp_{\Boxi}\{\mathbf{\Boxi}\notin \mathcal{Q}_B\} \delta^2\mathrm{tr}(\mathbf{W}^{-1}),
        \end{align}
        where the last inequality is from the deterministic upper bound $\sum_{j=1}^{d_\theta}\frac{\delta^2}{\lambda_j}\ind\{\hat{s}_j\neq \xi_j\} \leq \delta^2\mathrm{tr}(\mathbf{W}^{-1})$.

        \noindent\underline{\textbf{STEP 2.}} We first derive an upper bound for $\bbp_{\Boxi}\{\mathbf{\Boxi}\notin \mathcal{Q}_B\}$.
        By Hoeffding's inequality~\cite[Theorem~2.2.5]{VershyninHighDimBook}, 
        \begin{equation}
        \label{eq_upper_bound_on_xi_notin_C_b_set}
            \bbp_{\Boxi}\{\mathbf{\Boxi}\notin \mathcal{Q}_B\} = \bbp_{\Boxi}\left\{ \max_{\Boalp\in\E_s}\big|\mathbf{q}_\Boalp^\top \Bothe^{(\Boxi)}\big| >\sigma\log(B/B_0)\right\} \leq 2d^2\exp\left( -\frac{\sigma^2\log^2(B/B_0)}{2\delta^2\mu_Q d_\theta} \right),
        \end{equation}
        where the last inequality follows from $\Vert\boldsymbol{\Lambda}^{-1/2}\mathbf{H}^\top\mathbf{q}_\Boalp\Vert_2^2 \leq \max_{\Boalp\in\E_s}\mathbf{q}_\Boalp^\top\mathbf{W}^{-1}\mathbf{q}_\Boalp = \mu_Q d_\theta$ for any $\Boalp\in\E_s$, and the fact that $|\E_s|\leq d^2$. 
        By setting $\delta=\min\left\{ \frac{\sigma}{\sqrt{N}},\frac{\sigma\log(B/B_0)}{\sqrt{4\mu_Qd_\theta \log(4d)}} \right\}$, 
        we obtain by
        inequality~\eqref{eq_upper_bound_on_xi_notin_C_b_set} 
        that $\bbp_{\Boxi}\{\mathbf{\Boxi}\notin \mathcal{Q}_B\}\leq \frac{1}{8}$. 
        
        Next, by Lemma~\ref{lemma_Assouad_lemma} and Pinsker’s inequality
        \begin{align}
            \bbe_{\Boxi}\left[ \bbe_{\Bothe^{(\Boxi)}}^{(N)}\left[\sum_{j=1}^{d_\theta}\frac{\delta^2}{\lambda_j}\ind\{\hat{s}_j\neq \xi_j\}\right] \right] &= \frac{1}{2^{d_\theta}} \sum_{\mathbf{s}\in\mathcal{S}} \bbe_{\Bothe^{(\mathbf{s})}}^{(N)}\left[\sum_{j=1}^{d_\theta}\frac{\delta^2}{\lambda_j}\ind\{\hat{s}_j\neq s_j\}\right]\nonumber\\
            &\geq \frac{1}{2}\sum_{j=1}^{d_\theta} \frac{\delta^2}{\lambda_j}\left(1-\sup_{\mathbf{s},\mathbf{s}^{\prime}:\rho_H(\mathbf{s},\mathbf{s}^{\prime})=1}\dd_{TV}(\bbp_{\Bothe^{(\mathbf{s})}}^{(N)},\bbp_{\Bothe^{(\mathbf{s}^{\prime})}}^{(N)})\right)\nonumber\\
            \label{eq_two_exp_decoder_lower_by_KL}
            &\geq  \frac{1}{2}\sum_{j=1}^{d_\theta} \frac{\delta^2}{\lambda_j}\left(1-\sup_{s,s^{\prime}:\rho_H(s,s^{\prime})=1}\sqrt{\frac{1}{2}\dd_{KL}(\bbp_{\Bothe^{(\mathbf{s})}}^{(N)}, \bbp_{\Bothe^{(s^{\prime})}}^{(N)})}\right).
        \end{align}
        We then use Lemma~\ref{lemma_KL_between_general_exp_fam} to derive an upper bound for the KL divergence term
        \begin{align}
			\sup_{\mathbf{s},\mathbf{s}^{\prime}:\rho_H(\mathbf{s},\mathbf{s}^{\prime})=1}\dd_{KL}(\bbp_{\Bothe^{(\mathbf{s})}}^{(N)},\bbp_{\Bothe^{(\mathbf{s}^{\prime})}}^{(N)}) &\leq \sup_{\mathbf{s},\mathbf{s}^{\prime}:\rho_H(\mathbf{s},\mathbf{s}^{\prime})=1} \frac{N}{8\sigma^2}(\Bothe^{(\mathbf{s})}-\Bothe^{(\mathbf{s}^{\prime})})^\top \mathbf{W}(\Bothe^{(\mathbf{s})}-\Bothe^{(\mathbf{s}^{\prime})})\nonumber\\
			\label{eq_d_TV_sup_upperbound}
			& = \sup_{\mathbf{s},\mathbf{s}^{\prime}:\rho_H(\mathbf{s},\mathbf{s}^{\prime})=1} \frac{N\delta^2}{8\sigma^2}\Vert\mathbf{s}-\mathbf{s}^{\prime}\Vert_2^2 = \frac{N\delta^2}{2\sigma^2}\leq \frac{1}{2},
		\end{align}
    where the last equality is from the fact that $\mathbf{s},\mathbf{s}^\prime\in\{\pm1\}^{d_\theta}$, and the last inequality 
    is due to 
    the 
    specification of $\delta$. 
    By combining~\eqref{eq_lower_on_minimax_Euclid_diff_by_two_terms}
    -\eqref{eq_d_TV_sup_upperbound}, and the choice of $\delta$, we have
    \begin{align*}
        \sup_{\Bothe\in\Theta_B}\bbe_\Bothe^{(N)}\left[ \Vert\hat{\Bothe}-\Bothe\Vert_2^2 \right] &\geq \frac{1}{4}\sum_{j=1}^{d_\theta} \frac{\delta^2}{\lambda_j} - \frac{1}{8}\delta^2\mathrm{tr}(\mathbf{W}^{-1}) = \frac{1}{8}\delta^2\mathrm{tr}(\mathbf{W}^{-1})\\
        & =\frac{\sigma^2\mathrm{tr}(\mathbf{W}^{-1})}{8}\min\left\{\frac{1}{N}, \frac{\log^2(B/B_0)}{4\mu_Qd_\theta \log(4d)}\right\}.
    \end{align*}
    Taking infimum over $\hat{\Bothe}$ on both sides of the above inequality 
    gives rise to~\eqref{eq_minimax_exp_lowerbound}.

        \noindent \textbf{Proof of \eqref{eq_minimax_prob_lowerbound}.} First, by~\eqref{eq_MSE_w_s_to_decoder_1} and~\eqref{eq_MSE_w_s_to_decoder_2}
		\begin{equation}
        \label{eq_theta_hat_theta_s_lower_by_hamming}
			\Vert\hat{\Bothe}-\Bothe^{(\mathbf{s})}\Vert_2^2\geq \sum_{j=1}^{d_\theta}\frac{\delta^2}{\lambda_j}\ind\{\hat{s}_j\neq s_j\}, \,\forall \,\mathbf{s}\in\mathcal{Q}_B,
		\end{equation}
		$\bbp_{\Bothe^{(\mathbf{s})}}^{(N)}$ almost surely, where
        \begin{equation}
            \label{eq_decoding_error_bounded_as}
		  0\leq\sum_{j=1}^{d_\theta}\frac{\delta^2}{\lambda_j}\ind\{\hat{s}_j\neq s_j\} \leq \delta^2\mathrm{tr}(\mathbf{W}^{-1}).
		\end{equation}
        We adopt the same choice of $\delta$ as in the derivation above. That is, for all $t\in(0,\delta^2\mathrm{tr}(\mathbf{W}^{-1}))$
        \begin{align}
            \sup_{\Bothe\in\Theta_B} \bbp_{\Bothe}^{(N)}\left\{ \Vert\hat{\Bothe} - \Bothe\Vert_2^2 \geq t \right\} &\geq \frac{1}{|\mathcal{Q}_B|}\sum_{\mathbf{s}\in\mathcal{Q}_B}\bbp_{\Bothe^{(\mathbf{s})}}^{(N)}\left\{ \Vert\hat{\Bothe} - \Bothe^{(\mathbf{s})}\Vert_2^2 \geq t\right\}\nonumber\\
            &\geq \frac{1}{|\mathcal{Q}_B|}\sum_{\mathbf{s}\in\mathcal{Q}_B}\bbp_{\Bothe^{(\mathbf{s})}}^{(N)}\left\{ \sum_{j=1}^{d_\theta}\frac{\delta^2}{\lambda_j}\ind\{\hat{s}_j\neq s_j\} \geq t\right\}\nonumber\\
            &\geq \frac{|\mathcal{Q}_B|^{-1}\sum_{\mathbf{s}\in\mathcal{Q}_B}\bbe_{\Bothe^{(\mathbf{s})}}^{(N)}\left[ \sum_{j=1}^{d_\theta}\frac{\delta^2}{\lambda_j}\ind\{\hat{s}_j\neq s_j\} \right] - t }{\delta^2\mathrm{tr}(\mathbf{W}^{-1})-t}\nonumber\\
            & = \frac{ \frac{2^{d_\theta}}{|\mathcal{Q}_B|}\cdot \bbe_{\Boxi}\left[ \ind\{\Boxi\in\mathcal{Q}_B\}  \bbe_{\Bothe^{(\Boxi)}}^{(N)}\left[ \sum_{j=1}^{d_\theta}\frac{\delta^2}{\lambda_j}\ind\{\hat{s}_j\neq \xi_j\} \right]\right] - t }{\delta^2\mathrm{tr}(\mathbf{W}^{-1})-t}\nonumber\\
            \label{eq_minimax_euc_prob_lower_in_t}
            & \geq \frac{\frac{1}{8}\delta^2\mathrm{tr}(\mathbf{W}^{-1}) - t}{\delta^2\mathrm{tr}(\mathbf{W}^{-1})-t},
        \end{align}
        where the third inequality is from Lemma~\ref{eq_bounded_Rv_prob_lower}, and in the last inequality we use $\frac{2^{d_\theta}}{|\mathcal{Q}_B|}\geq 1$ and the lower bound on the r.h.s.~of~\eqref{eq_lower_on_minimax_Euclid_diff_by_two_terms} deduced in the proof of~\eqref{eq_minimax_exp_lowerbound}. By setting $t= \frac{\delta^2\mathrm{tr}(\mathbf{W}^{-1})}{16}$ and $\delta=\min\left\{ \frac{\sigma}{\sqrt{N}},\frac{\sigma\log(B/B_0)}{\sqrt{4\mu_Qd_\theta \log(4d)}} \right\}$, inequality~\eqref{eq_minimax_euc_prob_lower_in_t} yields
        \begin{align*}
             \sup_{\Bothe\in\Theta_B} \bbp_{\Bothe}^{(N)}\left\{ \Vert\hat{\Bothe} - \Bothe\Vert_2^2 \geq \frac{\delta^2\mathrm{tr}(\mathbf{W}^{-1})}{16} \right\} &= \sup_{\Bothe\in\Theta_B} \bbp_{\Bothe}^{(N)}\left\{ \Vert\hat{\Bothe} - \Bothe\Vert_2^2 \geq \frac{\sigma^2\mathrm{tr}(\mathbf{W}^{-1})}{16}\min\left\{ \frac{1}{N}, \frac{\log^2(B/B_0)}{4\mu_Qd_\theta \log(4d)} \right\} \right\}\\
             &\geq \frac{\delta^2\mathrm{tr}(\mathbf{W}^{-1})/8 - \delta^2\mathrm{tr}(\mathbf{W}^{-1})/16}{\delta^2\mathrm{tr}(\mathbf{W}^{-1})-\delta^2\mathrm{tr}(\mathbf{W}^{-1})/16} = \frac{1}{15},
        \end{align*}
        taking infimum over $\hat{\Bothe}$ on both sides proves~\eqref{eq_minimax_prob_lowerbound}.
        \qedbox

        \subsubsection{Proof of Corollary~\ref{coro_miniax_lowerbound_excess_risk}}
        \label{proof_coro_miniax_lowerbound_excess_risk}

        A direct corollary from Lemma~\ref{lema_sum_of_excess_loss_lower_bound_by_semi_one_sample} lower bounds the excess risk, i.e.,
            \begin{align}
                \E_{\mathcal{L}}(\Bov,\Bothe_1) + \E_{\mathcal{L}}(\Bov,\Bothe_2) &= \mathcal{L}_{\Bothe_1}(\Bov) - \mathcal{L}_{\Bothe_1}(\Bothe_1) + \mathcal{L}_{\Bothe_2}(\Bov) - \mathcal{L}_{\Bothe_2}(\Bothe_2)\nonumber\\
                & = \frac{1}{N}\sum_{\Boalp\in\E_s} n_\Boalp \left( \mathcal{L}_{\Bothe_1}^\Boalp(\Bov) - \mathcal{L}_{\Bothe_1}^\Boalp(\Bothe_1) + \mathcal{L}_{\Bothe_2}^\Boalp(\Bov) - \mathcal{L}_{\Bothe_2}^\Boalp(\Bothe_2) \right)\nonumber\\
                \label{eq_sum_of_excess_loss_lower_bound_by_semi}
                & \geq \frac{1}{16B\sigma^2} \sum_{\Boalp\in\E_s} \frac{n_\Boalp}{N} \Big|\mathbf{q}_\Boalp^\top(\Bothe_1-\Bothe_2)\Big|^2 = \frac{1}{16B\sigma^2} \Vert \Bothe_1-\Bothe_2 \Vert_{\mathbf{W}}^2,
            \end{align}
            for all $\Bothe_1,\Bothe_2\in\Theta_B$ and $\boldsymbol{v}\in\mathbb{R}^{d_\theta}$. We adopt the notation introduced in Section~\ref{proof_thm_minimax_loewe_bound_tr_inv}.

            First, let $\Bothe^{(\Boxi)}$ be defined as 
            in~\eqref{eq_def_of_theta_s_w_s}.
            Thus 
            $\bbp_{\Boxi}\{\Boxi\notin\mathcal{Q}_B\}\leq \frac{1}{8}$ when $\delta=\min\left\{ \frac{\sigma}{\sqrt{N}},\frac{\sigma\log(B/B_0)}{\sqrt{4\mu_Qd_\theta \log(4d)}} \right\}$. 
            Fix any $\hat{\Bothe}\in\mathcal{A}_N$, and define the decoder
            \begin{equation*}
                \hat{\mathbf{s}}:=\arg\min_{\mathbf{t}\in\mathcal{Q}_B}\E_{\mathcal{L}}(\hat{\Bothe},\Bothe^{(\mathbf{t})}),
            \end{equation*}
            where we 
            break ties according to a fixed ordering
            of $\mathcal{Q}_B$, and we have $\Bothe^{(\hat{\mathbf{s}})}\in\Theta_B$. Thus
            \begin{align}
                \sup_{\Bothe\in\Theta_B} \mathbb{E}_\Bothe^{(N)}\left[ \E_{\mathcal{L}}(\hat{\Bothe},\Bothe) \right]&
                \geq \frac{1}{|\mathcal{Q}_B|}\sum_{\mathbf{s}\in\mathcal{Q}_B} \bbe_{\Bothe^{(\mathbf{s})}}^{(N)} \left[ \E_{\mathcal{L}}(\hat{\Bothe},\Bothe^{(\mathbf{s})})\right]
                \nonumber\\
                &\geq \frac{1}{2|\mathcal{Q}_B|} \sum_{\mathbf{s}\in\mathcal{Q}_B} \mathbb{E}_{\Bothe^{(\mathbf{s})}}^{(N)} \left[ \E_{\mathcal{L}}(\hat{\Bothe},\Bothe^{(\mathbf{s})}) + \E_{\mathcal{L}}(\hat{\Bothe},\Bothe^{(\hat{\mathbf{s}})}) \right] \nonumber\\
                &\overset{\eqref{eq_sum_of_excess_loss_lower_bound_by_semi}}{\geq} \frac{1}{32B\sigma^2|\mathcal{Q}_B|}  \sum_{\mathbf{s}\in\mathcal{Q}_B} \mathbb{E}_{\Bothe^{(\mathbf{s})}}^{(N)} \left[\Vert \Bothe^{(\mathbf{s})} - \Bothe^{(\hat{\mathbf{s}})} \Vert_{\mathbf{W}}^2\right] 
                \nonumber\\
                &=\frac{\delta^2}{32B\sigma^2|\mathcal{Q}_B|}  \sum_{\mathbf{s}\in\mathcal{Q}_B} \mathbb{E}_{\Bothe^{(\mathbf{s})}}^{(N)} \left[\Vert \mathbf{s} - \hat{\mathbf{s}} \Vert_2^2\right] \nonumber\\
                & = \frac{\delta^2}{8B\sigma^2|\mathcal{Q}_B|}  \sum_{\mathbf{s}\in\mathcal{Q}_B} \mathbb{E}_{\Bothe^{(\mathbf{s})}}^{(N)} \left[\rho_H(\mathbf{s},\hat{\mathbf{s}})\right] \nonumber\\
                &= \frac{\delta^2}{8B\sigma^2}\frac{\bbe_{\Boxi}\left[\ind\{\Boxi\in\mathcal{Q}_B\}\bbe_{\Bothe^{(\Boxi)}}^{(N)}\left[ \rho_H(\Boxi,\hat{\mathbf{s}}) \right]\right]}{\bbp_{\Boxi}\{\Boxi\in\mathcal{Q}_B\}} \nonumber\\
                \label{eq_sup_excess_risk_lower_by_minimax_Hamming}
                &\geq \frac{\delta^2}{8B\sigma^2} \bbe_{\Boxi}\left[\ind\{\Boxi\in\mathcal{Q}_B\}\bbe_{\Bothe^{(\Boxi)}}^{(N)}\left[ \rho_H(\Boxi,\hat{\mathbf{s}}) \right]\right].
            \end{align}
            According to~\eqref{eq_lower_on_minimax_Euclid_diff_by_two_terms}, the r.h.s.~of~\eqref{eq_sup_excess_risk_lower_by_minimax_Hamming} can be lower bounded as
            \begin{align}
                \bbe_{\Boxi}\left[\ind\{\Boxi\in\mathcal{Q}_B\}\bbe_{\Bothe^{(\Boxi)}}^{(N)}\left[ \rho_H(\Boxi,\hat{\mathbf{s}}) \right]\right] &= \bbe_{\Boxi}\left[\bbe_{\Bothe^{(\Boxi)}}^{(N)}\left[ \rho_H(\Boxi,\hat{\mathbf{s}}) \right]\right] -\bbe_{\Boxi}\left[\ind\{\Boxi\notin\mathcal{Q}_B\}\bbe_{\Bothe^{(\Boxi)}}^{(N)}\left[ \rho_H(\Boxi,\hat{\mathbf{s}}) \right]\right]\nonumber\\
                & = \frac{1}{2^{d_\theta}}\sum_{\mathbf{s}\in\mathcal{S}}\bbe_{\Bothe^{(\mathbf{s})}}^{(N)}\left[ \rho_H(\mathbf{s},\hat{\mathbf{s}}) \right] -\bbe_{\Boxi}\left[\ind\{\Boxi\notin\mathcal{Q}_B\}\bbe_{\Bothe^{(\Boxi)}}^{(N)}\left[ \rho_H(\Boxi,\hat{\mathbf{s}}) \right]\right]\nonumber\\
                &\geq \frac{d_\theta}{2}\left( 1-\sup_{\mathbf{s},\mathbf{s}^{\prime}:\rho_H(\mathbf{s},\mathbf{s}^{\prime})=1}\dd_{TV}(\bbp_{\Bothe^{(\mathbf{s})}}^{(N)},\bbp_{\Bothe^{(\mathbf{s}^{\prime})}}^{(N)}) \right) - d_\theta\bbp_{\Boxi}\{\Boxi\notin\mathcal{Q}_B\}\nonumber\\
                \label{eq_exp_xi_minimax_Hamming_lower_by_TV_Pxi}
                &\geq \frac{d_\theta}{4} - d_\theta\bbp_{\Boxi}\{\Boxi\notin\mathcal{Q}_B\} \geq \frac{d_\theta}{8},
            \end{align}
            where the first inequality is from Lemma~\ref{lemma_Assouad_lemma} and the worst-case upper bound $\rho_H(\Boxi,\hat{\mathbf{s}})\leq d_\theta$ over all $\Boxi,\hat{\mathbf{s}}\in\mathcal{S}$. The second and last inequalities in~\eqref{eq_exp_xi_minimax_Hamming_lower_by_TV_Pxi} follow from analogous derivation as in~\eqref{eq_d_TV_sup_upperbound} and~\eqref{eq_upper_bound_on_xi_notin_C_b_set}, respectively.

            By substituting~\eqref{eq_exp_xi_minimax_Hamming_lower_by_TV_Pxi} into~\eqref{eq_sup_excess_risk_lower_by_minimax_Hamming}, we obtain 
            \begin{equation*}
                \sup_{\Bothe\in\Theta_B} \mathbb{E}_\Bothe^{(N)}\left[ \E_{\mathcal{L}}(\hat{\Bothe},\Bothe) \right] \geq \frac{\delta^2d_\theta}{64B\sigma^2} = \frac{d_\theta}{64B}\min\left\{ \frac{1}{N},\frac{\log^2(B/B_0)}{4\mu_Qd_\theta \log(4d)} \right\}.
            \end{equation*}
            Taking infimum over $\hat{\Bothe}$ on both sides gives rise to~\eqref{eq_minimax_excess_risk_in_exp_form}. Inequality~\eqref{eq_minimax_excess_risk_in_prob_form} follows from
            the deterministic inequality implied  by~\eqref{eq_sup_excess_risk_lower_by_minimax_Hamming}
            \begin{equation*}
                \E_{\mathcal{L}}(\hat{\Bothe},\Bothe^{(\mathbf{s})})\geq \frac{\delta^2}{8B\sigma^2}\rho_H(\mathbf{s},\hat{\mathbf{s}}),\,\forall\,\mathbf{s}\in\mathcal{Q}_B,
            \end{equation*}
            and
            \begin{equation*}
                0\leq \rho_H(\mathbf{s},\hat{\mathbf{s}})\leq d_\theta,\,\forall\,\mathbf{s}\in\mathcal{S},
            \end{equation*}
            and then mimicking  the proof of~\eqref{eq_minimax_prob_lowerbound} with $t=\frac{\delta^2 d_\theta}{128B\sigma^2}$. 
            \qedbox

		\subsubsection{Proof of Theorem~\ref{thm_residual_coherence_minimax}}
        \label{proof_thm_residual_coherence_minimax}

        It is enough to prove \eqref{eq_minimax_Q_infty_norm_Exp} via Lemma~\ref{lema_le_Cam_twopoints}. For any but fixed $\Boalp\in\E_s$ such that $\mathbf{q}_\Boalp\neq \mathbf{0}$, let $\mu_\Boalp:=\frac{1}{d_\theta}\mathbf{q}_\Boalp^\top\mathbf{W}^{-1}\mathbf{q}_\Boalp$ and $\mathbf{h}_\Boalp:=\frac{\mathbf{W}^{-1}\mathbf{q}_\Boalp}{\sqrt{\mu_\Boalp d_\theta}}$. Suppose that we can find $\delta>0$ such that $\Bothe_1:=\delta\mathbf{h}_\Boalp$, $\Bothe_2:=-\delta\mathbf{h}_\Boalp\in\Theta_B$. Then
            \begin{align*}
                \Vert\Bothe_1-\Bothe_2\Vert_{Q,\infty} &= \max_{\boldsymbol{\beta}\in\E_s} |\mathbf{q}_{\boldsymbol{\beta}}^\top(\Bothe_1-\Bothe_2)| = 2\delta\max_{\boldsymbol{\beta}\in\E_s} |\mathbf{q}_{\boldsymbol{\beta}}^\top\mathbf{h}_\Boalp| \\
                &\geq  2\delta |\mathbf{q}_\Boalp^\top\mathbf{h}_\Boalp|= 2\delta \frac{\mathbf{q}_\Boalp^\top\mathbf{W}^{-1}\mathbf{q}_\Boalp}{\sqrt{\mu_\Boalp d_\theta}} = 2\delta\sqrt{\mu_\Boalp d_\theta},
            \end{align*}
            from which we construct a $2\delta\sqrt{\mu_\Boalp d_\theta}-$separated pair measured in $\Vert\cdot\Vert_{Q,\infty}$.
            By Lemma \ref{lemma_KL_between_general_exp_fam} and setting $\delta = \frac{\sigma}{\sqrt{N}}$
            \begin{align*}
                \dd_{TV}(\bbp_{\Bothe_1}^{(N)},\bbp_{\Bothe_2}^{(N)})&\leq \sqrt{\dd_{KL} (\bbp_{\Bothe_1}^{(N)},\bbp_{\Bothe_2}^{(N)})/2}\leq \sqrt{\frac{N}{16\sigma^2}(\Bothe_1-\Bothe_2)^\top\mathbf{W}(\Bothe_1-\Bothe_2)} \\
                &= \sqrt{\frac{\delta^2N}{4\sigma^2}\mathbf{h}_\Boalp^\top\mathbf{W}\mathbf{h}_\Boalp} = \sqrt{\frac{\delta^2N}{4\sigma^2}}= \frac{1}{2},
            \end{align*}
            where the first inequality is from Pinsker's inequality.
            Therefore, invoking the first part of Lemma \ref{lema_le_Cam_twopoints} yields
            \begin{align}
                \mathcal{R}_N^\star(\Theta_B,\Vert\cdot\Vert_{Q,\infty}^2) &\geq \max_{\Boalp\in\E_s}\frac{\delta^2\mu_\Boalp d_\theta}{2}\left(1- \dd_{TV}(\bbp_{\Bothe_1}^{(N)},\bbp_{\Bothe_2}^{(N)})\right) \nonumber\\
                \label{eq_minimax_Q_infty_mu_bound_1}
                &\geq \max_{\Boalp\in\E_s}\frac{\delta^2\mu_\Boalp d_\theta}{4}= \max_{\Boalp\in\E_s}\frac{\mu_\Boalp d_\theta \sigma^2}{4N} = \frac{\mu_Q d_\theta \sigma^2}{4N}.
            \end{align}
            Next, we verify the claim that $\Bothe_1,\Bothe_2\in\Theta_B$. That is, by setting $\delta=\min\left\{\frac{\sigma}{\sqrt{N}},\frac{\sigma\log(B/B_0)}{\sqrt{\mu_Q d_\theta}}\right\}$
            \begin{align*}
                \max_{\Bobet\in\E_s}\frac{|\mathbf{a}_\Bobet^\top\mathbf{w}_o + \mathbf{a}_\Bobet^\top\mathbf{V}_{A^\perp}\Bothe_i|}{\sigma} &\leq \log(B_0) + \max_{\Bobet\in\E_s} \frac{|\mathbf{q}_\Bobet^\top\Bothe_i|}{\sigma}  = \log(B_0) + \max_{\Bobet\in\E_s} \frac{\delta \big|\mathbf{q}_\Bobet^\top\mathbf{W}^{-1}\mathbf{q}_\Boalp\big|}{\sigma\sqrt{\mu_\Boalp d_\theta}} \\
                & \leq \log(B_0) + \max_{\Bobet\in\E_s} \frac{\delta\Vert\mathbf{W}^{-1/2}\mathbf{q}_\Bobet\Vert_2 \Vert\mathbf{W}^{-1/2}\mathbf{q}_\Boalp\Vert_2}{\sigma\sqrt{\mu_\Boalp d_\theta}}\\
                & \leq \log(B_0) + \frac{\delta\sqrt{\mu_Qd_\theta}}{\sigma} \leq \log(B),
            \end{align*}
            for $i=1,2$, where the last inequality follows by setting $\delta\leq \frac{\sigma\log(B/B_0)}{\sqrt{\mu_Q d_\theta}}$. Whenever $\frac{\sigma\log(B/B_0)}{\sqrt{\mu_Q d_\theta}}\leq \frac{\sigma}{\sqrt{N}}$, we have $\dd_{TV}(\bbp_{\Bothe_1}^{(N)},\bbp_{\Bothe_2}^{(N)}) \leq \frac{1}{2}$. This, together with \eqref{eq_minimax_Q_infty_mu_bound_1}, implies that
            \begin{equation*}
                \mathcal{R}_N^\star(\Theta_B,\Vert\cdot\Vert_{Q,\infty}^2) \geq \min\left\{ \frac{\mu_Qd_\theta\sigma^2}{4N},\frac{\sigma^2\log^2(B/B_0)}{4}\right\}.
            \end{equation*}
            Inequality~\eqref{eq_minimax_Q_infty_norm_Prob} then follows from~\eqref{eq_LeCam_two_point_Prob} and simple adjustment of the constants.
            \qedbox
		\subsection{Proofs of MLE Error Bounds}
        \label{subsec_proofs_of_MLE_error_bounds}
        We begin with some extra notations that will be used throughout the proof. It is convenient to reindex the observations by a single sample index $k$. Since $N=\sum_{\Boalp\in\E_s} n_\Boalp$, fix an arbitrary bijection $\pi:[N]\rightarrow\left\{(\Boalp,t):\Boalp\in\E_s,\ t\in[n_\Boalp]\right\},\,\pi(k)=(\Boalp_k,t_k)$. We then write $Y_k:=Y_{\Boalp_k}^{(t_k)},\, k\in[N]$. For each $\Boalp\in\mathcal E_s$, define
        \begin{equation}
            \label{eq_flattenIndex_Sep_N_into_Salpha}
            S_\Boalp := \{k\in[N] \mid \Boalp_k=\Boalp\},\,[N] = \bigcup_{\Boalp\in\E_s} S_\Boalp,\,|S_\Boalp| = n_\Boalp.
        \end{equation}
        The particular choice of the bijection $\pi$ is immaterial. For each $k\in[N]$, define $\mathbf{q}_k := \mathbf{q}_{\Boalp_k}$, and
        \begin{equation}
            \label{eq_flatten_index_qk_eps_k_v_kstar}
            \varepsilon_{k}:=Y_{k}-\mathbb{E}_\BoTheS^{\Boalp_k}[Y_{k}] = Y_k - \Psi^{\prime}(\eta_{\Boalp_k}(\BoTheS)),\quad v_k^\star:=\var_\BoTheS^{\Boalp_k}(Y_{k})=\Psi^{\prime\prime}(\eta_{\Boalp_k}(\BoTheS)),
        \end{equation}
        where $\eta_{\Boalp}(\Bothe)$ is defined in~\eqref{eq_definition_q_alp}. Let the collection of the design vectors and ground-truth variances be
        \begin{equation}
        \label{eq_def_of_x_k_X_V_star}
            \mathbf{x}_k : = \frac{\mathbf{q}_k}{\sigma},\quad\mathbf{X}:=[\mathbf{x}_1,\dots,\mathbf{x}_N]^\top\in\mathbb{R}^{N\times d_\theta},\quad \mathbf{V}^\star:=\mathrm{diag}(v_1^\star,\dots,v_N^\star),
        \end{equation}
        respectively. It is also convenient to introduce $\eta_k^\star:=\eta_{\Boalp_k}(\BoTheS)$. Moreover, consider the matrix $\mathbf{A}\in\mathbb{R}^{N\times N}$ with its $(i,j)$-th entry $\mathbf{A}_{ij}$ being $\frac{\mathbf{x}_i^\top\nabla^2\ell(\BoTheS)^{-1}\mathbf{x}_j}{N}$, and let $A_{\max}:=\max_{i\in[N],j\in[N]}|\mathbf{A}_{ij}|$.
        \subsubsection{Proof of Theorem~\ref{thm_suffic_sample_l2_upper_bound}}
        \label{proof_thm_suffic_sample_l2_upper_bound}
        The core of the proof is to establish~\eqref{eq_ell_Q_error_upper_hp}. This reduces to understanding the components of $\Delta$, for which the high-level road-map is partly inspired by~\cite[Theorem 1]{SuMouPropScoreDeBias} and~\cite[Theorem 1]{LinSuJackQuadBarriZEst}, where the authors show that a sample complexity on the order of $d_\theta^{3/2}$ suffices to isolate the second-order bias effect for the canonical MLE with sub-Gaussian design vectors~\cite[Section 2.2]{SuMouPropScoreDeBias} and for the corresponding Z-estimator with deterministic design that admits a solution in the interior of a compact set~\cite[Section 2.1]{LinSuJackQuadBarriZEst}, respectively. However, these two regularity conditions cannot hold in our context, and we construct a compact and convex set on which the residual map is a self-map and obtain a finite MLE from its fixed point (Lemma~\ref{lema_Self-mapping_Property}).
        Within this set, the Hessian of the negative log-likelihood function tends to be well-behaved, resulting in $\ell(\Bothe)$ exhibiting local strong convexity. For this step, Chen~\cite{ChenYanXiRankRisistan} and Yang et al.'s~\cite{YangCongTopKMonoAd, YangRaschRandomMLE} strategy is the most analogous to ours among existing literature. They consider running preconditioned gradient descent, starting from the ground truth
        \begin{equation*}
            \Bothe_{t+1} = \Bothe_t - \eta \nabla^2\ell(\BoTheS)^{-1} \nabla\ell(\Bothe_t),\quad\Bothe_0 = \BoTheS.
        \end{equation*}
        Performing this recursively, it yields~\cite{YangRaschRandomMLE}
        \begin{equation}
        \label{eq_precond_GD_iteration_recursive}
            \Delta^{t+1} =-[1-(1-\eta)^{t+1}]\underbrace{\nabla^2\ell(\BoTheS)^{-1}\nabla\ell(\BoTheS)}_{-\Delta_{lin}} - \eta \sum_{j=0}^{t}(1-\eta)^{t-j}\underbrace{ \nabla^2\ell(\BoTheS)^{-1}[\bar{\mathbf{H}}^j - \nabla^2\ell(\BoTheS)]\Delta^j}_{\text{Quadratic and higher-order residuals}},
        \end{equation}
        where $\bar{\mathbf{H}}^t = \int_0^1 \nabla^2\ell(\BoTheS + s\Delta_t)\mathrm{d} s$ and $\Delta^{t}:=\Bothe^{t}-\BoTheS$. The localization step is accomplished by proving by induction that $\Bothe_t$ will remain close to $\BoTheS$, and this procedure will converge to the MLE solution. Yet the induction steps in~\cite[Appendix B]{YangCongTopKMonoAd} heavily rely on the specific graphical structure of the pairwise ranking model, which in general does not hold under our setup. Taking~\eqref{eq_ell_Q_error_upper_hp} as given, we can establish local strong convexity within a small region around the ground truth, from which we prove~\eqref{eq_seminorm_error_upper_hp}. This step utilizes some corollaries from Bach et al.~\cite{OstrovsBachSelfCMEst, BachSelfconcordantLogit}. 
        
        \paragraph{Taylor Expansion of High-Order Residuals}
        \label{subsubsec_taylor_expasion_of_high_order_res}
        The first step of the proof is to rephrase the ``main order + residual" structure depicted in~\eqref{eq_precond_GD_iteration_recursive} as a nonlinear equality system that encodes the first-order optimality condition, and write out the second-order bias explicitly. To ease the notation, let the nonlinear remainder around $\Psi^\prime(\eta_j^\star)$ be defined as 
        \begin{equation}
        \label{eq_nonlinear_remainder_r_j_Delta}
            r_j(\Delta):=\Psi^\prime(\eta_j^\star + \mathbf{x}_j^\top\Delta) - \Psi^\prime(\eta_j^\star) - \Psi^{\prime\prime}(\eta_j^\star)\mathbf{x}_j^\top\Delta,\,\forall\,j\in[N],
        \end{equation}
        where $\Delta=\Bothe-\BoTheS$ and $\Bothe\in\mathbb{R}^{d_\theta}$. 
        By construction,~\eqref{eq_precond_GD_iteration_recursive} can be viewed as a Taylor expansion of the likelihood score
        \begin{align}
            \nabla\ell(\BoTheS + \Delta) &= \nabla^2\ell(\BoTheS)\left(\Delta - \Delta_{lin}\right) + \int_0^1 \left[\nabla^2\ell(\BoTheS + s\Delta) - \nabla^2\ell(\BoTheS)\right]\Delta\mathrm{d}s \nonumber\\
            \label{eq_original_residual_expansion_into_second_order}
            & = \nabla^2\ell(\BoTheS)\left(\Delta - \Delta_{lin}\right) + \frac{1}{N}\sum_{j=1}^N r_j(\Delta) \mathbf{x}_j.
        \end{align}
        We are particularly interested in the structure of the second term on the r.h.s.~of~\eqref{eq_original_residual_expansion_into_second_order}, the following elementary expansion separates the leading quadratic component of $r_j(\Delta)$.
        \begin{lemma}
        \label{lema_r_j_Delta_Taylor_expand}
            Let $\Bothe\in\mathbb{R}^{d_\theta}$, and $\Delta=\Bothe-\BoTheS$. For any $j\in[N]$, the nonlinear remainder defined in~\eqref{eq_nonlinear_remainder_r_j_Delta} can be equivalently expressed as 
            \begin{equation}
            \label{eq_taylor_expansion_third_order_integral_remainder}
                r_j(\Delta) = \frac{1}{2}\Psi^{(3)}(\eta_j^\star)|\mathbf{x}_j^\top\Delta|^2 + \frac{1}{2}\int_0^1 (1-s)^2 \Psi^{(4)}(\eta_j^\star + s\mathbf{x}_j^\top\Delta)(\mathbf{x}_j^\top\Delta)^3 \mathrm{d}s. 
            \end{equation}
        \end{lemma}
        \begin{proof}
            This follows directly from Taylor expansion of a three-times continuously differentiable function $g:\mathbb{R}\to\mathbb{R}$ to the third term with integral remainder, i.e.,
            \begin{equation*}
                g(a+h) = g(a) + g^\prime(a)h + \frac{1}{2}g^{\prime\prime}(a)h^2 + \int_0^1 \frac{h^3}{2} (1-s)^2 g^{(3)}(a+sh)\mathrm{d}s,
            \end{equation*}
            for $g(\cdot)=\Psi^\prime(\cdot)$ with $a=\eta_j^\star$ and $h=\mathbf{x}_j^\top \Delta$. \qedbox
        \end{proof}
        
        \noindent Lemma~\ref{lema_r_j_Delta_Taylor_expand} also explains the particular form of the deterministic bias correction in~\eqref{eq_def_of_Delta_bias}. Indeed, by the definition of $\Delta_{lin}$ and $\mathbf{A}$, we have
        \begin{align}
            \bbe_{\BoTheS}^{(N)} \left[ \nabla^2\ell(\BoTheS)^{-1}\left[\frac{1}{2N} \sum_{j=1}^N \Psi^{(3)}(\eta_j^\star)|\mathbf{x}_j^\top\Delta_{lin}|^2 \mathbf{x}_j\right] \right] &= \nabla^2\ell(\BoTheS)^{-1}\left[\frac{1}{2N}\sum_{j=1}^N \Psi^{(3)}(\eta_j^\star)\mathbf{A}_{jj}\mathbf{x}_j\right]\nonumber\\
            \label{eq_identity_Exp_Q_2_Delta_bias}
            &= -\Delta_{bias}.
        \end{align} 
        Inspired by similar results under stronger regularity conditions~\cite[Section 2.2]{LinSuJackQuadBarriZEst} and the asymptotic rate proved by Portnoy~\cite[Section 3]{PortnoAsymExpFamily}, we infer that $\Delta - \Delta_{lin} - \Delta_{bias}$ is sufficiently small, when measured in the $\Vert\cdot\Vert_{Q,\infty}$ norm with high probability, provided that $\Bothe$ satisfies the first-order optimality condition (see Section~\ref{subsec_Existence_of_Fixed_Point}). 
        The following lemma converts the likelihood score equation into a fixed-point system for the modified residual vector $\Delta - \Delta_{lin} - \Delta_{bias}$. 
        \begin{lemma}[Residual Expansion with Bias]
        \label{lema_Residual_Expansion_with_Bias}
            For any $\Delta\in\mathbb{R}^{d_\theta}$, and $\Delta_{lin}$, $\Delta_{bias}$ defined as in~\eqref{eq_def_of_Delta_lin},~\eqref{eq_def_of_Delta_bias}, respectively, let $\Bothe:=\Delta+\BoTheS$, and $\mathbf{h}: = \Delta -\Delta_{lin} - \Delta_{bias}$. Then $\Bothe$ satisfies the first-order optimality condition $\nabla\ell(\Bothe)=\mathbf{0}$ iff $\mathbf{h}$ satisfies
            \begin{equation}
            \label{eq_h_Lin_Bias_fixed_point}
                \mathbf{h} = -\nabla^2\ell(\BoTheS)^{-1}\left[\frac{1}{N}\sum_{j=1}^N\mathbf{x}_jr_j(\Delta_{lin} + \Delta_{bias} + \mathbf{h}) - \frac{1}{2N}\sum_{j=1}^N \mathbf{A}_{jj}\Psi^{(3)}(\eta_j^\star) \mathbf{x}_j \right].
            \end{equation}
        \end{lemma}
        \begin{proof}
            See Appendix~\ref{proof_lema_Residual_Expansion_with_Bias}.\qedbox
        \end{proof}
        
        We next recall the pseudo self-concordance property of the cumulant function  $\Psi(\cdot)$, introduced by Bach et al.~\cite{BachSelfconcordantLogit, OstrovsBachSelfCMEst}, which will be utilized to control both the higher-order remainder and the variation of the Hessian of the negative log-likelihood function.
        \begin{definition}
            Let $g:[0,1]\to \mathbb{R}$ be a three-times continuously differentiable function. If $g^{\prime\prime}(0)>0$, and there exists some $S\in[0,\infty)$ such that 
            \begin{equation*}
            \label{eq_pseudo_self_concor_R_func}
                |g^{(3)}(t)|\leq S g^{\prime\prime}(t),\,\forall\, t\in[0,1],
            \end{equation*}
            then $g$ is said to be pseudo self-concordant with parameter $S$ (see ~\cite{BachSelfconcordantLogit, OstrovsBachSelfCMEst}). 
        \end{definition}
        By the chain rule and the fact that $|\Psi^{(3)}(\eta)|\leq \Psi^{\prime\prime}(\eta)$, it is straightforward to verify that for any $t\in[0,1]$, $\Psi(\eta_j^\star + t\mathbf{x}_j^\top \Delta)$, as a function of $t$, is pseudo self-concordant with parameter $|\mathbf{x}_j^\top \Delta|$, i.e.,
        \begin{equation}
        \label{eq_Psi_j_link_pseudo_self_concor_func}
            \Big|\Psi^{(3)}(\eta_j^\star + t\mathbf{x}_j^\top \Delta)\cdot (\mathbf{x}_j^\top \Delta)^3\Big| \leq |\mathbf{x}_j^\top \Delta| \cdot\left(\Psi^{\prime\prime}(\eta_j^\star + t\mathbf{x}_j^\top \Delta)|\mathbf{x}_j^\top \Delta|^2\right),\,\forall\, t\in[0,1],\,\forall\,j\in[N].
        \end{equation}
        Analogously, define the nonlinear remainder around $\Psi(\eta_j^\star)$ by
        \begin{equation}
            \delta_j(\Delta): = \Psi(\eta_j^\star + \mathbf{x}_j^\top \Delta) - \Psi(\eta_j^\star) - \Psi^\prime(\eta_j^\star)\mathbf{x}_j^\top \Delta,\,\forall\,j\in[N].
        \end{equation}
        Then, recursively integrating both sides of \eqref{eq_Psi_j_link_pseudo_self_concor_func} yields the following result.
        \begin{lemma}[{\cite[Lemma~1]{BachSelfconcordantLogit}}]
        \label{lema_single_variate_self_concord}
            Let $\Bothe\in\mathbb{R}^{d_\theta}$, $\Delta=\Bothe-\BoTheS$, and $t\in[0,1]$. Then for any $j\in[N]$\footnote{We adopt the limit definition when $\mathbf{x}_j^\top\Delta=0$, that is, for fixed $t$, $\lim_{S\downarrow 0 }\frac{1-e^{-St}}{S} =\lim_{S\downarrow 0 }\frac{e^{St}-1}{S} = t$.} 
            \begin{subequations}
                \begin{align}
                \label{eq_self_concor_on_R_gppt}
                e^{-|\mathbf{x}_j^\top\Delta|t} \Psi^{\prime\prime}(\eta_j^\star) &\leq \Psi^{\prime\prime}(\eta_j^\star + t\mathbf{x}_j^\top \Delta) \leq e^{|\mathbf{x}_j^\top\Delta|t}\Psi^{\prime\prime}(\eta_j^\star),\\
                \label{eq_self_concor_sourrgate_int_Psi_secnd}
                \frac{1-e^{-|\mathbf{x}_j^\top\Delta|t}}{|\mathbf{x}_j^\top\Delta|}\Psi^{\prime\prime}(\eta_j^\star)&\leq \int_{0}^t \Psi^{\prime\prime}(\eta_j^\star + s\mathbf{x}_j^\top \Delta)\mathrm{d}s\leq \frac{e^{|\mathbf{x}_j^\top\Delta|t}-1}{|\mathbf{x}_j^\top\Delta|}\Psi^{\prime\prime}(\eta_j^\star),\\
                \label{eq_self_concor_sourrgate_r_j_tDelta}
                (e^{-|\mathbf{x}_j^\top\Delta|t}+|\mathbf{x}_j^\top\Delta|t -1)\Psi^{\prime\prime}(\eta_j^\star) &\leq \delta_j(t\Delta) \leq (e^{|\mathbf{x}_j^\top\Delta|t}-|\mathbf{x}_j^\top\Delta|t -1)\Psi^{\prime\prime}(\eta_j^\star).
            \end{align}
            \end{subequations}
            When $\mathbf{x}_j^\top \Delta=0$, all three inequalities hold trivially.
        \end{lemma}
        Bach's convex localization strategy~\cite[Proposition~B.4]{OstrovsBachSelfCMEst} can be generalized to the unconstrained case. However, inequality~\eqref{eq_self_concor_sourrgate_r_j_tDelta} is not sharp enough to derive sample complexity of order $d_\theta^{3/2}$. Furthermore, Lemma~\ref{lema_single_variate_self_concord} indicates that the core of the proof lies in proper control of the term $\max_{m\in[N]}|\mathbf{x}_m^\top (\hat{\Bothe}-\BoTheS)|$. This will be the main focus of the next section.
        \paragraph{Existence of Fixed Point}
        \label{subsec_Existence_of_Fixed_Point}
        The existence of canonical MLE can be derived by existence of a fixed point of the nonlinear equality system~\eqref{eq_h_Lin_Bias_fixed_point}, underpinned by Brouwer's fixed point theorem. Specifically, Theorem~\ref{thm_suffic_sample_l2_upper_bound} builds on the following lemma.

        \begin{lemma}[Self-mapping Property]
        \label{lema_Self-mapping_Property}
                When $N\gtrsim B^2\mu_Q d_{\theta}^{3/2}\log^3 d$, there exist some positive constants $r_\infty$ and $r_H$, where
                \begin{align}
                \label{eq_r_infty_in_self_mapping_set}
                    r_\infty &\asymp \sqrt{A_{\max}\log d} + A_{\max}\sqrt{d_\theta} + BA_{\max}\log d,\\
                    r_H &\asymp \sqrt{\frac{A_{\max}d_\theta}{N}} + \sqrt{\frac{A_{\max}d_\theta B\log^2 d}{N}} + A_{\max}\log(d)\sqrt{\frac{d_\theta B\log d}{N}},
                \end{align}
                such that with probability at least $1-d^{-c}$,
                \begin{equation*}
                    \Phi(\mathbf{h}):=-\nabla^2\ell(\BoTheS)^{-1}\left[\frac{1}{N}\sum_{j=1}^N\mathbf{x}_jr_j(\Delta_{lin} + \Delta_{bias} + \mathbf{h}) - \frac{1}{2N}\sum_{j=1}^N \mathbf{A}_{jj}\Psi^{(3)}(\eta_j^\star) \mathbf{x}_j \right]
                \end{equation*}
                is a self-mapping on the non-empty, closed, convex and compact set 
                \begin{equation}
                    \mathcal{N}(r_H,r_\infty):=\left\{\mathbf{h}\in\mathbb{R}^{d_\theta}\,\Big|\, \Vert\mathbf{h}\Vert_{\nabla^2\ell(\BoTheS)}\leq r_H,\,\max_{m\in[N]}|\mathbf{x}_m^\top\mathbf{h}|\leq r_\infty\right\},
                \end{equation}
                that is, 
                \begin{equation}
                \label{eq_highprob_self_mapping}
                    \bbp_{\BoTheS}^{(N)}\Big\{\forall\,\mathbf{h}\in\mathcal{N}(r_H,r_\infty),\,\Phi(\mathbf{h})\in\mathcal{N}(r_H,r_\infty)\Big\}\geq 1-d^{-c}.
                \end{equation}
        \end{lemma}
        The proof is deferred to Appendix~\ref{proof_lema_Self-mapping_Property}.

        \noindent 
        \textbf{Proof of Theorem~\ref{thm_suffic_sample_l2_upper_bound}.}
        With Lemmas~\ref{lema_Residual_Expansion_with_Bias} and~\ref{lema_Self-mapping_Property},
        we are ready to prove Theorem~\ref{thm_suffic_sample_l2_upper_bound}. First,  
        observe that $\Phi:\mathcal{N}(r_H,r_\infty)\to\mathcal{N}(r_H,r_\infty)$ is continuous because each $r_j(\cdot)$ is continuous, and $\Phi(\cdot)$ is a finite linear combination of $r_j(\Delta_{lin}+\Delta_{bias}+(\cdot))$.
        Since $\mathcal{N}(r_H,r_\infty)$ is nonempty, compact and convex, 
        the self-mapping property
        enables us to apply Brouwer's fixed-point theorem which guarantees 
        existence of 
        a fixed point $\hat{\mathbf{h}}\in \mathcal{N}(r_H,r_\infty)$ such that
        \begin{equation}
        \label{eq_existence_of_fix_point_hat_h}
            \hat{\mathbf{h}} = \Phi(\hat{\mathbf{h}}).
        \end{equation}
        Let $\hat{\Delta}:=\Delta_{lin} + \Delta_{bias} + \hat{\mathbf{h}}$
        and 
        accordingly 
        $\hat{\Bothe}:= \BoTheS + \hat{\Delta}$. By Lemma~\ref{lema_Residual_Expansion_with_Bias}, the fixed point identity~\eqref{eq_existence_of_fix_point_hat_h} implies $\nabla\ell(\hat{\Bothe})=\mathbf{0}$. 
        Moreover, since $\Psi^{\prime\prime}(\eta)>0$ for any $\eta\in\mathbb{R}$ and $\mathbf{W}\succ\mathbf{0}$ by assumption, the negative log-likelihood function $\ell(\Bothe)$ is strictly convex, which means that $\hat{\Bothe}$ 
        is a unique global minimizer. Therefore the fixed point $\hat{\mathbf{h}}$ in~\eqref{eq_existence_of_fix_point_hat_h} is unique. 
        Finally, since $\hat{\mathbf{h}}\in \mathcal{N}(r_H,r_\infty)$, with probability at least $1-d^{-c}$, then
        \begin{alignat*}{2}
            \max_{m\in[N]}|\mathbf{x}_m^\top\hat{\Delta}|
            &\leq \max_{m\in[N]}|\mathbf{x}_m^\top\Delta_{lin}| + \max_{m\in[N]}|\mathbf{x}_m^\top\Delta_{bias}| + r_\infty\quad&\quad&\\
            &\lesssim \sqrt{A_{max}\log d} + A_{max}\log d + \frac{1}{2}A_{\max}\sqrt{d_\theta} + r_\infty\quad&\text{(by Lemmas~\ref{lema_Delta_lin_Q_infty_H_bounds} and~\ref{lema_Bias_Envelope_upperbound})}&\\
            &\lesssim \sqrt{A_{\max}\log d} + A_{\max}\sqrt{d_\theta} + BA_{\max}\log d\quad&\text{(by Lemma~\ref{lema_Self-mapping_Property})}&\\
            &\lesssim \sqrt{\frac{B\mu_Q d_\theta \log d}{N}} + \frac{B\mu_Q d_\theta^{3/2}}{N},
        \end{alignat*}
        where the last inequality is from Lemma~\ref{lema_property_of_A_matrix},
        for $N\gtrsim B^3 \mu_Q d_\theta \log d$. By noticing $\Vert\hat{\Bothe}-\BoTheS\Vert_{Q,\infty}=\Vert\hat{\Delta}\Vert_{Q,\infty} = \sigma \max_{m\in[N]}|\mathbf{x}_m^\top\hat{\Delta}|$, we 
        obtain \eqref{eq_ell_Q_error_upper_hp}.
       \paragraph{Local Strong Convexity}
        \label{subsubsec_restricted_strong_cvxity}
            Inequality~\eqref{eq_seminorm_error_upper_hp} is proved by the strong convexity of $\ell(\Bothe)$ over the set
            \begin{equation*}
                \mathfrak{W}_{\infty,1/2}:=\left\{\Bothe\in\mathbb{R}^{d_\theta}\,\Big|\, \max_{m\in[N]}|\mathbf{x}_m^\top(\Bothe-\BoTheS)|\leq \frac{1}{2}\right\}.
            \end{equation*}
            To prove this assertion, let $\Delta=\Bothe-\BoTheS$ for any $\Bothe\in\mathbb{R}^{d_\theta}$ and by the first inequality in~\eqref{eq_self_concor_on_R_gppt} and the structure of $\nabla^2\ell(\Bothe)$ in~\eqref{eq_hess_mle_in_theta}
            \begin{align*}
                \Delta^\top\nabla^2\ell(\BoTheS+ t\Delta)\Delta&=\frac{1}{N}\sum_{j=1}^N \Psi^{\prime\prime}(\eta_j^\star + t\mathbf{x}_j^\top \Delta)|\mathbf{x}_j^\top \Delta|^2  \geq \frac{1}{N}\sum_{j=1}^N e^{-|\mathbf{x}_j^\top \Delta|t} \Psi^{\prime\prime}(\eta_j^\star)|\mathbf{x}_j^\top \Delta|^2 \\
                &\geq \exp\left(-t\max_{m\in[N]}|\mathbf{x}_m^\top\Delta|\right)\frac{1}{N} \sum_{j=1}^N \Psi^{\prime\prime}(\eta_j^\star)|\mathbf{x}_j^\top \Delta|^2 \\
                &= \exp\left(-t\max_{m\in[N]}|\mathbf{x}_m^\top\Delta|\right) \Vert\Delta\Vert_{\nabla^2\ell(\BoTheS)}^2,\,\forall\,t\in[0,1].
            \end{align*}
            By integrating both sides of the above inequality twice w.r.t.~$t$, we obtain
            (see~\eqref{eq_self_concor_sourrgate_r_j_tDelta} or \cite[Proposition~1]{BachSelfconcordantLogit}~\cite[Proposition B.3]{OstrovsBachSelfCMEst})\footnote{Similarly, we take the limit definition when $\Delta=\mathbf{0}$, i.e., $\lim_{S\downarrow 0}\frac{e^{-S}+S -1}{S^2}=\frac{1}{2}.$}
            \begin{equation}
            \label{eq_strong_convxity_by_x_m_Delta}
                \ell(\Bothe) - \ell(\BoTheS) - \langle \nabla\ell(\BoTheS),\Delta \rangle \geq \frac{e^{-\max_{m\in[N]}|\mathbf{x}_m^\top\Delta|}+\max_{m\in[N]}|\mathbf{x}_m^\top\Delta|-1}{(\max_{m\in[N]}|\mathbf{x}_m^\top\Delta|)^2} \Vert\Delta\Vert_{\nabla^2\ell(\BoTheS)}^2.
            \end{equation}
            Using the basic inequality $\frac{x^2}{4}\leq e^{-x}+x-1$ 
            for $0\leq x\leq \frac{1}{2}$, we deduce from \eqref{eq_strong_convxity_by_x_m_Delta} that 
            \begin{equation}
            \label{eq_restricted_strong_convxity_in_N_infty_12}
                \ell(\Bothe) - \ell(\BoTheS) - \langle \nabla\ell(\BoTheS),\Delta \rangle \geq \frac{1}{4}\Vert\Delta\Vert_{\nabla^2\ell(\BoTheS)}^2,\,\forall\,\Bothe\in \mathfrak{W}_{\infty,1/2},
            \end{equation}
            where we prove the $\frac{1}{4}$-local strong convexity at $\BoTheS$ of $\ell(\Bothe)$.
            Moreover, by Lemma~\ref{lema_basic_cvx_loc_M_ester} and the assumption that $\nabla^2\ell(\BoTheS)\succ\mathbf{0}$, 
            \begin{equation}
                \Vert\Bothe-\BoTheS\Vert_{\nabla^2\ell(\BoTheS)} \leq 4 \Vert\nabla\ell(\BoTheS)\Vert_{\nabla^2\ell(\BoTheS)^{-1}},\,\forall\,\Bothe\in \mathfrak{W}_{\infty,1/2},\,\ell(\Bothe)\leq \ell(\BoTheS).
            \end{equation}
            Finally, under the assumptions of Lemma~\ref{lema_Self-mapping_Property}, by~\eqref{eq_r_infty_in_self_mapping_set}, the fact that $\ell(\hat{\Bothe})\leq \ell(\BoTheS)$, and Lemma~\ref{lema_Delta_lin_Q_infty_H_bounds}, we conclude that
            with probability at least $1-d^{-c}$, $\hat{\Bothe}\in \mathfrak{W}_{\infty,1/2}$, and
            \begin{equation*}
                \Vert\hat{\Bothe}-\BoTheS\Vert_{\nabla^2\ell(\BoTheS)} \lesssim \left(\sqrt{\frac{d_\theta}{N}}+ \sqrt{\frac{Bd_\theta \log d}{N}}  \right)  \lesssim \sqrt{\frac{Bd_\theta \log d}{N}}. 
            \end{equation*}
            This concludes the proof of~\eqref{eq_seminorm_error_upper_hp}. \qedbox
            \subsubsection{Proof of Theorem~\ref{thm_refinred_Q_infty_ell_2_residual_decompose_bound}}
            \label{proof_thm_refinred_Q_infty_ell_2_residual_decompose_bound}
            The proof relies on the proof of Lemma~\ref{lema_Self-mapping_Property}, and we will use the notation introduced there. Denote the MLE solution derived in Theorem~\ref{thm_suffic_sample_l2_upper_bound} and the fixed point deduced in~\eqref{eq_existence_of_fix_point_hat_h} by $\hat{\Bothe}$ and $\mathbf{\hat{h}}$, respectively. Let $\hat{\Delta}:=\hat{\Bothe} - \BoTheS = \Delta_{lin} + \Delta_{bias} + \mathbf{\hat{h}}$ and denote $\hat{\mathbf{e}}_{err}:=\Delta_{bias} + \mathbf{\hat{h}}$. By the derivation in Section~\ref{subsec_Existence_of_Fixed_Point}, $\hat{\mathbf{h}}$ satisfies the equality system~\eqref{eq_h_Lin_Bias_fixed_point}. This, together with Lemma~\ref{lema_r_j_Delta_Taylor_expand} and  identity~\eqref{eq_identity_Exp_Q_2_Delta_bias} gives
            \begin{align}
                 \label{eq_full_MLE_residual_expansion}
                    \hat{\Delta} &= \Delta_{lin} + \Delta_{bias} + \Phi(\hat{\mathbf{h}})\nonumber\\
                    & = \Delta_{lin} + \Delta_{bias} -\nabla^2\ell(\BoTheS)^{-1}\left[ Q_2(\Delta_{lin} + \hat{\mathbf{e}}_{err}) + Q_3(\Delta_{lin} + \hat{\mathbf{e}}_{err}) - \mathbb{E}_\BoTheS^{(N)}[Q_2(\Delta_{lin})]\right]\nonumber\\
                    &= \Delta_{lin} + \Delta_{bias} + \Delta_{quad} + \Delta_{\geq3},
            \end{align}
            where
            \begin{subequations}
                \begin{align}
                \label{eq_def_of_residual_quad_order}
                \Delta_{quad}&: = -\nabla^2\ell(\BoTheS)^{-1}\left[ Q_2(\Delta_{lin}) - \mathbb{E}_\BoTheS^{(N)}\left[Q_2(\Delta_{lin})\right] \right],\\
                \label{eq_def_of_residual_geq_3_order}
                \Delta_{\geq3} &: = -\nabla^2\ell(\BoTheS)^{-1}\left[ Q_2(\Delta_{lin} + \hat{\mathbf{e}}_{err}) - Q_2(\Delta_{lin}) + Q_3(\Delta_{lin} + \hat{\mathbf{e}}_{err})\right].
                \end{align}
            \end{subequations}
            \noindent\textbf{Proof of~\eqref{eq_Delta_minu_Deltalin_minu_Delta_bias_Q_infty}.} It suffices to derive the upper bounds on $\Vert\Delta_{quad}\Vert_{Q,\infty}$ and $\Vert\Delta_{\geq 3}\Vert_{Q,\infty}$. From STEP 2 in the proof of Lemma~\ref{lema_Self-mapping_Property} and the assumption $N\gtrsim B^3\mu_Q d_\theta^{3/2}\log^3 d$,
                \begin{align*}
                    \max_{m\in[N]} |\mathbf{x}_m^\top\Delta_{\geq 3}| &\leq \max_{m\in[N]} \Bigg|\mathbf{x}_m^\top \nabla^2\ell(\BoTheS)^{-1} \left[ Q_2(\Delta_{lin}+\hat{\mathbf{e}}_{err}) - Q_2(\Delta_{lin}) \right] \Bigg| + \max_{m\in[N]}\Bigg| \mathbf{x}_m^\top \nabla^2\ell(\BoTheS)^{-1} Q_3(\Delta_{lin}+\hat{\mathbf{e}}_{err})\Bigg| \\
                    &\lesssim \sqrt{A_{\max}N}\Vert \hat{\mathbf{e}}_{err} \Vert_{\nabla^2\ell(\BoTheS)}\left[ \sqrt{BA_{\max}\log d} + \max_{l\in[N]}|\mathbf{x}_l^\top \hat{\mathbf{e}}_{err}|\right] \\
                    & + \left( \sqrt{A_{\max}\log d} + A_{\max}\log d\right)\left( BA_{\max}\log d +A_{\max}\sqrt{d_\theta} \right) \\
                    & + \sqrt{A_{\max}N}\left( \max_{l\in[N]}|\mathbf{x}_l^\top\hat{\mathbf{e}}_{err}| \right)^2\Vert \hat{\mathbf{e}}_{err} \Vert_{\nabla^2\ell(\BoTheS)}\\
                    & \lesssim \left( \sqrt{A_{\max}\log d} + A_{\max}\log d\right)\left( BA_{\max}\log d +A_{\max}\sqrt{d_\theta} \right)\\
                    & + \sqrt{A_{\max}N}\left( \sqrt{\frac{A_{\max}d_\theta}{N}} + r_H \right) \left[ \sqrt{BA_{\max}\log d} + \frac{1}{2}A_{\max}\sqrt{d_\theta} +r_\infty + \left( A_{\max}\sqrt{d_\theta} + r_\infty \right)^2 \right]\\
                    &\lesssim BA_{\max}^{3/2}\log^{3/2}d +A_{\max}^{3/2}\sqrt{d_\theta \log d} + BA_{\max}^{3/2}\sqrt{d_\theta}\log^{3/2}d + \sqrt{B}A_{\max}^2 d_\theta\log d\\
                    &\lesssim BA_{\max}^{3/2}\sqrt{d_\theta}\log^{3/2}d + \sqrt{B}A_{\max}^2d_\theta\log d,
                \end{align*}
                where the second inequality follows by the upper bounds on $I_2$ and $I_3$ derived in Appendix~\ref{subsubsec_proof_of_Self_mapping_property}.
            On the other hand, by~\eqref{eq_I_1_Delta_quad_linear_form_upper}, $\max_{m\in[N]} |\mathbf{x}_m^\top \Delta_{quad}|  \lesssim BA_{\max}\log d$. Then, by combining the above two upper bounds and substituting the upper bound on $A_{\max}$, we obtain
            \begin{align*}
                \Vert\hat{\Delta} - \Delta_{lin} - \Delta_{bias}\Vert_{Q,\infty} &\leq \sigma\left(\max_{m\in[N]} |\mathbf{x}_m^\top\Delta_{\geq 3}| + \max_{m\in[N]} |\mathbf{x}_m^\top \Delta_{quad}|\right)\\
                &\lesssim \frac{\sigma B^2 \mu_Q d_\theta \log d}{N} \left[ \sqrt{\frac{B\mu_Q d_\theta^2 \log d}{N}} + \frac{\sqrt{B}\mu_Q d_\theta^2}{N} + 1 \right].
            \end{align*}
            Inequality~\eqref{eq_Delta_minu_Deltalin_minu_Delta_bias_Q_infty} is then directly recovered when $N\gtrsim B^3\mu_Q d_\theta^{3/2}\log^3 d$.
            
            \noindent\textbf{Proof of~\eqref{eq_Delta_minus_lin_bias_upper_bound}.} By~\eqref{eq_full_MLE_residual_expansion}, it suffices to find upper bounds on $\Vert\Delta_{\geq 3}\Vert_2$ and $\Vert\Delta_{quad}\Vert_2$. While the former can be derived via $\Vert\Delta_{\geq 3}\Vert_{\nabla^2\ell(\BoTheS)}$, i.e.,~\eqref{eq_upper_bound_H2_in_seminorm_selfmap} and~\eqref{eq_upper_bound_H3_in_seminorm_selfmap}, the latter requires some non-trivial tools for bounding the supremum of a collection of sub-Gaussian quadratic forms. Here, we utilize the Hanson-Wright inequality in vector space (which can be viewed as a corollary of the general theory developed by Adamczak et al.~\cite{AdamczakLatalaBanachHW}, restated in Lemma~\ref{lema_HilbertSpaceHW}). We refer interested readers to the paper and the references therein.

            \noindent \textbf{Bound for }$\Vert\Delta_{quad}\Vert_2$. We first recast $\Delta_{quad}$ in the form depicted in Lemma~\ref{lema_HilbertSpaceHW}. From the proof of~\eqref{eq_Centered_Quadratic_Biases1} in Lemma~\ref{lema_Centered_Quadratic_Biases_all3}, we can verify $\mathbb{E}_\BoTheS^{(N)}[\Delta_{quad}] = \mathbf{0}$, and $\Delta_{quad}$ can be written as follows
            \begin{align*}
                \Delta_{quad} &= -\frac{1}{2N}\sum_{j=1}^N \frac{\Psi^{(3)}(\eta_j^\star)}{v_j^\star} \nabla^2\ell(\BoTheS)^{-1}\mathbf{x}_j\left[ (\mathbf{e}_j^\top\mathbf{P}\boldsymbol{\xi})^2 - \mathbf{P}_{jj} \right]\\
                & = \sum_{(k,l)\in[N]^2} \mathbf{a}^{(kl)} \left(\xi_k\xi_l - \mathbb{E}_\BoTheS^{(N)}\left[\xi_k\xi_l\right]\right),
            \end{align*}
            where $\xi_i = \eps_i/\sqrt{v_i^\star}$, $i\in[N]$ are independent, zero mean, isotropic, and $\sqrt{B}$-sub-Gaussian, and $\boldsymbol{\xi}:=[\xi_1,\dots,\xi_N]^\top$, and $\mathbf{a}^{(kl)}:=-\frac{1}{2N}\sum_{j=1}^N  \frac{\Psi^{(3)}(\eta_j^\star)}{v_j^\star}\mathbf{P}_{jk}\mathbf{P}_{jl}\nabla^2\ell(\BoTheS)^{-1}\mathbf{x}_j\in\mathbb{R}^{d_\theta}$. In what follows, we derive bounds for the terms $T^2 := \sum_{(k,l)\in[N]^2}\Vert\mathbf{a}^{(kl)}\Vert_2^2$ and the $U,V$ defined in Lemma~\ref{lema_HilbertSpaceHW}. 

            For $T^2$, we use the fact that $\mathbf{P}\succeq
            \mathbf{0}$ and $\mathbf{P}^2 =\mathbf{P}$,
            \begin{align*}
                T^2 & = \frac{1}{4N^2} \sum_{(m,n)\in[N]^2} \frac{\Psi^{(3)}(\eta_m^\star)}{v_m^\star} \frac{\Psi^{(3)}(\eta_n^\star)}{v_n^\star} \left(\mathbf{x}_m^\top \nabla^2\ell(\BoTheS)^{-2}\mathbf{x}_n\right) \sum_{k\in[N]}\mathbf{P}_{mk}\mathbf{P}_{kn} \sum_{l\in[N]}\mathbf{P}_{ml}  \mathbf{P}_{ln}\\
                & = \frac{1}{4N^2} \left\langle \mathbf{P}\circ\mathbf{P}, \left(\frac{\Psi^{(3)}(\eta_m^\star)}{v_m^\star} \frac{\Psi^{(3)}(\eta_n^\star)}{v_n^\star} \mathbf{x}_m^\top \nabla^2\ell(\BoTheS)^{-2}\mathbf{x}_n\right)_{m,n\leq N} \right\rangle\\
                & \leq \frac{1}{4N^2}\Vert \mathbf{P}\circ\mathbf{P} \Vert \sum_{j=1}^N \left(\frac{\Psi^{(3)}(\eta_j^\star)}{v_j^\star}\right)^2\mathbf{x}_j^\top \nabla^2\ell(\BoTheS)^{-2}\mathbf{x}_j\leq \frac{1}{4N^2}\Vert \mathbf{P}\circ\mathbf{P} \Vert \sum_{j=1}^N\mathbf{x}_j^\top \nabla^2\ell(\BoTheS)^{-2}\mathbf{x}_j,
            \end{align*}
            where the second last inequality follows from the fact that both matrices in the inner product are positive semidefinite. Then, by the Gershgorin circle theorem and Lemma~\ref{lema_property_of_A_matrix}
            \begin{align*}
                \Vert \mathbf{P}\circ\mathbf{P} \Vert &\leq \max_{i\in[N]}\sum_{j=1}^N \Big|[\mathbf{P}\circ\mathbf{P}]_{ij}\Big| = \max_{i\in[N]} \sum_{j=1}^N v_i^\star v_j^\star \mathbf{A}_{ij}^2 =\max_{i\in[N]} v_i^\star \mathbf{A}_{ii} =\max_{i\in[N]} \mathbf{P}_{ii} \\
                & \leq \frac{A_{\max}}{4}.
            \end{align*}
            On the other hand, by Proposition~\ref{prop_bound_dymic_range}, $\sum_{j=1}^N \mathbf{x}_j\mathbf{x}_j^\top \preceq 4B \mathbf{X}^\top \mathbf{V}^\star \mathbf{X} = 4B N \nabla^2\ell(\BoTheS)$. Thus
            \begin{align*}
                \sum_{j=1}^N\mathbf{x}_j^\top \nabla^2\ell(\BoTheS)^{-2}\mathbf{x}_j &= \left\langle \nabla^2\ell(\BoTheS)^{-2}, \sum_{j=1}^N \mathbf{x}_j\mathbf{x}_j^\top \right\rangle \\
                &\leq \langle \nabla^2\ell(\BoTheS)^{-2}, 4B N \nabla^2\ell(\BoTheS)\rangle = 4BN \mathrm{tr}(\nabla^2\ell(\BoTheS)^{-1}),
            \end{align*}
            from which we have 
            \begin{equation}
            \label{eq_Hlibert_T_final_upper_bound}
                T \leq \frac{1}{2}\sqrt{\frac{BA_{\max}\mathrm{tr}(\nabla^2\ell(\BoTheS)^{-1})}{N}}.
            \end{equation}

            Next, we derive an upper bound for the second term in $U$. Observe
            \begin{align*}
                U_2 &:=\sup_{\Vert\mathbf{M}\Vert_F\leq 1} \Bigg\Vert\sum_{(k,l)\in[N]^2} \mathbf{a}^{(kl)} [\mathbf{M}]_{kl} \Bigg\Vert_2 = \sup_{\Vert\mathbf{v}\Vert_2\leq 1} \sup_{\Vert\mathbf{M}\Vert_F\leq 1}\Bigg|\sum_{(k,l)\in[N]^2} \mathbf{v}^\top\mathbf{a}^{(kl)} [\mathbf{M}]_{kl} \Bigg|\\
                & = \sup_{\Vert\mathbf{v}\Vert_2\leq 1} \Bigg\Vert\left(  \mathbf{v}^\top\mathbf{a}^{(kl)}\right)\Bigg\Vert_F = \sup_{\Vert\mathbf{v}\Vert_2\leq 1} \frac{1}{2N}\Vert\mathbf{P}\mathbf{D}_v \mathbf{P}\Vert_F,
            \end{align*}
            where $\mathbf{D}_v = \mathrm{diag}(\mathbf{d}_v)$, and $\mathbf{d}_v : = \left( (\frac{\Psi^{(3)}(\eta_i^\star)}{v_i^\star})\mathbf{v}^\top \nabla^2\ell(\BoTheS)^{-1}\mathbf{x}_i \right)_{i\leq N}$. Thus
            \begin{align}
                (U_2)^2 &= \sup_{\Vert\mathbf{v}\Vert_2\leq 1} \frac{1}{4N^2} \mathrm{tr}(\mathbf{D}_v\mathbf{P}\mathbf{D}_v\mathbf{P}) = \sup_{\Vert\mathbf{v}\Vert_2\leq 1} \frac{1}{4N^2}\mathbf{d}_v^\top (\mathbf{P}\circ\mathbf{P}) \mathbf{d}_v \nonumber\\
                & \leq \sup_{\Vert\mathbf{v}\Vert_2\leq 1} \frac{1}{4N^2} \Vert \mathbf{P}\circ\mathbf{P} \Vert \Vert\mathbf{d}_v\Vert_2^2 \leq \frac{A_{\max}}{16N^2} \sup_{\Vert\mathbf{v}\Vert_2\leq 1}\Vert\mathbf{d}_v\Vert_2^2 \nonumber\\
                & \leq \frac{A_{\max}}{16N^2} \sup_{\Vert\mathbf{v}\Vert_2\leq 1} \sum_{j=1}^N \left(\mathbf{v}^\top\nabla^2\ell(\BoTheS)^{-1}\mathbf{x}_j\right)^2 \leq \frac{BA_{\max}}{4N} \sup_{\Vert\mathbf{v}\Vert_2\leq 1} \mathbf{v}^\top \nabla^2\ell(\BoTheS)^{-1}\mathbf{v} \nonumber\\
                \label{eq_Hlibert_U2_final_upper_bound}
                &\leq \frac{BA_{\max}\Vert\nabla^2\ell(\BoTheS)^{-1}\Vert}{4N}. 
            \end{align}
            We then find an upper bound on $V$. Let $\Vert\mathbf{u}\Vert_2,\Vert\mathbf{v}\Vert_2\leq 1$ and $\mathbf{y}$ be a unit vector. Then
            \begin{align*}
                \Bigg\Vert\sum_{(k,l)\in[N]^2}\mathbf{a}^{(kl)}\mathbf{u}_k\mathbf{v}_l\Bigg\Vert_2 & = \sup_{\Vert\mathbf{y}\Vert_2=1} \frac{1}{2N}\Bigg| \sum_{j=1}^N \frac{\Psi^{(3)}(\eta_j^\star)}{v_j^\star}\left(\mathbf{y}^\top \nabla^2\ell(\BoTheS)^{-1}\mathbf{x}_j\right) \left( \mathbf{e}_j^\top\mathbf{Pu}\right)\left(\mathbf{e}_j^\top\mathbf{Pv}\right)\Bigg|\\
                & \leq \sup_{\Vert\mathbf{y}\Vert_2=1} \frac{1}{2N} \max_{m\in[N]}\Bigg| \mathbf{y}^\top \nabla^2\ell(\BoTheS)^{-1}\mathbf{x}_m \Bigg| \sum_{j=1}^N \Big|\mathbf{e}_j^\top\mathbf{Pu}\Big| \cdot \Big|\mathbf{e}_j^\top\mathbf{Pv}\Big|\\
                & \leq \frac{\max_{m\in[N]} \Vert\nabla^2\ell(\BoTheS)^{-1}\mathbf{x}_m\Vert_2}{2N} \Vert\mathbf{Pu}\Vert_2 \Vert\mathbf{Pv}\Vert_2,
            \end{align*}
            from which we have 
            \begin{align}
                V&:= \sup_{\Vert\mathbf{u}\Vert_2,\Vert\mathbf{v}\Vert_2\leq 1}\Bigg\Vert\sum_{(k,l)\in[N]^2}\mathbf{a}^{(kl)}\mathbf{u}_k\mathbf{v}_l\Bigg\Vert_2  \nonumber\\
                \label{eq_Hlibert_V_final_upper_bound}
                &\leq \frac{\max_{m\in[N]} \Vert\nabla^2\ell(\BoTheS)^{-1}\mathbf{x}_m\Vert_2}{2N}\leq \frac{1}{2} \sqrt{\frac{A_{\max}\Vert\nabla^2\ell(\BoTheS)^{-1}\Vert}{N}}.
            \end{align}
            Finally, for the first term in $U$, let $\Vert\mathbf{y}\Vert_2\leq 1$ and $\mathbf{d}_y :=\left(  \frac{\Psi^{(3)}(\eta_i^\star)}{v_i^\star} \right)_{i\leq N}  \circ \mathbf{Py}$. Then
            \begin{align*}
               \Bigg\Vert \sum_{k=1}^N \mathbf{a}^{(kl)}\mathbf{y}_k \Bigg\Vert_2^2 
               & = \frac{1}{4N^2} \Bigg\Vert \sum_{j=1}^N \frac{\Psi^{(3)}(\eta_j^\star)}{v_j^\star}\nabla^2\ell(\BoTheS)^{-1}\mathbf{x}_j\mathbf{P}_{jl}(\mathbf{e}_j^\top\mathbf{Py}) \Bigg\Vert_2^2\\
               & = \frac{1}{4N^2} \Bigg\Vert \nabla^2\ell(\BoTheS)^{-1}\mathbf{X}^\top \mathrm{diag}(\mathbf{d}_y)\mathbf{P}\mathbf{e}_l\Bigg\Vert_2^2,
            \end{align*}
            from which we have
            \begin{align}
                \sum_{l=1}^N \Bigg\Vert \sum_{k=1}^N \mathbf{a}^{(kl)}\mathbf{y}_k \Bigg\Vert_2^2 
                & = \frac{1}{4N^2} \Bigg\Vert \nabla^2\ell(\BoTheS)^{-1}\mathbf{X}^\top \mathrm{diag}(\mathbf{d}_y)\mathbf{P} \Bigg\Vert_F^2 \leq  \frac{1}{4N^2} \Bigg\Vert \nabla^2\ell(\BoTheS)^{-1}\mathbf{X}^\top \mathrm{diag}(\mathbf{d}_y)\Bigg\Vert_F^2 \nonumber\\
                & = \frac{1}{4N^2} \sum_{j=1}^N \left(  \frac{\Psi^{(3)}(\eta_j^\star)}{v_j^\star} \right)^2(\mathbf{e}_j^\top \mathbf{Py})^2 \mathbf{x}_j^\top \nabla^2\ell(\BoTheS)^{-2}\mathbf{x}_j \nonumber\\
                & \leq \frac{1}{4N^2}\max_{m\in[N]}\mathbf{x}_m^\top \nabla^2\ell(\BoTheS)^{-2}\mathbf{x}_m \Vert\mathbf{Py}\Vert_2^2 \nonumber\\
                \label{eq_U_1_term_final_upper_bound}
                & \leq \frac{1}{4N^2}\max_{m\in[N]}\mathbf{x}_m^\top \nabla^2\ell(\BoTheS)^{-2}\mathbf{x}_m \leq \frac{A_{\max}\Vert\nabla^2\ell(\BoTheS)^{-1}\Vert}{4N}.
            \end{align}
            Therefore
            \begin{equation*}
                \sup_{\Vert\mathbf{y}\Vert_2\leq 1}\sqrt{\sum_{l=1}^N \Bigg\Vert \sum_{k=1}^N \mathbf{a}^{(kl)}\mathbf{y}_k \Bigg\Vert_2^2} \leq \frac{1}{2}\sqrt{\frac{A_{\max}\Vert\nabla^2\ell(\BoTheS)^{-1}\Vert}{N}}.
            \end{equation*}
            Observe that every single term $\Vert \mathbf{a}^{(kl)}\mathbf{y}_k \Vert_2^2$ can be upper bounded by the r.h.s.~of~\eqref{eq_U_1_term_final_upper_bound} up to a constant factor, that is 
            \begin{align}
                U_1 &:= \sup_{\Vert\mathbf{y}\Vert_2\leq 1}\sqrt{\sum_{l=1}^N \Bigg\Vert \sum_{k\neq l}^N \mathbf{a}^{(kl)}\mathbf{y}_k \Bigg\Vert_2^2} \nonumber\\
                &\leq \sup_{\Vert\mathbf{y}\Vert_2\leq 1}\sqrt{\sum_{l=1}^N \Bigg\Vert \sum_{k= 1}^N \mathbf{a}^{(kl)}\mathbf{y}_k \Bigg\Vert_2^2} + \sup_{\Vert\mathbf{y}\Vert_2\leq1} \sqrt{\sum_{l=1}^N\Vert \mathbf{a}^{(ll)}\mathbf{y}_l \Vert_2^2} \nonumber\\
                &\leq \frac{1}{2}\sqrt{\frac{A_{\max}\Vert\nabla^2\ell(\BoTheS)^{-1}\Vert}{N}} + \max_{l\in[N]} \Vert\mathbf{a}^{(ll)}\Vert_2\nonumber\\
                \label{eq_Hlibert_U1_final_upper_bound}
                &\leq \frac{1}{2}\sqrt{\frac{A_{\max}\Vert\nabla^2\ell(\BoTheS)^{-1}\Vert}{N}} + V \leq\sqrt{\frac{A_{\max}\Vert\nabla^2\ell(\BoTheS)^{-1}\Vert}{N}}.
            \end{align}
            Now, using Lemma~\ref{lema_HilbertSpaceHW} and the aforementioned upper bounds~\eqref{eq_Hlibert_T_final_upper_bound},~\eqref{eq_Hlibert_U2_final_upper_bound},~\eqref{eq_Hlibert_U1_final_upper_bound}, and~\eqref{eq_Hlibert_V_final_upper_bound}, we have with probability at least $1-2e^{-t}$
            \begin{align}
                \Vert\Delta_{quad}\Vert_2&\lesssim B\left\{ T + U\sqrt{t} + Vt\right\}\nonumber\\
                \label{eq_final_upper_on_Delta_quad}
                &\lesssim \sqrt{\frac{B^3A_{\max}\mathrm{tr}(\nabla^2\ell(\BoTheS)^{-1})}{N}} + \sqrt{\frac{t B^3A_{\max}\Vert\nabla^2\ell(\BoTheS)^{-1}\Vert}{N}} + \sqrt{\frac{t^2 B^2A_{\max}\Vert\nabla^2\ell(\BoTheS)^{-1}\Vert}{N}}.
            \end{align}

        \noindent \textbf{Bound for }$\Vert\Delta_{\geq 3}\Vert_2$. It suffices to find upper bounds on $\Vert\Delta_{\geq 3}\Vert_{\nabla^2\ell(\BoTheS)}$, which, by the analysis in Section~\ref{subsubsec_proof_of_Self_mapping_property}, reduces to the upper bounds on $\Vert Q_2(\hat{\Delta}) - Q_2(\Delta_{lin})\Vert_{\nabla^2\ell(\BoTheS)^{-1}}$ and $\Vert Q_3(\hat{\Delta}) \Vert_{\nabla^2\ell(\BoTheS)^{-1}}$. By~\eqref{eq_upper_bound_H2_in_seminorm_selfmap}, we have
        \begin{equation*}
            \Vert Q_2(\hat{\Delta}) - Q_2(\Delta_{lin})\Vert_{\nabla^2\ell(\BoTheS)^{-1}} \lesssim \left[ \max_{m\in[N]} |\mathbf{x}_m^\top \Delta_{lin}| + \frac{1}{2}\max_{l\in[N]}|\mathbf{x}_l^\top\hat{\mathbf{e}}_{err}|\right] \Vert\hat{\mathbf{e}}_{err}\Vert_{\nabla^2\ell(\BoTheS)}.
        \end{equation*}
        Moreover, by Theorem~\ref{thm_suffic_sample_l2_upper_bound}, $\max_{m\in[N]}|\mathbf{x}_m^\top\hat{\Delta}|\leq \frac{1}{2}$ with probability at least $1-d^{-c}$ for sufficiently large $N$. Then by following the argument used to prove~\eqref{eq_upper_bound_H3_in_seminorm_selfmap}, we have
        \begin{equation*}
            \Vert Q_3(\hat{\Delta}) \Vert_{\nabla^2\ell(\BoTheS)^{-1}} \lesssim \left(\max_{m\in[N]}|\mathbf{x}_m^\top \Delta_{lin}|\right)^2 \Vert\Delta_{lin}\Vert_{\nabla^2\ell(\BoTheS)} + \left(\max_{m\in[N]}|\mathbf{x}_m^\top (\hat{\mathbf{e}}_{err})|\right)^2 \Vert\hat{\mathbf{e}}_{err}\Vert_{\nabla^2\ell(\BoTheS)}.
        \end{equation*}
        By Lemma~\ref{lema_Self-mapping_Property}, whenever Theorem~\ref{thm_suffic_sample_l2_upper_bound} holds, with probability at least $1-d^{-c}$, we have
        \begin{align*}
            \max_{m\in[N]} |\mathbf{x}_m^\top \Delta_{lin}| &\lesssim \sqrt{A_{\max}\log d} + A_{\max}\log d,\\
            \max_{m\in[N]} |\mathbf{x}_m^\top \hat{\mathbf{e}}_{err}| &\lesssim \sqrt{A_{\max}\log d} + A_{\max}\sqrt{d_\theta} + BA_{\max} \log d,\\
            \Vert \Delta_{lin} \Vert_{\nabla^2\ell(\BoTheS)} & \lesssim \sqrt{\frac{Bd_\theta\log d}{N}},
        \end{align*}
        where to obtain the second inequality, we also utilize Lemma~\ref{lema_Bias_Envelope_upperbound}. Under the setting of Theorem~\ref{thm_suffic_sample_l2_upper_bound}, with probability at least $1-d^{-c}$, we have
        \begin{alignat*}{2}
            \Vert\Delta_{\geq 3}\Vert_{\nabla^2\ell(\BoTheS)} &\lesssim \Vert Q_2(\hat{\Delta}) - Q_2(\Delta_{lin})\Vert_{\nabla^2\ell(\BoTheS)^{-1}} + \Vert Q_3(\hat{\Delta}) \Vert_{\nabla^2\ell(\BoTheS)^{-1}} \\
            & \lesssim\left[ \sqrt{A_{\max}\log d} + A_{\max}\sqrt{d_\theta} + BA_{\max} \log d \right]\left( \Vert \Delta_{bias} \Vert_{\nabla^2\ell(\BoTheS)} + \Vert \Delta_{quad} \Vert_{\nabla^2\ell(\BoTheS)} + \Vert \Delta_{\geq 3} \Vert_{\nabla^2\ell(\BoTheS)} \right) \\
            & + A_{\max}\log d \Vert\Delta_{lin}\Vert_{\nabla^2\ell(\BoTheS)}\\
            & \lesssim \left[ \sqrt{A_{\max}\log d} + A_{\max}\sqrt{d_\theta} + BA_{\max} \log d \right]\left( \Vert \Delta_{bias} \Vert_{\nabla^2\ell(\BoTheS)} + \Vert \Delta_{quad} \Vert_{\nabla^2\ell(\BoTheS)} \right) \\
            & + A_{\max}\log d \Vert\Delta_{lin}\Vert_{\nabla^2\ell(\BoTheS)}\\
            & \lesssim \left[\sqrt{BA_{\max} \log d} + A_{\max}\sqrt{d_\theta} \right]\sqrt{\frac{A_{\max} Bd_\theta \log^2 d}{N}} + A_{\max}\log d \sqrt{\frac{Bd_\theta\log d}{N}}\\
            &\lesssim BA_{\max}\sqrt{\frac{d_\theta\log^3 d}{N}} + \sqrt{B}A_{\max}^{3/2} \frac{d_\theta\log d}{\sqrt{N}},
        \end{alignat*}
        where the second inequality is from driving $N$ sufficiently large such that the square bracketed term is smaller than $\frac{1}{2}$, thus the $\Vert \Delta_{\geq 3} \Vert_{\nabla^2\ell(\BoTheS)}$ can be absorbed into the l.h.s.
        Using Lemma~\ref{lema_property_of_A_matrix}, we conclude that 
        \begin{align}
            \Vert\Delta_{\geq 3}\Vert_2 &\leq \Vert\Delta_{\geq 3}\Vert_{\nabla^2\ell(\BoTheS)}\sqrt{\Vert\nabla^2\ell(\BoTheS)^{-1}\Vert} \nonumber\\
            \label{eq_final_upper_on_Delta_geq_3}
            &\lesssim \sqrt{\Vert\nabla^2\ell(\BoTheS)^{-1}\Vert}\left[\frac{B^2\mu_Q d_\theta^{3/2}\log^{3/2}d}{N^{3/2}} + \frac{B^2 \mu_Q^{3/2}d_\theta^{5/2}\log d}{N^2}\right]. 
        \end{align}
        
        \noindent\textbf{Combining the Bounds.} Setting $t=C\log d$ for some sufficiently large $C>0$ in~\eqref{eq_final_upper_on_Delta_quad}, combining it with~\eqref{eq_final_upper_on_Delta_geq_3}, and using Lemma~\ref{lema_property_of_A_matrix} again to upper bound $A_{\max}$ gives
        \begin{align*}
            \Vert\hat{\Delta} - \Delta_{lin} - \Delta_{bias}\Vert_2 &\lesssim \frac{\sqrt{\Vert\nabla^2\ell(\BoTheS)^{-1}\Vert B\mu_Q} d_\theta}{N} \left[ \frac{B^{3/2}\sqrt{\mu_Q d_\theta}\log^{3/2} d}{\sqrt{N}} + \frac{B^{3/2}\mu_Q d_\theta^{3/2}\log d}{N}\right]\\
            & + \sqrt{\frac{\mathrm{tr}(\nabla^2\ell(\BoTheS)^{-1})\log d_{\theta}}{N}} B^2 \left[\sqrt{\frac{\mu_Q d_\theta}{N\log d_{\theta}}} + \sqrt{\frac{\mu_Q d_\theta \log d}{N\log d_\theta}} + \sqrt{\frac{\mu_Q d_\theta \log^2 d}{BN\log d_\theta}}\right],
        \end{align*}
        where we use the fact that $\Vert\nabla^2\ell(\BoTheS)^{-1}\Vert\leq \mathrm{tr}(\nabla^2\ell(\BoTheS)^{-1})$ to further simplify~\eqref{eq_final_upper_on_Delta_quad}. 
        Then by setting $N\gtrsim \max\left\{B^3\mu_Q d_\theta^{3/2}\log^3 d,B^4\mu_Q d_\theta \log d\right\}$ the above inequality yields~\eqref{eq_Delta_minus_lin_bias_upper_bound}.\qedbox
\appendix
        \section{Proofs of Auxiliary Results}
        \label{appendi_A}
        In this appendix, we include proofs of a number of intermediate results needed in the main part of this paper.
        \subsection{Proof of Proposition \ref{prop_indentifibility_complete_obv}}
        \label{proof_prop_indentifibility_complete_obv}
        The reparameterization map $\Bothe\mapsto \mathbf{w}_o + \mathbf{V}_{A^\perp}\Bothe$ is injective, so it suffices to work with $\Bothe\in\mathbb{R}^{d_\theta}$. 
        Since the function $\frac{1}{1+e^{-t}}$ is strictly increasing in $t$, for any two feasible parameters $\mathbf{w} = \mathbf{w}_o + \mathbf{V}_{A^\perp}\Bothe$ and $\mathbf{w}^\prime = \mathbf{w}_o + \mathbf{V}_{A^\perp}\Bothe^\prime$, the condition $p_\Boalp(\mathbf{w})=p_\Boalp(\mathbf{w}^\prime),\,\forall\,\Boalp\in\E_s$ is equivalent to 
        \begin{equation}
            \label{eq_z_equality_on_connt_comp}
            (\mathbf{e}_i-\mathbf{e}_j)^\top \mathfrak{A}\mathbf{V}_{A^\perp}(\Bothe-\Bothe^\prime)=0,\,\forall\,\Boalp=(i,j)\in\mathcal{E}_s.
        \end{equation}
        Equation~\eqref{eq_z_equality_on_connt_comp} 
        states that the vector $\mathfrak{A}\mathbf{V}_{A^\perp}(\Bothe-\Bothe^\prime)$ is constant on connected components of $\mathcal{G}([d],\mathcal{E}_s,1)$, i.e., there exist coefficients $[c_1,\dots,c_K]$ such that 
        \begin{equation}
            \label{eq_UVperpA_belongs_to_connectC_I}
            \mathfrak{A}\mathbf{V}_{A^\perp}(\Bothe-\Bothe^\prime) = \sum_{t=1}^K c_t\mathds{1}_{\mathcal{C}_t}\Leftrightarrow\, 
            \begin{bmatrix}
                \mathbf{C}_I & \mathfrak{A}\mathbf{V}_{A^\perp}
            \end{bmatrix}
            \begin{bmatrix}
                -[c_1,\dots,c_K]^\top\\
                \Bothe-\Bothe^\prime
            \end{bmatrix}=\mathbf{0}.
        \end{equation}
        We first prove~\eqref{eq_identifiable_rank_constraint} implies identifiability. Suppose that $[\mathbf{C}_I,\,\mathfrak{A}\mathbf{V}_{A^\perp}]\in\mathbb{R}^{d\times (d_\theta + K)}$ has full column rank $d_\theta + K$. Then~\eqref{eq_UVperpA_belongs_to_connectC_I} implies $[c_1,\dots,c_K]^\top=\mathbf{0}$ and $\Bothe-\Bothe^\prime=\mathbf{0}$. Since the reparameterization is injective, we have $\mathbf{w}=\mathbf{w}^\prime$ accordingly. 
        
        Conversely, suppose that~\eqref{eq_identifiable_rank_constraint} fails. Then $\mathrm{rank}([\mathbf{C}_I,\,\mathfrak{A}\mathbf{V}_{A^\perp}])<d_\theta+K$ and hence there exists a non-zero solution $\left[-[c_1,\dots,c_K], \Bothe-\Bothe^\prime\right]^\top$ satisfying the linear system of equations at the r.h.s. of~\eqref{eq_UVperpA_belongs_to_connectC_I}.
        Since the columns of $\mathbf{C}_I$ are linearly independent, we assert that $\Bothe\neq \Bothe^\prime$, because otherwise, we would have $\mathbf{C}_I[c_1,\dots,c_K]^\top=\mathbf{0}$ for $[c_1,\dots,c_K]^\top\neq \mathbf{0}$. Let $\mathbf{w}, \mathbf{w}^\prime\in\{\mathbf{w}\in\mathbb{R}^p\mid\mathbf{Aw}=\mathbf{b}\}$ be defined according to the $\Bothe,\,\Bothe^\prime$. Then by the equivalence between~\eqref{eq_UVperpA_belongs_to_connectC_I} and~\eqref{eq_z_equality_on_connt_comp}, $\mathbf{a}_{\Boalp}^\top\mathbf{w} = \mathbf{a}_{\Boalp}^\top\mathbf{w}^\prime$ holds for all $\Boalp\in\E_s$. This further implies $p_\Boalp(\mathbf{w})=p_\Boalp(\mathbf{w}^\prime),\,\forall\,\Boalp\in\E_s$ for $\mathbf{w}\neq \mathbf{w}^\prime$, which means that the model is not identifiable over $\{\mathbf{w}\in\mathbb{R}^p\mid\mathbf{Aw}=\mathbf{b}\}$.
        \qedbox
        \subsection{Proof of Corollary \ref{coro_from_identifibality}}
        \label{proof_coro_from_identifibality}
        The first part is derived  from identifying the null space of $\mathbf{W}$. From standard results in spectral graph theory, $\mathrm{span}(\{\mathbf{1}_{\mathcal{C}_t}\}_{t=1}^K) = \mathrm{span}(\mathbf{C}_I)=\mathrm{Null}(\mathbf{L})$. Suppose that there exists $\Bothe\neq\mathbf{0}$ such that $\Bothe^\top\mathbf{W}\Bothe=0$. Since $\mathbf{L}\succeq\mathbf{0}$, equivalently, this gives $(\mathfrak{A}\mathbf{V}_{A^\perp}\Bothe)^\top\mathbf{L}(\mathfrak{A}\mathbf{V}_{A^\perp}\Bothe)=0$, and $\mathfrak{A}\mathbf{V}_{A^\perp}\Bothe\in\mathrm{Null}(\mathbf{L})$. Thus
        \begin{equation}
        \label{eq_nullspace_of_K_explicit}
            \mathrm{Null}(\mathbf{W})=\{\Bothe\in\mathbb{R}^{d_\theta}\mid \mathfrak{A}\mathbf{V}_{A^\perp}\Bothe\in \mathrm{span}(\mathbf{C}_I)\}
        \end{equation}
        Since the matrix $\mathbf{C}_I$ has linearly independent columns, whenever \eqref{eq_identifiable_rank_constraint} holds, we will have 
        \begin{equation}
        \label{eq_indentifibility_complete_obv_equalform}
            \mathrm{span}(\mathbf{C}_I)\,\cap\,\mathrm{Range}(\mathfrak{A}\mathbf{V}_{A^\perp}) = \{\mathbf{0}\},\quad \mathrm{rank}(\mathfrak{A}\mathbf{V}_{A^\perp})=d_\theta,
        \end{equation}
        and vice versa. If $\mathrm{Null}(\mathbf{W})= \{\mathbf{0}\}$, then~\eqref{eq_nullspace_of_K_explicit} yields~\eqref{eq_indentifibility_complete_obv_equalform}, and consequently~\eqref{eq_identifiable_rank_constraint}. Conversely, suppose that~\eqref{eq_identifiable_rank_constraint} holds. Then~\eqref{eq_indentifibility_complete_obv_equalform} implies that the only $\Bothe$ which satisfies~\eqref{eq_nullspace_of_K_explicit} is identically zero. Thus $\mathrm{Null}(\mathbf{W})= \{\mathbf{0}\}$.

        For the second part, let $\mathrm{rank}(\mathbf{W})=r<d_\theta$, and the full eigen decomposition of $\mathbf{W}$ be
        \begin{equation*}
            \mathbf{W}  = [\mathbf{Q}_+,\mathbf{Q}_\perp]
            \begin{bmatrix}
                \boldsymbol{\Lambda} & \mathbf{0}\\
                \mathbf{0} & \mathbf{0}
            \end{bmatrix}
            [\mathbf{Q}_+,\mathbf{Q}_\perp]^\top.
        \end{equation*}
        Then, the refined linear equality system can be expressed as 
        \begin{equation*}
            \tilde{\mathbf{A}}\mathbf{w}=\tilde{\mathbf{b}},\quad \tilde{\mathbf{A}}:=
            \begin{bmatrix}
                \mathbf{A}\\
                \mathbf{Q}_\perp^\top \mathbf{V}_{A^\perp}^\top
            \end{bmatrix},\quad
            \tilde{\mathbf{b}}:=
            \begin{bmatrix}
                \mathbf{b}\\
                \mathbf{0}
            \end{bmatrix}.
        \end{equation*}
        By following the same procedure and reparameterizing as in \eqref{eq_transfrom_w_to_theta}, the counterpart of $\mathbf{W}$ in the refined model is
        \begin{equation*}
            \tilde{\mathbf{W}}:= \mathbf{Q}_+^\top\mathbf{V}_{A^{\perp}}^\top\mathfrak{A}^\top \mathbf{L}\mathfrak{A}\mathbf{V}_{A^{\perp}}\mathbf{Q}_+ = \mathbf{Q}_+^\top\mathbf{W}\mathbf{Q}_+=\boldsymbol{\Lambda}\succ\mathbf{0}.
        \end{equation*}
        By the first part of the proof, we know the refined model enforces \eqref{eq_identifiable_rank_constraint} and $\mathrm{rank}([\mathbf{C}_I,\mathfrak{A}\mathbf{V}_{A^\perp}\mathbf{Q}_+])=r+K$, thus the refined model is identifiable.

        Finally, fix an arbitrary $\mathbf{w}^\star$ that satisfies Assumption \ref{assumpt1_well_Logit}, and set $\tilde{\mathbf{w}}^\star=\mathbf{w}_o+\mathbf{V}_{A^\perp}\mathbf{Q}_+\mathbf{Q}_+^\top\BoTheS$. We verify that: (i) $p_\Boalp(\tilde{\mathbf{w}}^\star)=p_\Boalp(\mathbf{w}^\star),\,\forall\,\Boalp\in\E_s$, and (ii) $\tilde{\mathbf{w}}^\star$ is unique under the refined model, thereby completing the proof. For $(i)$, notice that $\Bothe^\star-\mathbf{Q}_+\mathbf{Q}_+^\top\BoTheS =\mathbf{Q}_\perp\mathbf{Q}_\perp^\top\BoTheS\in\mathrm{Null}(\mathbf{W})$. By \eqref{eq_nullspace_of_K_explicit} we know $\mathfrak{A}\mathbf{V}_{A^\perp}(\Bothe^\star-\mathbf{Q}_+\mathbf{Q}_+^\top\BoTheS)\in\mathrm{span}(\mathbf{C}_I)$, thus for any $\Boalp = (i,j)\in\E_s$
        \begin{equation*}
            (\mathbf{e}_i-\mathbf{e}_j)^\top\mathfrak{A}\mathbf{V}_{A^\perp}(\Bothe^\star-\mathbf{Q}_+\mathbf{Q}_+^\top\BoTheS)= (\mathbf{e}_i-\mathbf{e}_j)^\top\mathfrak{A}(\mathbf{w}^\star-\tilde{\mathbf{w}}^\star) = \mathbf{a}_{\Boalp}^\top(\mathbf{w}^\star-\tilde{\mathbf{w}}^\star)=0,
        \end{equation*}
        which proves (i). For (ii), suppose there exists $\tilde{\mathbf{w}}^\star_o = \mathbf{w}_o + \mathbf{V}_{A^\perp}\mathbf{Q}_+\mathbf{Q}_+^\top\BoTheS_o$ that satisfies $\tilde{\mathbf{A}}\mathbf{w}=\tilde{\mathbf{b}}$ and induces the same probability distribution, then following \eqref{eq_UVperpA_belongs_to_connectC_I}, we have $\mathfrak{A}\mathbf{V}_{A^\perp}(\mathbf{Q}_+\mathbf{Q}_+^\top\BoTheS-\mathbf{Q}_+\mathbf{Q}_+^\top\BoTheS_o)\in\mathrm{span}(\mathbf{C}_I)$. By \eqref{eq_nullspace_of_K_explicit} and the definition of $\mathbf{Q}_+$, this implies $\mathbf{Q}_+\mathbf{Q}_+^\top(\BoTheS-\BoTheS_o)\in\mathrm{Null}(\mathbf{W})\bigcap\mathrm{Range}(\mathbf{W}) = \{\mathbf{0}\}$ hence $\tilde{\mathbf{w}}^\star = \tilde{\mathbf{w}}_o^\star$. \qedbox
            \subsection{Proof of Proposition~\ref{prop_lim_B_to_infty_minimax_risk_Euclide_to_infty}}
            \label{proof_prop_lim_B_to_infty_minimax_risk_Euclide_to_infty}
            This result can be viewed as a generalization of~\cite[Proposition~17]{ShahEstPariCompGrapTop}. Fix a unit vector $\Bothe_o\in\mathbb{R}^{d_\theta}$ and consider the feasible ray $\{\Bothe^\prime\in\mathbb{R}^{d_\theta} \mid \Bothe^\prime = c\Bothe_o, c\geq 0\}$. For query $\Boalp$ the utility difference along the ray is $\mathbf{w}_o^\top \mathbf{a}_{\Boalp} + c\mathbf{q}_\Boalp^\top\Bothe_o$. Since $\E_s$ is a finite set, there exists some $c_0<\infty$ such that for every $\Boalp\in\E_s$, $\mathrm{sgn}(\mathbf{w}_o^\top \mathbf{a}_{\Boalp} + c\mathbf{q}_\Boalp^\top\Bothe_o)$ remains unchanged for all $c\geq c_0$. Hence, by~\cite[Proposition~17]{ShahEstPariCompGrapTop}, there exists a realization $\tilde{\boldsymbol{y}}^{(N)}\in\{0,1\}^{N}$ of the response sequence $Y^{(N)}$ agreeing with these signs such that 
            \begin{equation*}
                \bbp_{c\Bothe_o}^{(N)}\left\{ Y^{(N)} = \tilde{\boldsymbol{y}}^{(N)} \right\} = \prod_{\Boalp\in\E_s}\prod_{t\in[n_\Boalp]}\rP_{c\Bothe_o}^{\Boalp}\left\{ Y_\Boalp^{(t)} =  \tilde{y}_{\Boalp}^{(t)}\right\} \geq \frac{1}{2^N},\,\forall\,c\geq c_0.
            \end{equation*}
            Fix any $\bar{c}\geq c_0$, then $\{\Bothe^\prime\in\mathbb{R}^{d_\theta}\mid \Bothe^\prime=c\Bothe_o, c\in[c_0,\bar{c}]\}\subseteq\Theta_B$ for every sufficiently large $B$. Consequently, for any estimator $\hat{\Bothe}$
            \begin{align*}
                \sup_{\Bothe\in\Theta_B}\mathbb{E}_\Bothe^{(N)}\left[\Vert\hat{\Bothe}-\Bothe\Vert_2^2\right]&\geq \sup_{c\in[c_0,\bar{c}]}\mathbb{E}_{c\Bothe_o}^{(N)}\left[\Vert\hat{\Bothe}-c\Bothe_o\Vert_2^2\right] = \sup_{c\in[c_0,\bar{c}]}\sum_{\boldsymbol{y}^{(N)}\in\mathcal{Y}^N}\bbp_{c\Bothe_o}^{(N)}\left\{Y^{(N)}=\boldsymbol{y}^{(N)}\right\}\Vert\hat{\Bothe}(\boldsymbol{y}^{(N)})-c\Bothe_o\Vert_2^2\\
                &\geq \sup_{c\in[c_0,\bar{c}]} \bbp_{c\Bothe_o}^{(N)}\left\{Y^{(N)}=\tilde{\boldsymbol{y}}^{(N)}\right\}\Vert\hat{\Bothe}(\tilde{\boldsymbol{y}}^{(N)})-c\Bothe_o\Vert_2^2 \geq \frac{1}{2^N}\sup_{c\in[c_0,\bar{c}]} \Vert\hat{\Bothe}(\tilde{\boldsymbol{y}}^{(N)})-c\Bothe_o\Vert_2^2\\
                &\geq \frac{1}{2^N}\max\left\{ \Vert \hat{\Bothe}(\tilde{\boldsymbol{y}}^{(N)}) -c_0\Bothe_o \Vert_2^2, \Vert \hat{\Bothe}(\tilde{\boldsymbol{y}}^{(N)}) -\bar{c}\Bothe_o \Vert_2^2 \right\} \geq \frac{(\bar{c}-c_0)^2}{2^{N+2}}.
            \end{align*} 
            Since $\bar{c}>c_0$ is arbitrary, for every $M>0$, one may choose $\bar{c}$ sufficiently large such that $\frac{(\bar{c}-c_0)^2}{2^{N+2}}>M$. Consequently, $\lim_{B\to\infty}\sup_{\Bothe\in\Theta_B}\mathbb{E}_\Bothe^{(N)}\left[\Vert\hat{\Bothe}-\Bothe\Vert_2^2\right] = \infty$.
            \qedbox
            \subsection{Proof of Lemma~\ref{lema_sum_of_excess_loss_lower_bound_by_semi_one_sample}}
            \label{proof_lema_sum_of_excess_loss_lower_bound_by_semi_one_sample}
            The KL divergence between Bernoulli distributions can be written as the first-order Taylor remainder of $\Psi(\cdot)$, see, e.g.~\cite{ShahEstPariCompGrapTop}. That is 
                \begin{align*}
                    \dd_{KL}(\rP_{\Psi^\prime}^a,\rP_{\Psi^\prime}^z) &= \mathbb{E}_{\rP_{\Psi^\prime}^a} \left[\log\left(\frac{\mathrm{d}\rP_{\Psi^\prime}^a(Y)}{\mathrm{d}\rP_{\Psi^\prime}^z(Y)}\right)\right] \\
                    & = \mathbb{E}_{\rP_{\Psi^\prime}^a} \left[ Ya - \Psi(a) - Yz + \Psi(z) \right] = \Psi(z) - \Psi(a) - \Psi^\prime(a)(z-a),
                \end{align*}
                where $\mathbb{E}_{\rP_{\Psi^\prime}^a}$ denotes the expectation under $\rP_{\Psi^\prime}^a$, and the last two equalities are from Proposition~\ref{observ_Bernoulli_exp_calss}. Observe that 
                \begin{align}
                    \dd_{KL}(\rP_{\Psi^\prime}^a,\rP_{\Psi^\prime}^z) + \dd_{KL}(\rP_{\Psi^\prime}^b,\rP_{\Psi^\prime}^z) &= \Psi(z) - \Psi(a) - \Psi^\prime(a)(z-a) + 
                    \Psi(z) - \Psi(b) - \Psi^\prime(b)(z-b)\nonumber\\
                    & \geq \inf_{z\in\mathbb{R}} \Psi(z) - \Psi(a) - \Psi^\prime(a)(z-a) + 
                    \Psi(z) - \Psi(b) - \Psi^\prime(b)(z-b), \nonumber
                \end{align}
                and the infimum is attained at $z^\star$ such that $\Psi^{\prime}(z^\star) = \frac{\Psi^{\prime}(a) + \Psi^{\prime}(b)}{2}$. Since $\Psi^\prime(\cdot)$ is strictly increasing, w.l.o.g.~assume that $a\leq b$, we have $z^\star\in[a,b]$. Thus, there exist $\xi_1,\xi_2 \in [-\log B,\log B]$ such that
                \begin{align}
                    \dd_{KL}(\rP_{\Psi^\prime}^a,\rP_{\Psi^\prime}^z) + \dd_{KL}(\rP_{\Psi^\prime}^b,\rP_{\Psi^\prime}^z) & \geq \frac{1}{2}\left( \Psi^{\prime\prime}(\xi_1)(z^\star-a)^2 + \Psi^{\prime\prime}(\xi_2)(z^\star-b)^2\right) \nonumber\\
                    \label{eq_Bernoulli_Jensen_shannon_divergence_one_dim}
                    & \geq \frac{1}{8B}\left( (z^\star-a)^2 +(z^\star-b)^2 \right) \geq \frac{1}{16B}(a-b)^2,
                \end{align}
                where the second last inequality is by Proposition~\ref{prop_bound_dymic_range} and the fact that $\mathrm{Var}_{\rP_{\Psi^\prime}^a}(Y) = \Psi^{\prime\prime}(a)$, and the last inequality follows from $(z^\star-a)^2 +(z^\star-b)^2 = 2(z^\star - \frac{a+b}{2})^2 + \frac{(a-b)^2}{2}\geq \frac{(a-b)^2}{2}$.
                Next, observe that 
                \begin{align*}
                    \mathcal{L}^{\Boalp}_{\Bothe_1}(\Bov) -\mathcal{L}^{\Boalp}_{\Bothe_1}(\Bothe_1) &= \Psi\left(\frac{\mathbf{w}_o^\top\mathbf{a}_{\Boalp} + \mathbf{q}_\Boalp^\top \Bov}{\sigma}\right) - \Psi\left(\frac{\mathbf{w}_o^\top\mathbf{a}_{\Boalp} + \mathbf{q}_\Boalp^\top \Bothe_1}{\sigma}\right) - \Psi^\prime\left(\frac{\mathbf{w}_o^\top\mathbf{a}_{\Boalp} + \mathbf{q}_\Boalp^\top \Bothe_1}{\sigma}\right)\frac{\mathbf{q}_\Boalp^\top(\Bov - \Bothe_1)}{\sigma}\\
                    & = \dd_{KL}(\rP_{\Bothe_1}^\Boalp,\rP_{\Bov}^\Boalp),\\
                    \mathcal{L}^{\Boalp}_{\Bothe_2}(\Bov) -\mathcal{L}^{\Boalp}_{\Bothe_2}(\Bothe_2) &= \Psi\left(\frac{\mathbf{w}_o^\top\mathbf{a}_{\Boalp} + \mathbf{q}_\Boalp^\top \Bov}{\sigma}\right) - \Psi\left(\frac{\mathbf{w}_o^\top\mathbf{a}_{\Boalp} + \mathbf{q}_\Boalp^\top \Bothe_2}{\sigma}\right) - \Psi^\prime\left(\frac{\mathbf{w}_o^\top\mathbf{a}_{\Boalp} + \mathbf{q}_\Boalp^\top \Bothe_2}{\sigma}\right)\frac{\mathbf{q}_\Boalp^\top(\Bov - \Bothe_2)}{\sigma}\\
                    & = \dd_{KL}(\rP_{\Bothe_2}^\Boalp,\rP_{\Bov}^\Boalp).
                \end{align*}
                Thus, by~\eqref{eq_Bernoulli_Jensen_shannon_divergence_one_dim}
                \begin{align*}
                    \mathcal{L}^{\Boalp}_{\Bothe_1}(\Bov) -\mathcal{L}^{\Boalp}_{\Bothe_1}(\Bothe_1) + \mathcal{L}^{\Boalp}_{\Bothe_2}(\Bov) -\mathcal{L}^{\Boalp}_{\Bothe_2}(\Bothe_2) &=\dd_{KL}(\rP_{\Bothe_1}^\Boalp,\rP_{\Bov}^\Boalp) +\dd_{KL}(\rP_{\Bothe_2}^\Boalp,\rP_{\Bov}^\Boalp) \\
                    &\geq \frac{1}{16B\sigma^2}\Big|\mathbf{q}_\Boalp^\top(\Bothe_1-\Bothe_2)\Big|^2,
                \end{align*}
                which proves~\eqref{eq_sum_of_excess_loss_lower_bound_by_semi_one_sample}.
                \qedbox
        \subsection{Proof of Lemma~\ref{lema_Residual_Expansion_with_Bias}}
        \label{proof_lema_Residual_Expansion_with_Bias}
        First, by adding and subtracting terms, the gradient $\nabla\ell(\BoTheS+\Delta)$ can be equivalently written as
            \begin{align*}
                \nabla\ell(\BoTheS+\Delta)  &=\nabla\ell(\BoTheS) + \nabla^2\ell(\BoTheS)\Delta + \left( \nabla\ell(\BoTheS+\Delta)- \nabla\ell(\BoTheS)- \nabla^2\ell(\BoTheS)\Delta\right) \\
                & = \nabla\ell(\BoTheS) + \nabla^2\ell(\BoTheS)\Delta + \frac{1}{N}\sum_{j=1}^N\mathbf{x}_jr_j(\Delta),
            \end{align*}
            where $r_j(\Delta)$ is defined in~\eqref{eq_nonlinear_remainder_r_j_Delta}.
            Using the relation $\Delta = \Delta_{lin} + \Delta_{bias} + \mathbf{h}$, and the definition of $\Delta_{lin}, \Delta_{bias}$ in~\eqref{eq_def_of_Delta_lin_and_Delta_bias}, we can present the first-order optimality condition as
            \begin{align*}
                \mathbf{0}= \nabla\ell(\BoTheS+\Delta) &= \nabla\ell(\BoTheS) + \nabla^2\ell(\BoTheS)\Delta_{lin} + \nabla^2\ell(\BoTheS)\Delta_{bias} + \nabla^2\ell(\BoTheS)\mathbf{h} + \frac{1}{N}\sum_{j=1}^N\mathbf{x}_jr_j(\Delta)\\
                & = - \frac{1}{2N}\sum_{j=1}^N \mathbf{A}_{jj}\Psi^{(3)}(\eta_j^\star) \mathbf{x}_j + \nabla^2\ell(\BoTheS)\mathbf{h} + \frac{1}{N}\sum_{j=1}^N\mathbf{x}_jr_j(\Delta)\\
                & = - \frac{1}{2N}\sum_{j=1}^N \mathbf{A}_{jj}\Psi^{(3)}(\eta_j^\star) \mathbf{x}_j + \nabla^2\ell(\BoTheS)\mathbf{h} + \frac{1}{N}\sum_{j=1}^N\mathbf{x}_jr_j(\Delta_{lin} + \Delta_{bias} + \mathbf{h}),
            \end{align*}
            where the second last equality follows from the definition of $\Delta_{lin}$ and $\Delta_{bias}$. Rearranging the terms gives rise to~\eqref{eq_h_Lin_Bias_fixed_point}. \qedbox
            \subsection{Proof of Lemma~\ref{lema_Self-mapping_Property}}
            \label{proof_lema_Self-mapping_Property}
                Lemma~\ref{lema_Self-mapping_Property} relies on several intermediate results stated below, which control the $\Vert\cdot\Vert_{\nabla^2\ell(\BoTheS)}$ and $\Vert\cdot
                \Vert_{Q,\infty}$ norms of the deterministic second-order bias, and tail behavior of certain sub-Gaussian quadratic forms. We use the notation introduced at the beginning of Section~\ref{subsec_proofs_of_MLE_error_bounds}.
                \subsubsection{Deterministic residuals}
            \begin{lemma}
            \label{lema_property_of_A_matrix}
            Let $\mathbf{A}:=\frac{1}{N}\mathbf{X}\nabla^2\ell(\BoTheS)^{-1}\mathbf{X}^\top\in\mathbb{R}^{N\times N}$. 
            Then the following assertions hold.
            \begin{itemize}
           
            \item[(i)]            $\mathbf{A}\mathbf{V}^\star\mathbf{A}=\mathbf{A}$, i.e.,
            \begin{equation}
            \label{eq_cycle_propert_of_A}
                \sum_{j=1}^N v_j^\star\mathbf{A}_{ij}^2 = \mathbf{A}_{ii},\quad\sum_{j=1}^N v_j^\star\mathbf{A}_{jj} = d_\theta.
            \end{equation}
          
          \item[(ii)] $
          A_{\max}:=\max_{i\in[N],j\in[N]}|\mathbf{A}_{ij}| = \max_{i\in[N]}\mathbf{A}_{ii}\leq \frac{4B\mu_Q d_\theta}{N}$.
               
        \end{itemize}
        \end{lemma}
        The proof is deferred to Appendix~\ref{proof_lema_property_of_A_matrix}. As a straightforward corollary from Lemma~\ref{lema_property_of_A_matrix}, the matrix 
        \begin{equation}
            \label{eq_def_of_orthonomal_P}\mathbf{P}:=\sqrt{\mathbf{V}^\star}\mathbf{A}\sqrt{\mathbf{V}^\star}
        \end{equation}
        is an orthogonal projection, i.e., $\mathbf{P}^\top=\mathbf{P}$ and $\mathbf{P}^2 = \sqrt{\mathbf{V}^\star}\mathbf{A}\mathbf{V}^\star\mathbf{A}\sqrt{\mathbf{V}^\star} = \sqrt{\mathbf{V}^\star}\mathbf{A}\sqrt{\mathbf{V}^\star}=\mathbf{P}$.
        \begin{lemma}
        \label{lema_Bias_Envelope_upperbound}
            Consider the following linear functional
            \begin{align*}
                -\mathbf{x}_m^\top \Delta_{bias} = \mathbf{x}_m^\top \nabla^2\ell(\BoTheS)^{-1}\left(\frac{1}{2N}\sum_{j=1}^N \mathbf{A}_{jj}\Psi^{(3)}(\eta_j^\star) \mathbf{x}_j\right)= \frac{1}{2}\sum_{j=1}^N \mathbf{A}_{mj}\Psi^{(3)}(\eta_j^\star)\mathbf{A}_{jj},\,m\in[N],
            \end{align*}
            where $\Delta_{bias}$ is defined as in~\eqref{eq_def_of_Delta_bias}.
            It satisfies
            \begin{equation}
                \max_{m\in[N]}\bigg|-\mathbf{x}_m^\top \Delta_{bias}\bigg|\leq \frac{1}{2}A_{\max}\sqrt{d_\theta}\lesssim \frac{B\mu_Q d_\theta^{3/2}}{N}.
            \end{equation}
        \end{lemma}
        \begin{proof}
            The upper bound follows by Cauchy–Schwarz inequality
            \begin{align*}
                \bigg|-\mathbf{x}_m^\top \Delta_{bias}\bigg| &\leq \frac{1}{2}\sum_{j=1}^N |\mathbf{A}_{mj}|v_j^\star\mathbf{A}_{jj} \leq \frac{1}{2}\left( \sum_{j=1}^N v_j^\star\mathbf{A}_{mj}^2 \right)^{1/2}\left( 
                \sum_{j=1}^N v_j^\star\mathbf{A}_{jj}^2\right)^{1/2}\\
                &\leq \frac{1}{2}\sqrt{\mathbf{A}_{mm}} \sqrt{A_{\max}\sum_{j=1}^N v_j^\star\mathbf{A}_{jj}}\leq \frac{1}{2}A_{\max}\sqrt{d_\theta},
            \end{align*}
            where the last two inequalities follow from~\eqref{eq_cycle_propert_of_A}. \qedbox
        \end{proof}
        \begin{lemma}
        \label{lema_bias_term_in_semi_norm_upper}
            Let $\Delta_{bias}$ be defined as in~\eqref{eq_def_of_Delta_bias}. Then 
            \begin{equation}
                \Vert\Delta_{bias}\Vert_{\nabla^2\ell(\BoTheS)}\leq \frac{1}{2}\sqrt{\frac{A_{\max}d_\theta}{N}}.
            \end{equation}
        \end{lemma}
        \begin{proof}
            This follows from reproducing the orthogonal projection $\mathbf
            P$, and then using $\Vert\mathbf{P}\Vert\leq 1$. Let $\mathbf{d}_{v^\star}:=\left(\mathbf{A}_{ii}\frac{\Psi^{(3)}(\eta_i^\star)}{v_i^\star}\sqrt{v_i^\star}\right)_{i\leq N}$ for the moment, then
            \begin{align*}
                \Vert\Delta_{bias}\Vert_{\nabla^2\ell(\BoTheS)}^2
                & = \frac{1}{4N} \mathbf{d}_{v^\star}^\top \mathbf{P}\mathbf{d}_{v^\star}\leq \frac{1}{4N} \sum_{j=1}^N \mathbf{A}_{jj}^2 \left(\frac{\Psi^{(3)}(\eta_j^\star)}{v_j^\star}\right)^2v_j^\star \leq \frac{1}{4N} \sum_{j=1}^N \mathbf{A}_{jj}^2 v_j^\star\\
                &\leq A_{\max}\frac{1}{4N} \sum_{j=1}^N \mathbf{A}_{jj} v_j^\star \leq \frac{A_{\max}d_\theta}{4N},
            \end{align*}
            where the second inequality follows from the fact that $|\Psi^{(3)}(\eta)|\leq \Psi^{\prime\prime}(\eta)$ for any $\eta\in\mathbb{R}$ (see Proposition~\ref{observ_Bernoulli_exp_calss}), and the last inequality follows by Lemma~\ref{lema_property_of_A_matrix}. \qedbox        
            \end{proof}
            \subsubsection{Random residuals}
            \begin{lemma}[Tail Bounds on $\Delta_{lin}$]
            \label{lema_Delta_lin_Q_infty_H_bounds}
                For 
                $m\in[N]$, the term $\Delta_{lin}$ satisfies
                \begin{align}
                    \label{eq_Delta_lin_coherence_tail}
                    \bbp_{\BoTheS}^{(N)}\left\{|\mathbf{x}_m^\top \Delta_{lin}|\geq t\right\} &\lesssim \exp\left( -c\min\left( \frac{t^2}{A_{\max}},\frac{t}{A_{\max}} \right) \right),\\
                    \label{eq_Delta_lin_seminorm_tail}
                    \bbp_{\BoTheS}^{(N)}\left\{\Bigg|\Vert\Delta_{lin}\Vert_{\nabla^2\ell(\BoTheS)}^2 - \frac{\mathrm{tr}(\mathbf{I}_{d_\theta})}{N}\Bigg|\geq t\right\} &\lesssim \exp\left( -c\min\left( \frac{N^2t^2}{B^2d_\theta},\frac{Nt}{B} \right) \right),
                \end{align}
                for every $t\geq 0$.
            \end{lemma}
            Its proof is deferred to Appendix~\ref{proof_lema_Delta_lin_Q_infty_H_bounds}. Lemma~\ref{lema_Centered_Quadratic_Biases_all3} to be presented below can be viewed as a corollary from the general tail behavior of a certain sub-Gaussian quadratic form.
            \begin{lemma}
                Let $\boldsymbol{\xi}\in \mathbb{R}^{N}$ be a zero-mean, isotropic, sub-Gaussian random vector with independent coordinates, such that $\max_{m\in[N]} \Vert\boldsymbol{\xi}_m\Vert_{\psi_2}\leq \sqrt{K}$. Let $\mathbf{P}$, $\mathbf{D}\in\mathbb{R}^{N\times N}$ be any but fixed orthogonal projection and diagonal matrix, respectively. Then, for $\mathbf{M}:=\mathbf{PDP}$ and every $t\geq0$
                \begin{equation}
                \label{eq_HW_PDP_xi_subG_form}
                    \bbp_{\boldsymbol{\xi}} \left\{ \Big| \boldsymbol{\xi}^\top\mathbf{M}\boldsymbol{\xi} - \mathrm{tr}(\mathbf{M}) \Big| \geq t\right\} \lesssim\exp\left( -c\min\left( \frac{t^2}{K^2 \mathrm{tr}(\mathbf{D}^2\mathbf{P})},\frac{t}{K D_{\max}} \right) \right),
                \end{equation}
                where $D_{\max}:=\max_{i\in[N]}|\mathbf{D}_{ii}|$.
            \end{lemma}
            \begin{proof}
                Our focus is on controlling $\Vert\mathbf{M}\Vert$ and $\Vert\mathbf{M}\Vert_F$, where 
                \begin{equation*}
                    \Vert\mathbf{M}\Vert = \sup_{\Vert\mathbf{v}\Vert_2=1}\Big|\mathbf{v}^\top \mathbf{M} \mathbf{v}\Big| \leq \max_{j\in[N]}|\mathbf{D}_{jj}|\sup_{\Vert\mathbf{v}\Vert_2=1}\Vert\mathbf{Pv}\Vert_2\leq D_{\max}.
                \end{equation*}
                On the other hand $ \Vert\mathbf{M}\Vert_F^2 = \Vert\mathbf{PDP}\Vert_F^2\leq \Vert\mathbf{DP}\Vert_F^2=\mathrm{tr}(\mathbf{D}^2\mathbf{P})$. Then by Lemma~\ref{lema_HsuHW_ineq} we conclude~\eqref{eq_HW_PDP_xi_subG_form}. \qedbox 
            \end{proof}
            \begin{lemma}[Centered Quadratic Biases]
            \label{lema_Centered_Quadratic_Biases_all3}
                For $m\in[N]$, let
                \begin{subequations}
                \begin{align}
                    \label{eq_Centered_Quadratic_Biases1}
                    Q_{lin}^{m} & := \frac{1}{2}\sum_{j=1}^N\mathbf{A}_{mj}\Psi^{(3)}(\eta_j^\star)\left[\left(\sum_{i=1}^N\mathbf{A}_{ji}\eps_i\right)^2 - \mathbf{A}_{jj}\right],\\
                    \label{eq_Centered_Quadratic_Biases2}
                    Q_{Err,C}^m &: = \sum_{j=1}^N v_j^\star \mathbf{A}^2_{mj}\left[\left(\sum_{i=1}^N\mathbf{A}_{ji}\eps_i\right)^2 - \mathbf{A}_{jj}\right],\\
                    \label{eq_Centered_Quadratic_Biases3}
                    Q_{Abs}^{m} &: = \sum_{j=1}^N v_j^\star|\mathbf{A}_{mj}| \left[\left(\sum_{i=1}^N\mathbf{A}_{ji}\eps_i\right)^2 - \mathbf{A}_{jj}\right],
                \end{align}
                \end{subequations}
                be centered sub-Gaussian quadratic forms. Then $Q_{lin}^{m}$, $Q_{Err,C}^m$ and $ Q_{Abs}^{m}$ satisfy
            \begin{subequations}
                \begin{align}
                \label{eq_tail_of_Q_lin_m}
                \bbp_{\BoTheS}^{(N)}\left\{ |Q_{lin}^{m}| \geq t \right\} &\lesssim \exp\left(-c\min\left(\frac{t^2}{B^2A_{\max}^2},\frac{t}{BA_{\max}}\right)\right),\\
                \label{eq_tail_of_Q_Err_C_m}
                \bbp_{\BoTheS}^{(N)}\left\{| Q_{Err,C}^m |\geq t \right\} &\lesssim \exp\left(-c\min\left(\frac{t^2}{B^2A_{\max}^4},\frac{t}{BA_{\max}^2}\right)\right),\\
                \label{eq_tail_of_Q_Abs_m}
                \bbp_{\BoTheS}^{(N)}\left\{ |Q_{Abs}^{m}| \geq t \right\} &\lesssim \exp\left(-c\min\left(\frac{t^2}{B^2A_{\max}^2},\frac{t}{BA_{\max}}\right)\right),
            \end{align}
            \end{subequations}
            for every $t\geq 0$, respectively. Moreover, the following upper bounds hold
            \begin{align*}
                \sum_{j=1}^N v_j^\star \mathbf{A}^2_{mj}\mathbf{A}_{jj} \leq A_{\max}^2,\quad \sum_{j=1}^N v_j^\star|\mathbf{A}_{mj}|\mathbf{A}_{jj}\leq A_{\max}\sqrt{d_\theta}.
            \end{align*}
            \end{lemma}
            The proofs of these bounds are deferred to Appendix~\ref{proof_lema_Centered_Quadratic_Biases_all3}.
            \subsubsection{Self-mapping property}
            \label{subsubsec_proof_of_Self_mapping_property}
            We are now ready to prove Lemma~\ref{lema_Self-mapping_Property}. It is straightforward to verify that the set $\mathcal{N}(r_H,r_\infty)$ is non-empty, closed, convex and compact, since $\{\mathbf{x}_m\}_{m=1}^N$ linearly spans $\mathbb{R}^{d_\theta}$ and $\mathbf{0}\in \mathcal{N}(r_H,r_\infty)$. Consider the following events 
                \begin{align*}
                    \mathscr{E}_{\infty} &:= \left\{\max_{m\in[N]}|\mathbf{x}_m^\top \Delta_{lin}| \lesssim \sqrt{A_{\max}\log d} + A_{\max}\log d\right\},\\
                    \mathscr{E}_{H} &:= \left\{\Bigg|\Vert \Delta_{lin} \Vert_{\nabla^2\ell(\BoTheS)}^2 - \frac{\mathrm{tr}(\mathbf{I}_{d_\theta})}{N}\Bigg|^{1/2}\lesssim \sqrt{\frac{Bd_\theta \log d}{N}}\right\},\\
                    \mathscr{E}_{Q,lin} &:= \left\{ \max_{m\in[N]} |Q_{lin}^m| \lesssim \sqrt{B^2 A_{\max}^2 \log d} + BA_{\max}\log d \right\},\\
                    \mathscr{E}_{Q,Err} &:= \left\{ \max_{m\in[N]} |Q_{Err, C}^m| \lesssim \sqrt{B^2 A_{\max}^4\log d} + BA_{\max}^2\log d \right\},\\
                    \mathscr{E}_{Q,Abs} &:= \left\{ \max_{m\in[N]} |Q_{Abs}^m| \lesssim \sqrt{B^2 A_{\max}^2 \log d} + BA_{\max}\log d \right\},\\
                    \mathscr{E} &: = \mathscr{E}_{\infty} \cap \mathscr{E}_{H} \cap \mathscr{E}_{Q,lin} \cap \mathscr{E}_{Q,Err} \cap \mathscr{E}_{Q,Abs}.
                \end{align*}
                Let $c^\prime>5$ be a fixed absolute constant. By Lemma~\ref{lema_Delta_lin_Q_infty_H_bounds} and a union bound, there exists a sufficiently large absolute constant $C>0$ such that
                \begin{equation*}
                    \bbp_\BoTheS^{(N)}\{ \bar{\mathscr{E}}_{\infty}\} \leq |\E_s| \max_{m\in[N]}\bbp_\BoTheS^{(N)}\left\{|\mathbf{x}_m^\top\Delta_{lin}|\geq C(\sqrt{A_{\max}\log d} + A_{\max}\log d)\right\} \leq d^{-c^\prime +2},
                \end{equation*}
                where the first inequality is from the fact that $\max_{m\in[N]} |\mathbf{x}_m^\top \Delta_{lin}| = \max_{\Bobet\in\E_s}|\mathbf{x}_\Bobet^\top \Delta_{lin}|$, and the second inequality is due to $|\E_s|\leq d^2$. Similar bounds can be established for $\bbp_\BoTheS^{(N)}\{\bar{\mathscr{E}}_{H}\}$, $\bbp_\BoTheS^{(N)}\{\bar{\mathscr{E}}_{Q,lin}\}$, $\bbp_\BoTheS^{(N)}\{\bar{\mathscr{E}}_{Q,Err}\}$, and $ \bbp_\BoTheS^{(N)}\{\bar{\mathscr{E}}_{Q,Abs}\}$ by virtue of Lemmas~\ref{lema_Delta_lin_Q_infty_H_bounds},~\ref{lema_Centered_Quadratic_Biases_all3}, respectively. Consequently, we have
                \begin{align*}
                    \bbp_\BoTheS^{(N)}\{\bar{\mathscr{E}} \} &\leq 
                    \bbp_\BoTheS^{(N)}\{ \bar{\mathscr{E}}_{H} \} + \bbp_\BoTheS^{(N)}\{ \bar{\mathscr{E}}_{\infty} \} + \bbp_\BoTheS^{(N)}\{ \bar{\mathscr{E}}_{Q,lin} \} + \bbp_\BoTheS^{(N)}\{ \bar{\mathscr{E}}_{Q,Err} \} + \bbp_\BoTheS^{(N)}\{ \bar{\mathscr{E}}_{Q,Abs} \} \\
                    & \leq 5d^{-c^\prime+2} \leq d^{-c^\prime +5} = d^{-c},
                \end{align*}
                and the last inequality follows from $d\geq 2$. Setting $c^\prime>6, c=c^\prime-5$ is enough to ensure $c>1$. 
                
                We are now ready to prove inequality~\eqref{eq_highprob_self_mapping} on the high-probability event $\mathscr{E}$. We proceed with the proof by applying the Taylor expansion in Lemma~\ref{lema_r_j_Delta_Taylor_expand} in the first place. For notational simplicity, let 
                $\mathbf{e}_{err} := \Delta_{bias} +\mathbf{h}$. Then
                \begin{align*}
                    Q_h&:=\frac{1}{N}\sum_{j=1}^N\mathbf{x}_jr_j(\Delta_{lin} + \Delta_{bias} + \mathbf{h})  = \underbrace{\frac{1}{2N}\sum_{j=1}^N\mathbf{x}_j\Psi^{(3)}(\eta_j^\star) |\mathbf{x}_j^\top(\Delta_{lin}+\mathbf{e}_{err})|^2}_{Q_2(\Delta_{lin}+\mathbf{e}_{err})} \\
                    &+ \underbrace{\frac{1}{2N}\int_0^1 (1-s)^2 \sum_{j=1}^N\mathbf{x}_j\Psi^{(4)}(\eta_j^\star + s\mathbf{x}_j^\top(\Delta_{lin}+\mathbf{e}_{err}))(\mathbf{x}_j^\top(\Delta_{lin}+\mathbf{e}_{err}))^3 \mathrm{d}s}_{Q_3(\Delta_{lin}+\mathbf{e}_{err})}\\
                    & = Q_2(\Delta_{lin}) + \left[Q_2(\Delta_{lin}+\mathbf{e}_{err}) - Q_2(\Delta_{lin})\right] + Q_3(\Delta_{lin}+\mathbf{e}_{err}).
                \end{align*}
                It suffices to derive bounds of the three terms at the r.h.s.~of the last equality respectively. Observe that $Q_2(\Delta_{lin})$ contributes to the second-order bias of the MLE, and the other two terms are higher-order residuals. We proceed in two steps.

                \noindent \underline{\textbf{STEP 1:}} \textbf{$\Vert\cdot\Vert_{\nabla^2\ell(\BoTheS)}$ Self-mapping}. We use the dual-norm definition of $\Vert\cdot\Vert_{\nabla^2\ell(\BoTheS)^{-1}}$ to derive bounds, that is 
                \begin{equation*}
                    \Vert\cdot\Vert_{\nabla^2\ell(\BoTheS)^{-1}} : = \sup_{\Vert\mathbf{v}\Vert_{\nabla^2\ell(\BoTheS)}=1} \mathbf{v}^\top (\cdot).
                \end{equation*}
                To show $\Vert\Phi(\mathbf{h})\Vert_{\nabla^2\ell(\BoTheS)}\leq r_H$, it is enough to find proper upper bounds on 
                \begin{align*}
                    H_1 &: = \Bigg\Vert Q_2(\Delta_{lin}) - \frac{1}{2N}\sum_{j=1}^N \mathbf{A}_{jj}\Psi^{(3)}(\eta_j^\star) \mathbf{x}_j \Bigg\Vert_{\nabla^2\ell(\BoTheS)^{-1}},\\
                    H_2 &: = \Bigg\Vert Q_2(\Delta_{lin}+\mathbf{e}_{err}) - Q_2(\Delta_{lin})\Bigg\Vert_{\nabla^2\ell(\BoTheS)^{-1}},\\
                    H_3 &: = \Bigg\Vert Q_3(\Delta_{lin}+\mathbf{e}_{err})\Bigg\Vert_{\nabla^2\ell(\BoTheS)^{-1}}.
                \end{align*}

                \noindent \textbf{Bound of $H_1$.} We prove that on the event of $\mathscr{E}$,
                \begin{equation}
                \label{eq_upper_bound_H1_in_seminorm_selfmap}
                    H_1 \lesssim \sqrt{\frac{A_{\max}d_\theta}{N}} + \sqrt{A_{\max}\log{d}} \sqrt{\frac{Bd_\theta \log d}{N}}.
                \end{equation}
                By the triangle inequality and Lemma~\ref{lema_bias_term_in_semi_norm_upper}
                \begin{align}
                    H_1 &\leq \Vert Q_2(\Delta_{lin})\Vert_{\nabla^2\ell(\BoTheS)^{-1}} + \Bigg\Vert\frac{1}{2N}\sum_{j=1}^N \mathbf{A}_{jj}\Psi^{(3)}(\eta_j^\star) \mathbf{x}_j \Bigg\Vert_{\nabla^2\ell(\BoTheS)^{-1}}\nonumber\\
                    &= \Vert Q_2(\Delta_{lin})\Vert_{\nabla^2\ell(\BoTheS)^{-1}} + \Vert\Delta_{bias}\Vert_{\nabla^2\ell(\BoTheS)}\nonumber \\
                    \label{eq_H_1_first_step_upper_bound}
                    &\leq  \Vert Q_2(\Delta_{lin})\Vert_{\nabla^2\ell(\BoTheS)^{-1}} + \frac{1}{2}\sqrt{\frac{A_{\max}d_\theta}{N}}.
                \end{align}
                Using the relation $|\mathbf{x}_j^\top \Delta_{lin}|^2 = \left(\sum_{i=1}^N\mathbf{A}_{ji}\eps_i\right)^2$, we have
                \begin{align}
                     \Vert Q_2(\Delta_{lin})\Vert_{\nabla^2\ell(\BoTheS)^{-1}} &= \sup_{\Vert\mathbf{v}\Vert_{\nabla^2\ell(\BoTheS)}=1} \langle\mathbf{v},Q_2(\Delta_{lin})\rangle \leq \sup_{\Vert\mathbf{v}\Vert_{\nabla^2\ell(\BoTheS)}=1} |\langle\mathbf{v},Q_2(\Delta_{lin})\rangle|\nonumber\\
                     & \leq \sup_{\Vert\mathbf{v}\Vert_{\nabla^2\ell(\BoTheS)}=1}  \frac{1}{2N} \sum_{j=1}^N v_j^\star |\mathbf{x}_j^T\mathbf{v}|\left(\sum_{i=1}^N\mathbf{A}_{ji}\eps_i\right)^2\nonumber\\
                     & \leq \frac{1}{2} \max_{m\in[N]} \Bigg|\sum_{i=1}^N\mathbf{A}_{mi}\eps_i\Bigg|\cdot \sup_{\Vert\mathbf{v}\Vert_{\nabla^2\ell(\BoTheS)}=1} \left( \frac{1}{N} \sum_{j=1}^N v_j^\star |\mathbf{x}_j^\top \mathbf{v}|^2 \right)^{1/2} \cdot\left(\frac{1}{N}\sum_{j=1}^N v_j^\star \left(\sum_{i=1}^N\mathbf{A}_{ji}\eps_i\right)^2 \right)^{1/2}\nonumber\\
                     & \leq \frac{1}{2} \max_{m\in[N]} \Bigg|\sum_{i=1}^N\mathbf{A}_{mi}\eps_i\Bigg|\cdot \sup_{\Vert\mathbf{v}\Vert_{\nabla^2\ell(\BoTheS)}=1}(\mathbf{v}^\top \nabla^2\ell(\BoTheS)\mathbf{v}) \cdot \Vert\Delta_{lin}\Vert_{\nabla^2\ell(\BoTheS)}\nonumber\\
                     \label{eq_Q_2_Delta_lin_H_star_inv_norm}
                     & = \frac{1}{2}\max_{m\in[N]} \Big|\mathbf{x}_m^\top \Delta_{lin}\Big|\cdot\Vert\Delta_{lin}\Vert_{\nabla^2\ell(\BoTheS)}.
                \end{align}
                On the other hand, 
                on event $\mathscr{E}_\infty\cap\mathscr{E}_{H}$
                \begin{align}
                    \max_{m\in[N]} \Big|\mathbf{x}_m^\top \Delta_{lin}\Big|\cdot\Vert\Delta_{lin}\Vert_{\nabla^2\ell(\BoTheS)} &\leq \max_{m\in[N]} \Big|\mathbf{x}_m^\top \Delta_{lin}\Big|\cdot\left[\Bigg|\Vert\Delta_{lin}\Vert_{\nabla^2\ell(\BoTheS)}^2 - \frac{\mathrm{tr}(\mathbf{I}_{d_\theta})}{N} \Bigg|^{1/2} + \sqrt{\frac{\mathrm{tr}(\mathbf{I}_{d_\theta})}{N}}\right] \nonumber\\
                    &\lesssim\left(\sqrt{A_{\max}\log{d}} + A_{\max}\log d\right) \left(\sqrt{\frac{d_\theta}{N}}+ \sqrt{\frac{Bd_\theta \log d}{N}}  \right)\nonumber\\
                    \label{eq_x_m_Delta_lin_Delta_seminorm}
                    &\lesssim \sqrt{A_{\max}\log d}\sqrt{\frac{Bd_\theta \log d}{N}},
                \end{align}
                where the last inequality follows from the assumption on sample complexity, i.e., $N\gtrsim B\mu_Q d_\theta\log d$.
                A combination of~\eqref{eq_H_1_first_step_upper_bound}--\eqref{eq_x_m_Delta_lin_Delta_seminorm} gives rise to~\eqref{eq_upper_bound_H1_in_seminorm_selfmap}.

                \noindent \textbf{Bound of $H_2$.} We prove that on the event of  $\mathscr{E}$
                \begin{equation}
                \label{eq_upper_bound_H2_in_seminorm_selfmap}
                    H_2 \lesssim \left(\sqrt{A_{\max}\log d} + \max_{m\in[N]}|\mathbf{x}_m^\top \mathbf{e}_{err}|\right) \Vert \mathbf{e}_{err}\Vert_{\nabla^2\ell(\BoTheS)}.
                \end{equation}
                Observe that 
                \begin{align*}
                    Q_2(\Delta_{lin}+\mathbf{e}_{err}) - Q_2(\Delta_{lin}) &= \frac{1}{2N}\sum_{j=1}^N\mathbf{x}_j\Psi^{(3)}(\eta_j^\star) \left( |\mathbf{x}_j^\top(\Delta_{lin}+\mathbf{e}_{err})|^2 - |\mathbf{x}_j^\top(\Delta_{lin})|^2\right)\\
                    & = \underbrace{\frac{1}{N}\sum_{j=1}^N\mathbf{x}_j\Psi^{(3)}(\eta_j^\star)\left(  \mathbf{x}_j^\top\Delta_{lin}  \right) \left(  \mathbf{x}_j^\top\mathbf{e}_{err}\right)}_{H_{2,1}} + \underbrace{\frac{1}{2N}\sum_{j=1}^N\mathbf{x}_j\Psi^{(3)}(\eta_j^\star)\left(  \mathbf{x}_j^\top\mathbf{e}_{err}\right)^2}_{H_{2,2}}.
                \end{align*}
                For $H_{2,1}$, recall $|\Psi^{(3)}(\eta_j^\star)|\leq v_j^\star$, then use the dual norm method again
                \begin{align*}
                    \Vert H_{2,1}\Vert_{\nabla^2\ell(\BoTheS)^{-1}} &\leq \sup_{\Vert\mathbf{v}\Vert_{\nabla^2\ell(\BoTheS)}=1} \frac{1}{N}\sum_{j=1}^N v_j^\star |\mathbf{x}_j^\top\mathbf{v}|\cdot |\mathbf{x}_j^\top\Delta_{lin}| \cdot |\mathbf{x}_j^\top\mathbf{e}_{err}|\\
                    &\leq \max_{m\in[N]} |\mathbf{x}_m^\top\Delta_{lin}| \cdot \sup_{\Vert\mathbf{v}\Vert_{\nabla^2\ell(\BoTheS)}=1} \left( \frac{1}{N}\sum_{j=1}^N v_j^\star |\mathbf{x}_j^\top\mathbf{v}|^2 \right)^{1/2} \cdot \left( \frac{1}{N}\sum_{j=1}^N v_j^\star |\mathbf{x}_j^\top\mathbf{e}_{err}|^2\right)^{1/2}\\
                    & = \max_{m\in[N]} |\mathbf{x}_m^\top\Delta_{lin}| \cdot \Vert \mathbf{e}_{err} \Vert_{\nabla^2\ell(\BoTheS)} \lesssim \sqrt{A_{\max}\log d} \Vert \mathbf{e}_{err}\Vert_{\nabla^2\ell(\BoTheS)},
                \end{align*}
                where the last step follows from the event $\mathscr{E}_\infty$ and the fact that $N\gtrsim B\mu_Q d_\theta\log d$. For the deterministic term $H_{2,2}$, we have
                \begin{align}
                    \Vert H_{2,2}\Vert_{\nabla^2\ell(\BoTheS)^{-1}} & \leq \sup_{\Vert\mathbf{v}\Vert_{\nabla^2\ell(\BoTheS)}=1} \frac{1}{2N}\sum_{j=1}^N v_j^\star |\mathbf{x}_j^\top \mathbf{v}| \cdot |\mathbf{x}_j^\top\mathbf{e}_{err}|^2\nonumber\\
                    \label{eq_H_22_H_star_inv_norm}
                    &\leq \frac{1}{2}\max_{m\in[N]}|\mathbf{x}_m^\top\mathbf{e}_{err}|\cdot \left( \frac{1}{N}\sum_{j=1}^N v_j^\star |\mathbf{x}_j^\top\mathbf{e}_{err}|^2\right)^{1/2} = \frac{1}{2} \Vert\mathbf{e}_{err}\Vert_{\nabla^2\ell(\BoTheS)} \max_{m\in[N]}|\mathbf{x}_m^\top\mathbf{e}_{err}|.
                \end{align}
                By combining $H_{2,1}$ and $H_{2,2}$, we arrive at~\eqref{eq_upper_bound_H2_in_seminorm_selfmap}.

                \noindent \textbf{Bound of $H_3$.}  We prove that on the event of  $\mathscr{E}$
                \begin{equation}
                \label{eq_upper_bound_H3_in_seminorm_selfmap}
                    H_3\lesssim A_{\max}\log d  \sqrt{\frac{Bd_\theta \log d}{N}} +  \left(\max_{m\in[N]}|\mathbf{x}_m^\top \mathbf{e}_{err}|\right)^2 \Vert \mathbf{e}_{err} \Vert_{\nabla^2\ell(\BoTheS)}.
                \end{equation}
                First, observe that for each $j\in[N]$, the coefficient $Q_{3,j}^c$ in $Q_3(\Delta_{lin}+\mathbf{e}_{err})$ admits the following bound
                \begin{align}
                    |Q_{3,j}^c|&:=\Bigg|\int_0^1 (1-s)^2 \Psi^{(4)}(\eta_j^\star + s\mathbf{x}_j^\top(\Delta_{lin}+\mathbf{e}_{err}))(\mathbf{x}_j^\top(\Delta_{lin}+\mathbf{e}_{err}))^3 \mathrm{d}s\Bigg|\nonumber\\
                    &\leq \int_0^1 (1-s)^2 |\Psi^{(4)}(\eta_j^\star + s\mathbf{x}_j^\top(\Delta_{lin}+\mathbf{e}_{err}))|\cdot|\mathbf{x}_j^\top(\Delta_{lin}+\mathbf{e}_{err})|^3 \mathrm{d}s\nonumber\\
                    & \leq \int_0^1 (1-s)^2 \Psi^{\prime\prime}(\eta_j^\star + s\mathbf{x}_j^\top(\Delta_{lin}+\mathbf{e}_{err}))\cdot|\mathbf{x}_j^\top(\Delta_{lin}+\mathbf{e}_{err})|^3 \mathrm{d}s\nonumber\\
                    &\leq \int_0^1 (1-s)^2 \exp\left(s\max_{m\in[N]}\big|\mathbf{x}_m^\top(\Delta_{lin}+\mathbf{e}_{err})\big|\right)v_j^\star|\mathbf{x}_j^\top(\Delta_{lin}+\mathbf{e}_{err})|^3 \mathrm{d}s\nonumber\\
                    \label{eq_Q_3j_coffient_upper_bound}
                    &\leq \frac{1}{2}v_j^\star|\mathbf{x}_j^\top(\Delta_{lin}+\mathbf{e}_{err})|^3,
                \end{align}
                where the second last inequality is from the pseudo self-concordance property introduced in Lemma~\ref{lema_single_variate_self_concord}, and the last inequality is from the assumption on sample complexity, i.e., $N\gtrsim B^2 \mu_Q d_\theta^{3/2}\log^3 d$, under which
                \begin{align*}
                    \max_{m\in[N]}\Big|\mathbf{x}_m^\top(\Delta_{lin}+\mathbf{e}_{err})\Big|& \leq \max_{m\in[N]}|\mathbf{x}_m^\top\Delta_{lin}| + \max_{m\in[N]}|\mathbf{x}_m^\top\Delta_{bias}| + \max_{m\in[N]}|\mathbf{x}_m^\top\mathbf{h}| \\
                    & \lesssim \sqrt{A_{\max}\log d}+A_{\max}\log d + \frac{1}{2}A_{\max}\sqrt{d_\theta} + r_\infty<1.
                \end{align*}
                Here, the second inequality is from event $\mathscr{E}_\infty$ and Lemma~\ref{lema_Bias_Envelope_upperbound}. Next, substituting~\eqref{eq_Q_3j_coffient_upper_bound} back into $H_3$ yields
                \begin{align}
                    H_3&= \Bigg\Vert\frac{1}{2N}\sum_{j=1}^N \mathbf{x}_j Q_{3,j}^c\Bigg\Vert_{\nabla^2\ell(\BoTheS)^{-1}} \leq  \sup_{\Vert\mathbf{v}\Vert_{\nabla^2\ell(\BoTheS)}=1}\frac{1}{2N}\sum_{j=1}^N |\mathbf{x}_j^\top\mathbf{v}| \cdot |Q_{3,j}^c|\nonumber\\
                    &\leq \sup_{\Vert\mathbf{v}\Vert_{\nabla^2\ell(\BoTheS)}=1}\frac{1}{4N} \sum_{j=1}^Nv_j^\star|\mathbf{x}_j^\top\mathbf{v}| \cdot |\mathbf{x}_j^\top(\Delta_{lin}+\mathbf{e}_{err})|^3\nonumber\\
                    \label{eq_frac_of_H_3_into_cross_err}
                    &\leq \underbrace{\sup_{\Vert\mathbf{v}\Vert_{\nabla^2\ell(\BoTheS)}=1}\frac{1}{N} \sum_{j=1}^Nv_j^\star|\mathbf{x}_j^\top\mathbf{v}| \cdot |\mathbf{x}_j^\top\Delta_{lin}|^3}_{H_3^1}  + \underbrace{\sup_{\Vert\mathbf{v}\Vert_{\nabla^2\ell(\BoTheS)}=1}\frac{1}{N} \sum_{j=1}^Nv_j^\star|\mathbf{x}_j^\top\mathbf{v}| \cdot |\mathbf{x}_j^\top\mathbf{e}_{err}|^3}_{H_3^2},
                \end{align}
                where the last inequality follows from $|a+b|^3\leq 4|a|^3 + 4|b|^3$.
                The terms $H_3^1$, $H_3^2$ can be estimated in a similar manner as~\eqref{eq_Q_2_Delta_lin_H_star_inv_norm}~\eqref{eq_x_m_Delta_lin_Delta_seminorm} and~\eqref{eq_H_22_H_star_inv_norm} respectively. That is, on event $\mathscr{E}_\infty\cap\mathscr{E}_H$
                \begin{align*}
                    H_3^1 &\leq \left(\max_{m\in[N]}|\mathbf{x}_m^\top \Delta_{lin}|\right)^2 \sup_{\Vert\mathbf{v}\Vert_{\nabla^2\ell(\BoTheS)}=1} \left( \frac{1}{N}\sum_{j=1}^N v_j^\star |\mathbf{x}_j^\top\mathbf{v}|^2 \right)^{1/2} \left( \frac{1}{N}\sum_{j=1}^N v_j^\star |\mathbf{x}_j^\top\Delta_{lin}|^2\right)^{1/2}\\
                    & \leq \left(\max_{m\in[N]}|\mathbf{x}_m^\top \Delta_{lin}|\right)^2 \Vert\Delta_{lin}\Vert_{\nabla^2\ell(\BoTheS)} \lesssim A_{\max}\log d  \sqrt{\frac{Bd_\theta \log d}{N}},
                \end{align*}
                where the last step is from the assumption on sample complexity, i.e., $N\gtrsim B \mu_Q d_\theta \log d$. Likewise,
                \begin{align*}
                    H_3^2 \leq \left(\max_{m\in[N]}|\mathbf{x}_m^\top \mathbf{e}_{err}|\right)^2 \Vert \mathbf{e}_{err} \Vert_{\nabla^2\ell(\BoTheS)}.
                \end{align*}
                Combining the above inequalities, we obtain~\eqref{eq_upper_bound_H3_in_seminorm_selfmap}.

                \noindent \textbf{Combining the Bounds:} To complete \textbf{STEP 1}, we verify $\Vert\Phi(\mathbf{h})\Vert_{\nabla^2\ell(\BoTheS)}\leq r_H$ for any $\mathbf{h}\in\mathcal{N}(r_H,r_\infty)$ on event $\mathscr{E}$, that is 
                \begin{align*}
                    \Vert\Phi(\mathbf{h})\Vert_{\nabla^2\ell(\BoTheS)} &\leq H_1 + H_2 + H_3\\
                    &\overset{(i)}{\lesssim} \sqrt{\frac{A_{\max}d_\theta}{N}} + \sqrt{A_{\max}\log{d}} \sqrt{\frac{Bd_\theta \log d}{N}} + A_{\max}\log d  \sqrt{\frac{Bd_\theta \log d}{N}}\\
                    & + \left(\sqrt{A_{\max}\log d} + \max_{m\in[N]}|\mathbf{x}_m^\top \mathbf{e}_{err}| + \left(\max_{m\in[N]}|\mathbf{x}_m^\top \mathbf{e}_{err}|\right)^2\right) \Vert \mathbf{e}_{err}\Vert_{\nabla^2\ell(\BoTheS)}\\
                    &\overset{(ii)}{\lesssim}  \sqrt{\frac{A_{\max}d_\theta}{N}} + \sqrt{A_{\max}\log{d}} \sqrt{\frac{Bd_\theta \log d}{N}} + A_{\max}\log d  \sqrt{\frac{Bd_\theta \log d}{N}} + \frac{1}{4}\Vert \mathbf{e}_{err}\Vert_{\nabla^2\ell(\BoTheS)}\\
                    &\overset{(iii)}{\lesssim}  \sqrt{\frac{A_{\max}d_\theta}{N}} + \sqrt{A_{\max}\log{d}} \sqrt{\frac{Bd_\theta \log d}{N}} + A_{\max}\log d  \sqrt{\frac{Bd_\theta \log d}{N}} + \frac{1}{4}r_H\leq r_H.
                \end{align*}
                Here, step $(i)$ is obtained by~\eqref{eq_upper_bound_H1_in_seminorm_selfmap},~\eqref{eq_upper_bound_H2_in_seminorm_selfmap}, and~\eqref{eq_upper_bound_H3_in_seminorm_selfmap}. Step $(ii)$ is obtained by setting $N$ sufficiently large, the bracketed term is smaller than $\frac{1}{4}$, i.e.,
                \begin{alignat*}{2}
                    \sqrt{A_{\max}\log d} + \max_{m\in[N]}|\mathbf{x}_m^\top \mathbf{e}_{err}| + &\left(\max_{m\in[N]}|\mathbf{x}_m^\top \mathbf{e}_{err}|\right)^2& &\lesssim \sqrt{A_{\max}\log d} + \max_{m\in[N]}|\mathbf{x}_m^\top \mathbf{e}_{err}| \\
                    && &\lesssim \sqrt{\frac{B\mu_Q d_\theta \log d}{N}} + \max_{m\in[N]}|\mathbf{x}_m^\top \Delta_{bias}| + \max_{m\in[N]} |\mathbf{x}_m^\top \mathbf{h}|\\
                    &(\text{by Lemma~\ref{lema_Bias_Envelope_upperbound}})& &\lesssim \sqrt{\frac{B\mu_Q d_\theta \log d}{N}} + \frac{B\mu_Q d_\theta^{3/2}}{N} + r_\infty \leq \frac{1}{4},
                \end{alignat*}
                when $N\gtrsim B^2\mu_Q d_\theta^{3/2} \log d$. Finally, step $(iii)$ is established by observing
                \begin{equation*}
                    \Vert \mathbf{e}_{err}\Vert_{\nabla^2\ell(\BoTheS)}\leq \Vert\Delta_{bias}\Vert_{\nabla^2\ell(\BoTheS)} + \Vert\mathbf{h}\Vert_{\nabla^2\ell(\BoTheS)}\leq \frac{1}{2}\sqrt{\frac{A_{\max}d_\theta}{N}} + r_H,
                \end{equation*}
                where the last inequality is from Lemma~\ref{lema_bias_term_in_semi_norm_upper}.
                
                \noindent \underline{\textbf{STEP 2:}} \textbf{$\Vert\cdot\Vert_{Q,\infty}$ Self-mapping.} We aim to find upper bounds on the following three terms
                \begin{subequations}
                    \begin{align}
                \label{eq_def_of_delf_map_prop_I_1}
                    I_1 &:= \max_{m\in[N]} \Bigg|\mathbf{x}_m^\top \nabla^2\ell(\BoTheS)^{-1}\left[ Q_2(\Delta_{lin}) - \frac{1}{2N}\sum_{j=1}^N \mathbf{A}_{jj}\Psi^{(3)}(\eta_j^\star) \mathbf{x}_j \right]\Bigg|,\\
                    \label{eq_def_of_delf_map_prop_I_2}
                    I_2 &:= \max_{m\in[N]} \Bigg|\mathbf{x}_m^\top \nabla^2\ell(\BoTheS)^{-1} \left[ Q_2(\Delta_{lin}+\mathbf{e}_{err}) - Q_2(\Delta_{lin}) \right] \Bigg|,\\
                    \label{eq_def_of_delf_map_prop_I_3}
                    I_3 &:= \max_{m\in[N]}\Bigg| \mathbf{x}_m^\top \nabla^2\ell(\BoTheS)^{-1} Q_3(\Delta_{lin}+\mathbf{e}_{err})\Bigg|.
                \end{align}
                \end{subequations}
                
                \noindent \textbf{Bound of $I_1$:} By plugging the definition of $Q_2(\Delta_{lin})$ in $I_1$, we obtain
                \begin{equation*}
                    \mathbf{x}_m^\top \nabla^2\ell(\BoTheS)^{-1} Q_2(\Delta_{lin}) = \frac{1}{2}\sum_{j=1}^N\mathbf{A}_{mj}\Psi^{(3)}(\eta_j^\star)\left(\sum_{i=1}^N\mathbf{A}_{ji}\eps_i\right)^2.
                \end{equation*} 
                The upper bound follows directly from~\eqref{eq_tail_of_Q_lin_m}, i.e., on event $\mathscr{E}_{Q,lin}$, we have
                \begin{equation}
                \label{eq_I_1_Delta_quad_linear_form_upper}
                    I_1\lesssim \sqrt{B^2 A_{\max}^2 \log d} + B A_{\max} \log d \lesssim B A_{\max} \log d.
                \end{equation}

                \noindent \textbf{Bound of $I_2$:} From direct calculation
                \begin{align*}
                    I_2 \leq \underbrace{\max_{m\in[N]}\Bigg|\sum_{j=1}^N \mathbf{A}_{mj} \Psi^{(3)}(\eta_j^\star)\left(\sum_{i=1}^N \mathbf{A}_{ji}\eps_i\right) \mathbf{x}_j^\top \mathbf{e}_{err}\Bigg|}_{I_2^1} + \underbrace{\max_{m\in[N]}\Bigg|\frac{1}{2} \sum_{j=1}^N \mathbf{A}_{mj} \Psi^{(3)}(\eta_j^\star)(\mathbf{x}_j^\top \mathbf{e}_{err})^2\Bigg|}_{I_2^2}.
                \end{align*}
                The term $I_2^1$ can be bounded using the Cauchy–Schwarz inequality, ~\eqref{eq_tail_of_Q_Err_C_m}, and a union bound. That is, on event $\mathscr{E}_{Q,Err}$
                \begin{align*}
                    I_2^1 &\leq \max_{m\in[N]}\left( \sum_{j=1}^N v_j^\star \mathbf{A}_{mj}^2 \left(\sum_{i=1}^N \mathbf{A}_{ji}\eps_i\right)^2 \right)^{1/2} \left( \sum_{l=1}^N v_l^\star (\mathbf{x}_l^\top \mathbf{e}_{err})^2\right)^{1/2}\\
                    & \lesssim \sqrt{BA_{\max}^2\log d} \cdot \sqrt{N}\Vert\mathbf{e}_{err}\Vert_{\nabla^2\ell(\BoTheS)}.
                \end{align*}
                The upper bound for the deterministic part $I_2^2$ is meanwhile straightforward
                \begin{align*}
                    I_2^2 &\leq \max_{m\in[N]}\frac{1}{2} \sum_{j=1}^N |\mathbf{A}_{mj}|v_j^\star(\mathbf{x}_j^\top \mathbf{e}_{err})^2\leq \frac{1}{2} \max_{m\in[N]}|\mathbf{x}_m^\top \mathbf{e}_{err}| \left(\sum_{j=1}^N v_j^\star \mathbf{A}_{mj}^2\right)^{1/2}\left( \sum_{j=1}^N v_j^\star (\mathbf{x}_j^\top \mathbf{e}_{err})^2 \right)^{1/2}\\
                    & \leq \frac{1}{2}\sqrt{A_{max} N} \max_{m\in[N]}|\mathbf{x}_m^\top \mathbf{e}_{err}| \cdot \Vert\mathbf{e}_{err}\Vert_{\nabla^2\ell(\BoTheS)}.
                \end{align*}

                \noindent \textbf{Bound of $I_3$:} From~\eqref{eq_Q_3j_coffient_upper_bound},~\eqref{eq_frac_of_H_3_into_cross_err}, and setting $N$ sufficiently large, we have 
                \begin{align*}
                    I_3 = \max_{m\in[N]}\Bigg| \frac{1}{2}\sum_{j=1}^N \mathbf{A}_{mj}Q_{3,j}^c  \Bigg| \leq \underbrace{\max_{m\in[N]} \sum_{j=1}^N v_j^\star |\mathbf{A}_{mj}| \cdot |\mathbf{x}_j^\top \Delta_{lin}|^3}_{I_3^1} + \underbrace{\max_{m\in[N]} \sum_{j=1}^N v_j^\star |\mathbf{A}_{mj}| \cdot |\mathbf{x}_j^\top \mathbf{e}_{err}|^3}_{I_3^2}.
                \end{align*}
                The upper bounds on $I_3^1$, $I_3^2$ are relatively easy to establish. On event $\mathscr{E}_\infty\cap \mathscr{E}_{Q,Abs}$
                \begin{alignat*}{2}
                    I_3^1 &\leq \max_{l\in[N]} |\mathbf{x}_l^\top \Delta_{lin}| \cdot \max_{m\in[N]}\sum_{j=1}^N v_j^\star |\mathbf{A}_{mj}|\left(\sum_{i=1}^N \mathbf{A}_{ji}\eps_i\right)^2\quad &\quad&\\
                    &\lesssim (\sqrt{A_{\max}\log d} + A_{\max}\log d)(BA_{\max}\log d + A_{\max} \sqrt{d_\theta}) \quad &\text{(by Lemma~\ref{lema_Delta_lin_Q_infty_H_bounds} and~\eqref{eq_tail_of_Q_Abs_m})}&\\
                    &\lesssim \sqrt{A_{\max}\log d}.  \quad&(\text{by }N\gtrsim B^2\mu_Q d_\theta^{3/2}\log d)&
                \end{alignat*}
                Moreover,
                \begin{alignat*}{2}
                    I_3^2 &\leq \left( \max_{l\in[N]}|\mathbf{x}_l^\top \mathbf{e}_{err}| \right)^2 \cdot \max_{m\in[N]} \sum_{j=1}^N v_j^\star |\mathbf{A}_{mj}| \cdot |\mathbf{x}_j^\top \mathbf{e}_{err}|  \quad &\quad&\\
                    &\leq \left( \max_{l\in[N]}|\mathbf{x}_l^\top \mathbf{e}_{err}| \right)^2 \sqrt{A_{\max}N}\Vert\mathbf{e}_{err}\Vert_{\nabla^2\ell(\BoTheS)}.\quad &(\text{by Cauchy-Schwarz and Lemma~\ref{lema_property_of_A_matrix}})&
                \end{alignat*}

                \noindent \textbf{Combining the Bounds:} The final step of this part is to verify $\max_{m\in[N]}|\mathbf{x}_m^\top \Phi(\mathbf{h})|\leq r_\infty$ for any $\mathbf{h}\in\mathcal{N}(r_H,r_\infty)$ on event $\mathscr{E}$. We proceed by first showing 
                \begin{alignat*}{2}
                    &\Vert\mathbf{e}_{err}\Vert_{\nabla^2\ell(\BoTheS)}\sqrt{A_{\max}N}& &\leq \Vert\mathbf{e}_{err}\Vert_{\nabla^2\ell(\BoTheS)}\sqrt{A_{\max}BN} \\
                    &&&\leq \left(\Vert\Delta_{bias}\Vert_{\nabla^2\ell(\BoTheS)} + \Vert\mathbf{h}\Vert_{\nabla^2\ell(\BoTheS)}\right)\sqrt{A_{\max}BN}\\
                    &\quad(\text{by Lemma~\ref{lema_bias_term_in_semi_norm_upper}})&&\lesssim  A_{\max}\sqrt{Bd_\theta} + BA_{\max}\sqrt{d_\theta}\log d + BA_{\max}^{3/2}\sqrt{d_\theta}(\log d)^{3/2} <\frac{1}{4},
                \end{alignat*}
                when $N\gtrsim B^2\mu_Q d_\theta^{3/2}\log^2 d$, and
                \begin{equation*}
                    \max_{m\in[N]}|\mathbf{x}_m^\top \mathbf{e}_{err}| \leq \max_{m\in[N]} |\mathbf{x}_m^\top \Delta_{bias}| + r_\infty \lesssim A_{\max}\sqrt{d_\theta} + r_\infty <1,\quad(\text{by Lemma~\ref{lema_Bias_Envelope_upperbound}})
                \end{equation*}
                when $N\gtrsim B^2\mu_Q d_\theta^{3/2}\log^2 d$. Then, on the event of $\mathscr{E}$
                \begin{alignat*}{2}
                    \max_{m\in[N]}|\mathbf{x}_m^\top \Phi(\mathbf{h})| &\leq&\,& I_1 + I_2 + I_3 \\
                    &\lesssim&\,&  B A_{\max} \log d + \sqrt{A_{\max}\log d} + \Vert\mathbf{e}_{err}\Vert_{\nabla^2\ell(\BoTheS)}\sqrt{A_{\max}N}\times\\
                    &&\,&\left( \max_{m\in[N]}|\mathbf{x}_m^\top \mathbf{e}_{err}| + \left(\max_{l\in[N]}|\mathbf{x}_l^\top \mathbf{e}_{err}| \right)^2\right) + \Vert\mathbf{e}_{err}\Vert_{\nabla^2\ell(\BoTheS)}\sqrt{A_{\max}BN}\sqrt{ A_{\max} \log d}\\
                    &\lesssim&\,& B A_{\max} \log d + \sqrt{A_{\max}\log d} + \frac{1}{4} \left(\max_{m\in[N]}|\mathbf{x}_m^\top \mathbf{e}_{err}| +\sqrt{ A_{\max} \log d}\right)\\
                    &\lesssim&\,& B A_{\max} \log d + \sqrt{A_{\max}\log d} + A_{\max}\sqrt{d_\theta} + \frac{1}{4} r_\infty \leq r_\infty.
                \end{alignat*}
                By combining \textbf{STEP 1} and \textbf{STEP 2}, we conclude that on the event $\mathscr{E}$ with $\bbp_\BoTheS^{(N)}\{\mathscr{E}\}\geq 1-d^{-c}$ uniformly over $\mathbf{h}\in\mathcal{N}(r_H,r_\infty)$, the event $\Phi(\mathbf{h})\in \mathcal{N}(r_H,r_\infty)$ occurs, this proves~\eqref{eq_highprob_self_mapping}. \qedbox
        \subsection{Proof of Lemma~\ref{lema_property_of_A_matrix}}
            \label{proof_lema_property_of_A_matrix}
            
            Part (i). By~\eqref{eq_hess_mle_in_theta}, $\nabla^2\ell(\BoTheS)= \frac{1}{N}\mathbf{X}^\top\mathbf{V}^\star\mathbf{X}$, where $\mathbf{X}$ and $\mathbf{V}^\star$ are defined in~\eqref{eq_def_of_x_k_X_V_star}. Thus
            \begin{align*}
                \mathbf{A}\mathbf{V}^\star\mathbf{A} &= \frac{1}{N^2} \mathbf{X}\nabla^2\ell(\BoTheS)^{-1}(\mathbf{X}^\top \mathbf{V}^\star\mathbf{X})\nabla^2\ell(\BoTheS)^{-1}\mathbf{X}^\top\\
                & = \frac{1}{N^2} \mathbf{X}\nabla^2\ell(\BoTheS)^{-1}(N\nabla^2\ell(\BoTheS))\nabla^2\ell(\BoTheS)^{-1}\mathbf{X}^\top= \frac{1}{N}\mathbf{X}\nabla^2\ell(\BoTheS)^{-1}\mathbf{X}^\top = \mathbf{A}.
            \end{align*}
            By the symmetry of $\mathbf{A}$, $\mathbf{A}_{ii} = \mathbf{e}_i^\top\mathbf{A}\mathbf{e}_i = \mathbf{e}_i^\top\mathbf{A}\mathbf{V}^\star\mathbf{A}\mathbf{e}_i = \sum_{j=1}^N v_j^\star\mathbf{A}_{ij}^2$.
            Moreover, by the definition of $\mathbf{A}$ and the cyclic property of trace 
            \begin{equation}
                \sum_{j=1}^N v_j^\star\mathbf{A}_{jj} = \mathrm{tr}(\mathbf{V}^\star\mathbf{A}) = \mathrm{tr}\left( \mathbf{V}^\star\frac{1}{N}\mathbf{X}\nabla^2\ell(\BoTheS)^{-1}\mathbf{X}^\top \right) = \mathrm{tr}\left( \nabla^2\ell(\BoTheS)^{-1} \frac{1}{N}\mathbf{X}^\top\mathbf{V}^\star\mathbf{X}\right) = d_\theta,
            \end{equation}
            where $d_\theta$ is defined in~\eqref{eq_transfrom_w_to_theta}.           

    \noindent Part (ii). Since $\mathbf
    A$ is symmetric, and $\mathbf{A}\succeq\mathbf{0}$, there exists some $\mathbf{B}$ such that $\mathbf{A}=\mathbf{BB}^\top$. Thus $|\mathbf{A}_{ij}| = |\mathbf{e}_i^\top \mathbf{B}\mathbf{B}^\top \mathbf{e}_j|\leq \Vert\mathbf{B}^\top \mathbf{e}_i\Vert_2 \Vert\mathbf{B}^\top \mathbf{e}_j\Vert_2 = \sqrt{\mathbf{A}_{ii}\mathbf{A}_{jj}}\leq \max_{i\in[N]}\mathbf{A}_{ii}$. The upper bound on $\max_{i\in[N]}\mathbf{A}_{ii}$ is by Definition~\ref{Def_incoherent_level_of_leverages}, and Proposition~\ref{prop_bound_dymic_range}.\qedbox
        \subsection{Proof of Lemma~\ref{lema_Delta_lin_Q_infty_H_bounds}}
                \label{proof_lema_Delta_lin_Q_infty_H_bounds}
                We prove inequalities~\eqref{eq_Delta_lin_coherence_tail} and~\eqref{eq_Delta_lin_seminorm_tail} by virtue of Bernstein's inequality and Hanson-Wright inequality (restated in Lemma~\ref{lema_HsuHW_ineq}), respectively.

                    \noindent\textbf{Proof of~\eqref{eq_Delta_lin_coherence_tail}.} Observe that 
            \begin{equation*}
                |\mathbf{x}_m^\top \Delta_{lin}| = \Big|\mathbf{x}_m^\top\nabla^2\ell(\BoTheS)^{-1}\nabla\ell(\BoTheS)\Big|=\Bigg|\mathbf{x}_m^\top\nabla^2\ell(\BoTheS)^{-1}\frac{1}{N}\sum_{j=1}^N \mathbf{x}_j\eps_j\Bigg|=\Bigg|\sum_{j=1}^N \mathbf{A}_{mj}\eps_j\Bigg|.
            \end{equation*}
            $\{\mathbf{A}_{mj}\eps_j\}_{j\in[N]}$ are zero-mean, bounded random variables, i.e.,
            \begin{equation*}
                |\mathbf{A}_{mj}\eps_j| \leq |\mathbf{A}_{mj}|\leq A_{\max},\,a.s.,\,
            \end{equation*}
            where the last inequality is by Lemma~\ref{lema_property_of_A_matrix}.
            The summation of these random variables has a bounded variance in that 
            \begin{equation*}
                \mathbb{E}_\BoTheS^{(N)}\left[\left(\sum_{j=1}^N \mathbf{A}_{mj}\eps_j\right)^2\right] = \sum_{j=1}^N \mathbf{A}_{mj}^2 v_j^\star  = \mathbf{A}_{mm} \leq A_{\max}.
            \end{equation*}
            With the established bounds, we can obtain~\eqref{eq_Delta_lin_coherence_tail} by virtue of Bernstein's inequality (see, e.g.,~\cite[Theorem 2.8.4]{VershyninHighDimBook}).
            
            \noindent\textbf{Proof of~\eqref{eq_Delta_lin_seminorm_tail}.} By a simple maneuver, we can obtain the following relationship 
            \begin{equation*}
                \Bigg|\Vert\Delta_{lin}\Vert_{\nabla^2\ell(\BoTheS)}^2 - \frac{\mathrm{tr}(\mathbf{I}_{d_\theta})}{N}\Bigg| = \frac{1}{N}\Bigg| \boldsymbol{\xi}^\top \mathbf{P} \boldsymbol{\xi} - \mathrm{tr}(\mathbf{P}) \Bigg|,
            \end{equation*}
            for the isotropic, $\sqrt{B}$ sub-Gaussian random vector $\boldsymbol{\xi}:=\left[\frac{\eps_1}{\sqrt{v_1^\star}},\dots,\frac{\eps_N}{\sqrt{v_N^\star}}\right]^\top$, i.e., $\Vert\boldsymbol{\xi}\Vert_{\psi_2}\lesssim\sqrt{B}$. Since $\mathrm{tr}(\mathbf{I}_{d_\theta}) = \mathrm{tr}(\mathbf{P}) =d_\theta$, we can directly verify that $\bbe_\BoTheS^{(N)}\left[ \boldsymbol{\xi}^\top \mathbf{P} \boldsymbol{\xi} \right] = \mathrm{tr}(\mathbf{P})$, $\Vert\mathbf{P}\Vert_F^2 = \mathrm{tr}(\mathbf{P}^2) = d_\theta$, and $\Vert\mathbf{P}\Vert\leq 1$. Applying Lemma~\ref{lema_HsuHW_ineq} to this quadratic form yields
            \begin{equation*}
                \bbp_{\BoTheS}^{(N)}\left( \frac{1}{N}\Bigg| \boldsymbol{\xi}^\top \mathbf{P} \boldsymbol{\xi} - \mathrm{tr}(\mathbf{P}) \Bigg| \geq t \right) \lesssim \exp\left( -c\min\left\{ \frac{N^2t^2}{d_\theta B^2},\frac{Nt}{B} \right\} \right),
            \end{equation*}
            which implies~\eqref{eq_Delta_lin_seminorm_tail}.
            \qedbox
            \subsection{Proof of Lemma~\ref{lema_Centered_Quadratic_Biases_all3}}
            \label{proof_lema_Centered_Quadratic_Biases_all3}
            \textbf{Proof of~\eqref{eq_tail_of_Q_lin_m}.} By plain calculation,
                \begin{equation*}
                    \mathbb{E}_{\BoTheS}^{(N)}\left[\left(\sum_{i=1}^N\mathbf{A}_{ji}\eps_i\right)^2\right] = \mathbb{E}_{\BoTheS}^{(N)}\left[\sum_{i=1}^N \mathbf{A}_{ji}^2 \eps_i^2\right] = \sum_{i=1}^N \mathbf{A}_{ji}^2 v_i^\star = \mathbf{A}_{jj},
                \end{equation*}
                where the first equality is due to the independence of samples.
                Consider
                \begin{equation*}
                    \mathbf{D}:=\mathrm{diag}\left( \left(\frac{\Psi^{(3)}(\eta_i^\star)}{v_i^\star}\mathbf{A}_{mi}\right)_{i\leq N} \right),\quad \mathbf{M}_m : = \mathbf{P}\mathbf{D}\mathbf{P},
                \end{equation*}
                where $\mathbf{P}$ is defined in~\eqref{eq_def_of_orthonomal_P}, and the isotropic sub-Gaussian random vector $\boldsymbol{\xi}:=\left[\frac{\eps_1}{\sqrt{v_1^\star}},\dots,\frac{\eps_N}{\sqrt{v_N^\star}}\right]^\top$. Then $Q_{lin}^{m}$ admits the following equivalent form
                \begin{equation}
                    Q_{lin}^{m} =\frac{1}{2}\left( \boldsymbol{\xi}^\top \mathbf{M}_m\boldsymbol{\xi} - \mathrm{tr}(\mathbf{M}_m)\right).
                \end{equation}
                Since $D_{\max}:=\max_{j\in[N]}\Big|\frac{\Psi^{(3)}(\eta_j^\star)}{v_j^\star}\mathbf{A}_{mj}\Big| \leq A_{\max}$ by Lemma~\ref{lema_property_of_A_matrix}, our focus is on controlling $\mathrm{tr}(\mathbf{D}^2\mathbf{P})$, that is
                \begin{align*}
                    \mathrm{tr}(\mathbf{D}^2\mathbf{P}) = \sum_{j=1}^N \left(\frac{\Psi^{(3)}(\eta_j^\star)}{v_j^\star}\mathbf{A}_{mj}\right)^2 v_j^\star \mathbf{A}_{jj}\leq \sum_{j=1}^N v_j^\star\mathbf{A}_{jj} \mathbf{A}_{mj}^2\leq A_{\max} \sum_{j=1}^N v_j^\star\mathbf{A}_{mj}^2 \leq A_{\max}^2,
                \end{align*}
            where the first inequality is due to $|\Psi^{(3)}(\eta_j^\star)|\leq v_j^\star$. With the established bounds of $D_{\max}$ and $\mathrm{tr}(\mathbf{D}^2\mathbf{P})$, we can obtain the tail bound~\eqref{eq_tail_of_Q_lin_m} via inequality~\eqref{eq_HW_PDP_xi_subG_form}.
            
            \noindent\textbf{Proof of~\eqref{eq_tail_of_Q_Err_C_m}.}
            First, notice that 
                \begin{equation*}
                    \sum_{j=1}^N v_j^\star \mathbf{A}^2_{mj}\mathbf{A}_{jj} \leq A_{\max} \sum_{j=1}^N v_j^\star \mathbf{A}^2_{mj} = A_{\max}\mathbf{A}_{mm}\leq A_{\max}^2,
                \end{equation*}
                which yields the deterministic upper bound. Again, we can identify the matrix 
                \begin{equation*}
                    \mathbf{D}:=\mathrm{diag}\left( \mathbf{A}_{m1}^2,\dots, \mathbf{A}_{mN}^2\right),\quad \mathbf{M}_{m,2} : = \mathbf{P}\mathbf{D}\mathbf{P},
                \end{equation*}
                and the isotropic sub-Gaussian random vector $\boldsymbol{\xi}:=\left[\frac{\eps_1}{\sqrt{v_1^\star}},\dots,\frac{\eps_N}{\sqrt{v_N^\star}}\right]^\top$ such that $\Vert\boldsymbol{\xi}\Vert_{\psi_2}\lesssim \sqrt{B}$, and
                \begin{equation*}
                    Q_{Err,C}^m = \boldsymbol{\xi}^\top \mathbf{M}_{m,2}\boldsymbol{\xi} - \mathrm{tr}(\mathbf{M}_{m,2}).
                \end{equation*}
                Moreover $D_{\max}:=\max_{i\in[N]}\mathbf{A}_{mi}^2\leq A_{\max}^2$, and
                \begin{align*}
                    \mathrm{tr}(\mathbf{D}^2\mathbf{P}) = \sum_{j=1}^N \mathbf{A}_{mj}^4 v_j^\star\mathbf{A}_{jj} \leq A_{\max}^2\sum_{j=1}^N v_j^\star \mathbf{A}_{jj}\mathbf{A}_{mj}^2\leq A_{\max}^4.
                \end{align*}
                Using inequality~\eqref{eq_HW_PDP_xi_subG_form}, we obtain~\eqref{eq_tail_of_Q_Err_C_m}. 
                
                \noindent\textbf{Proof of~\eqref{eq_tail_of_Q_Abs_m}.}
                The deterministic upper bound is from Cauchy-Schwarz inequality,
                \begin{equation*}
                    \sum_{j=1}^N v_j^\star|\mathbf{A}_{mj}|\mathbf{A}_{jj} \leq \left( \sum_{j=1}^N v_j^\star\mathbf{A}_{mj}^2 \right)^{1/2}\left( \sum_{j=1}^N v_j^\star\mathbf{A}_{jj}^2 \right)^{1/2}\leq \sqrt{A_{\max}}\sqrt{A_{\max}d_\theta} = A_{\max}\sqrt{d_\theta}.
                \end{equation*}
                The tail bound is from identifying the matrix
                \begin{equation*}
                    \mathbf{M}_{m,|\cdot|} : = \mathbf{P}\mathrm{diag}\left( |\mathbf{A}_{m1}|,\dots, |\mathbf{A}_{mN}|\right)\mathbf{P},
                \end{equation*}
                the isotropic sub-Gaussian random vector $\boldsymbol{\xi}:=\left[\frac{\eps_1}{\sqrt{v_1^\star}},\dots,\frac{\eps_N}{\sqrt{v_N^\star}}\right]^\top$, i.e., $\Vert\boldsymbol{\xi}\Vert_{\psi_2}\lesssim\sqrt{B}$, such that 
                \begin{equation*}
                    Q_{Abs}^{m} =  \boldsymbol{\xi}^\top \mathbf{M}_{m,|\cdot|} \boldsymbol{\xi} - \mathrm{tr}(\mathbf{M}_{m,|\cdot|}).
                \end{equation*}
                Then $D_{\max}:=\max_{j\in[N]}|\mathbf{A}_{mj}|\leq A_{\max}$, and
                \begin{align*}
                    \mathrm{tr}(\mathbf{D}^2\mathbf{P}) & = \sum_{j=1}^N \mathbf{A}_{mj}^2 \mathbf{P}_{jj} = \sum_{j=1}^N \mathbf{A}_{mj}^2 v_j^\star\mathbf{A}_{jj} \leq \mathbf{A}_{mm} A_{\max}\leq A_{\max}^2.
                \end{align*}
                Using inequality~\eqref{eq_HW_PDP_xi_subG_form}, we conclude the proof. \qedbox 
        \subsection{Proof of Corollary~\ref{coro_euclidean_upper_in_detailed_terms}}
        \label{proof_coro_euclidean_upper_in_detailed_terms}
        It suffices to control $\Vert\Delta_{lin}\Vert_2$ and $\Vert\Delta_{bias}\Vert_2$. The latter can be directly estimated from $\Vert\Delta_{bias}\Vert_{\nabla^2\ell(\BoTheS)}$,
        \begin{align}
            \Vert\Delta_{bias}\Vert_2 &\leq \sqrt{\Vert \nabla^2\ell(\BoTheS)^{-1}  \Vert} \Vert\Delta_{bias}\Vert_{\nabla^2\ell(\BoTheS)} \lesssim \sqrt{\Vert \nabla^2\ell(\BoTheS)^{-1}  \Vert} \sqrt{\frac{A_{\max}d_\theta}{N}} \nonumber\\
            \label{eq_upper_bound_on_Delta_bias_ell2}
            &\lesssim  \sqrt{\Vert \nabla^2\ell(\BoTheS)^{-1}  \Vert} \frac{\sqrt{B\mu_Q}d_\theta}{N},
        \end{align}
        where the second and last inequalities follow by Lemmas~\ref{lema_bias_term_in_semi_norm_upper} and~\ref{lema_property_of_A_matrix}, respectively. The upper bound on $\Vert\Delta_{lin}\Vert_2$ can be derived from the matrix Bernstein inequality, restated here as Lemma~\ref{lema_TroppMatrixConcen_ineq}. Recall that $\nabla\ell(\BoTheS) = -\frac{1}{N}\sum_{j=1}^N \mathbf{x}_j\eps_j$, and $\nabla^2\ell(\BoTheS) = \frac{1}{N}\sum_{j=1}^N v_j^\star \mathbf{x}_j\mathbf{x}_j^\top$, from which $\Delta_{lin}$ can be written as the sum of zero-mean, independent random vectors,
        \begin{equation*}
            \Delta_{lin} = \sum_{j=1}^N \frac{1}{N}\nabla^2\ell(\BoTheS)^{-1}\mathbf{x}_j\eps_j : = \sum_{j=1}^N \mathbf{z}_j.
        \end{equation*}
        The random vector $\mathbf{z}_j\in\mathbb{R}^{d_\theta}$ has Euclidean norm bounded a.s.,
        \begin{align*}
            \Vert \mathbf{z}_j \Vert_2 \leq \sqrt{\frac{1}{N^2} \mathbf{x}_j^\top \nabla^2\ell(\BoTheS)^{-2}\mathbf{x}_j} \leq \sqrt{\frac{\Vert \nabla^2\ell(\BoTheS)^{-1} \Vert}{N}\mathbf{A}_{jj}} \leq \sqrt{\frac{\Vert \nabla^2\ell(\BoTheS)^{-1} \Vert A_{\max}}{N}}.
        \end{align*}
        And the variance terms satisfy 
        \begin{align*}
            \Bigg\Vert \sum_{j=1}^N \mathbb{E}_\BoTheS^{(N)}\left[ \mathbf{z}_j\mathbf{z}_j^\top \right]\Bigg\Vert &= \Bigg\Vert\frac{1}{N^2} \nabla^2\ell(\BoTheS)^{-1} \left(\sum_{j=1}^N v_j^\star \mathbf{x}_j\mathbf{x}_j^\top\right)\nabla^2\ell(\BoTheS)^{-1}\Bigg\Vert = \frac{\Vert\nabla^2\ell(\BoTheS)^{-1}\Vert}{N}\\
            \Bigg| \sum_{j=1}^N \mathbb{E}_\BoTheS^{(N)}\left[ \mathbf{z}_j^\top\mathbf{z}_j \right] \Bigg| &= \Bigg|\frac{1}{N^2}\sum_{j=1}^N v_j^\star \mathbf{x}_j^\top \nabla^2\ell(\BoTheS)^{-2} \mathbf{x}_j\Bigg| = \frac{1}{N}\mathrm{tr}\left(\nabla^2\ell(\BoTheS)^{-2}\frac{1}{N}\sum_{j=1}^N v_j^\star\mathbf{x}_j\mathbf{x}_j^\top\right)\\
            & = \frac{\mathrm{tr}(\nabla^2\ell(\BoTheS)^{-1})}{N}.
        \end{align*}
        Applying Lemma~\ref{lema_TroppMatrixConcen_ineq} and substituting the upper bound on $A_{\max}$, it yields
        \begin{align}
            \Vert\Delta_{lin}\Vert_2 = \Bigg\Vert\sum_{j=1}^N \mathbf{z}_j\Bigg\Vert_2 &\lesssim \sqrt{\frac{\mathrm{tr}(\nabla^2\ell(\BoTheS)^{-1})\log d_\theta}{N}} + \sqrt{\frac{B\mu_Q d_\theta\Vert \nabla^2\ell(\BoTheS)^{-1} \Vert\log^2 d_\theta}{N^2}}  \nonumber\\
            \label{eq_upper_bound_on_Delta_lin_ell2}
            & \lesssim \sqrt{\frac{\mathrm{tr}(\nabla^2\ell(\BoTheS)^{-1})\log d_\theta}{N}},
        \end{align}
        with probability at least $1-d_\theta^{-c}$, where the last inequality is from setting $N\gtrsim \frac{B\mu_Qd_\theta \log d_\theta}{r_{\mathrm{eff}}}$. Comparing~\eqref{eq_upper_bound_on_Delta_lin_ell2} with~\eqref{eq_upper_bound_on_Delta_bias_ell2}, setting $N\gtrsim \frac{B\mu_Q d_\theta^2}{r_{\mathrm{eff}}\log d_\theta}$ suffices to ensure that the upper bound in~\eqref{eq_upper_bound_on_Delta_lin_ell2} dominates that in~\eqref{eq_upper_bound_on_Delta_bias_ell2}. By setting
        \begin{equation*}
            N\gtrsim \max\left\{ B^3\mu_Q d_\theta^{3/2}\log^3 d,\, B^4\mu_Qd_\theta \log d,\,\frac{B\mu_Qd_\theta \log d_\theta}{r_{\mathrm{eff}}},\,\frac{B\mu_Q d_\theta^2}{r_{\mathrm{eff}}\log d_\theta}\right\},
        \end{equation*}
        the r.h.s.~of~\eqref{eq_Delta_minus_lin_bias_upper_bound} is also controlled by the r.h.s.~of~\eqref{eq_upper_bound_on_Delta_lin_ell2}. Using the triangle inequality then gives~\eqref{eq_delta_hat_upper_bound_ell2_to_minimax_rate}. 
        \qedbox
		\section{Supporting Lemmas}
        Throughout this appendix, all random objects appearing in the same statement are assumed to be defined on a common probability space $(\Omega,\mathcal{F},\bbp)$, and $\bbe$ denotes expectation under $\bbp$. The underlying probability space may vary from lemma to lemma. When applying these results to the statistical model developed in the main part of the paper, $(\mathbb P,\mathbb E)$ is instantiated as $(\bbp_{\Bothe}^{(N)},\bbe_{\Bothe}^{(N)})$ for the relevant value of ${\Bothe}$.
		\begin{lemma}
			\label{eq_bounded_Rv_prob_lower}
            Let $X$ be a bounded, non-negative random variable with $0\leq X \leq C<\infty$ almost surely. Then
			\begin{equation}
            \label{eq_BoundedPos_inv_Markov}
				\mathbb{P}(X\geq t)\geq \frac{\mathbb{E}[X]-t}{C-t},\,\,\forall\,t\in(0,C).
			\end{equation}
		\end{lemma}
		\begin{proof}
			By direct decomposition, we have
			\begin{equation*}
				\mathbb{E}[X] = \mathbb{E}[X\cdot\ind\{X< t\}] + \mathbb{E}[X\cdot\ind\{X\geq t\}]\leq t\,\mathbb{P}(X< t) + C\,\mathbb{P}(X\geq t) = t + (C-t)\mathbb{P}(X\geq t),
			\end{equation*}
			rearranging the inequality gives rise to~\eqref{eq_BoundedPos_inv_Markov}. \qedbox
		\end{proof}
		\begin{lemma}[{\cite[Theorem~1.1]{RudelsonVershy_HW_subG}}]
			\label{lema_HsuHW_ineq}
			Let $\boldsymbol{\xi}=[\xi_1,\dots,\xi_n]^\top\in\mathbb{R}^n$ be a random vector with independent components $\xi_i$ which satisfy $\bbe[\xi_i] = 0$ and $\Vert \xi_i \Vert_{\psi_2}\leq K$.
            Then for any but fixed matrix $\mathbf{M}\in\mathbb{R}^{n\times n}$ and every $t\geq 0$,
            \begin{equation}
            \label{eq_HW_independent_coordinate_subGK}
				\bbp\left\{\Bigg|\boldsymbol{\xi}^\top\mathbf{M}\boldsymbol{\xi}-\mathbb{E}[\boldsymbol{\xi}^\top\mathbf{M}\boldsymbol{\xi}]\Bigg|\geq t\right\}\leq 2\exp\left( -c\min\left\{ \frac{t^2}{K^4\Vert\mathbf{M}\Vert_F^2},\frac{t}{K^2\Vert\mathbf{M}\Vert} \right\} \right).
			\end{equation}
		\end{lemma}
        \begin{lemma}[{\cite[Corollary~16]{AdamczakLatalaBanachHW}}]
        \label{lema_HilbertSpaceHW}
            Let $\xi_1,\dots,\xi_n$ be independent mean-zero $K-$sub-Gaussian random variables and let $\mathbf{H}:=(\mathbf{h}^{(ij)})_{i,j\leq n}$ be any but fixed symmetric matrix with values in the normed space $(\mathbb{R}^{m},\Vert\cdot\Vert_2)$. For $t\geq CK^2\sqrt{\sum_{i=1}^n\sum_{j=1}^n \Vert \mathbf{h}^{(ij)} \Vert_2^2}$, we have
            \begin{equation*}
                \bbp\left\{ \Bigg\Vert \sum_{(i,j)\in[n]^2} \mathbf{h}^{(ij)}\left(\xi_i\xi_j - \mathbb{E}[\xi_i\xi_j]\right) \Bigg\Vert_2 \geq t\right\} \leq 2\exp\left(-\frac{1}{C}\min\left\{ \frac{t^2}{K^4 U^2},\frac{t}{K^2 V} \right\}\right),
            \end{equation*}
            where
            \begin{align*}
                U  &= \sup_{\Vert\mathbf{u}\Vert_2\leq 1} \sqrt{\sum_{j=1}^n \Big\Vert \sum_{i\neq j}^n\mathbf{h}^{(ij)}\mathbf{u}_i \Big\Vert_2^2} + \sup_{\Vert\mathbf{Y}\Vert_F\leq 1} \Big\Vert \sum_{(i,j)\in[n]^2} \mathbf{h}^{(ij)}[\mathbf{Y}]_{ij} \Big\Vert_2,\\
                V &= \sup_{\Vert\mathbf{u}\Vert_2\leq 1,\Vert\mathbf{v}\Vert_2\leq 1}\Big\Vert \sum_{(i,j)\in[n]^2} \mathbf{h}^{(ij)}\mathbf{u}_i\mathbf{v}_j \Big\Vert_2.
            \end{align*}
        \end{lemma}
		\begin{lemma}[{\cite[Theorem~1.6]{TroppMatrixConcentra}}]
		  \label{lema_TroppMatrixConcen_ineq}
		    Let $\{\mathbf{H}_k\}$ be a finite sequence of independent, random matrices in $\mathbb{R}^{n_1\times n_2}$ satisfying $\mathbb{E}[\mathbf{H}_k]=\mathbf{0}$, $\Vert\mathbf{H}_k\Vert\leq L$ almost surely. Let the norm of total variance be
			\begin{equation*}
				\nu^2=\max\left\{ \Bigg\Vert\sum_k \mathbb{E}\left[\mathbf{H}_k\mathbf{H}_k^\top\right]\Bigg\Vert,\,\Bigg\Vert\sum_k \mathbb{E}\left[\mathbf{H}_k^\top\mathbf{H}_k\right]\Bigg\Vert \right\}.
			\end{equation*} 
			Then for all $t\geq 0$
			\begin{equation*}
				\bbp\left\{\Bigg\Vert\sum_k \mathbf{H}_k\Bigg\Vert\geq t\right\}\leq (n_1+n_2)\exp\left( \frac{-t^2}{2\nu^2 +2Lt/3} \right).
			\end{equation*}
		\end{lemma}
        \begin{small}
			\bibliographystyle{plain}
            \bibliography{bib}
		\end{small}
	\end{document}

%% file: Tikz_flow_charts/flowchart_style.tex
\definecolor{FCink}{gray}{0}
\definecolor{FCmuted}{gray}{0.35}
\definecolor{FCedge}{gray}{0.50}
\definecolor{FChead}{gray}{0.88}
\definecolor{FCshade}{gray}{0.95}

\newcommand{\FCRef}[1]{{\fontsize{6.6}{7.8}\selectfont\color{FCmuted}#1}}

\tikzset{
  FCdiagram/.style={
    x=1cm,y=1cm,
    execute at begin picture={
      \hyphenpenalty=10000\relax\exhyphenpenalty=10000\relax
      \ifdefined\hypersetup\hypersetup{hidelinks}\fi
    },
    every node/.style={
      font=\normalfont\fontsize{7.5}{8.8}\selectfont,
      text=FCink,inner sep=0pt,align=center
    }
  },
  FCArrow/.style={
    -{Stealth[length=1.7mm,width=1.3mm]},draw=FCink,line width=.6pt
  },
  FCPanelBox/.style={draw=FCedge,line width=.5pt,fill=white}
}

%% file: Tikz_flow_charts/theory_tikz_auto.tex
\begingroup
\def\TheoryHeight{3.02cm}%
\def\TheoryHeaderHeight{0.40cm}%
\def\TheoryPadding{0.12cm}%
\def\TheoryGap{0.40cm}%
\setlength{\fboxsep}{0pt}%

\newcommand{\TheoryAutoBox}[3]{%
    \vbox{%
        \hbox{\colorbox{FChead}{%
            \parbox[c][\TheoryHeaderHeight][c]{#1}{%
                \centering\bfseries #2}%
        }}%
        \nointerlineskip
        \hbox{%
            \hspace*{\TheoryPadding}%
            \begin{minipage}[c]
                [\dimexpr\TheoryHeight-\TheoryHeaderHeight\relax][s]
                {\dimexpr#1-\TheoryPadding-\TheoryPadding\relax}
                \setlength{\parindent}{0pt}
                \setlength{\parskip}{0pt}
                \centering
                \vspace*{\TheoryPadding}%
                #3\par
                \vspace*{\TheoryPadding}%
            \end{minipage}%
            \hspace*{\TheoryPadding}%
        }%
    }%
}%

\newcommand{\TheoryAutoResult}[2]{%
    \begin{tabular*}{\linewidth}{@{}l@{\extracolsep{\fill}}r@{}}
        #1 & \FCRef{#2}
    \end{tabular*}\par
}%

\begin{tikzpicture}[
    FCdiagram,
    theory box/.style={FCPanelBox,inner sep=0pt,outer sep=0pt}
]

\node[theory box] (model) {%
    \TheoryAutoBox{4.10cm}{Model setup (Section~\ref{sec_model_setup})}{%
        Parametric BTL model\\
        \FCRef{(Definition~\ref{def_para_utility_func})}

        \par\vfill
        Identifiability $\Longleftrightarrow\mathbf{W}\succ\mathbf{0}$\\
        \FCRef{(Proposition~\ref{prop_indentifibility_complete_obv};
                Corollary~\ref{coro_from_identifibality})}

        \par\vfill
        Statistical model and MLE\\
        \FCRef{(Definition~\ref{def_MLE_func})}
    }%
};

\node[theory box,right=\TheoryGap of model] (bounds) {%
    \TheoryAutoBox{5.55cm}{Lower bounds (Section~\ref{sec_minimax_LB})}{%
        CRLB for locally unbiased estimators\\
        \FCRef{(Proposition~\ref{prop_the_CRLB}, Corollary~\ref{coro_of_CRLB})}

        \par\vfill
        {\color{FCedge!55}\hrule height 0.4pt}
        \vfill
        \textbf{Minimax lower bounds}

        \par\vfill
        \TheoryAutoResult{Euclidean error}
            {(Theorem~\ref{thm_minimax_loewe_bound_tr_inv})}
        \vfill
        \TheoryAutoResult{Excess risk}
            {(Corollary~\ref{coro_miniax_lowerbound_excess_risk})}
        \vfill
        \TheoryAutoResult{Query-wise error}
            {(Theorem~\ref{thm_residual_coherence_minimax})}
    }%
};

\node[theory box,right=\TheoryGap of bounds] (mle) {%
    \TheoryAutoBox{6.35cm}{Non-asymptotic MLE analysis (Section~\ref{MLE_upper_bound})}{%

        \TheoryAutoResult{MLE existence: $d_\theta^{1.5}$ sample size}
            {(Theorem~\ref{thm_suffic_sample_l2_upper_bound})}
        \vfill
        \TheoryAutoResult{\textbf{Refined decomposition}}
            {(Theorem~\ref{thm_refinred_Q_infty_ell_2_residual_decompose_bound})}
        $\hat{\boldsymbol{\theta}}-\boldsymbol{\theta}^{\star}
          =\boldsymbol{\Delta}_{lin}
           +\boldsymbol{\Delta}_{bias}+\mathrm{remainder}$

        \par\vfill
        \colorbox{FCshade}{%
            \begin{minipage}{\linewidth}
                \centering
                \textbf{Minimax Optimality}\par
                \TheoryAutoResult{Query-wise: $d_\theta^2$ sample size}
                    {(Theorem~\ref{thm_refinred_Q_infty_ell_2_residual_decompose_bound})}
                \TheoryAutoResult{Euclidean: $\mathrm{tr}(\mathcal{I}(\boldsymbol{\theta}^\star)^{-1})/N$ rate}
                    {(Corollary~\ref{coro_euclidean_upper_in_detailed_terms})}
            \end{minipage}%
        }%
    }%
};

\draw[FCArrow] (model.east) -- (bounds.west);
\draw[FCArrow] (bounds.east) -- (mle.west);
\end{tikzpicture}%
\endgroup%

%% file: Tikz_flow_charts/survey_tikz_auto.tex
\begingroup%
\def\SurveyHeight{3.02cm}%
\def\SurveyHeaderHeight{0.40cm}%
\def\SurveyPadding{0.12cm}%
\def\SurveyGap{0.40cm}%
\def\SurveyFeaturesHeight{\dimexpr(\SurveyHeight-\SurveyStackGap)/2\relax}%
\def\SurveyStackGap{0.17cm}%
\setlength{\fboxsep}{0pt}%

\newcommand{\SurveyAutoBox}[5][\SurveyPadding]{%
    \vbox{%
        \hbox{\colorbox{FChead}{%
            \parbox[c][\SurveyHeaderHeight][c]{#2}{%
                \centering\bfseries #4}%
        }}%
        \nointerlineskip
        \hbox{%
            \hspace*{#1}%
            \begin{minipage}[c]
                [\dimexpr#3-\SurveyHeaderHeight\relax][s]
                {\dimexpr#2-#1-#1\relax}
                \setlength{\parindent}{0pt}
                \setlength{\parskip}{0pt}
                \centering
                \vspace*{#1}%
                #5\par
                \vspace*{#1}%
            \end{minipage}%
            \hspace*{#1}%
        }%
    }%
}%

\begin{tikzpicture}[
    FCdiagram,
    survey box/.style={FCPanelBox,inner sep=0pt,outer sep=0pt}
]

\node[survey box] (features) {%
    \SurveyAutoBox[0.04cm]{3.25cm}{\SurveyFeaturesHeight}
        {Alternative features}{%
        $\mathbf{x}_1,\dots,\mathbf{x}_d\in\mathbb{R}^n$
        \par\vfill
        (e.g., $d$ products or services)
    }%
};

\node[survey box,below=\SurveyStackGap of features] (utility) {%
    \SurveyAutoBox[0.06cm]{3.25cm}
        {\dimexpr\SurveyHeight-\SurveyFeaturesHeight-\SurveyStackGap\relax}
        {Utility model}{
        $U(\mathbf{x}_i)=\mathbf{w}^{\top}\boldsymbol{\phi}(\mathbf{x}_i)$

        \par\vfill
        $\mathbf{Aw}=\mathbf{b}$

        \par\vfill
        \FCRef{(Definition~\ref{def_para_utility_func})}
    }%
};

\node[survey box,right=\SurveyGap of features.north east,anchor=north west]
    (design) {%
    \SurveyAutoBox{5.50cm}{\SurveyHeight}{Pairwise questionnaire design}{%
        Select $\mathcal{E}_s\subseteq\mathbb{I}$: Optimally / Manually

        \par\vfill
        Reparameterize: $\boldsymbol{\theta}\in\mathbb{R}^{d_\theta}$,
        $\mathbf{q}_{\boldsymbol{\alpha}}
            =\mathbf{V}_{A^\perp}^{\top}\mathbf{a}_{\boldsymbol{\alpha}}$\\
        \FCRef{\eqref{eq_def_of_U_phi_matrix},
               \eqref{eq_transfrom_w_to_theta},
               \eqref{eq_definition_q_alp}}

        \par\vfill
        \colorbox{FCshade}{%
            \begin{minipage}{\linewidth}
                \centering
                \strut
                \textbf{Comparison graph} \FCRef{(Definition~\ref{def_comparGraph})}\\
                \textbf{Design matrix $\mathbf{W}\succ\mathbf{0}$}
                \FCRef{\eqref{eq_def_of_matrix_K}}\\
                (with identifiability)\\
                \strut\textbf{FIM $\mathcal{I}(\boldsymbol{\theta}^\star)\succ\mathbf{0}$} \FCRef{~\eqref{eq_FIM_in_theta}}\strut
            \end{minipage}%
        }%
    }%
};

\node[survey box,right=\SurveyGap of design] (responses) {%
    \SurveyAutoBox{3.15cm}{\SurveyHeight}{Collect responses}{%
        Present $(i,j)\in\mathcal{E}_s$\\
        DM chooses $i$ or $j$

        \par\vfill
        Ordinal response $Y_{\boldsymbol{\alpha}}^{(t)}\in\{0,1\}$\par\vfill
        
        BTL model~\eqref{eq_BTL_basic_scaled_alpha}

        \par\vfill
        \FCRef{(Assumptions~\ref{assumpt1_well_Logit}--\ref{assump_gap_wp_w_B_r0})}
    }%
};

\node[survey box,right=\SurveyGap of responses] (estimate) {%
    \SurveyAutoBox{3.70cm}{\SurveyHeight}{Estimation}{%
        $\hat{\boldsymbol{\theta}}
            \in\arg\min_{\boldsymbol{\theta}\in\mathbb{R}^{d_\theta}}
            \ell(\boldsymbol{\theta})$

        \par\vfill
        \FCRef{(\eqref{eq_def_of_ell_theta}, Definition~\ref{def_MLE_func})}

        \par\vfill
        \textbf{Recover the utility}

        \par\vfill
        $\hat{\mathbf{w}}=\mathbf{w}_o
            +\mathbf{V}_{A^\perp}\hat{\boldsymbol{\theta}}$

        \par\vfill
        $\hat U(\mathbf{x})
            =\hat{\mathbf{w}}^{\top}\boldsymbol{\phi}(\mathbf{x})$
    }%
};

\coordinate (features-out) at ([yshift=-0.5*\SurveyHeaderHeight]features.east);
\coordinate (utility-out) at ([yshift=-0.5*\SurveyHeaderHeight]utility.east);
\draw[FCArrow] (features-out) -- (design.west |- features-out);
\draw[FCArrow] (utility-out) -- (design.west |- utility-out);
\draw[FCArrow] (design.east) -- (responses.west);
\draw[FCArrow] (responses.east) -- (estimate.west);
\end{tikzpicture}%
\endgroup%

%% file: Tikz_flow_charts/numerical_tikz_auto.tex
\begingroup%
\def\NumericalHeight{3.02cm}%
\def\NumericalHeaderHeight{0.40cm}%
\def\NumericalPadding{0.12cm}%
\def\NumericalGap{0.35cm}%
\setlength{\fboxsep}{0pt}%

\newcommand{\NumericalAutoBox}[3]{%
    \vbox{%
        \hbox{\colorbox{FChead}{%
            \parbox[c][\NumericalHeaderHeight][c]{#1}{%
                \centering\bfseries #2}%
        }}%
        \nointerlineskip
        \hbox{%
            \hspace*{\NumericalPadding}%
            \begin{minipage}[c]
                [\dimexpr\NumericalHeight-\NumericalHeaderHeight\relax][s]
                {\dimexpr#1-\NumericalPadding-\NumericalPadding\relax}
                \setlength{\parindent}{0pt}
                \setlength{\parskip}{0pt}
                \centering
                \vspace*{\NumericalPadding}%
                #3\par
                \vspace*{\NumericalPadding}%
            \end{minipage}%
            \hspace*{\NumericalPadding}%
        }%
    }%
}%

\begin{tikzpicture}[
    FCdiagram,
    numerical box/.style={FCPanelBox,inner sep=0pt,outer sep=0pt}
]

\node[numerical box] (pool) {%
    \NumericalAutoBox{6.50cm}{Candidate query pool and ground truth}{%
        \textbf{Synthetic: Gaussian differences}\FCRef{~\eqref{eq_candidate_q_alpha_and_W_mat_def}--\eqref{eq_candidate_theta_star_def}}

        \par\vfill
        $\mathbf{q}_{\boldsymbol{\alpha}}
            =\mathbf{R}\tilde{\mathbf{W}}^{-1/2}
             \xi_{\boldsymbol{\alpha}}(\boldsymbol{g}_i-\boldsymbol{g}_j)$

        \par\vfill
        $\boldsymbol{\theta}^\star
            =\mathbf{R}^{-1}\bigl(\rho\mathbf{u}
             +\sqrt{1-\rho^2}\,\mathbf{v}\bigr)$

        \par\vfill
        {\color{FCedge!55}\hrule height 0.4pt}
        \vfill
        \textbf{Real world: Green-Flight}\FCRef{~\eqref{eq_contextual_ques_pool_set}}

        \par\vfill
        $\mathbf{q}_{\mathfrak{D},\boldsymbol{\alpha},n}
            =\operatorname{vec}\bigl[
             (\mathbf{q}_{\mathfrak{D},\boldsymbol{\alpha_A}}
              -\mathbf{q}_{\mathfrak{D},\boldsymbol{\alpha_B}})
             \mathbf{c}_n^\top\bigr]\in\mathbb{R}^{24}$

        \par\vfill
        $\boldsymbol{\theta}_\mathfrak{D}^\star$ by Firth correction
    }%
};

\node[numerical box,right=\NumericalGap of pool] (design) {%
    \NumericalAutoBox{3.00cm}{D-optimal design}{%
        Solve $(\mathcal{E}_s,\{n_{\boldsymbol{\alpha}}\})$\FCRef{~\eqref{eq_D_optimal_problem}}

        \par\vfill
        $\mathbf{W}\succ\mathbf{0}$; $\mu_Q<2$

        \par\vfill
        \colorbox{FCshade}{%
            \begin{minipage}{\linewidth}
                \centering
                \strut\textbf{Synthetic}\\
                Replicate / redesign\\
                \textbf{Green-Flight}\\
                Redesign for each $N$\strut
            \end{minipage}%
        }%
    }%
};

\node[numerical box,right=\NumericalGap of design] (simulation) {%
    \NumericalAutoBox{3.00cm}{BTL simulation}{%
        $Y_{\boldsymbol{\alpha}}^{(t)}
            \sim\operatorname{Ber}\left(\frac{1}{1+e^{-\mathbf{q}_{\boldsymbol{\alpha}}^\top
                              \boldsymbol{\theta}^\star}}\right)$

        \par\vfill
        $3000$ independent trials in each setup

        \par\vfill
        \FCRef{Section~\ref{subsec_implemen_detail}}
    }%
};

\node[numerical box,right=\NumericalGap of simulation] (estimation) {%
    \NumericalAutoBox{3.25cm}{MLE and Firth}{%
        MLE existence (LP)
        \par\vfill
        $\hat{\boldsymbol{\theta}}$ computed by trust-region method
        \par\vfill
        \colorbox{FCshade}{%
            \begin{minipage}{\linewidth}
                \centering
                \strut RMSE versus CRLB\\
                Error decomposition\\
                \FCRef{Sections~\ref{subsec_Synthetic_Questionnaires}--\ref{subsec_green_flight_question}}\strut
            \end{minipage}%
        }%
    }%
};

\draw[FCArrow] (pool.east) -- (design.west);
\draw[FCArrow] (design.east) -- (simulation.west);
\draw[FCArrow] (simulation.east) -- (estimation.west);
\end{tikzpicture}%
\endgroup%